\documentclass{article} 
\usepackage{iclr2027_conference,times}

\usepackage{amsmath,amsfonts,bm}

\def\eqref#1{equation~\ref{#1}}

\def\1{\bm{1}}

\DeclareMathAlphabet{\mathsfit}{\encodingdefault}{\sfdefault}{m}{sl}
\SetMathAlphabet{\mathsfit}{bold}{\encodingdefault}{\sfdefault}{bx}{n}

\newcommand{\R}{\mathbb{R}}

\newcommand{\KL}{D_{\mathrm{KL}}}

\DeclareMathOperator*{\argmin}{arg\,min}

\usepackage{hyperref}
\usepackage{url}
\usepackage{graphicx}       
\usepackage{booktabs}       
\usepackage{subcaption}     
\usepackage{amsmath}        
\usepackage{mathtools}
\usepackage{amsthm}
\usepackage{amsfonts}       
\usepackage{bbm}            
\usepackage{nicefrac}       
\usepackage{xcolor}         
\usepackage{float}          
\usepackage{graphicx}

\newtheorem{theorem}{Theorem}[section]
\newtheorem{definition}[theorem]{Definition}
\newtheorem{proposition}[theorem]{Proposition}
\newtheorem{corollary}[theorem]{Corollary}
\newtheorem{remark}[theorem]{Remark}

\newcommand{\bitem}{\begin{itemize}}
\newcommand{\eitem}{\end{itemize}}
\newcommand{\mc}[1]{\mathcal{#1}}

\newcommand{\N}{\mathbb{N}}
\newcommand{\EE}{\mathbb{E}}

\newcommand{\bpm}{\begin{pmatrix}}
\newcommand{\epm}{\end{pmatrix}}
\newcommand{\bvm}{\begin{vmatrix}}
\newcommand{\bsm}{\left(\begin{smallmatrix}}
\newcommand{\esm}{\end{smallmatrix}\right)}
\newcommand{\T}{\top}

\newcommand{\wh}[1]{\widehat{#1}}
\newcommand{\wt}[1]{\widetilde{#1}}
\newcommand{\la}{\langle}
\newcommand{\mrm}[1]{\mathrm{#1}}

\newcommand{\veps}{\varepsilon}

\newcommand{\eins}{\mathbbm{1}}

\definecolor{defred}{RGB}{200,0,0}

\DeclareMathSymbol{\mydiv}{\mathbin}{symbols}{"04}

\DeclareMathOperator{\NLL}{NLL}

\DeclareMathOperator{\MHD}{MHD}

\DeclareMathOperator{\Cross}{Cross-Entropy}

\DeclareMathOperator{\round}{round}

\DeclareMathOperator{\Span}{span}

\DeclareMathOperator{\CFM}{CFM}

\makeatletter
\def\widebreve{\mathpalette\wide@breve}
\def\wide@breve#1#2{\sbox\z@{$#1#2$}%
     \mathop{\vbox{\m@th\ialign{##\crcr
\kern0.08em\brevefill#1{0.8\wd\z@}\crcr\noalign{\nointerlineskip}%
                    $\hss#1#2\hss$\crcr}}}\limits}
\def\brevefill#1#2{$\m@th\sbox\tw@{$#1($}%
  \hss\resizebox{#2}{\wd\tw@}{\rotatebox[origin=c]{90}{\upshape(}}\hss$}
\makeatother

\title{Generative Modeling of Discrete Data \\ Using Geometric Latent Subspaces}

\author{%
Daniel Gonzalez-Alvarado
\textsuperscript{1,2,3}\thanks{Correspondence to: \texttt{daniel.gonzalez@iwr.uni-heidelberg.de}.}
\quad
Jonas Cassel\textsuperscript{1,4}
\quad
Stefania Petra\textsuperscript{1,2,3}
\quad
Christoph Schnörr\textsuperscript{1,3,4}
\AND
\textsuperscript{1}Institute for Mathematics, Heidelberg University\\
\textsuperscript{2}Mannheim Institute for Intelligent Systems in Medicine, Heidelberg University\\
\textsuperscript{3}IWR, Heidelberg University\\
\textsuperscript{4}Research Station Geometry and Dynamics, Heidelberg University
}

\iclrfinalcopy 
\begin{document}

\maketitle
\lhead{} 
\renewcommand{\headrulewidth}{0pt}
\begingroup
\renewcommand{\thefootnote}{}
\footnotetext{Preprint.}
\endgroup

\begin{abstract}
We propose a geometric latent-subspace framework for generative modeling of discrete data. Specifically, we introduce latent subspaces in the exponential parameter space of product manifolds of categorical distributions as a novel approach to learning low-dimensional representations of high-dimensional discrete data. The resulting low-dimensional latent space captures statistical dependencies and removes redundant degrees of freedom among the categorical variables. We equip the parameter domain with a Riemannian geometry such that the latent subspace and induced data manifold are related isometrically, enabling consistent flow matching. Exploiting this structure, we propose a geometry-aware dimensionality reduction objective, called geometric PCA (GPCA), which we formulate as a regularized cross-entropy minimization that encourages small Riemannian distances between the data and their reconstructions. In particular, under the induced geometry, geodesics correspond to straight lines in the latent parameter space, allowing flow matching to be performed directly in reduced coordinates. Empirical results show that low-dimensional latent representations suffice to accurately model high-dimensional discrete data and enable substantially more computationally efficient flow matching.
\end{abstract}

\section{Introduction}\label{sec:Introduction}
Generative modeling of discrete data constitutes a major area of research in machine learning. One strand of this research develops Flow Matching methods on categorical statistical manifolds to learn high dimensional joint distributions of discrete random variables, see \cite{Stark:2024,Boll:2024ab,Boll:2025aa,Davis:2024,Cheng:2024,Cheng:2025,Williams:2025}, among others. Key questions in this context concern the \textit{representation} of discrete data and the \textit{geometry} of the spaces on which the corresponding probability distributions are defined.

Our \textbf{main contribution} in this paper combines 
\textit{continuous and geometric} modeling to provide
\textbf{(1)} the accurate \textit{encoding} and \textit{representation} of high dimensional \textit{discrete} data using Geometric PCA (GPCA),
\textbf{(2)} a \textit{Riemannian geometry} that isometrically relates the latent GPCA subspace to the induced data manifold, and
\textbf{(3)} an efficient \textit{generative model} based on Flow Matching in low dimensional manifold structure.

To this end, we extend the \textit{generalized PCA} approach \cite{Collins:2001aa} to categorical statistical manifolds \cite{Amari:2000aa,Amari:2016}. We equip the statistical manifold with the $e$-metric rather than the Fisher-Rao metric, which makes the latent subspace isometric to the induced data manifold. We then learn this subspace using a regularized cross-entropy objective that controls the Riemannian reconstruction error. We call the representation \textit{Geometric PCA (GPCA)}. Under this geometry, \textit{geodesics} are \textit{straight lines} in the low-dimensional GPCA coordinates, which makes Flow Matching computationally efficient. 
Figure~\ref{fig:GPCA-approach} summarizes our approach, and Appendix~\ref{appendix:representation-synthetic} gives a geometric picture through synthetic examples. 
We demonstrate empirically on realistic high-dimensional datasets that GPCA achieves significant dimensionality reduction.
This highlights the basic manifold assumption of data science and machine learning 
\cite{Fefferman:2016aa,Goldt:2020aa,Meila:2024aa}, here specifically for \textit{discrete} data and generative models.

Our paper is organized as follows. Section~\ref{sec:CD-GPCA} introduces the basic information geometry of categorical statistical manifolds and generalized PCA. Section~\ref{sec:Riemann-Isometry} introduces the exponential geometry compatible with the GPCA parametrization. Section~\ref{sec:Flow-Matching} describes the training of the generative model using Flow Matching in the GPCA representation. A more detailed discussion of related work and our contribution is provided in Section \ref{sec:Related-Work-Discussion}. Section \ref{sec:Experiments} underpin quantitatively our claims. 

Frequently used symbols and basic notation are listed in Appendix~\ref{appendix:notation}. Mathematical proofs and additional illustrative figures are provided in Appendices~\ref{appendix:proofs}, \ref{appendix:gpca-empirical-performance}, \ref{sec:generation-metrics}, and~\ref{sec:generation-novelty}.

\begin{figure}[H]
\centering
\includegraphics[width=\linewidth]{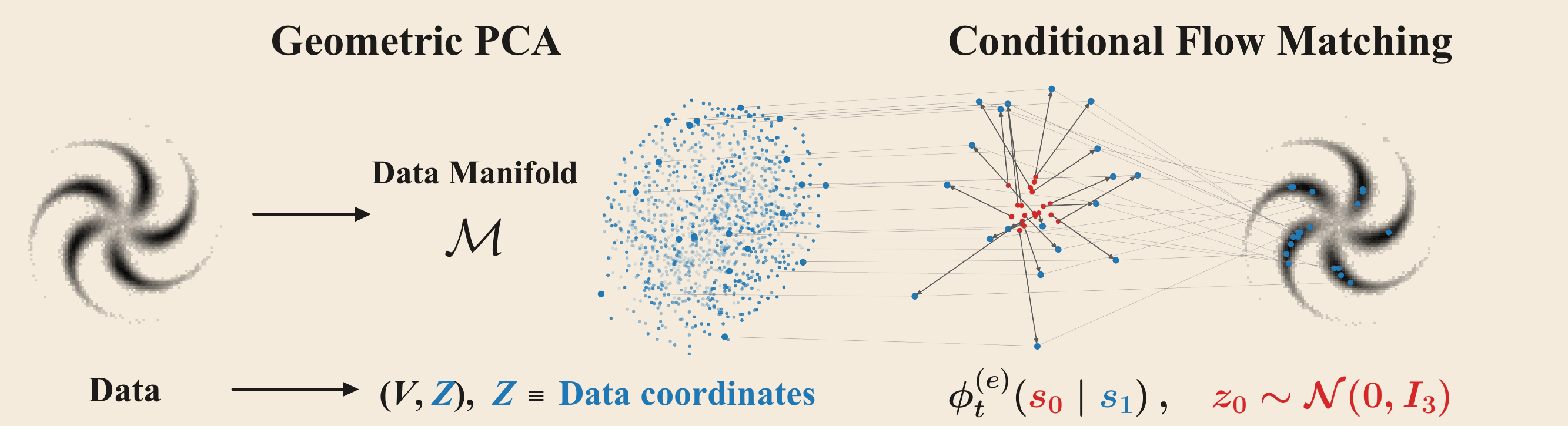}
\caption{
\textbf{Approach (sketch):}
Geometric PCA represents the data on a nonlinear low dimensional manifold $\mathcal M$ through an exponential parametrization specified by $(V,Z)$. The conditional diffeomorphism $\phi_t^{(e)}$, defines a geodesic interpolation between base samples $s_0=\varphi^{-1}(V\odot z_0)$, marked in red, and data reconstructions $s_1=\varphi^{-1}(V\odot z)$, represented in blue. In exponential coordinates, these geodesic interpolants are represented by straight lines, making flow matching computationally efficient while remaining compatible with the underlying manifold geometry.}
\label{fig:GPCA-approach}
\end{figure}

\section{Discrete distributions and generalized PCA}\label{sec:CD-GPCA}
\subsection{Geometry of categorical distributions}
\label{sec:categorical-statistical-manifold}
We recall some basic concepts from information geometry of discrete distributions \citep{Amari:2000aa,Amari:2016}. Every discrete probability measure $p$ on a finite set of categories $\mc C=[c]:=\{1,\dotsc,c\}$ is uniquely determined by its singleton probabilities $p_j=p(\{j\})$, $j\in\mc C$, and can therefore be identified with a point in the probability simplex. Throughout this work, we restrict attention to the relative interior of the simplex, corresponding to categorical distributions with full support, which we denote by
$\mc S_c :=\Bigl\{p\in\R_{>0}^c:\ \sum_{j\in[c]}p_j=1\Bigr\}.$
The set $\mc S_c$ is a smooth $(c-1)$-dimensional statistical manifold, and the additive log-ratio transformation
\begin{equation}
\varphi:\mc S_c\to\R^{c-1},\qquad \varphi(p) :=\left(\log\frac{p_j}{p_c}\right)_{j\in[c-1]} \label{eq:exp-coord-chart}
\end{equation}
is a diffeomorphism and defines a global coordinate chart on $\mc S_c$. Its inverse $p_\theta:=\varphi^{-1}(\theta)$ yields a minimal exponential-family parametrization of categorical distributions \citep{Barndorff-Nielsen:1978aa, Aitchison:1982vt, Brown:1986vy}. More precisely, defining the sufficient statistic $T(j):=(\delta_{jk}),$ $k\in[c-1]$, $j\in[c]$, where $\delta_{jk}$ denotes the Kronecker delta, we have
\begin{equation}
p_\theta(j) =\exp\bigl(\la\theta,T(j)\rangle-\psi(\theta)\bigr), \quad j\in[c], \qquad
\psi(\theta) :=\log\left(1+\sum_{k=1}^{c-1}e^{\theta_k}\right), \quad \theta\in\R^{c-1}. \label{eq:categorical-log-partition}
\end{equation}
Here, in particular, $T(c)=0\in\R^{c-1}$ and therefore ${p_\theta(c)}={\bigl(1+\sum_{k=1}^{c-1}e^{\theta_k}\bigr)^{-1}}$. 

The log-partition function $\psi:\R^{c-1}\to\R$ is smooth, strictly convex, and its Hessian gives the coordinate representation of the Fisher–Rao metric in natural coordinates \cite[Theorem~2.1]{Amari:2016}.
The gradient $\nabla\psi(\theta)$ coincides with the first $c-1$ components of $\varphi^{-1}(\theta)$, with the last component recovered by normalization. The convex conjugate $\psi^\ast$ is the negative entropy \citep[Section~1.4]{Amari:2016}, and its Bregman divergence \cite{Zhang:2004aa,Basseville:2013aa} coincides with the Kullback--Leibler divergence
\begin{equation}
    D_{\psi^\ast}(\eta,\wt\eta)
    = \KL(p,\wt p).
    \label{eq:categorical-bregman-kl}
\end{equation}
Here, $\eta$ and $\wt\eta$ contain the first $c-1$ probabilities
of $p,\wt p\in\mc S_c$, respectively.

Throughout this work, we consider product distributions with $n$ categorical factors on $\mc S_c^n := \mc S_c \times \dotsb \times \mc S_c$.
We regard $\mc S_c^n$ as a product statistical manifold and, depending on the context, equip it with a Riemannian metric, such as the Fisher-Rao metric or with the exponential geometry introduced in Section~\ref{sec:Riemann-Isometry}.
We emphasize that, although our geometric construction is formulated on the product space $\mc S_c^n$, the distribution learned by our model is not required to factorize and may represent a general discrete joint distribution.  

\textbf{Data representation.}
We work with data $\mc X_N^n=\{x_i\}_{i\in[N]}$ consisting of $N$ data points $x_i=(x_{i1},\dotsc,x_{in})^\T\in\mc X^n$, where each component takes one of $c$ categorical values in $\mc C$. Under the one-hot encoding, each data point is identified with a corner of the product simplex. We perturb these points slightly toward the barycenter and work on the full-supported manifold $\mc S_c^n$. 

\subsection{Generalized PCA and induced data manifolds} \label{sec:GPCA-Latent}
Following the \textit{generalized PCA} framework of \citet{Collins:2001aa}, we parametrize categorical distributions through a linear subspace of the natural-parameter space $\Theta:=\R^{n\times(c-1)}$,
\begin{equation}
\mc U :=\Span\{V^1,\dotsc,V^d\} \subset\R^{n\times(c-1)}, \qquad \textbf{(GPCA subspace)}
\label{eq:gpca-subspace}
\end{equation}
of dimension $ d\ll n(c-1)$, where $V^\ell\in\R^{n\times(c-1)}$ denotes the $\ell$-th basis element. Given latent coordinates $z\in \R^d$, we write the corresponding natural parameter as 
\begin{equation}
\wh \theta :=\sum_{\ell=1}^d z_\ell V^\ell =:V\odot z \in\mc U. \label{eq:gpca-natural-parameter}
\end{equation}
The basis vectors $V=(V^{1},\dots,V^{d})$ along with latent coefficient vectors $Z=(Z_1,\dotsc,Z_N)^\T$ associated with the training data $\mc X_N^n$ are obtained by minimizing the empirical objective
\begin{equation}
    \mc L(V,Z):=\frac{1}{Nn}\sum_{i=1}^N\sum_{j=1}^n \KL\left(x_{ij}, \varphi^{-1}\bigl((V\odot Z_i)_j\bigr)    \right). 
    \label{eq:gpca-objective}
\end{equation}
Here and below, $\varphi^{-1}$ is applied factorwise, with the last probability recovered by normalization as described above. We note that \citet{Collins:2001aa} originally formulated generalized PCA as a Bregman divergence minimization. We write it here in terms of KL divergence by \eqref{eq:categorical-bregman-kl}. Since the entropy of the data is fixed, minimizing this objective is equivalent to minimizing the corresponding cross-entropy loss. The optimization can be carried out by alternating over convex minimization subproblems or by gradient-based methods using automatic differentiation. We refer to Appendix~\ref{sec:GPCA-Loss} for further details on the objective and its optimization.

The GPCA subspace $\mc U$ induces the nonlinear \textbf{data manifold} $\mc M := \varphi^{-1}(\mc U)$.
The inverse coordinate map $\varphi^{-1}$ acts as a \textit{decoder}. Conversely, $\varphi$ maps categorical distributions back to natural parameters. Once the GPCA parameters $(\wh V,\wh Z)$ are obtained by minimizing the objective \eqref{eq:gpca-objective}, new data can be projected onto $\mc M$ via
\begin{equation}\label{eq:bregman-projection}
\Pi_{\mc{M}}(x)=  \varphi^{-1}(\wh V \odot \wh z),\; \wh z= \argmin_{z\in \R^d} \frac{1}{n} \sum_{j=1}^n \KL\big(x_{j},\varphi^{-1}((\wh V\odot z)_{j})\big). 
\end{equation} 
By \eqref{eq:categorical-bregman-kl}, this corresponds to a Bregman projection.
Figure~\ref{fig:GPCA-approach} sketches the GPCA procedure and our approach for learning generative models of discrete data. 


\section{Exponential geometry and geometric PCA} \label{sec:Riemann-Isometry}
\subsection{Riemannian structure induced on $\mc M$}
The generalized PCA construction defines a linear subspace $\mc U\subset\R^{n\times(c-1)}$ and the reconstructed data manifold $\mc M= \varphi^{-1}(\mc U)$. To formulate flow matching on $\mc M$ in the next section, the geometry on $\mc S_c^n$ should preserve the relation between the linear latent coordinates in $\mc U$ and their reconstructed distributions in $\mc M$. The Fisher–Rao metric does not generally have this property, since its Riemannian geodesics are not straight lines in natural coordinates and geodesics between two points of $\mc M$ need not remain in $\mc M$. We therefore pull back the Euclidean metric on $\R^{c-1}$ through the natural-coordinate chart $\varphi$ defined in \eqref{eq:exp-coord-chart}. The resulting geometry makes $\varphi$ a Riemannian isometry, so that straight lines in $\mc U$ are mapped to geodesics that remain in $\mc M$. This leads to the following definition.
\begin{definition}[$e$-metric]\label{def:e-metric-nodes}
We define the $e$-metric on $\mc S_c$ as the Riemannian metric obtained by pulling back the Euclidean metric on $\R^{c-1}$ through the natural-coordinate chart $\varphi$ defined in \eqref{eq:exp-coord-chart}. For $p\in\mc S_c$ and $u,v\in T_p\mc S_c$, it is given by
\begin{equation}
    g_p^{(e)}(u,v) := \left\langle \mathrm d\varphi_p[u],\mathrm d\varphi_p[v]\right\rangle. \label{eq:e-metric-nodes}
\end{equation}
For $v\in T_p\mc S_c$, we denote the induced tangent norm by $\|v\|_{e,p}^2 := g_p^{(e)}(v,v)$.
\end{definition}

The coordinate matrix of the $e$-metric in natural coordinates is the identity matrix $I_{c-1}$. Hence, the global coordinate chart $\varphi$ is a Riemannian isometry from $(\mc S_c,g^{(e)})$ to $(\R^{c-1},\langle\cdot,\cdot\rangle)$. The length of a curve in $\mc S_c$ therefore agrees with the Euclidean length of its image under $\varphi$, and the induced Riemannian distance is
\begin{equation}
    d_e(p,q)=\left\|\varphi(p)-\varphi(q)\right\|_2,\qquad p,q\in\mc S_c \quad \textbf{(Riemannian distance)}. \label{eq:e-RiemannDistance}
\end{equation}
We extend the $e$-metric and its distance factorwise to $\mc S_c^n$, so that $d_e^2(s,\widetilde s)=\sum_{j=1}^n d_e^2(s_j,\widetilde s_j)$.
Furthermore, Proposition~\ref{prop:e-connection} shows that the associated Levi--Civita connection coincides with the exponential connection $\nabla^{(e)}$, studied in information geometry \cite{Amari:2000aa,Amari:2016}. For $s_0,s_1\in\mc S_c^n$, we write $\exp_{s_0}^{(e)}(\cdot)$ and $\log_{s_0}^{(e)}(\cdot)$ for the exponential and logarithmic maps associated with $\nabla^{(e)}$ at $s_0$, respectively. The $e$-geodesic interpolant between $s_0$ and $s_1$ is then given by
\begin{equation}
    \phi_t^{(e)}(s_0\mid s_1)=\exp_{s_0}^{(e)}\left(t\log_{s_0}^{(e)}(s_1)\right),\qquad t\in[0,1] \quad \textbf{(geodesic interpolant)}. \label{eq:geodesic-interpolant-e}
\end{equation}

In the following propositions we show that the Riemannian structure on $\mc{S}_{c}^{n}$ defined above captures the $e$-geometry. 
In particular, geodesic interpolants between reconstructed points in $\mc M$ (cf.~Figure \ref{fig:GPCA-approach}) remain close, in terms of the Riemannian distance \eqref{eq:e-RiemannDistance}, to the corresponding geodesic interpolants in $\mc S_c^n$ whenever the corresponding initial and terminal points are close.

\begin{proposition}\label{prop:e-connection}
The Riemannian connection associated with the $e$-metric $g_e$ on $\mc{S}_{c}^{n}$ coincides with the $e$-connection $\nabla^{(e)}$.
\end{proposition}


\begin{proposition}[isometry relations]\label{prop:isom}
	The mapping $ \varphi^{-1}: \mc{U} \to \mc M \subset \mc{S}_c^{n}$ defines an isometric embedding of $\mc U$ equipped with the Euclidean metric, into $\mc{S}_{c}^{n}$ is equipped with the metric $g_e$. 
	In particular, $\varphi^{-1}$ maps geodesics in $\mc U$, i.e. lines of the form  
	$\theta_t {=} (1-t)\theta_0 + t \theta_1$, to $e$-geodesics in $\mc{S}_c^{n}$.
\end{proposition}

\begin{proposition}[geodesic interpolants on $\mc M$ and $\mc S_c^n$]
\label{prop:geod-approx}
Let $p_0,p_1\in\mc S_c^n$ and let $\wh p_0,\wh p_1\in\mc M$ denote their latent subspace approximations. If the initial and terminal
points are close to their approximations in Riemannian distance, then
the corresponding geodesic interpolants remain uniformly close
\begin{equation}
{d}_e(p_i,\wh p_i)\leq\varepsilon,\quad i=0,1 \quad\Longrightarrow\quad {d}_e\bigl( \phi_t^{(e)}(p_0\mid p_1), \phi_t^{(e)}(\wh p_0\mid\wh p_1) \bigr) \leq\varepsilon,\quad t\in[0,1].
\end{equation}
\end{proposition}
We refer to Appendix~\ref{sec:app-proofs-riemann-isom} for the proof.

\subsection{Geometric PCA}
Proposition~\ref{prop:geod-approx} shows that geodesic interpolants on the induced submanifold $\mc M = \varphi^{-1}(\mc U)$ remain close to the corresponding geodesics interpolants in the data space, provided that the reconstructions are close to the data in Riemannian distance. In this regime, Corollary~\ref{cor:cfm-approximation-ds} further provides a corresponding approximation bound for the flow matching objective on the data space. 

In the next proposition, we show that the Riemannian reconstruction error decreases together with the GPCA loss.

\begin{proposition}\label{prop:kl-edistance-compact}
Let $K \subset \mc S_c$ be a compact subset such that $\varphi(K)$ is convex. Then there exist constants $C_1,C_2 >0$ such that, for all $p,p'\in K$,
\begin{equation}
C_1 {d}_e^2(p,p') \;\le\;\KL(p,p') \;\le\;C_2 {d}_e^2(p,p').
\end{equation}
In particular, on $K$, the KL divergence and the squared $e$-distance are equivalent up to multiplicative constants.
\end{proposition}

\textbf{Remark.} In practice, minimizing the GPCA loss in \eqref{eq:gpca-objective} does not necessarily ensure that the reconstructions remain within a fixed compact subset of $\mc S_c^n$. Consequently the cross-entropy loss may decrease substantially while the Riemannian reconstruction error does not. Consequently, we address this issue by regularizing the GPCA objective to discourage reconstructed points from approaching the boundary.

\textbf{Geometric regularization.}
We introduce an explicit $e$-distance penalty and minimize the \textbf{GPCA Loss}
\begin{equation}\label{eq:loss-gpca-e_reg}
\mc L_{\lambda}(V, Z){=}  \frac{1}{Nn} \sum_{i=1}^{N} \sum_{j=1}^n \KL(x_{ij} , \varphi^{-1}((V \odot Z_i)_j)) {+} \lambda \, {d}_e^2(x_{ij}, \varphi^{-1}((V \odot Z_i)_j)) .
\end{equation}
The parameter $\lambda \ge 0$ controls the trade-off between KL reconstruction error and Riemannian reconstruction error. This regularization provides improved control of the $e$-geodesic reconstruction on $\mc M$ relative to that on $\mc S_c^n$. 
Empirically, we find that minimizing the objective in \eqref{eq:loss-gpca-e_reg} reduces the Riemannian reconstruction error while retaining low cross-entropy and mean Hamming reconstruction errors.
The regularization discourages reconstructions from approaching the boundary of $\mc S_c^n$.
Section~\ref{sec:Experiments} and Appendix~\ref{appendix:gpca-empirical-performance} demonstrate these effects and compare the previously introduced objectives.
%
%
Due to this geometric set-up we call our method \textit{Geometric PCA (GPCA)}. 

\section{Flow matching on $\mc{M}$}\label{sec:Flow-Matching}

The goal of \textit{Flow Matching (FM)} \cite{Liu:2022,Albergo:2022,Lipman:2023},  is to learn a time-dependent vector field which generates a probability path $P_t$ interpolating between a base distribution $P_0$ and the data distribution $P$. 
When the data distribution is discrete, several works \cite{Stark:2024, Boll:2024ab,Boll:2025aa,Davis:2024,Cheng:2024,Cheng:2025,Williams:2025} proposed to work within the simplex geometry, following the Riemannian FM framework \cite{Chen:2023}.

The goal is to learn a time-dependent vector field 
\begin{equation}\label{eq:fm-vf}
    u_t: \mc S_c^n \; \to \; T_0^n,
\end{equation}
which induces a flow that generates the probability path $P_t = (\phi_t)_\sharp P_0$ satifying $P_{t=0} = P_0$ and $P_{t=1} = P$. The induced flow $\phi_t$ is the unique solution of the ordinary differential equation (ODE)
\begin{equation}\label{eq:flow-equation}
\frac{d}{dt} \phi_t(x) = u_t(\phi_t(x)), \quad \phi_0(x) = x.
\end{equation}
The choice of the base distribution $P_0$ is arbitrary for these models. 
Common choices are uniform distribution on the simplex or a transformed Gaussian distribution on the simplex.
The time horizon of the flows can also be extended to $\R_{+}$.

Since neither $P_t$ nor $u_t$ are available in closed form, the FM framework instroduces geodesic interpolants $\phi_t(\cdot| x_1) : \mc S_c^n \to \mc S_c^n$ \textit{conditioned} on the data, and to minimize a conditional FM objective.
The choice of conditional interpolants and underlying geometry used on the FM objective may vary. 
We propose to work within the $e$-geometry, and minimize the conditional FM objective
\begin{equation}\label{eq:cfm-obj}
    \mc L^{(e)}_{\CFM}(w) = \EE_{t, x_0, x_1} \Big[ \| v_t(x_t; w) - \partial_t \phi^{(e)}_t(x_0| x_1) \|_{e,x_t}^2 \Big],
\end{equation}
with $t \sim \mrm{unif}(0,1)$, $x_0 \sim P_0$, $x_1 \sim P_1$, $\| \cdot \|_{e,x_t}$ being the norm in Definition \ref{def:e-metric-nodes}, and $x_t = \phi_t^{(e)}(x_0\mid x_1)$ being the geodesic interpolants in \eqref{eq:geodesic-interpolant-e}.
Our choice of conditional interpolants and the Riemannian structure allows
the following formulation in \textit{natural} parameters.

\begin{proposition}[CFM in natural parameters]\label{prop:cfm-objective-equivalence}
The CFM objective \eqref{eq:cfm-obj} is equivalent to
\begin{equation}\label{eq:cfm-objective-U}
    \mc L^{(e)}_{\CFM}(w) = \EE_{t, \theta_0, \theta_1} \Big[ \|  v^{\psi}_t(\theta_t; w) - (\theta_1 - \theta_0) \|_2^2 \Big],
\end{equation}
where $\theta_0 \sim \varphi_{\sharp}P_0$ and $\theta_1 \sim \varphi_{\sharp}P_1$ are the natural parameters corresponding to $x_0$ and $x_1$, ${\theta_t = t\theta_1 + (1-t)\theta_0}$, and $v_t^{\psi}(\theta_t; w) \coloneqq d\varphi_{\varphi^{-1}(\theta_t)}\big[v_t(\varphi^{-1}(\theta_t); w)\big].$
\end{proposition}
The objective function \eqref{eq:cfm-objective-U} shows that flow matching with respect to the $e$-geometry of the simplex is performed via \textit{straight line interpolation} in the space of natural parameters.

\begin{remark}[computational savings]\label{rem:com-savings}
The mapping 
\begin{equation} 
    \mc{Z} \xrightarrow{V} \mc{U} \xrightarrow{\varphi^{-1}} \mc{M}\subset \mc{S}_c^{n}
\end{equation}
allows us to pull back the metric $g_e$ from $\mc{S}_{c}^{n}$ to $\mc{Z}$. The resulting scalar product on $\mc{Z}$ is given by 
$\la Vz, Vz' \rangle$, for all $z,z'$ in the tangent bundle $T\mc{Z}$. We denote the corresponding norm function by $\|\cdot \|_{V}$.
The CFM objective in $\mc Z$ then reads 
\begin{equation}\label{eq:cfm-objective-Z}
    \mc L_{\CFM}^{\mc Z}(w_z) = \EE_{t, z_0, z_1} \Big[ \| v_t(z_t; w) - (z_1- z_0)  \|_{V}^2 \Big].
\end{equation}
Choosing $p_z = \mc N(0, I_d)$ yields the normal distribution $V_\sharp p_z$ in $\mc{U}$, which we can use in the objective the objective in \eqref{eq:cfm-objective-U} to identify both CFM objectives. 
This allows for significant computational savings through lower-dimensional representations and smaller network architectures, since the network architecture operates in the lower-dimensional space $\mc Z$.
\end{remark}

\begin{corollary}[approximate flow matching]\label{cor:cfm-approximation-ds}
    Let $(V,Z)$ be the GPCA parameters obtained from minimizing the GPCA objective \eqref{eq:loss-gpca-e_reg}, and assume that $v_t^{\psi}(\cdot;w)$, defined in Proposition \ref{prop:cfm-objective-equivalence}, is L-Lipschitz uniformly in $t$. If for some $\varepsilon>0$, the reconstruction error satisfies
        \begin{equation} \label{eq:assumption-approximation}
        d_{e}(x_i, \wh x_i)\leq \sqrt{\varepsilon}, \quad x_i \sim P_i,  \; \wh x_i = \Pi_{\mc{M}}(x_i), \; i=0,1, 
    \end{equation}
    then it also holds
    \begin{multline}
    \EE_{t, x_0, x_1} \Big[ \| v_t(x_t; w) - \partial_t \phi^{(e)}_t(x_0| x_1) \|_{e,x_t}^2 \Big] \\
    \le (1+ \sqrt{\veps})\EE_{t, \wh x_0, \wh x_1} \Big[ \| v_t(\wh x_t; w) - \partial_t \phi^{(e)}_t(\wh x_0| \wh x_1) \|_{e,\wh x_t}^2 \Big] +
    (L+2)^2(\varepsilon+\sqrt{\varepsilon}),
    \end{multline}
    where $t \sim \mrm{unif}(0,1),x_0 \sim P_0,x_1 \sim P_1$ and 
    ${\wh x_0 \sim (\Pi_{\mc{M}})_{\sharp} P_0},\wh x_1 \sim (\Pi_{\mc{M}})_{\sharp} P_1$.
\end{corollary}

Corollary~\ref{cor:cfm-approximation-ds} and Proposition~\ref{prop:cfm-objective-equivalence} motivate flow matching on $\mathcal M$ and $\mathcal U$, respectively. We choose the target distribution $\widehat P_1=(\Pi_{\mathcal M})_\sharp P_1$, whose reconstructions have natural parameters in $\mathcal U$. The assumption in \eqref{eq:assumption-approximation}, for $i=1$, is addressed with the objective \eqref{eq:loss-gpca-e_reg}, which controls the reconstruction error. See also Section~\ref{sec:Experiments} and Appendix~\ref{appendix:gpca-empirical-performance} for data evaluations. Choosing $P_0$ supported on $\mathcal M$ ensures zero reconstruction error and the assumption in \eqref{eq:assumption-approximation} for $i=0$. In practice, $v_t$ is parametrized by a neural network, for which we assume the uniform Lipschitz bound required by Corollary~\ref{cor:cfm-approximation-ds}.
We refer to Appendix \ref{sec:app-proofs-fm} for the proofs.

\section{Related work and discussion}
\label{sec:Related-Work-Discussion}


We distinguish between \textit{prior} and \textit{recent and current} related work, briefly reporting both in order to position our contribution and to point out limitations and perspectives.

\textbf{Prior related work.} 
A major component of our approach, viz.~employing \textit{geometric} PCA in the canonical (exponential) parameter space of discrete distributions, has been motivated by a precursor,  namely
\textit{generalized} PCA, introduced by \cite{Collins:2001aa}. 
The basic idea is to regard the squared Euclidean norm, which is used as the objective function for Euclidean data compression in classical principal component analysis (PCA) \cite{Jolliffe:2002aa}, as a special divergence function \cite{Basseville:2013aa} and to replace it with a general Bregman divergence induced by the log-partition function of a regular, minimally represented exponential family of distributions 
\cite{Brown:1986vy}, \cite{Barndorff-Nielsen:1978aa}. 
From a \textit{geometric} viewpoint, Bregman divergences play a prominent role in the general theory of contrast functions \cite{Matumoto:1993aa,Zhang:2004aa}, in particular in connection with statistical manifolds and information geometry \cite{Lauritzen:1987aa,Amari:2000aa,Ay:2017aa}.
\\[0.1cm]
Historically, the set-up of GPCA relates through the notion of a `link function' associated with the generalized linear models of classical statistics for regression \cite{McCullagh:1989aa,Critchley:1994aa,Warmuth:2001aa} and applications to online learning \cite{Kivinen:1997aa,Mahony:2001aa}. The gradients of the aforementioned log-partition functions of exponential families of distributions provide a set of canonical link functions.
\textbf{Our approach} elaborates this idea for joint distributions of \textit{discrete} high-dimensional data and combines it with a \textit{geometric} latent space suited for generative modeling (see Sections \ref{sec:CD-GPCA} and \ref{sec:Riemann-Isometry}).


\textbf{Recent and current related work.} 
Several recent works formulate generative modeling of discrete data through flow matching on statistical manifolds. The information-geometric approaches of \cite{Boll:2024ab,Boll:2025aa} consider statistical product manifolds of categorical distributions with full support, equipped with the canonical Fisher--Rao metric and the e-connection.
This set-up provides a specific novel instance of the general geometric approach \cite{Chen:2023} to flow matching \cite{Liu:2022,Albergo:2022, Lipman:2023}. The recent works \cite{Cheng:2024,Cheng:2025} also fall into this category. In addition, the latter paper studies various $\alpha$-connections besides the e-connection corresponding to $\alpha=1$. 
A common limitation of these approaches is that they do not explicitly employ latent subspaces. In particular, \citep[pp. 10]{Cheng:2024} point out the fixed input size and the independence assumption between classes as limitations. \textbf{Our GPCA subspace approach} addresses these issues in three respects: it enables efficient encoding and compression of discrete data, represents \textit{non}-factorizing discrete joint distributions and hence statistical dependencies, and carries a natural Riemannian geometry that isometrically relates latent and data spaces. See Sections \ref{sec:Experiments} and \ref{appendix:gpca-empirical-performance} for further analysis and empirical results. 

Regarding \textit{latent space models}, we point out that our encoding map (cf.~Figure \ref{fig:GPCA-approach}) is \textit{directly} given by information geometry and the geometric representation of discrete joint distributions as statistical manifolds. 
This stands in contrast to encoder maps that are parametrized and \textit{learned} by a network in more general \textit{non}-discrete scenarios like, e.g., in auto-encoder architectures \cite{Dao:2023,samaddar2025efficient}. However, despite the \textit{absence} of a \textit{pre-trained} encoding map and hence using \textit{purely geometric} and \textit{unsupervised} latent space embeddings in our approach (in particular, \underline{no} encoder training is used), the resulting data representations 
appear to be natural and semantically meaningful, cf. Figures~\ref{fig:MNIST-GPCA-embdding-2D} and \ref{fig:GPCA-embedding-3D}

Research on \textit{generative models} for \textit{discrete} data has become very active recently, e.g.~\cite{Stark:2024,Boll:2024ab,Boll:2025aa,Davis:2024,Cheng:2024,Cheng:2025}. We cannot provide a comprehensive review here. 
An interesting \textit{alternative geometric} approach has been proposed recently in the paper \cite{Williams:2025}. The authors employ Aitchison's geometry known from the subarea of statistical research concerned with compositional data \cite{Aitchison:1982vt,Filzmoser:2018vq} and point out as major ingredient of their approach a bijection between the probability simplex and a Euclidean domain, which enables to model and work in the latter space. In this context, we note that information geometry equips statistical manifolds with a dual geometric structure \cite{Amari:2000aa}. For the manifold of categorical distributions considered here, this structure is dually flat. 
A detailed comparison with compositional data analysis is given by \cite{Erb:2021aa}, who emphasize the invariance properties of the Fisher information metric and the role of Bregman divergences and the Pythagorean theorem of information geometry.

\section{Experiments and limitations}

\label{sec:Experiments}


\subsection{GPCA reconstruction error}\label{subsec:gpca-reconstruction-error}

We evaluate reconstruction quality using the mean Hamming distance (MHD), which measures the fraction of mislabeled nodes, cf. \eqref{eq:hamming-distance}, and additionally report the mean squared Riemannian distance ($e$-distance) in \eqref{eq:gpca-e-distance} to assess the geometric approximation error underlying the bounds in Proposition~\ref{prop:geod-approx} and Corollary~\ref{cor:cfm-approximation-ds}.
We vary the latent dimension $d$ and regularization parameter $\lambda$ across synthetic categorical data, binarized MNIST~\cite{Lecun:2010}, binarized Fashion-MNIST~\cite{Xiao:2017}, Cityscapes~\cite{Cordts:2016}, and promoter and enhancer DNA sequence datasets, following the experimental setups of \citet{Avdeyev:2023} and \citet{Stark:2024}, respectively.
For brevity, we report representative Cityscapes results in Table~\ref{tab:cityscapes-combined} and provide further dataset specifications, experimental configurations, and extended results in Appendix~\ref{appendix:gpca-empirical-performance}.

Across multiple datasets, we observe that the proposed GPCA objective decreases consistently with the mean Hamming distance, yielding low reconstruction error at relatively small latent dimensions. However, the corresponding Riemannian distances do not necessarily decrease at the same rate. Table~\ref{tab:cityscapes-combined} illustrates this behavior and the effect of regularization. At fixed $\lambda=10^{-3}$, increasing the latent dimension $d$ reduces cross-entropy, MHD, and $e$-distance, even achieving zero label reconstruction error on the full dataset for the reported latent dimensions $d\geq512$. In contrast, at fixed $d=256$, decreasing $\lambda$ reduces cross-entropy and MHD but increases the Riemannian distance. Thus, exact label reconstruction does not ensure small geometric approximation error, and the regularization parameter controls the trade-off between the two.

\begin{table}[H]
    \centering
\caption{Cityscapes reconstruction metrics for varying $d$ and $\lambda$.}    \label{tab:cityscapes-combined}
    \setlength{\tabcolsep}{3pt}
    \begin{tabular}{lcccccccc}
        \toprule
        & \multicolumn{4}{c}{Varying $\boldsymbol{d}$ with fixed $\boldsymbol{\lambda{=}10^{-3}}$}
        & \multicolumn{4}{c}{Varying $\boldsymbol{\lambda}$ with fixed $\boldsymbol{d{=}256}$} \\
        \cmidrule(r){2-5} \cmidrule(l){6-9}
        & $\boldsymbol{d{=}256}$ & $\boldsymbol{d{=}512}$ & $\boldsymbol{d{=}1024}$ & $\boldsymbol{d{=}2048}$
        & $\boldsymbol{\lambda{=}1}$ & $\boldsymbol{\lambda{=}0.01}$ & $\boldsymbol{\lambda{=}0.001}$ & $\boldsymbol{\lambda{=}0}$ \\
        \midrule
        \textbf{Cross-Entropy} & $0.119$ & $0.087$ & $0.074$ & $0.0689$
                      & $5.133$ & $0.207$ & $0.119$ & $0.099$ \\
        \textbf{MHD}       & $0.0004$ & $0.00$ & $0.00$ & $0.00$
                      & $0.144$ & $0.012$ & $0.0004$ & $0$ \\
        \textbf{e-Distance}  & $18.204$ & $8.418$ & $3.286$ & $0.868$
                      & $4.91$ & $7.24$ & $18.20$ & $58.23$ \\
        \bottomrule
    \end{tabular}
\end{table}

\subsubsection{GPCA vs PCA}\label{sec:gpca-vs-pca}
\enlargethispage*{0pt} 
\begin{figure}[H]
\centering
\makebox[0.2\linewidth][c]{\includegraphics[width=0.18\linewidth]{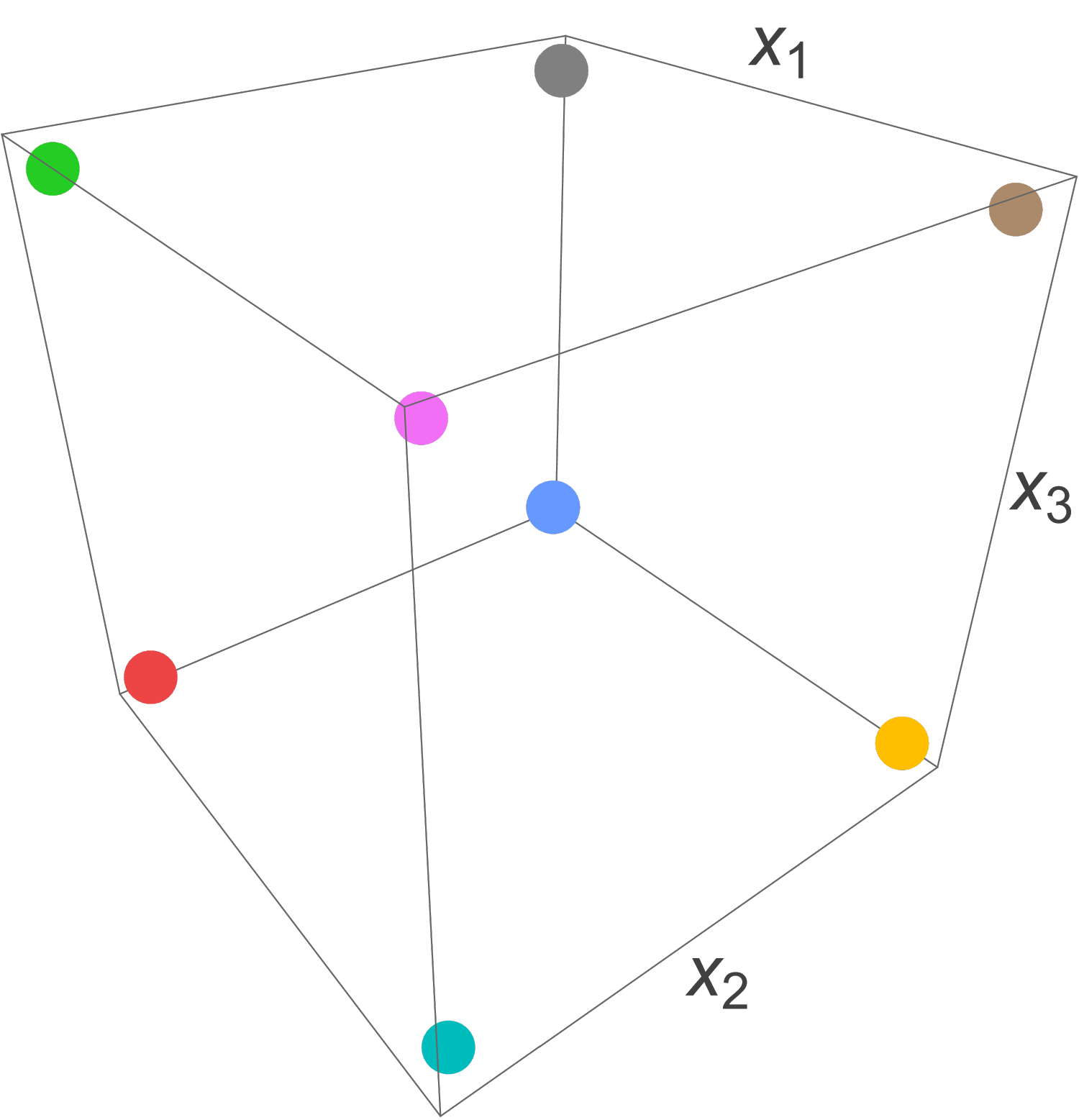}}
\hspace{0.02\linewidth}
\makebox[0.2\linewidth][c]{\includegraphics[width=0.18\linewidth]{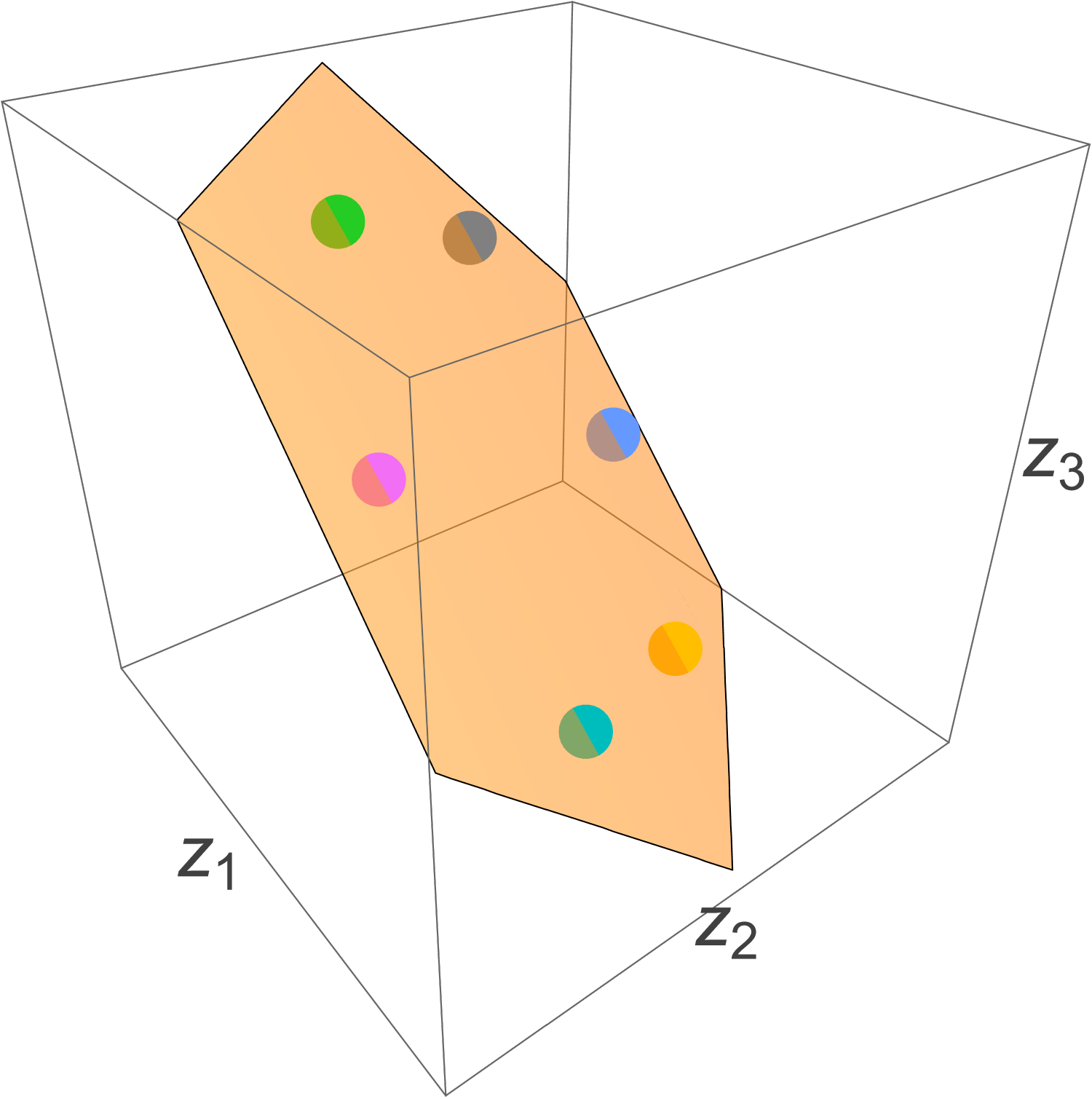}}
\hspace{0.02\linewidth}
\makebox[0.2\linewidth][c]{\includegraphics[width=0.18\linewidth]{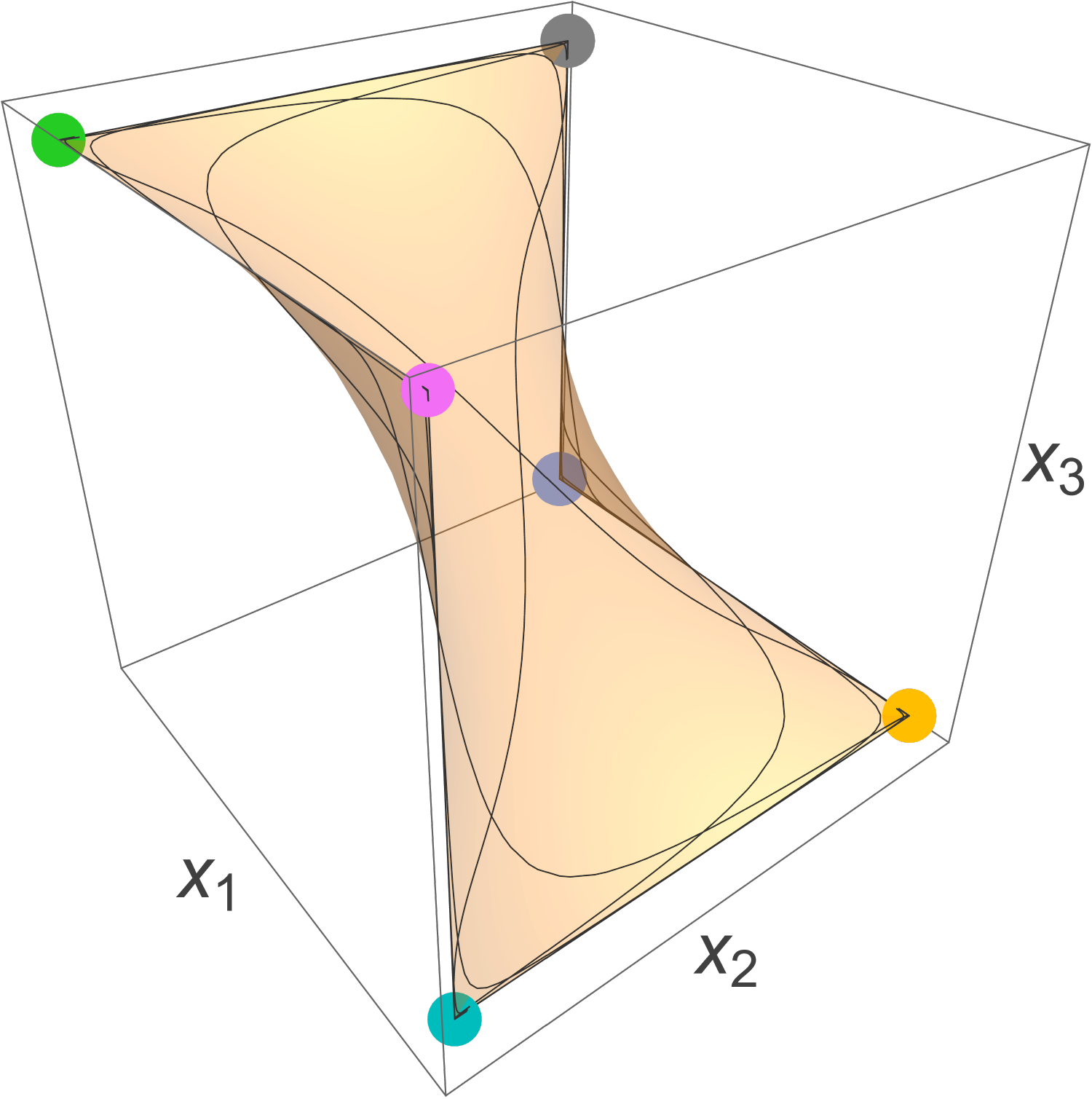}}
\hspace{0.02\linewidth}
\makebox[0.2\linewidth][c]{\includegraphics[width=0.18\linewidth]{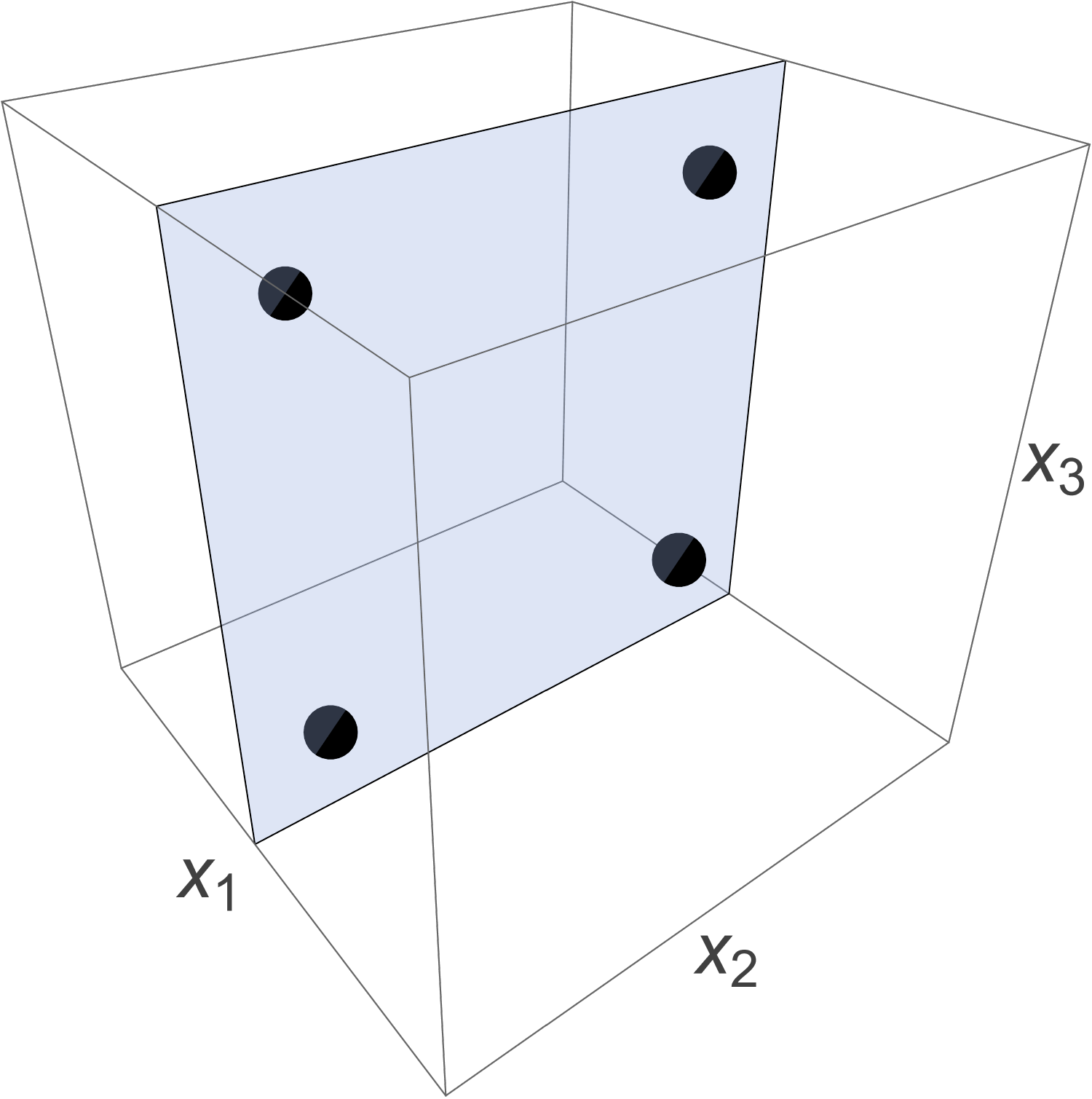}}
\parbox{0.17\linewidth}{\centering\footnotesize (a)}
\hspace{0.02\linewidth}
\parbox{0.17\linewidth}{\centering\footnotesize (b)}
\hspace{0.02\linewidth}
\parbox{0.17\linewidth}{\centering\footnotesize (c)}
\hspace{0.02\linewidth}
\parbox{0.17\linewidth}{\centering\footnotesize (d)}
\caption{\textbf{GPCA- vs.~PCA-subspaces.}
Toy example: nonlinear GPCA decoding recovers more discrete points
than the illustrated PCA projection at equal latent dimension.
\textbf{(a)} The hypercube $\mc X{=}\{0,1\}^3$.
Two-dimensional affine projection with rounding recovers at most
four vertices exactly.
\textbf{(b)} A two-dimensional \textit{geometric latent} GPCA
subspace recovers six vertices using \textit{real} coordinates.
\textbf{(c)} The \textit{nonlinear} manifold $\mc M$ induced by GPCA.
\textbf{(d)} since $\mc X$ is isotropic, PCA reveals no dominant direction.
The illustrated projection yields four nonintegral points outside $\mc X$.
Appendix~\ref{appendix:gpca-pca} compares GPCA and PCA reconstruction
errors across datasets at equal latent dimensions.}
\label{fig:H3-GPCA}
\end{figure}

%

\subsection{Generation quality}
\subsubsection{Binarized MNIST}
\noindent\begin{minipage}[t]{0.48\linewidth}
    \vspace{0pt}
We evaluate generation quality on binarized MNIST~\cite{Lecun:2010}, with $28^2$ binary pixels. Table~\ref{tab:mnist-generation-comparison} compares our method with D3PM~\cite{Austin:2021}, DDSM~\cite{Avdeyev:2023}, DirichletFM~\cite{Stark:2024}, DFM~\cite{Gat:2024}, and variants of SFM~\cite{Cheng:2024}, using baseline FID values from \citet{Cheng:2024}. Since learning is performed in the low-dimensional space $\mc Z$, we use a different architecture, as theirs is not suited to this representation. This reduces training time to under 20 minutes, compared with the 6 hours reported in \cite{Cheng:2024}. Appendix~\ref{sec:time-complexity} provides hardware and training details.
Table~\ref{tab:mnist-regularization-generation} compares MHD, $e$-distance, and FID for our model across regularization weights. Increasing $\lambda$ reduces the Riemannian distance but increases the Hamming distance, reflecting the trade-off discussed in Section~\ref{subsec:gpca-reconstruction-error}. The lowest FID occurs at $\lambda{=}10^{-4}$, indicating that moderate regularization benefits generation quality. Appendix~\ref{sec:generation-metrics} provides additional qualitative classifier metrics across latent dimensions in a separate experiment.
\end{minipage}\hfill
\begin{minipage}[t]{0.49\linewidth}
    \vspace{0pt}
    \captionsetup{type=table}
    \centering
        \caption{FID on binarized MNIST.}
        \label{tab:mnist-generation-comparison}
        \begin{tabular*}{\linewidth}{@{\extracolsep{\fill}}lc@{}}
            \toprule 
            \textbf{Model} & \textbf{FID} $\downarrow$ \\
            \midrule
            SFM w/ OT & $4.62$ \\
            SFM w/o OT & $5.15$ \\
            SFM w/o $\pi$ & $8.21$ \\
            LinearFM & $5.91$ \\
            DirichletFM & $77.35$ \\
            DDSM & $7.79$ \\
            D3PM & $67.36$ \\
            DFM & $15.34$ \\
            \midrule
            GPCA + FM (ours), $d=32$, $\lambda=10^{-4}$ & $\mathbf{1.89}$ \\
            \bottomrule
        \end{tabular*}

    \vspace{10pt}
        \caption{Metrics for binarized MNIST at $d=32$.}
        \label{tab:mnist-regularization-generation}
        \setlength{\tabcolsep}{3pt}
        \begin{tabular*}{\linewidth}{@{\extracolsep{\fill}}cccc@{}}
            \toprule
            \boldsymbol{$\lambda$} & \textbf{MHD} $\downarrow$ & \boldsymbol{$d_e^2\downarrow$} & \textbf{FID} $\downarrow$ \\
            \midrule
            $0$ & $0.0077$ & $ 14.453$ & $2.39$ \\
            $10^{-5}$ & $0.0078$ & $13.743$ & $2.10$ \\
            $10^{-4}$ & $0.0083$ & $10.065$ & $\mathbf{1.89}$ \\
            $10^{-3}$ & $0.0102$ & $4.011$ & $2.35$ \\
            $10^{-2}$ & $0.0143$ & $1.699$ & $3.30$ \\
            $10^{-1}$ & $0.0202$ & $1.374$ & $5.74$ \\
            \bottomrule
        \end{tabular*}
\end{minipage}

\subsubsection{Sample generation}
Figure~\ref{fig:fm-toy-datasets} shows samples generated by our FM model (Section~\ref{sec:Flow-Matching}). Appendix~\ref{sec:generation-novelty} assesses novelty by comparing additional samples with their nearest training examples. Appendices~\ref{sec:gpca-configuration} and~\ref{appendix:DNA} provide dataset configurations and further experiments on promoter and enhancer DNA datasets following the experimental setups from 
\citet{Avdeyev:2023} and \citet{Stark:2024}, respectively.

\begin{figure}[H]
    \centering
    \captionsetup[subfigure]{labelformat=parens,labelfont=small,textfont=small,skip=1pt}
    \begin{subfigure}[t]{0.20\textwidth}
        \centering
        {\setlength{\tabcolsep}{0pt}
        \begin{tabular}{@{}cccccc@{}}
            \includegraphics[width=0.166\linewidth]{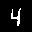} &
            \includegraphics[width=0.166\linewidth]{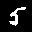} &
            \includegraphics[width=0.166\linewidth]{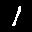} &
            \includegraphics[width=0.166\linewidth]{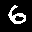} &
            \includegraphics[width=0.166\linewidth]{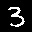} &
            \includegraphics[width=0.166\linewidth]{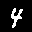} \\[-1.0ex]
            \includegraphics[width=0.166\linewidth]{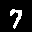} &
            \includegraphics[width=0.166\linewidth]{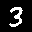} &
            \includegraphics[width=0.166\linewidth]{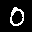} &
            \includegraphics[width=0.166\linewidth]{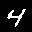} &
            \includegraphics[width=0.166\linewidth]{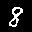} &
            \includegraphics[width=0.166\linewidth]{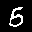} \\[-1.0ex]
            \includegraphics[width=0.166\linewidth]{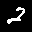} &
            \includegraphics[width=0.166\linewidth]{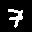} &
            \includegraphics[width=0.166\linewidth]{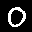} &
            \includegraphics[width=0.166\linewidth]{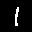} &
            \includegraphics[width=0.166\linewidth]{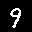} &
            \includegraphics[width=0.166\linewidth]{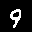} \\[-1.0ex]
            \includegraphics[width=0.166\linewidth]{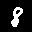} &
            \includegraphics[width=0.166\linewidth]{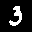} &
            \includegraphics[width=0.166\linewidth]{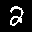} &
            \includegraphics[width=0.166\linewidth]{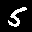} &
            \includegraphics[width=0.166\linewidth]{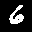} &
            \includegraphics[width=0.166\linewidth]{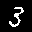} \\[-1.0ex]
            \includegraphics[width=0.166\linewidth]{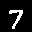} &
            \includegraphics[width=0.166\linewidth]{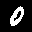} &
            \includegraphics[width=0.166\linewidth]{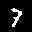} &
            \includegraphics[width=0.166\linewidth]{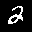} &
            \includegraphics[width=0.166\linewidth]{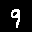} &
            \includegraphics[width=0.166\linewidth]{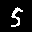} \\[-1.0ex]
            \includegraphics[width=0.166\linewidth]{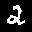} &
            \includegraphics[width=0.166\linewidth]{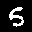} &
            \includegraphics[width=0.166\linewidth]{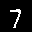} &
            \includegraphics[width=0.166\linewidth]{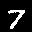} &
            \includegraphics[width=0.166\linewidth]{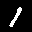} &
            \includegraphics[width=0.166\linewidth]{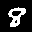}
        \end{tabular}}
        
        \subcaption{}
        \label{fig:fm-four-gaussians}
    \end{subfigure}\hfill
    \begin{subfigure}[t]{0.20\textwidth}
        \centering
        {\setlength{\tabcolsep}{0pt}
        \begin{tabular}{@{}cccccc@{}}
            \includegraphics[width=0.166\linewidth]{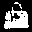} &
            \includegraphics[width=0.166\linewidth]{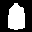} &
            \includegraphics[width=0.166\linewidth]{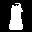} &
            \includegraphics[width=0.166\linewidth]{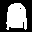} &
            \includegraphics[width=0.166\linewidth]{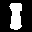} &
            \includegraphics[width=0.166\linewidth]{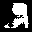} \\[-1.0ex]
            \includegraphics[width=0.166\linewidth]{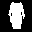} &
            \includegraphics[width=0.166\linewidth]{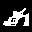} &
            \includegraphics[width=0.166\linewidth]{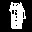} &
            \includegraphics[width=0.166\linewidth]{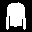} &
            \includegraphics[width=0.166\linewidth]{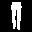} &
            \includegraphics[width=0.166\linewidth]{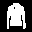} \\[-1.0ex]
            \includegraphics[width=0.166\linewidth]{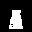} &
            \includegraphics[width=0.166\linewidth]{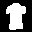} &
            \includegraphics[width=0.166\linewidth]{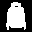} &
            \includegraphics[width=0.166\linewidth]{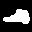} &
            \includegraphics[width=0.166\linewidth]{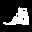} &
            \includegraphics[width=0.166\linewidth]{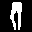} \\[-1.0ex]
            \includegraphics[width=0.166\linewidth]{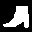} &
            \includegraphics[width=0.166\linewidth]{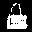} &
            \includegraphics[width=0.166\linewidth]{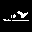} &
            \includegraphics[width=0.166\linewidth]{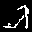} &
            \includegraphics[width=0.166\linewidth]{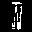} &
            \includegraphics[width=0.166\linewidth]{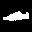} \\[-1.0ex]
            \includegraphics[width=0.166\linewidth]{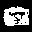} &
            \includegraphics[width=0.166\linewidth]{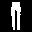} &
            \includegraphics[width=0.166\linewidth]{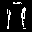} &
            \includegraphics[width=0.166\linewidth]{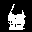} &
            \includegraphics[width=0.166\linewidth]{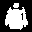} &
            \includegraphics[width=0.166\linewidth]{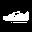} \\[-1.0ex]
            \includegraphics[width=0.166\linewidth]{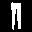} &
            \includegraphics[width=0.166\linewidth]{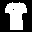} &
            \includegraphics[width=0.166\linewidth]{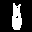} &
            \includegraphics[width=0.166\linewidth]{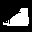} &
            \includegraphics[width=0.166\linewidth]{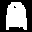} &
            \includegraphics[width=0.166\linewidth]{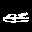}
        \end{tabular}}
        
        \subcaption{}
        \label{fig:fm-pinwheel}
    \end{subfigure}\hfill\hspace{0.01\textwidth}
    \begin{subfigure}[t]{0.24\textwidth}
        \centering
        {\setlength{\tabcolsep}{1pt}
        \begin{tabular}{@{}cc@{}}
            \includegraphics[width=0.48\linewidth]{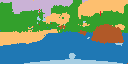} &
            \includegraphics[width=0.48\linewidth]{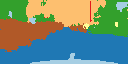} \\
            \includegraphics[width=0.48\linewidth]{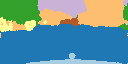} &
            \includegraphics[width=0.48\linewidth]{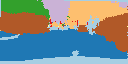} \\
            \includegraphics[width=0.48\linewidth]{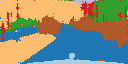} &
            \includegraphics[width=0.48\linewidth]{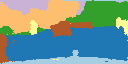}
        \end{tabular}}
        \vspace{0.1em}
        \subcaption{}
        \label{fig:fm-two-moons}
    \end{subfigure}\hfill
    \begin{subfigure}[t]{0.20\textwidth}
        \centering
        \begin{tabular}{@{}cc@{}}
            \includegraphics[width=0.47\linewidth]{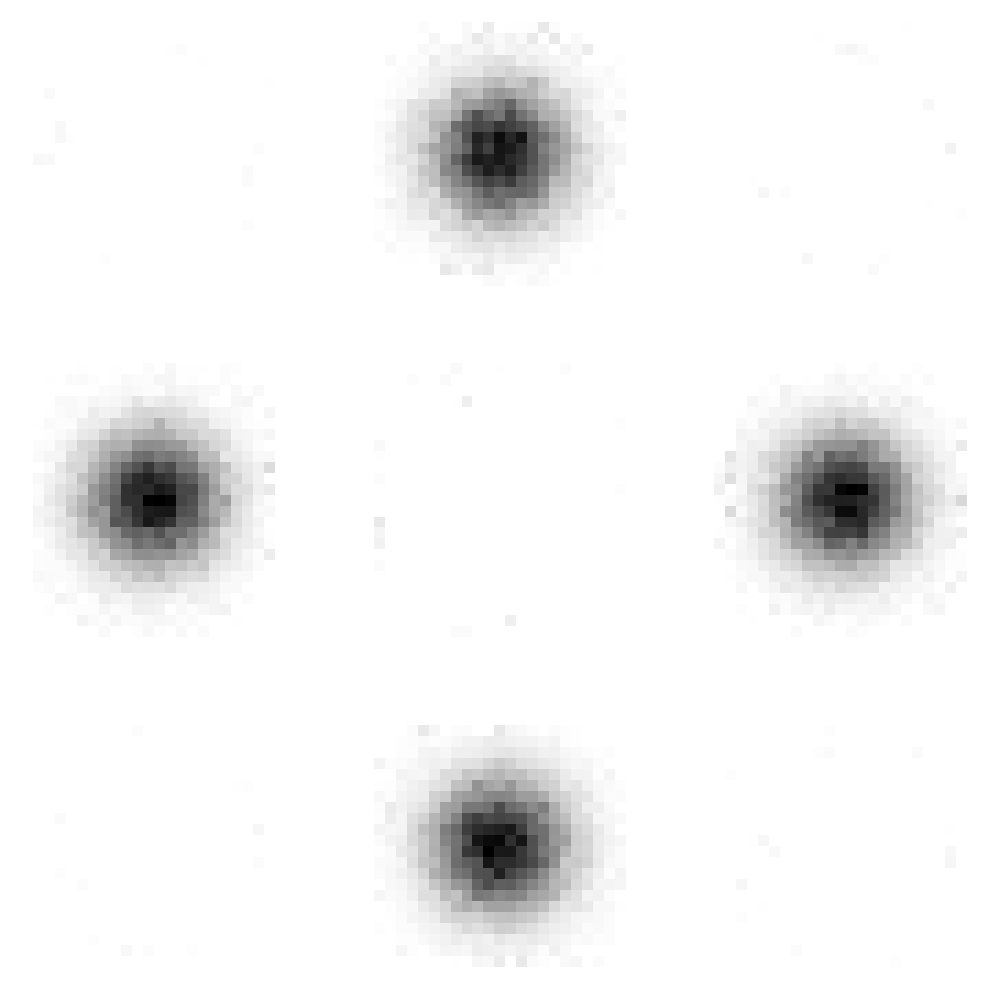} &
            \includegraphics[width=0.47\linewidth]{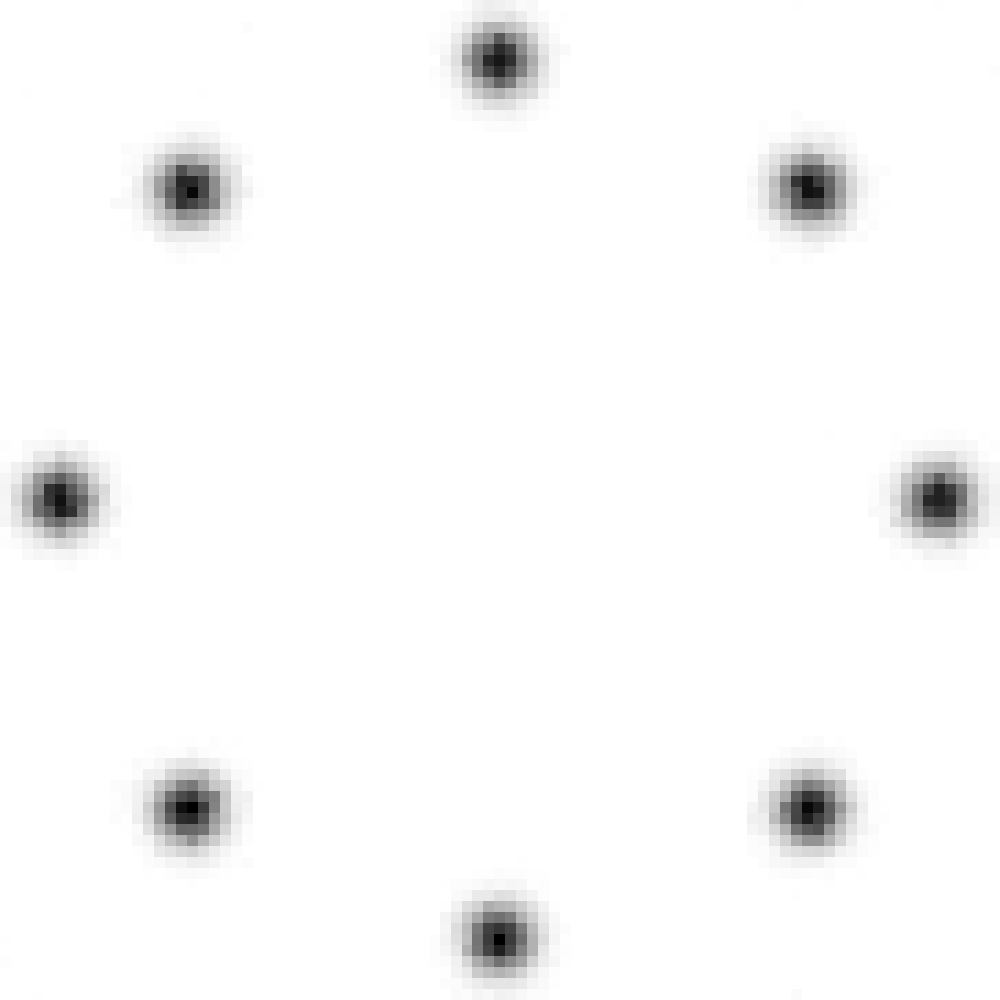} \\
            \includegraphics[width=0.47\linewidth]{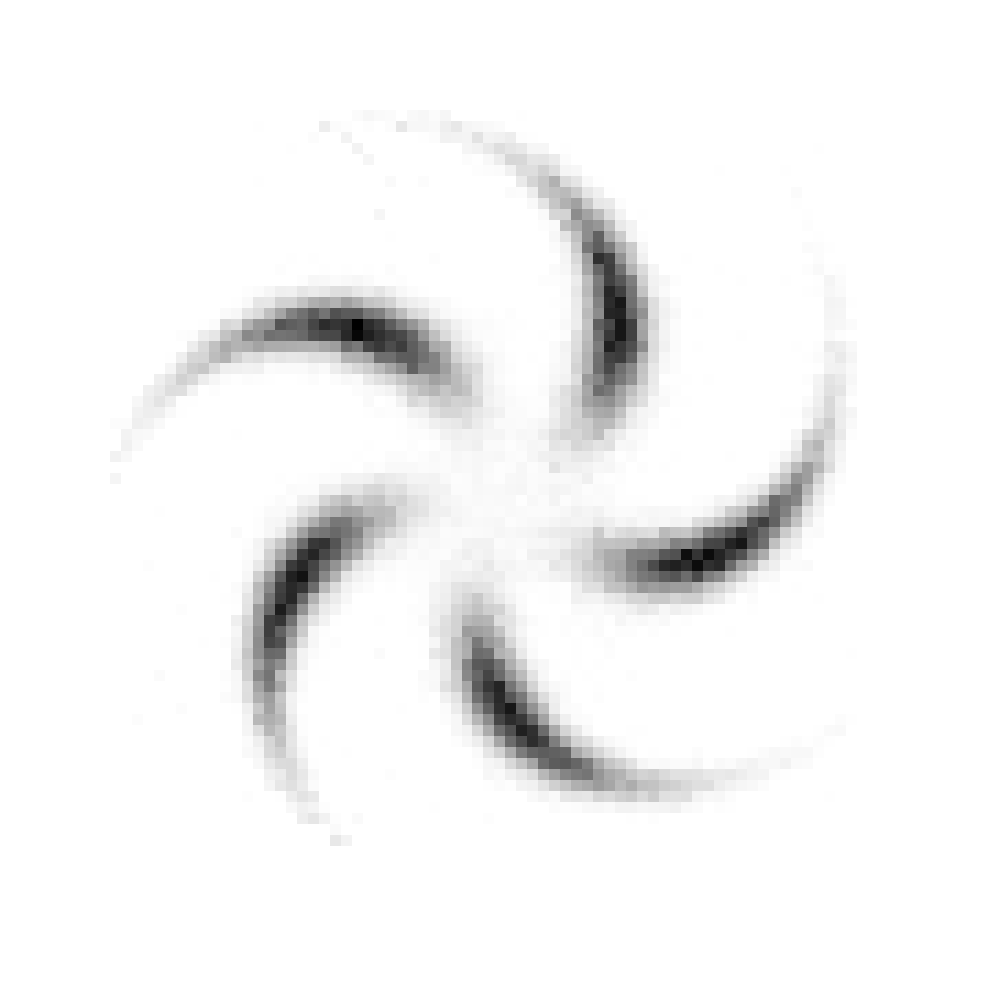} &
            \includegraphics[width=0.47\linewidth]{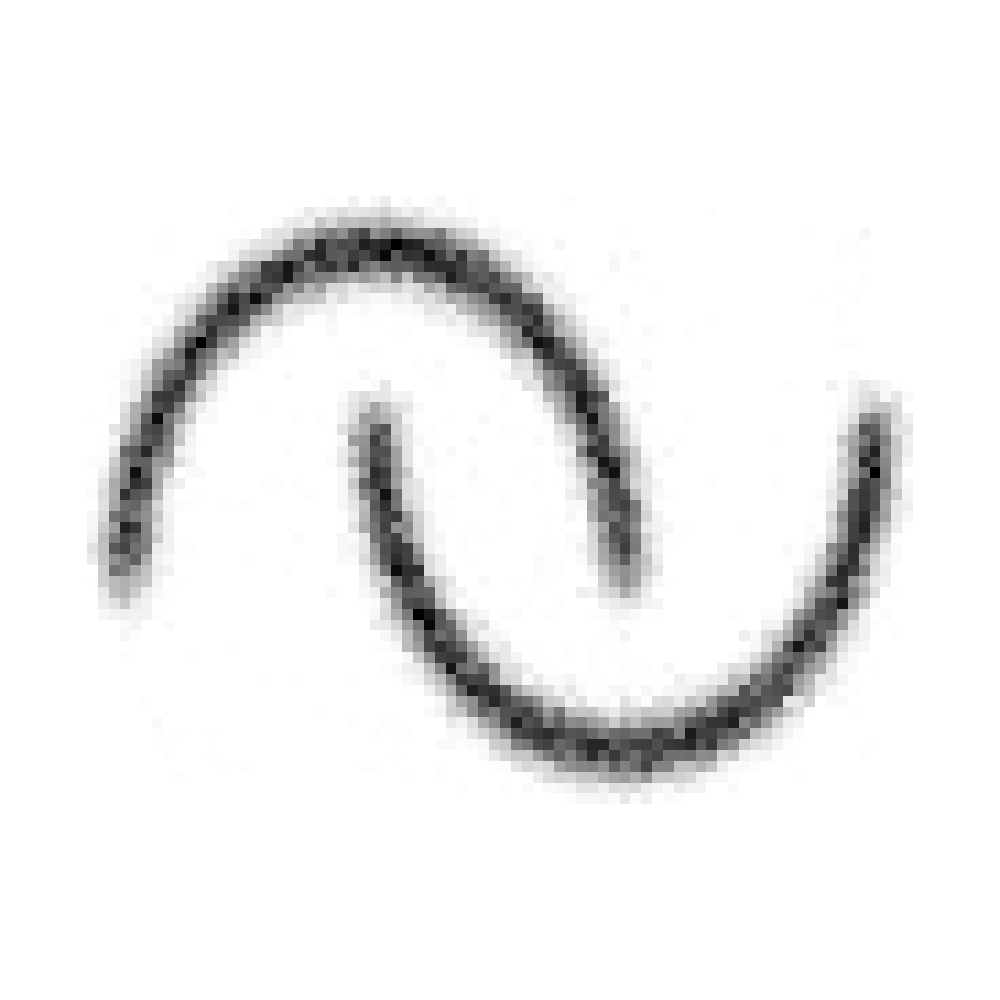}
        \end{tabular}
        
        \subcaption{}
        \label{fig:fm-gaussian-mixture}
    \end{subfigure}
    \caption{\textbf{Samples of FM model described in Section~\ref{sec:Flow-Matching}}. Generated samples for (a) MNIST with $d=32$, $\lambda=10^{-2}$, (b) binarized Fashion-MNIST with $d=64$, $\lambda=10^{-2}$, (c) Cityscapes with $d=512$, $\lambda=10^{-3}$, and (d) the empirical joint distribution with $d=16$, $\lambda=10^{-3}$.}
    \label{fig:fm-toy-datasets}
\end{figure}

\subsubsection{Toy data KL}

\begin{figure}[H]
    \centering
    \begin{minipage}[t]{0.48\linewidth}
        \vspace{0pt}
        \raggedright
We compare our method (GPCA + FM) with Linear FM, DirichletFM~\cite{Stark:2024}, FisherFM~\cite{Davis:2024}, and GSFM~\cite{Tang:2025} on synthetic categorical data with varying category counts $K$. The reference experiment uses length-four sequences drawn from a random categorical distribution and evaluates the KL divergence between generated and target distributions. Setup and baselines follow \citet[Table~7]{Tang:2025}. Our curves report the minimum KL divergence over five seeds for $d=\lfloor K/2\rfloor$ and $\lambda\in{10^{-3},10^{-4}}$, testing generation quality in a lower-dimensional GPCA subspace. 
    \end{minipage}\hfill
    \begin{minipage}[t]{0.49\linewidth}
        \vspace{0pt}
        \centering
        \includegraphics[width=\linewidth]{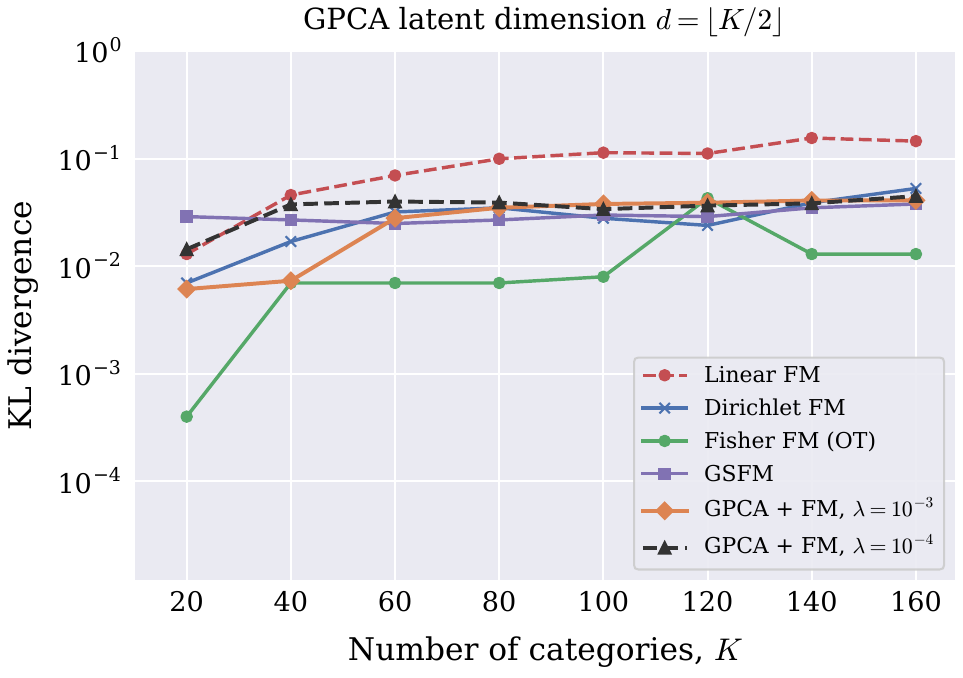}
        \caption{KL divergence across categories.}
        \label{fig:toy-kl-k-over-three-comparison}
    \end{minipage}
\end{figure}

\subsection{Flexible applicability, computational efficiency and limitations}
Once the GPCA representation is learned, FM training uses a unified, dataset-independent formulation, allowing the same architecture and training configuration to be reused across datasets. This simplifies implementation and facilitates reusability. Computationally, the FM objective in \eqref{eq:cfm-objective-Z} offers numerical advantages by applying the network to low-dimensional coordinates in $\mc Z$. Appendix~\ref{sec:time-complexity} reports GPCA compute and memory requirements.

Our approach is limited by the compressibility of the data. For the DNA datasets for example, low $e$-metric and low Hamming distance cannot be achieved simultaneously even at relative small compression factors, cf. Tables~\ref{tab:kl-dna-promoter}--\ref{tab:kl-dna-edistance}. A second limitation is the regularization parameter $\lambda$ in \eqref{eq:loss-gpca-e_reg}, which controls this reconstruction–distance trade-off and is selected empirically. A systematic, data-dependent choice is left for future work.

\section{Conclusion}
\label{sec:Conclusion}

We introduced a novel generative model based on geometric latent subspaces for discrete data, combining nonlinear decoding, compression, and efficient flow matching in the low-dimensional latent space. An isometry links latent and data manifolds. Future work will explore GPCA-based consistency models and discrete-data analysis via this isometry.

\section*{Acknowledgements}
This work is funded by the Deutsche Forschungsgemeinschaft (DFG) under Germany’s Excellence Strategy EXC-2181/1 - 390900948 (the Heidelberg STRUCTURES Excellence Cluster). This work was funded by the
Deutsche Forschungsgemeinschaft (DFG), grant SCHN 457/17-2, within the priority programme SPP 2298: Theoretical Foundations of Deep Learning.

\clearpage
\appendix
\section{Basic notation}\label{appendix:notation}

\noindent
\textbf{Spaces}
\begin{flalign*}
&N,n,c\colon\; &&\text{number of data points, nodes, and categories, respectively} & \\
&\mc{C} :=\; 
\{1,\dotsc,c\} =: [c],\quad c\in\N
&
&\text{set of categories}
& \\
&\mc{X} :=\; 
\{e_{1},\dotsc, e_{c}\}\subset\{0,1\}^{c} 
&
&\text{one-hot encoding of $\mc{C}$}
& \\
&\mc{X}^{n} :=\; 
\mc{X}\times\dotsb\times\mc{X} 
&
&\text{one-hot encoding of $\mc{C}^{n}$}
& \\
&\mc{X}^{n}_{N}:=\; 
\{x_{1},\dotsc,x_{N}\}\subset (\text{either} \;\mc{X}^{n} \text{ or } \mc{S}_{c}^{n})
&
&\text{given data, training set consisting of $N$ data points}
& \\
&\mc{S}_{c} :=\; 
 \{p\in \R^c_{>0} : \la\eins_{c}, p\rangle = 1\}
&
&\text{discrete distributions with full support}
& \\
&T_{0} :=\;
T_0\mc{S}_{c} = \{v\in\R^{c}\colon\la\eins_{c},v\rangle=0\}
&
&\text{tangent space of $\mc{S}_{c}$ in ambient coordinates}
& \\
&\eins_{\mc S_c} :=\; 
\frac{1}{c}(1,1,\dotsc,1)^{\T}\in\mc S_c,
&
&\text{barycenter of $\mc{S}_c$}
& \\
&\mc{S}_{c}^{n} :=\; 
\mc{S}_{c}\times\dotsb\times\mc{S}_{c}
&
&\text{product space of factorizing distributions on $\mc{C}^{n}$}
& \\
&T_{0}^{n} :=\;
 T_{0}\times\dotsb\times T_{0}
&
&\text{tangent space to $\mc{S}_{c}^{n}$ in ambient coordinates}
&
\end{flalign*}
\textbf{Parametrizations}
\begin{flalign*}
&\theta :=\; \varphi(p)\in\R^{c-1} &&\text{natural parameters of $p\in\mc S_c$} & \\
&\Theta :=\; \R^{n\times(c-1)} &&\text{natural-parameter space for $n$ nodes} & \\
&\psi(\theta) :=\; \log\Bigl(1+\sum_{l=1}^{c-1}e^{\theta_l}\Bigr) &&\text{log-partition function, cf. \eqref{eq:categorical-log-partition}} & \\
&\eta :=\; \nabla\psi(\theta)\in\R^{c-1} &&\text{first $c-1$ components of $p_\theta \in \mc S_c$} & \\
&\psi^\ast(\eta)\colon\;  &&\text{convex conjugate of $\psi$, negative entropy} & \\
&\varphi\colon\; \mc S_c\to\R^{c-1} &&\text{coordinate map (encoder), cf. \eqref{eq:exp-coord-chart}} & \\
&\varphi^{-1}\colon\; \R^{c-1}\to\mc S_c &&\text{inverse coordinate map (decoder)} & \\
&p_\theta:=\varphi^{-1}(\theta)\in\mc S_c &&\text{categorical distribution with natural parameters $\theta$} &
\end{flalign*}
The maps $\varphi$ and $\varphi^{-1}$ extend factorwise between $\mc S_c^n$ and $\Theta$.

\textbf{Latent subspace and approximation}
\begin{flalign*}
&d\colon\;  &&\text{latent subspace dimension} & \\
&\mc U :=\; \Span\{V^1,\dotsc,V^d\}\subset\Theta &&\text{GPCA subspace, cf. \eqref{eq:gpca-subspace}} & \\
&\mc Z :=\; \R^d &&\text{latent coefficient space} & \\
&V :=\; (V^1,\dotsc,V^d) &&\text{basis elements $V^\ell\in \Theta=\R^{n\times(c-1)}$ obtained from GPCA} & \\
&Z :=\; (Z_1,\dotsc,Z_N)^\T\in\R^{N\times d} &&\text{latent coefficients of training data obtained from GPCA} & \\
&\wh \theta:= V\odot z :=\; \sum_{\ell=1}^d z_\ell V^\ell &&\text{natural parameters associated with $z\in\mc Z$, cf. \eqref{eq:gpca-natural-parameter}} & \\
&\mc M :=\; \varphi^{-1}(\mc U)\subset\mc S_c^n &&\text{induced data manifold, cf. Section~\ref{sec:GPCA-Latent}} & \\
&\Pi_{\mc M}\colon\; \text{see \eqref{eq:bregman-projection}} &&\text{Bregman projection onto $\mc M$} & \\
&\wh p_i :=\; \varphi^{-1}(\wh\theta_i) &&\text{reconstruction of data point $x_i$, $\wh\theta_i$ is the reconstructed natural parameter} & \\
&\lambda\colon\; \lambda\geq0 &&\text{regularization weight in \eqref{eq:loss-gpca-e_reg}} &
\end{flalign*}

\textbf{Flow matching and $e$-geometry}
\begin{flalign*}
&g_p^{(e)}\colon\; \text{see Definition~\ref{def:e-metric-nodes}} &&\text{$e$-metric; also denoted by $g_e$} & \\
&d_e\colon\; \text{see \eqref{eq:e-RiemannDistance}} &&\text{Riemannian distance induced by the $e$-metric} & \\
&\|\cdot\|_{e,p}\colon\; \|v\|_{e,p}^2=g_p^{(e)}(v,v) &&\text{norm on the tangent space at $p$} & \\
&\phi_t^{(e)}(p_0\mid p_1)\colon\; \text{see \eqref{eq:geodesic-interpolant-e}} &&\text{$e$-geodesic interpolant} & \\
&u_t\colon\; \mc S_c^n\to T_0^n &&\text{flow-generating vector field, cf. \eqref{eq:fm-vf}} & \\
&v_t(\cdot;w)\colon\;  &&\text{learned vector field with network parameters $w$} & \\
&v_t^\psi(\cdot;w)\colon\; \text{see Proposition~\ref{prop:cfm-objective-equivalence}} &&\text{pushforward of $v_t(\cdot;w)$ through $\varphi$} & \\
&P_0,P_1,P_t\colon\;  &&\text{base distribution, target distribution, and probability path} & \\
&\|\cdot\|_V\colon\; \|z\|_V^2=\langle Vz,Vz\rangle &&\text{induced latent norm in $\mc Z$, cf. \eqref{eq:cfm-objective-Z}} &
\end{flalign*}

\clearpage

\section{Proofs and additional details}\label{appendix:proofs}

\subsection{Proofs of Section~\ref{sec:Riemann-Isometry}} \label{sec:app-proofs-riemann-isom}
\begin{proof}[Proof of Proposition \ref{prop:e-connection}]
By Proposition 1.9.1 in \cite{calin2014geometric}, the Christoffel symbols of the $e$-connection $\nabla^{e}=\nabla^{(1)}$ vanish in natural coordinates $\theta=\varphi(p)$. In these coordinates, the $e$-metric has constant identity matrix, so the Christoffel symbols of its Levi--Civita connection also vanish. Hence, the two connections coincide, since their Christoffel symbols agree in the same global coordinate chart. The argument extends factorwise to $\mc S_c^n$.
\end{proof}

\begin{proof}[Proof of Proposition \ref{prop:isom}]
The map $\varphi^{-1}:\mc U\to\mc M$ is the restriction of the inverse coordinate chart $\varphi^{-1}:\R^{n\times(c-1)}\to\mc S_c^n$ to the linear subspace $\mc U$. By definition of the $e$-metric, $\varphi^{-1}$ is a Riemannian isometry on the full parameter space. Its restriction therefore defines an isometric embedding of $\mc U$ into $\mc S_c^n$, with image $\mc M$. Since straight lines in $\mc U$ are also geodesics of the ambient Euclidean space, their images under $\varphi^{-1}$ are $e$-geodesics in $\mc S_c^n$.
\end{proof}

\begin{proof}[Proof of Proposition \ref{prop:geod-approx}]
Since $\varphi$ is a Riemannian isometry, it maps $e$-geodesics to Euclidean straight lines. Consequently, the $e$-geodesic in \eqref{eq:geodesic-interpolant-e} satisfies
\begin{align}\label{eq:aux-geodesic-interpolant}
\varphi\big(\phi_t^{(e)}(p_0\mid p_1)\big)=(1-t)\varphi(p_0)+t\varphi(p_1), \qquad t\in[0,1].
\end{align}
Writing $\theta_i=\varphi(p_i)$ and $\wh\theta_i=\varphi(\wh p_i)$, the assumption gives
\begin{align}
\|\theta_i-\wh\theta_i\|_2=d_e(p_i,\wh p_i)\leq\varepsilon, \qquad i=0,1.
\end{align}
Hence, 
\begin{subequations}
\begin{align}
d_e\big(\phi_t^{(e)}(p_0\mid p_1),\phi_t^{(e)}(\wh p_0\mid\wh p_1)\big) &\overset{\eqref{eq:e-RiemannDistance}}{=} \| \varphi\big(\phi_t^{(e)}(p_0\mid p_1)\big) - \varphi\big(\phi_t^{(e)}(\wh p_0\mid\wh p_1)\big) \|_2 \\
\\ & \overset{\eqref{eq:aux-geodesic-interpolant}}{=} \| (1-t) \varphi(p_0) + t \varphi(p_1) - \big((1-t) \varphi(\wh p_0) + t \varphi(\wh p_1)\big) \|_2 \\
& \overset{\theta_i = \varphi(p_i), \wh\theta_i = \varphi(\wh p_i)}{=} \big\|(1-t)(\theta_0-\wh\theta_0)+t(\theta_1-\wh\theta_1)\big\|_2 \\
&\leq (1-t)\|\theta_0-\wh\theta_0\|_2+t\|\theta_1-\wh\theta_1\|_2 \\
&\leq (1-t)\varepsilon+t\varepsilon=\varepsilon.
\end{align}
\end{subequations}
This proves the bound uniformly for $t\in[0,1]$.
\end{proof}


\begin{proof}[Proof of Proposition~\ref{prop:kl-edistance-compact}]
First note that, since $K$ is compact and $\varphi$ is continuous, $U_K:=\varphi(K)$ is compact.  Moreover, $U_K$ is convex by assumption. The Hessian $\nabla^2\psi(\theta)$ is the coordinate representation of the Fisher-Rao metric in natural coordinates, cf. \citet[Section~2.1]{Amari:2016}, and hence it is positive definite for the minimal categorical family.

The smallest and largest eigenvalues of $\nabla^2\psi(\theta)$ depend continuously on $\theta$. Hence, by compactness of $U_K$, they attain their extrema, and we may define
\begin{align}
m_K:=\min_{\theta\in U_K}\lambda_{\min}\big(\nabla^2\psi(\theta)\big)>0, \qquad M_K:=\max_{\theta\in U_K}\lambda_{\max}\big(\nabla^2\psi(\theta)\big)<\infty.
\end{align}
To obtain the corresponding matrix bounds, fix $\theta\in U_K$ and let $q_1,\dots,q_{c-1}$ be an orthonormal eigenbasis of $\nabla^2\psi(\theta)$, with corresponding eigenvalues $\lambda_1,\dots,\lambda_{c-1}$. Writing $v=\sum_{k=1}^{c-1}a_kq_k$, we obtain
\begin{align}
v^\T\nabla^2\psi(\theta)v=\sum_{k=1}^{c-1}\lambda_k a_k^2, \qquad \|v\|_2^2=\sum_{k=1}^{c-1}a_k^2.
\end{align}
Since $m_K\leq\lambda_k\leq M_K$ for every $k \in [c-1]$, it follows that
\begin{align}
m_K\|v\|_2^2\leq v^\T\nabla^2\psi(\theta)v\leq M_K\|v\|_2^2, \qquad v\in\R^{c-1}.
\end{align}
Equivalently,
\begin{align}
m_K I_{c-1}\preceq\nabla^2\psi(\theta)\preceq M_K I_{c-1}, \qquad \theta\in U_K.
\end{align}
For $p,p'\in K$, write $\theta=\varphi(p)$ and $\theta'=\varphi(p')$. Since $U_K$ is convex, the Hessian lower bound implies strong convexity in $U_K$, and hence \cite[Equation 9.8]{Vandenberghe:2004} gives
\begin{align}
\psi(\theta')\geq\psi(\theta)+\langle\nabla\psi(\theta),\theta'-\theta\rangle+\frac{m_K}{2}\|\theta'-\theta\|_2^2.
\end{align}
Rearranging terms yields
\begin{align}
D_\psi(\theta',\theta)&=\psi(\theta')-\psi(\theta)-\langle\nabla\psi(\theta),\theta'-\theta\rangle \
\geq\frac{m_K}{2}\|\theta'-\theta\|_2^2.
\end{align}
Likewise, the upper Hessian bound implies the quadratic upper bound in \cite[Equation 9.13]{Vandenberghe:2004}
\begin{align}
\psi(\theta')\leq\psi(\theta)+\langle\nabla\psi(\theta),\theta'-\theta\rangle+\frac{M_K}{2}\|\theta'-\theta\|_2^2,
\end{align}
which gives 
\begin{align}
D_\psi(\theta',\theta)\leq\frac{M_K}{2}\|\theta'-\theta\|_2^2.
\end{align}
Using $\KL(p,p')=D_\psi(\theta',\theta)$, cf. \citet[Equation~(2.16)]{Amari:2016} and $d_e(p,p')=\|\theta-\theta'\|_2$ from \eqref{eq:e-RiemannDistance}, it follows that
\begin{align}
\frac{m_K}{2}d_e^2(p,p')\leq\KL(p,p')\leq\frac{M_K}{2}d_e^2(p,p').
\end{align}
The claim follows with $C_1=m_K/2$ and $C_2=M_K/2$.
\end{proof}

\subsection{Proofs of Section~\ref{sec:Flow-Matching}}\label{sec:app-proofs-fm}
\begin{proof}[Proof of Proposition \ref{prop:cfm-objective-equivalence}]
The natural parameters corresponding to $x_0,x_1\in\mc S_c^n$ are given factorwise by
\begin{align}
\theta_{k,j}=\varphi(x_{k,j})=\left(\log\frac{x_{k,jl}}{x_{k,jc}}\right)_{l\in[c-1]}, \qquad k=0,1, \quad j\in[n].
\end{align}
Since $\varphi$ is a Riemannian isometry, it maps $e$-geodesics to Euclidean straight lines. Writing $x_t=\phi_t^{(e)}(x_0\mid x_1)$, we have
\begin{align}
\varphi(x_t)=(1-t)\varphi(x_0)+t\varphi(x_1)=(1-t)\theta_0+t\theta_1=\theta_t.
\end{align}
In particular,
\begin{align}
\partial_t\theta_t=\theta_1-\theta_0.
\end{align}
Moreover, by the chain rule we have
\begin{align}\label{eq:aux-chainrule}
\partial_t\theta_t=d\varphi_{x_t}[\partial_t x_t].
\end{align}
By Definition~\ref{def:e-metric-nodes}, extended factorwise to $\mc S_c^n$, we compute
\begin{align}
\|v_t(x_t;w)-\partial_t x_t\|_{e,x_t}^2 &= \big\|d\varphi_{x_t}\big[v_t(x_t;w)-\partial_t x_t\big]\big\|_2^2 \\
&= \big\|d\varphi_{x_t}\big[v_t(x_t;w)\big]-d\varphi_{x_t}[\partial_t x_t]\big\|_2^2 \\
&\overset{\eqref{eq:aux-chainrule}}{=} \big\|d\varphi_{x_t}\big[v_t(x_t;w)\big]-\partial_t\theta_t\big\|_2^2 \\
&= \big\|v_t^\psi(\theta_t;w)-(\theta_1-\theta_0)\big\|_2^2.
\end{align}
where the last equation uses the definition $v_t^\psi(\theta_t;w) := d\varphi_{x_t}\big[v_t(x_t;w)\big]$. Taking expectations on both sides, with $\theta_k=\varphi(x_k)$ distributed according to $\varphi_\sharp P_k$, $k=0,1$, concludes the proof.
\end{proof}

\begin{proof}[Proof of Corollary~\ref{cor:cfm-approximation-ds}]
We write $x_t=\phi_t^{(e)}(x_0\mid x_1)$, $\wh x_t=\phi_t^{(e)}(\wh x_0\mid\wh x_1)$, and denote the corresponding natural parameters by
\begin{align}
\theta_i=\varphi(x_i), \qquad \wh\theta_i=\varphi(\wh x_i), \qquad i=0,1.
\end{align}
Since $\varphi$ preserves distances as a Riemannian isometry, the reconstruction assumption \eqref{eq:assumption-approximation} gives
\begin{align}\label{eq:aux-assumption-approx}
\|\theta_i-\wh\theta_i\|_2\leq\sqrt{\varepsilon}, \qquad i=0,1.
\end{align}
Writing $\theta_t=\varphi(x_t)$ and $\widehat\theta_t=\varphi(\widehat x_t)$, the straight-line representation of the interpolants gives
\begin{subequations}\label{eq:aux-theta-approx}
\begin{align}
\|\theta_t-\wh\theta_t\|_2 &= \| ( (1-t)\theta_0 + t\theta_1 ) - ( (1-t)\wh\theta_0 + t\wh\theta_1 ) \|_2 \\
&= \|(1-t)(\theta_0-\wh\theta_0)+t(\theta_1-\wh\theta_1)\|_2 \\
&\leq(1-t)\|\theta_0-\wh\theta_0\|_2+t\|\theta_1-\wh\theta_1\|_2\leq\sqrt{\varepsilon}.
\end{align}
\end{subequations}
Furthermore the L-Lipschitz assumption implies
\begin{align}\label{eq:aux-L-Lipschitz}
\|v_t^\psi(\theta_t;w)-v_t^\psi(\wh\theta_t;w)\|_2\leq L\|\theta_t-\wh\theta_t\|_2 \overset{\eqref{eq:aux-theta-approx}}{\leq} L\sqrt{\varepsilon}.
\end{align}
Adding and subtracting the reconstructed vector field and target velocity, we compute
\begin{subequations}
\begin{align}
&\|v_t^\psi(\theta_t;w)-(\theta_1-\theta_0)\|_2 \\
&=\Big\|v_t^\psi(\theta_t;w)-v_t^\psi(\wh\theta_t;w)+v_t^\psi(\wh\theta_t;w)-(\wh\theta_1-\wh\theta_0)+(\wh\theta_1-\wh\theta_0)-(\theta_1-\theta_0)\Big\|_2 \\
&\leq\|v_t^\psi(\theta_t;w)-v_t^\psi(\wh\theta_t;w)\|_2+\|v_t^\psi(\wh\theta_t;w)-(\wh\theta_1-\wh\theta_0)\|_2+\|(\wh\theta_1-\wh\theta_0)-(\theta_1-\theta_0)\|_2 \\
& \overset{\eqref{eq:aux-L-Lipschitz}}{\leq}\|v_t^\psi(\wh\theta_t;w)-(\wh\theta_1-\wh\theta_0)\|_2+L\sqrt{\varepsilon}+\|\wh\theta_1-\theta_1\|_2+\|\wh\theta_0-\theta_0\|_2 \\
&\overset{\eqref{eq:aux-assumption-approx}}{\leq}\|v_t^\psi(\wh\theta_t;w)-(\wh\theta_1-\wh\theta_0)\|_2+(L+2)\sqrt{\varepsilon}.
\end{align}
\end{subequations}

Since, for any $\eta>0$,
\begin{align}
(a+b)^2\leq(1+\eta)a^2+\left(1+\frac{1}{\eta}\right)b^2,
\end{align}
setting
\begin{align}
a=\|v_t^\psi(\wh\theta_t;w)-(\wh\theta_1-\wh\theta_0)\|_2, \qquad b=(L+2)\sqrt{\varepsilon}, \qquad \eta=\sqrt{\varepsilon},
\end{align}
we obtain
\begin{align}
\|v_t^\psi(\theta_t;w)-(\theta_1-\theta_0)\|_2^2 
\leq(1+\sqrt{\varepsilon})\|v_t^\psi(\wh\theta_t;w)-(\wh\theta_1-\wh\theta_0)\|_2^2+(L+2)^2(\varepsilon+\sqrt{\varepsilon}).
\end{align}
Taking expectations on both sides and applying Proposition~\ref{prop:cfm-objective-equivalence} to the original and reconstructed interpolants concludes the proof.
\end{proof}

\clearpage

\section{Example: Representation of synthetic datasets}\label{appendix:representation-synthetic}
We identify points in the product space $\mc S_c^n$ with their images in the \textit{meta-simplex} $\mc S_{c^n}$ through the standard identification of marginal distributions with their corresponding product distributions, cf. \cite[Section 5.2]{Drton:2009}, 
\begin{align}
\mc T:\mc S_c^n &\hookrightarrow \mc S_{c^n},
\qquad
p=(p_1,\dots,p_n)
\mapsto
p_1\otimes\cdots\otimes p_n=:P.
\end{align}
The entries of $P=p_1\otimes\cdots\otimes p_n$ are indexed by configurations $(i_1,\dots,i_n)\in[c]^n$ and given by
\begin{align}
P(i_1,\dots,i_n) = \prod_{k=1}^n p_k(i_k), \qquad i_k\in\{1,\dots,c\}.
\end{align}
In this representation, $P\in\mc S_{c^n}$ is the joint probability vector over the $c^n$ possible configurations, while the image $\mc T(\mc S_c^n)\subset\mc S_{c^n}$ consists of those probability vectors admitting the tensor-product factorization above.
As a simple illustrative example, consider $c=n=2$. In this case, the product simplex $\mc S_2^2$ is embedded as a two-dimensional surface in the three-dimensional meta-simplex $\mc S_4$. Figure~\ref{fig:factordistr-embedding-S4} shows this embedded product manifold from three different viewing angles.
\begin{figure}[H]
    \centering
    \includegraphics[width=0.18\linewidth]{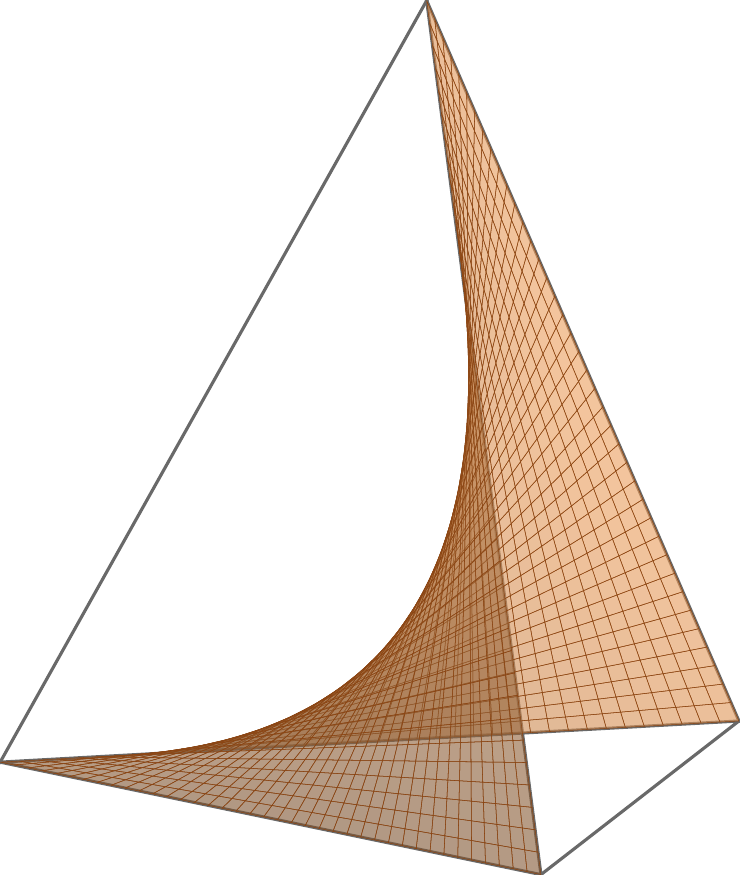}%
    \hspace{0.05\linewidth}%
    \includegraphics[width=0.18\linewidth]{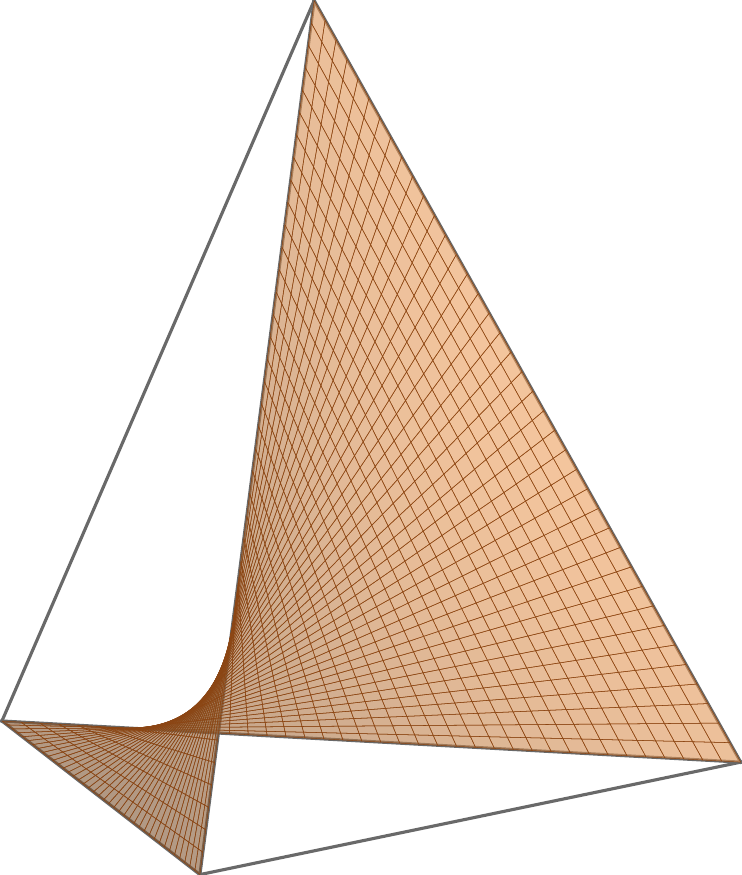}%
    \hspace{0.05\linewidth}%
    \includegraphics[width=0.18\linewidth]{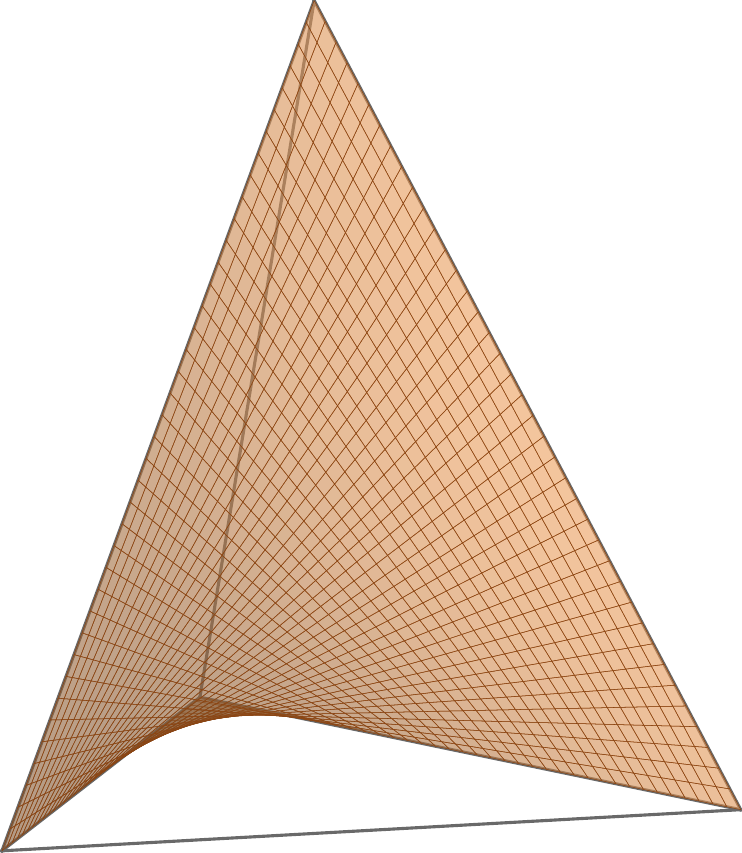}
\caption{Three views of the embedded product manifold
$\mc T(\mc S_2^2)\subset\mc S_4$.}
    \label{fig:factordistr-embedding-S4}
\end{figure}

Flow-based models following Riemannian Flow Matching on the product space $\mc S_c^n$, such as those in \cite{Stark:2024,Boll:2024ab,Boll:2025aa,Davis:2024,Cheng:2024,Cheng:2025,Williams:2025}, learn a vector field on $\mc S_c^n$ that defines a flow $\phi_t$ transporting a base distribution $P_0\in\mc P(\mc S_c^n)$ into a distribution $\wh P_1 \in \mc P(\mc S_c^n)$, whose samples simulate samples from the data distribution $P_1$.
As an illustrative example, we consider again $c=n=2$ and, for $0<\varepsilon<1/4$, the target joint distribution  
\begin{align}\label{eq:example_P1}
P_1 = (0.5-\varepsilon, \varepsilon, \varepsilon, 0.5-\varepsilon)^\top \in \mc S_4.
\end{align}
Although $P_1 \notin \mc T(\mc S_2^2)$, samples from $P_1$ are represented by vertices of the closure of $\mc S_2^2$. The flow generates samples whose empirical distribution approximates $P_1$ rather than representing $P_1$ as a single product distribution. Figure~\ref{fig:synthetic-histogram} shows the empirical distribution of samples drawn from $P_1$. Figure~\ref{fig:synthetic-metasimplex} shows $P_1$ in the meta-simplex $\mc S_4$, emphasising that $P_1 \notin \mc T(\mc S_2^2)$. Figure~\ref{fig:synthetic-flow} illustrates how a flow $\phi_t$ on $\mc S_2^2$ transports samples from the base distribution to samples approximately distributed according to $P_1$.
\begin{figure}[H]
    \centering
    \begin{subfigure}[t]{0.20\linewidth}
        \centering
        \includegraphics[width=\linewidth]{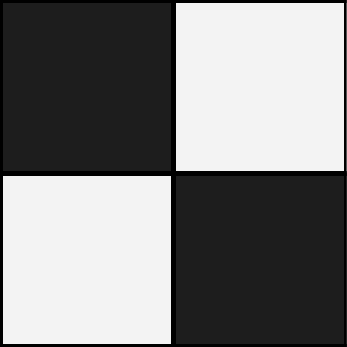}
        \caption{Empirical approximation of $P_1$.}
        \label{fig:synthetic-histogram}
    \end{subfigure}\hspace{0.08\linewidth}%
    \begin{subfigure}[t]{0.20\linewidth}
        \centering
        \includegraphics[width=\linewidth]{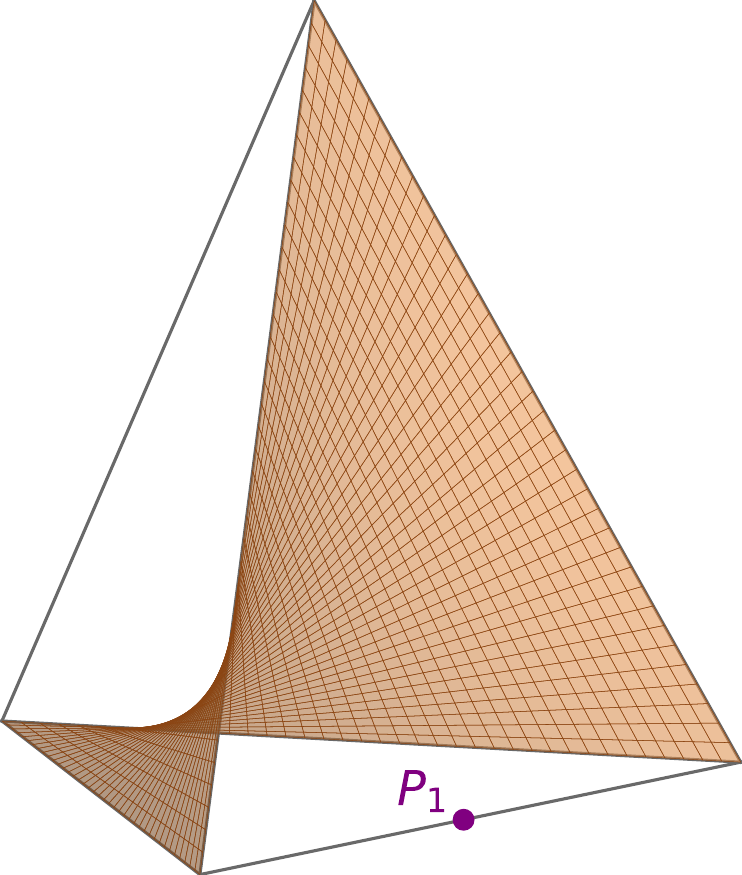}
        \caption{Target $P_1$ in the meta-simplex.}
        \label{fig:synthetic-metasimplex}
    \end{subfigure}\hspace{0.08\linewidth}%
    \begin{subfigure}[t]{0.20\linewidth}
        \centering
        \includegraphics[width=\linewidth]{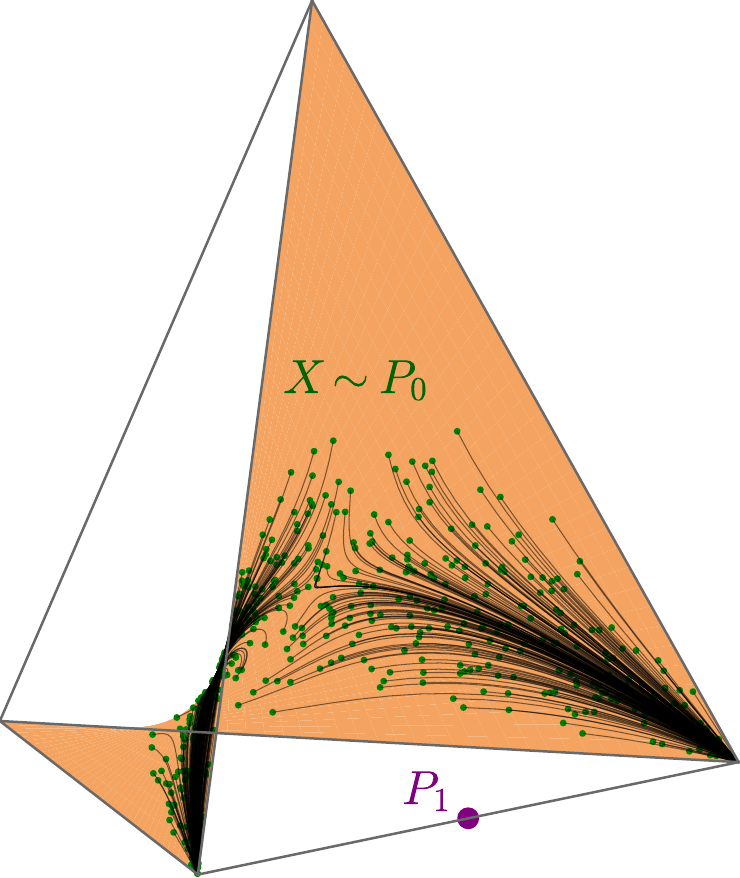}
        \caption{Flow on $\mc S_2^2$.}
        \label{fig:synthetic-flow}
    \end{subfigure}
\caption{Illustration of the target distribution $P_1$ defined in
\eqref{eq:example_P1} and its approximation by a flow-based model.
(a) Histogram of $N=10^4$ samples drawn from $P_1$.
(b) The target $P_1 \in \mc S_4$ lies outside the embedded product
manifold $\mc T(\mc S_2^2)$.
(c) A flow on $\mc S_2^2$ transports base samples to samples
approximately distributed according to $P_1$.}
    \label{fig:synthetic}
\end{figure}

Our approach learns a latent subspace $\mc U = \Span(V)$ and the corresponding low-dimensional manifold $\mc M = \varphi^{-1}(\mc U)$. This provides a lower-dimensional space in which to learn the flow.
Continuing the example above, we apply GPCA to samples from $P_1$ to learn the one-dimensional submanifold $\mc M \subset \mc S_2^2$, whose embedded image $\mc T(\mc M)$ is shown in blue in Figure~\ref{fig:synthetic-gpca-manifold}.
\begin{figure}[H]
    \centering
    \includegraphics[width=0.24\linewidth]{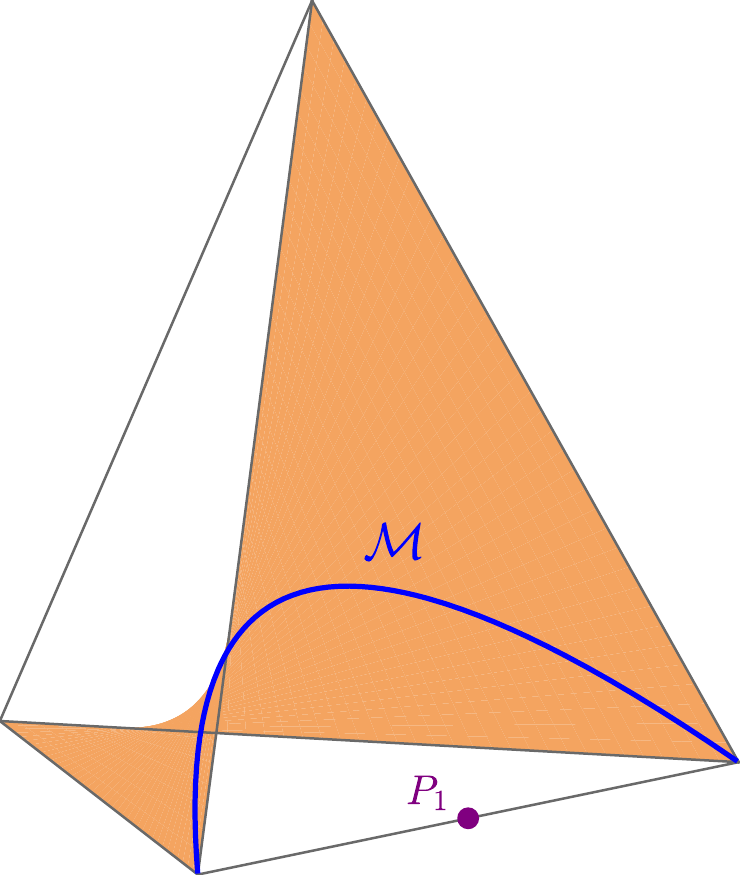}
\caption{The embedded image $\mc T(\mc M)$ of the one-dimensional
manifold learned by GPCA from samples of $P_1$ defined in
\eqref{eq:example_P1}, shown in blue in the meta-simplex $\mc S_4$.}
    \label{fig:synthetic-gpca-manifold}
\end{figure}

The same geometric picture provides intuition for the discrete synthetic datasets used throughout this work, shown in Figure~\ref{fig:synthetic-datasets}, where we use $n=2$ and $c=92$. We refer to     Appendices~\ref{sec:gpca-configuration} and~\ref{appendix:gpca-visual-comparison} for data specifications and visual experiments on these datasets. These histograms should be interpreted analogously to the empirical distribution in Figure~\ref{fig:synthetic-histogram}, representing discrete realizations of the underlying continuous distributions. Figure~\ref{fig:pinwheel-pointcloud-flow} illustrates, on the left, the coordinate representations $Z$ of the data, shown in blue, obtained by GPCA, see Section~\ref{sec:GPCA-Latent}. The middle image shows the interpolations between base samples and selected data points, while the rightmost image shows the corresponding samples as empirical histograms. By choosing the geometry as described in Section~\ref{sec:Riemann-Isometry}, these geodesic interpolations correspond to straight lines in the objective \eqref{eq:cfm-objective-U}. Choosing a different geometry, such as the Fisher–Rao geometry,     would result in curved interpolations and reduce the computational gains noted in Remark~\ref{rem:com-savings}.
\begin{figure}[H]
    \centering
    \includegraphics[width=0.22\linewidth,trim=6.5bp 6.5bp 6.5bp 6.5bp,clip]{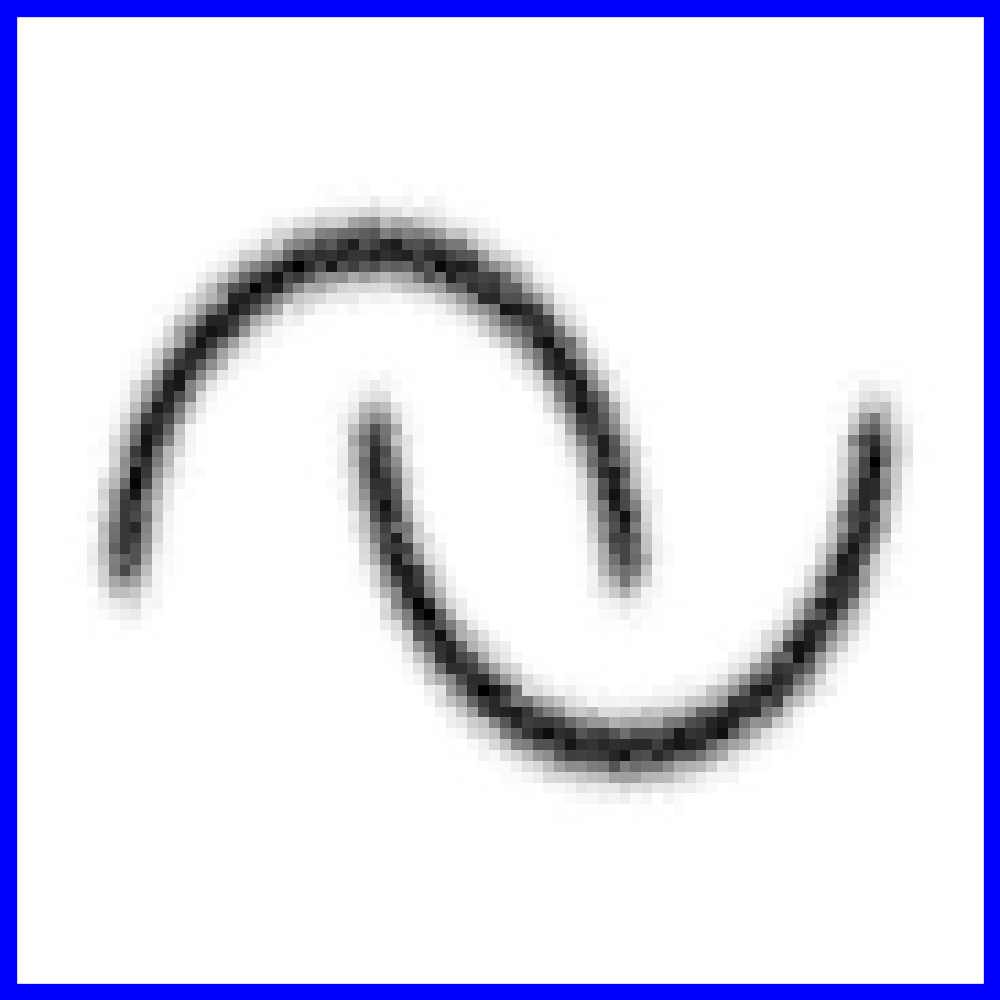}\hspace{0.02\linewidth}%
    \includegraphics[width=0.22\linewidth,trim=6.5bp 6.5bp 6.5bp 6.5bp,clip]{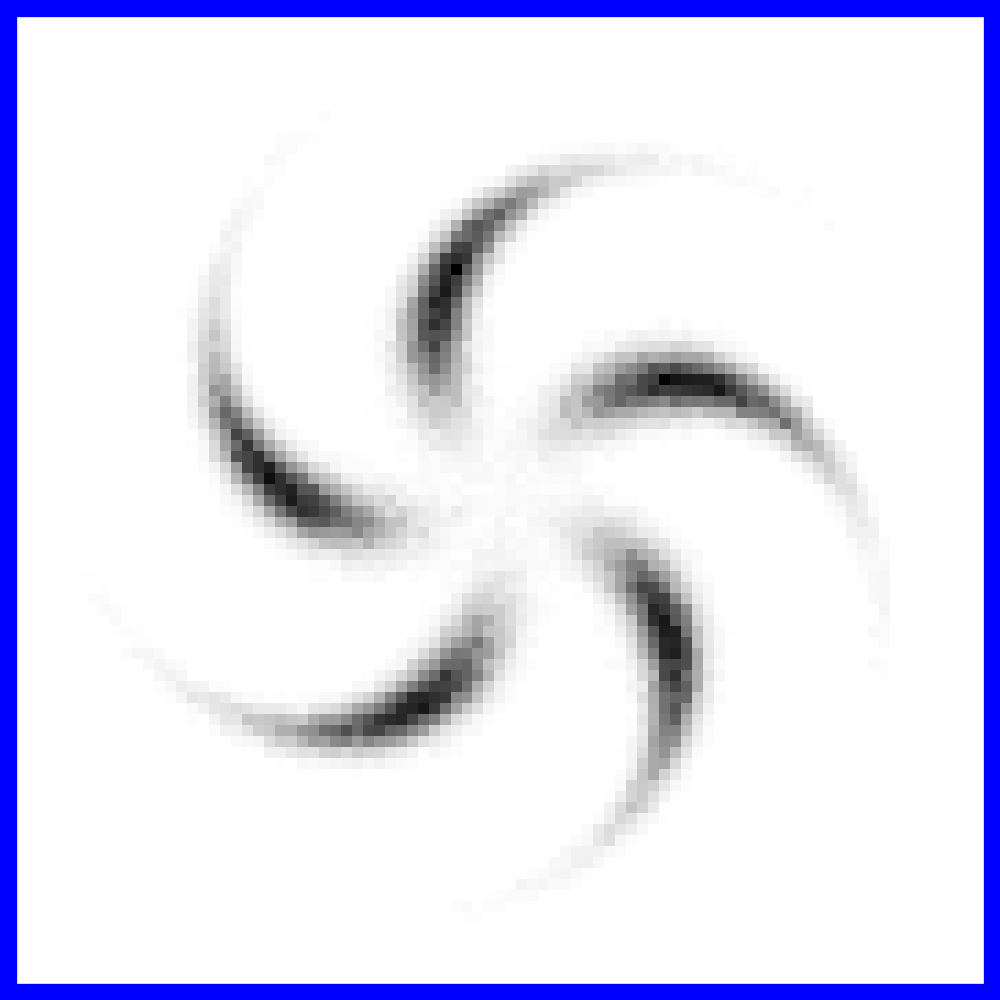}\hspace{0.02\linewidth}%
    \includegraphics[width=0.22\linewidth,trim=6.5bp 6.5bp 6.5bp 6.5bp,clip]{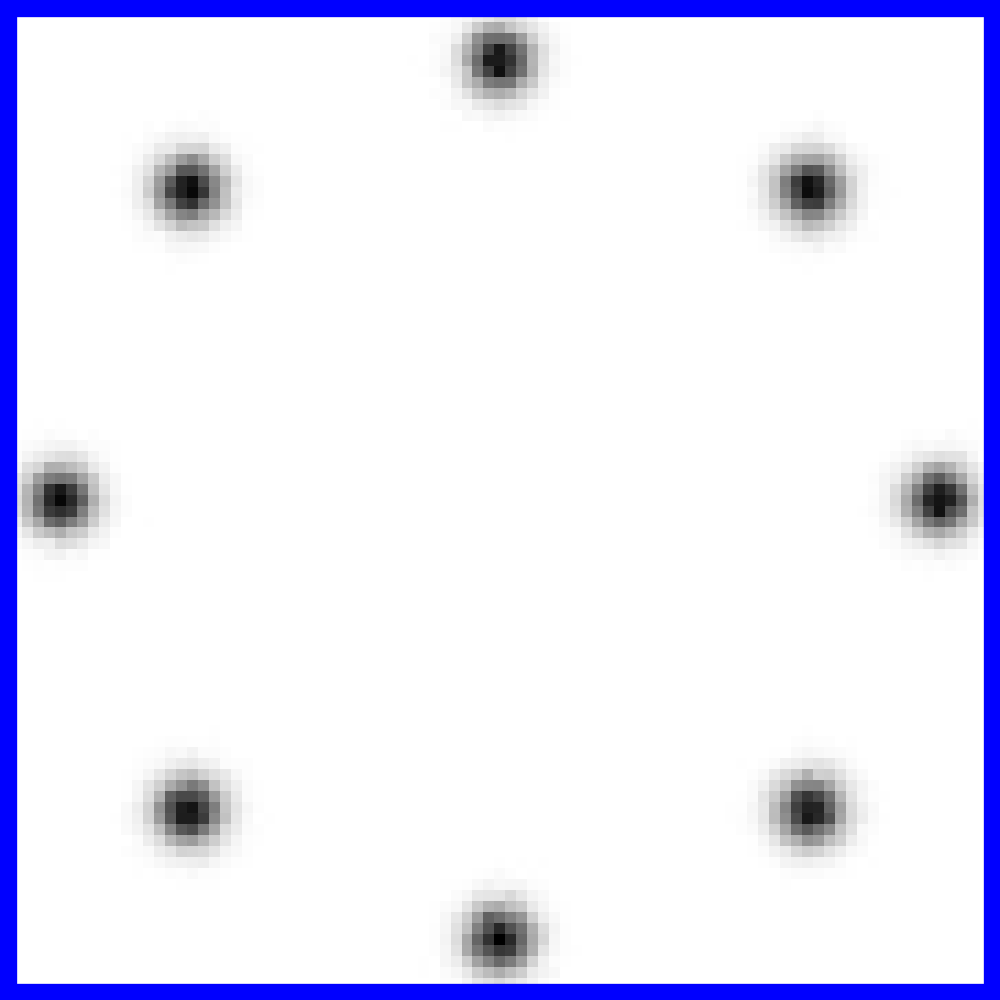}
    \caption{Synthetic datasets: two moons, pinwheel, and Gaussian mixture (left to right).}
    \label{fig:synthetic-datasets}
\end{figure}

\begin{figure}[H]
    \centering
     \includegraphics[width=0.85\linewidth]{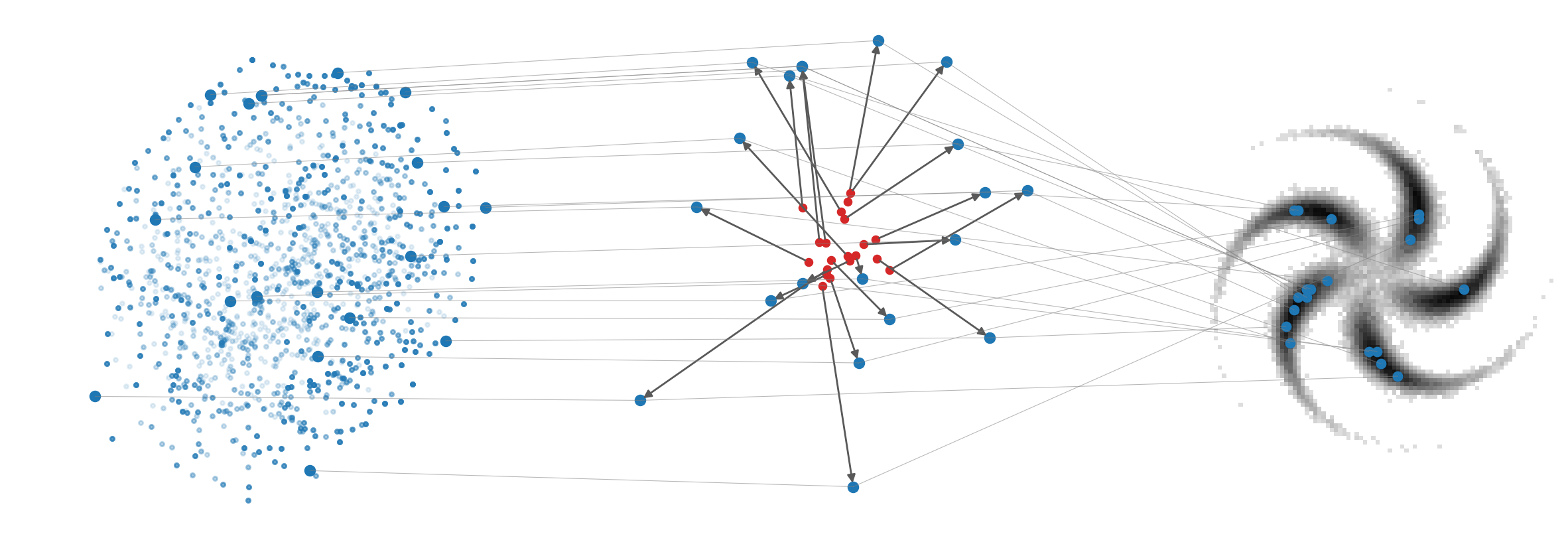}
    \caption{(left) Coordinate representations of the $N=10^4$ samples of the pinwheel dataset, (center) geodesic interpolations from normal samples to selected coordinate representations, and (right) the corresponding pinwheel samples.}
    \label{fig:pinwheel-pointcloud-flow}
\end{figure}

\clearpage


\section{GPCA loss, $e$-distance \& data representation}\label{sec:GPCA-Loss}
Throughout this section, let $i\in [N]$ index data points, $j \in [n]$ index nodes, and $l \in [c]$ index classes. We represent factorizing distributions $p\in \mc S_c^n$ in the minimal parameterization as
\begin{align}
    p_{\theta;j}(l) = \exp\big( \langle \theta_j, T(l)\rangle - \psi(\theta_j) \big), \quad 
     \psi(\theta_j) = \log\Big(1 + \sum_{l=1}^{c-1} e^{\theta_{jl}}\Big), \quad j\in [n],
\end{align}
cf. Section~\ref{sec:CD-GPCA}, where $\theta_j$ denote the natural parameters, $\psi$ is the log-partition function, and $T(l) = (\delta_{lk})_{k\in[c-1]}$ denote the sufficient statistic of class $l$. The mean parameters are given by the gradient of the log-partition function
\begin{align}
   \eta_j := \nabla \psi(\theta_j) =  \Big(\frac{e^{\theta_{jl}}}{1 + \sum_{k=1}^{c-1} e^{\theta_{jk}}} \Big)_{l\in [c-1]}, \quad  j\in [n],
\end{align}
and the last component is obtained by normalization, i.e.
\begin{align}
   \eta_{jc} = 1 - \sum_{l=1}^{c-1} \eta_{jl}.
\end{align}
The conjugate of the log-partition function is the negative entropy \citep[Section 2.1]{Amari:2016}
\begin{align}
    \psi^*(\eta_j) = \sup_{\theta_j\in\mathbb R^{c-1}} \left( \langle \theta_j,\eta_j\rangle-\psi(\theta_j) \right)= \sum_{l=1}^{c-1}\eta_{jl}\log\eta_{jl} + \eta_{jc}\log\eta_{jc}.
\end{align}

For hard-label data $x_1, \dots ,x_N \in \mc X_N^n$, \citet{Collins:2001aa} proposed to learn $V\in \R^{n \times (c-1)\times d }$ and $Z \in \R^{N \times d}$ such that the natural parameters are represented as $\wh \theta_{i} =  V \odot Z_i$, cf. ~\eqref{eq:gpca-natural-parameter}. The representation is estimated by minimizing the Bregman divergence objective
\begin{align}
   \frac{1}{Nn} \sum_{i=1}^{N} \sum_{j=1}^n D_{\psi^*}(x_{ij}^-, \nabla \psi(  (V \odot Z_i)_j ) ),
\end{align}
where $x_{ij}^- = (x_{ijl})_{l\in [c-1]}$ contains the first $c-1$ components of the one-hot encoded data points. \citet{Collins:2001aa} point out that this objective equals the negative log likelihood (NLL) of the data under the model. Indeed, since 
\begin{align}
\psi^{\ast}(x_{ij}^-) = \sum_{l=1}^{c} x_{ijl}\log x_{ijl} = 0
\end{align}
for one-hot data, with the convention that $0\log 0 = 0$, Legendre duality gives
\begin{align}
    D_{\psi^*}(x_{ij}^-, \nabla \psi(  (V \odot Z_i)_j ) ) &= \psi^*(x_{ij}^-) + \psi(\wh \theta_{ij}) - \langle \wh \theta_{ij}, x_{ij}^- \rangle \\
    &= \psi(\wh \theta_{ij}) - \langle \wh \theta_{ij}, x_{ij}^- \rangle = - \sum_{l=1}^{c}  x_{ijl} \log p_{ \wh \theta;{ij}}(l).
\end{align}
The last expression is the negative logarithm of the probability assigned to the observed class, and hence the objective studied in \cite{Collins:2001aa}, for hard-label data, is equivalent to 
\begin{align}\label{eq:nll-objective}
\NLL(V,Z) = - \frac{1}{Nn} \sum_{i=1}^{N} \sum_{j=1}^n \sum_{l=1}^{c}  x_{ijl} \log p_{ \wh \theta;{ij}}(l).
\end{align}

The node-wise reconstruction error is measured by the mean Hamming distance (MHD) between the original data points $x_1,\dots, x_N \in \mc X_N^n$ and their reconstructions, and is defined as
\begin{align}\label{eq:hamming-distance}
\MHD = \frac{1}{Nn}\sum_{i=1}^N \sum_{j=1}^n d_H(x_{ij}, \round\big(\varphi^{-1}(\wh \theta_{ij})) \big), \quad d_H(p,q) = \frac{1}{2}\sum_{l=1}^c \eins_{p_l \neq q_l} , \; p,q \in \{e_1, \dots, e_c\},
\end{align}
where $\wh \theta_{i} = V \odot Z_i$, and $(V,Z)$ are the parameters estimated by minimizing the geometric or generalized PCA loss. Here, \(\operatorname{round}\) returns the one-hot encoding of the class with largest probability.

The optimization of the NLL is numerically stable and can be carried out efficiently using reverse-mode automatic differentiation. 
For small latent dimensions, the Hamming distance decreases together with the NLL, but generally does not reach zero. For larger latent dimensions, which still provide a substantial dimensionality reduction, zero Hamming distance is reached much faster in practice. 
Section~\ref{appendix:gpca-empirical-performance} provides a more detailed empirical analysis of the relation between the NLL, Hamming distance, and $e$-distance for different latent dimensions on benchmark datasets.

\textbf{Data representation.}
As discussed in the main text, hard-label data points $x \in \mc X_N^n$ are embedded into the interior of the product simplex $\mc S_c^n$ by a smoothing transformation. In practice, we use
\begin{align} \label{eq:eta-smooth-transformation}
    \widetilde{x}_{ij} = \epsilon \eins_{\mc S_c} + (1-\epsilon)x_{ij}, \qquad 0 < \epsilon < 1,
\end{align}
where $\eins_{\mc S_c}$ denotes the barycenter of $\mc S_c$. This transformation ensures that all entries of $\widetilde{x}_{ij}$ are strictly positive, so that KL divergences and logarithms are well-defined.

For points in $p, \wt p \in \mc S_c$, let $\eta, \wt \eta \in \R^{c-1}$ contain their first $c-1$ components, respectively. By \eqref{eq:categorical-bregman-kl}, the Bregman divergence generated by the convex conjugate $\psi^*$ coincides with the Kullback-Leibler divergence
\begin{align}
    D_{\psi^*}(\eta, \wt \eta) = \KL(p, \wt p).
\end{align}
Consequently, denoting the smoothed data $\widetilde{x}$ again by $x$ for simplicity, and writing $\eta_{ij} = (x_{ijl})_{l\in[c-1]}$, the generalized PCA loss reads
\begin{align}
\mc L(V,Z) = \frac{1}{Nn} \sum_{i=1}^{N} \sum_{j=1}^n D_{\psi^*}(\eta_{ij}, \widehat{\eta}_{ij}) = \frac{1}{Nn} \sum_{i=1}^{N} \sum_{j=1}^n \KL(x_{ij}, \widehat{p}_{ij}),
\end{align}
where $\wh \eta_{ij} = \nabla \psi(\wh \theta_{ij})$, and $\widehat{p}_{ij} = \varphi^{-1}(\wh \theta_{ij})$, with $\wh \theta_{i} = V \odot Z_i$.

Since
\begin{align}
    \KL\bigl({x}_{ij}, \widehat{p}_{ij}\bigr) = \sum_{l=1}^{c} {x}_{ijl}  \log \frac{{x}_{ijl}}{\widehat{p}_{ijl}}
    = - \sum_{l=1}^{c} {x}_{ijl} \log \widehat{p}_{ijl} +  \sum_{l=1}^{c} {x}_{ijl} \log {x}_{ijl},
\end{align}
the KL divergence decomposes into a cross-entropy term and a negative entropy term that depends only on the data. Therefore, the entropy term is constant with respect to the optimization variables $V$ and $Z$. Minimizing the KL loss is thus equivalent to minimizing the cross-entropy loss
\begin{align}\label{eq:gpca-cross-entropy}
    \Cross(V,Z)    =  -\frac{1}{Nn} \sum_{i=1}^{N}\sum_{j=1}^{n} \sum_{l=1}^{c}  {x}_{ijl}\log \wh p_{ijl},
\end{align}
which is the objective that we minimize in practice.

In practice, for the flow matching framework, we want the reconstruction error between data points $x$ and its reconstruction $\wh p\in \mc M$ to be small, with respect to underlying Riemannian distance in \eqref{eq:e-RiemannDistance}. In our setting, this means that the mean squared $e$-distance
\begin{align}\label{eq:gpca-e-distance}
e\text{-Distance} = \frac{1}{Nn} \sum_{i=1}^N \sum_{j=1}^n d_e^2(x_{ij}, \wh p_{ij}) 
\end{align}
should be small. When the initial and terminal points are close to their reconstructions, Proposition \ref{prop:geod-approx} guarantees that the $e$-geodesic interpolants in \eqref{eq:geodesic-interpolant-e} on the submanifold $\mc M =  \varphi^{-1}(\mc U)$ remain close to $e$-geodesic interpolants on the data space. Moreover, Corollary \ref{cor:cfm-approximation-ds} provides a corresponding approximation bound for the Flow Matching objective on the data space. 

Empirically, we find that the Hamming distance decreases along with the GPCA loss. The $e$-distance, however, does not decrease as consistently. In particular, minimizing the cross-entropy on the smoothed data instead of the NLL on the hard-label data leads to larger Hamming distances. We refer to Section~\ref{appendix:gpca-empirical-performance} for a more detailed empirical analysis of the relation between the NLL, cross-entropy, Hamming distance, and $e$-distance for different latent dimensions on benchmark datasets.



\clearpage

 \label{appendix:GPCA}

\section{Experiments: GPCA vs PCA}\label{appendix:gpca-pca}

We compare PCA and GPCA in terms of discrete reconstruction quality on different categorical datasets: binarized MNIST~\cite{Lecun:2010}, binarized FashionMNIST~\cite{Xiao:2017}, Cityscapes~\cite{Cordts:2016}, and the \textsc{Promoter}, \textsc{Melanoma} and \textsc{FlyBrain} datasets. For further specifications on the datasets and experimental configurations, see Sections~\ref{sec:gpca-configuration} and \ref{appendix:gpca-dna}, which provide the corresponding references.

For PCA, we determine the latent dimension needed to retain at least $99\%$ of the total variance, following the cumulative percentage of total variation criterion in \cite[Section~6.1.1]{Jolliffe:2002aa}. More concretely, each data point is represented by $nc$ one-hot coordinates, and its empirical covariance matrix has size $nc\times nc$. Let $\ell_1\geq\cdots\geq\ell_{nc}\geq0$ denote its eigenvalues, where $\ell_k$ is the variance of the corresponding principal component. The ratio of the sum of the first $d$ eigenvalues to the sum of all eigenvalues gives the cumulative proportion of total variance explained by these components. We write
\begin{align}
d_{99}:=\min\left\{d\in\{1,\ldots,nc\}:\frac{\sum_{k=1}^{d}\ell_k}{\sum_{k=1}^{nc}\ell_k}\geq0.99\right\}.
\end{align}
Table~\ref{tab:gpca-pca} compares the discrete reconstruction errors of PCA and GPCA. The first group of columns reports $d_{99}$ and the mean Hamming distance \eqref{eq:hamming-distance}, between the original data and their discrete PCA reconstructions at this dimension (PCA MHD). The second group compares PCA and GPCA at the same latent dimension $d$, with PCA MHD and GPCA MHD denoting their respective reconstruction errors. All GPCA results use $\lambda=10^{-4}$. 

The MHD measures the mean fraction of mislabeled nodes after rounding or categorical label selection. Thus, an MHD with value $x$ for a dataset with $n$ nodes per sample corresponds to $xn$ mislabeled nodes per sample on average. For instance, for the Melanoma dataset with $n=500$, the reconstruction errors at $d=500$ correspond to an average of $101.5$ mislabeled nodes per sample for PCA and $0.0015$ for GPCA.

\begin{table}[H]
\centering
\caption{PCA and GPCA reconstruction errors. The comparison at fixed $d$ uses the same latent dimension for both methods. GPCA uses $\lambda=10^{-4}$.}
\label{tab:gpca-pca}
\setlength{\tabcolsep}{5pt}
\begin{tabular}{llrrrrr}
\toprule
& & \multicolumn{2}{c}{\textbf{PCA: $\boldsymbol{99\%}$ explained variance}}
& \multicolumn{3}{c}{\textbf{Mean Hamming distance}} \\
\cmidrule(lr){3-4}\cmidrule(lr){5-7}
\textbf{Dataset} & $\boldsymbol{n\times c}$ & $\boldsymbol{d_{99}}$ & \textbf{PCA MHD} & $\boldsymbol{d}$ & \textbf{PCA MHD} & \textbf{GPCA MHD} \\
\midrule
Binarized MNIST & $1024\times2$ & $395$ & $3.51 \cdot 10^{-4}$ & $128$ & $3.37 \cdot 10^{-3}$ & $47 \cdot 10^{-6}$ \\
Binarized FashionMNIST & $1024\times2$ & $594$ & $5.29 \cdot 10^{-4}$ & $256$ & $7.06 \cdot 10^{-3}$ & $6 \cdot 10^{-6}$ \\
Cityscapes & $8192\times8$ & $2515$ & $5 \cdot 10^{-6}$ & $512$ & $2.66 \cdot 10^{-2}$ & $0$ \\
Promoter & $1024\times4$ & $3017$ & $3.2 \cdot 10^{-5}$ & $1024$ & $2.05 \cdot 10^{-1}$ & $41 \cdot 10^{-6}$ \\
Melanoma & $500\times4$ & $1475$ & $0$ & $500$ & $2.03 \cdot 10^{-1}$ & $3 \cdot 10^{-6}$ \\
Flybrain & $500\times4$ & $1477$ & $0$ & $500$ & $2.36 \cdot 10^{-1}$ & $0$ \\
\bottomrule
\end{tabular}
\end{table}

At equal latent dimensions, GPCA achieves smaller Hamming reconstruction errors than PCA on all six datasets. These results indicate that GPCA provides more accurate discrete reconstructions than PCA at the same level of compression. The differences are particularly pronounced for the biological datasets: at the selected latent dimensions, PCA misclassifies approximately $20$--$24 \%$ of nodes, whereas GPCA achieves nearly error-free discrete reconstructions. GPCA achieves these errors with substantially fewer latent dimensions than PCA requires to retain $99 \%$ of the total variance.

\clearpage

\section{Experiments: computational costs and time complexity} \label{sec:time-complexity}

We briefly discuss the computational cost of our method. We report wall-clock runtimes separately for the GPCA optimization and the subsequent flow-matching procedure. All experiments were conducted on a single NVIDIA GeForce RTX 4090 GPU with 24 GB of VRAM. Peak GPU memory usage was approximately 15 GB. We use PyTorch $\geq 2.10$ with a CUDA 12.8.

For the optimization of the objective in \eqref{eq:loss-gpca-e_reg}, we employ full-batch Adam and compute gradients via PyTorch automatic differentiation. The regularization weight $\lambda$ varies across experimental settings, as specified in the corresponding results. The resulting runtimes for the datasets used in our experiments are summarized in Table~\ref{tab:gpca_runtime}. Empirically, we find that only a moderate number of iterations, approximately $5000$ in most cases, is required to attain values close to the optima reported in Tables~\ref{tab:nll-mnist-edistance}--\ref{tab:cityscapes-256-loss-lambda}. The optimization dynamics are stable, with gradient norms decreasing quickly and smoothly, as illustrated in Figures~\ref{fig:gpca-cityscapes-gradients-loss-nll} and~\ref{fig:gpca-cityscapes-gradients-loss}. Since this optimization is not computationally demanding, we run it for $30000$ iterations in our experiments to ensure convergence.

We refer to Section~\ref{sec:gpca-configuration} for details on the datasets used, and to~\cite{Avdeyev:2023,Stark:2024} for a description of the configurations for the \textsc{promoter} and \textsc{enhancer} DNA datasets.

\begin{table}[H]
\centering
\caption{Estimated GPCA optimization speed and wall-clock time for 30000 iterations.}
\label{tab:gpca_runtime}
\begin{tabular}{lcc}
\toprule
\textbf{Dataset} & \textbf{iter/s} & \textbf{Time for 30000 iterations} \\
\midrule
MNIST, $d=64$ & 39.88 & 12m 32s \\
Fashion-MNIST, $d=64$ & 39.85 & 12m 33s \\
Cityscapes, $d=512$ & 27.52 & 18m 10s \\
Synthetic Large-Class datasets, $d=16$ & 420.98 & 1m 11s \\
\textsc{Melanoma} dataset, $d=1024$ & 46.02 & 10m 52s \\
\textsc{FlyBrain} dataset, $d=1024$ & 44.60 & 11m 12s \\
\textsc{Promoter} dataset, $d=1024$ & 10.54 & 47m 26s \\
\bottomrule
\end{tabular}
\end{table}

Regarding the flow-matching procedure, we emphasize that the computational cost depends on the chosen architecture. In particular, there is an important distinction between minimizing the objective in \eqref{eq:cfm-obj} and minimizing the reduced objective in \eqref{eq:cfm-objective-Z}. While learning directly in the data space may lead to faster convergence in terms of the number of epochs, it incurs a substantially higher computational cost per iteration. By contrast, when working in the latent coordinate space $\mc Z$, the computational cost is significantly reduced.

As an illustrative example, we trained on MNIST with latent dimension $d=64$, using an MLP with approximately 8M parameters in the latent coefficient space $\mathcal Z$. Both the latent MLP and the data-space U-Net were trained for $600$ epochs with batch size $256$, using Adam with an initial learning rate of $10^{-4}$ and a cosine-annealing learning-rate schedule. The MLP uses eight layers, hidden dimension $512$, and an expansion ratio of $4$. The U-Net uses $32$ base channels, channel multipliers $(1,2,2,2)$, two residual blocks per level, and attention at spatial resolution $16$. Including evaluation, the corresponding wall-clock runtimes were approximately 11 minutes for the latent MLP and around 1 hour for the data-space U-Net. Despite its larger parameter count, the MLP benefits from the reduced objective in \eqref{eq:cfm-objective-Z}, which expresses flow matching entirely in the latent coordinates $z$. In this formulation, the vector field $v_t(\cdot;w)$ acts on $d$-dimensional inputs rather than on the full ambient representation. Consequently, the forward and backward passes operate on compact latent representations, contributing to the observed reduction in computational cost. The learned transport dynamics are thus modeled directly in the reduced coordinate space $\mathcal Z$, with the decoder applied when mapping the resulting latent samples back to the data manifold.
For sampling, we integrate the learned vector field using the adaptive Heun solver over $t\in[0,1]$, with absolute and relative tolerances both set to $10^{-4}$.
\clearpage

\section{Experiments: empirical performance of geometric PCA} \label{appendix:gpca-empirical-performance}

\subsection{Datasets and GPCA configuration}\label{sec:gpca-configuration}

We compare our approach, which minimizes the regularized GPCA objective in \eqref{eq:loss-gpca-e_reg}, with NLL minimization in \eqref{eq:nll-objective} on one-hot data and cross-entropy minimization on smoothed data. These objectives and data representations are described in Appendix~\ref{sec:GPCA-Loss}. The experiments are conducted on the following datasets:
\begin{enumerate}
    \item \textbf{Synthetic data with many classes and few nodes.}
    We consider four synthetic datasets, each represented by a histogram  of a discrete distribution on $n=2$ variables with $c=92$ categories. For each dataset, we draw $N=3 \cdot 10^4$ samples and
    then apply GPCA, as described in Section~\ref{sec:GPCA-Latent}.

    \item \textbf{MNIST and Fashion-MNIST.}
    We use MNIST~\cite{Lecun:2010} and Fashion-MNIST~\cite{Xiao:2017},    each with $N=6 \cdot 10^4$ training samples. The images are padded to $n=32 \cdot 32$ pixels, and we use $c=2$ classes. Since the pixel values lie in $[0,1]$, we apply a rounding transformation to obtain binary data. For MNIST, this transformation has only a minor visual effect, and the digits remain clearly recognizable. For Fashion-MNIST, however, it provides another valuable benchmark for generating \textit{discrete} data, cf. in Figure~\ref{fig:fashion-mnist-original-transformed}.

    \item \textbf{Cityscapes.}
    We use Cityscapes~\cite{Cordts:2016}, consisting of $N=2975$ training samples. The images are resized to $n=256 \cdot 256$ nodes, and we use $c=8$ semantic classes.
\end{enumerate}

\begin{figure}[H]
    \centering
    \begin{subfigure}[t]{0.481\textwidth}
        \centering
        \includegraphics[width=\linewidth]{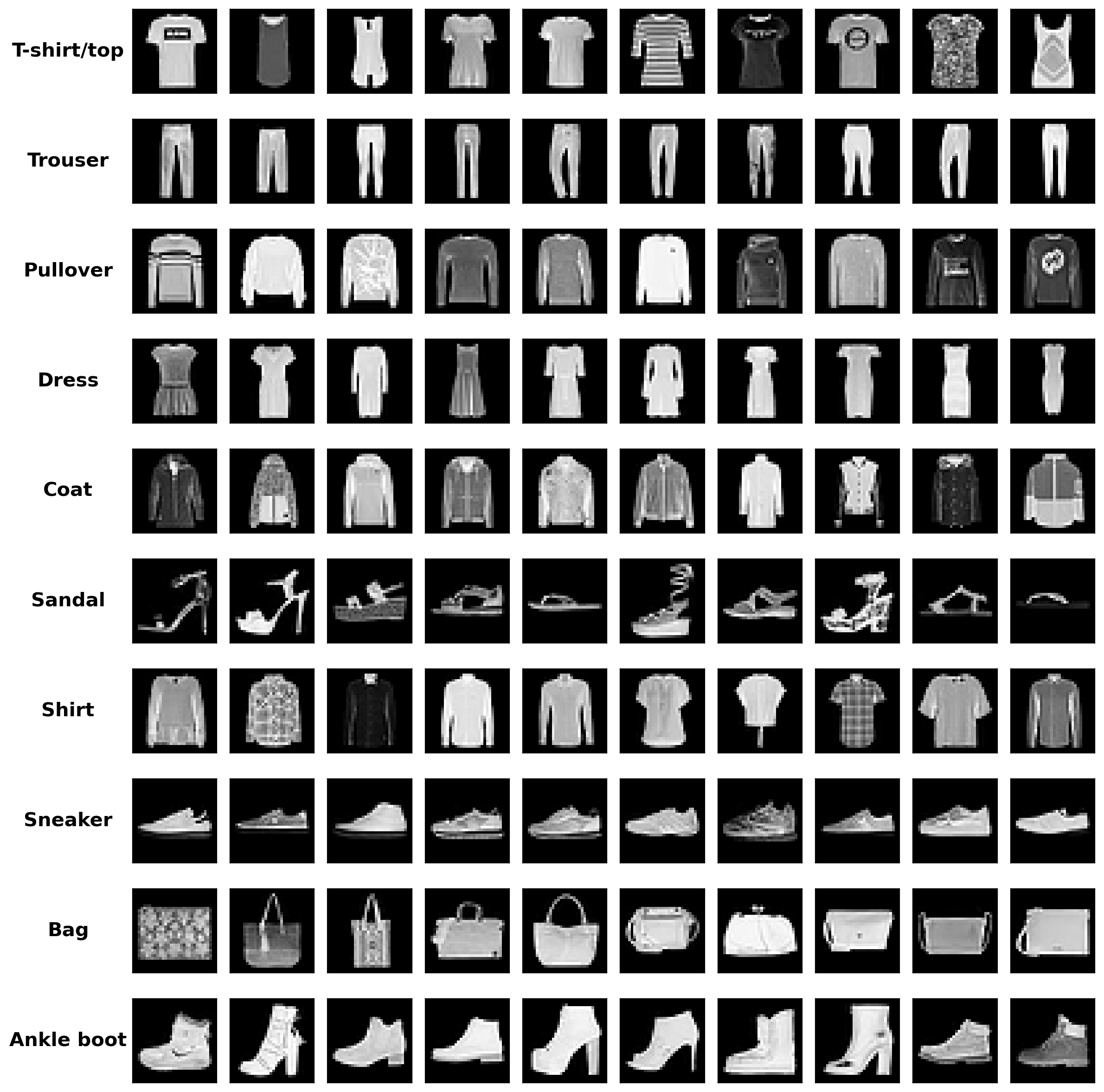}
        \subcaption{Original Fashion-MNIST samples.}
        \end{subfigure} \hspace{0.005\textwidth}
    \begin{subfigure}[t]{0.475\textwidth}
        \centering
        \includegraphics[width=\linewidth]{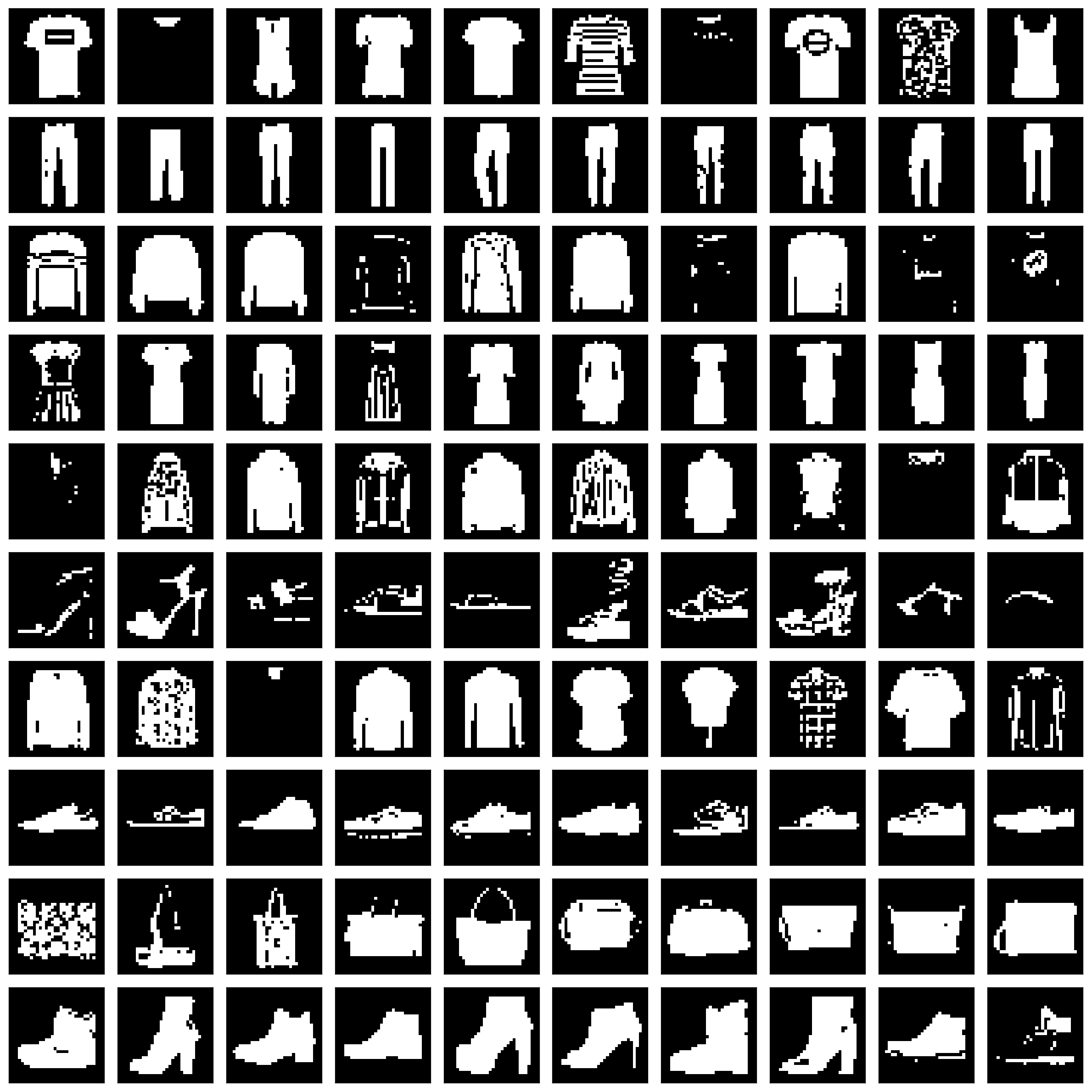}
        \subcaption{Samples of binarized Fashion-MNIST.}
    \end{subfigure}
    \caption{Visual comparison of Fashion-MNIST samples before and after rounding.
    We included the binarized Fashion-MNIST data into our set of benchmark experiments.
    }
    \label{fig:fashion-mnist-original-transformed}
\end{figure}

For the GPCA optimization, the latent codes $Z \in \R^{N \times d}$ and basis parameters $V \in \mathbb{R}^{n \times (c-1) \times d}$ are learned jointly using full-batch Adam with learning rate $10^{-2}$ for between $10^4$ and $3 \cdot 10^4$ iterations, depending on the experiment. For cross-entropy minimization and the regularized GPCA objective, we apply the smoothing transformation described in Appendix~\ref{sec:GPCA-Loss} with $\epsilon=10^{-2}$. Gradients with respect to both $Z$ and $V$ are computed by automatic differentiation using PyTorch.

\subsection{Negative log-likelihood \& KL minimization}
We first illustrate the behavior of the GPCA loss of \cite{Collins:2001aa} on one-hot data and on data smoothed using \eqref{eq:eta-smooth-transformation} with $\epsilon=10^{-2}$. As discussed in Appendix~\ref{sec:GPCA-Loss}, the objective equals the NLL in \eqref{eq:nll-objective} for one-hot data. For smoothed data, we minimize the corresponding KL divergence, equivalently the cross-entropy loss. We observe that NLL minimization leads to stable gradients and a consistent decrease in Hamming distance, with zero Hamming distance achieved at relatively small latent dimensions.

Figure~\ref{fig:gpca-cityscapes-gradients-loss-nll} shows an optimization example on Cityscapes with latent dimension $d=256$ over the first $450$ iterations. We compare the NLL, the Hamming distance, and the Riemannian distance, also referred to as the $e$-distance. We observe that both the NLL and the Hamming distance decrease consistently and quickly reach very small values, whereas the $e$-distance remains comparatively high.
\begin{figure}[H]
    \centering
    \includegraphics[width=0.48\textwidth]{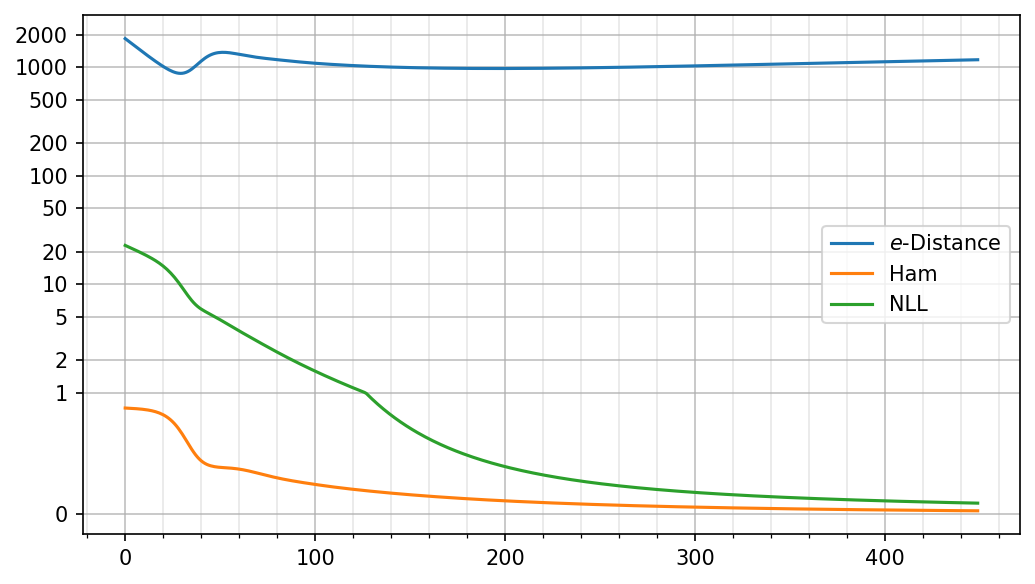}\hspace{0.02\textwidth}
    \includegraphics[width=0.48\textwidth]{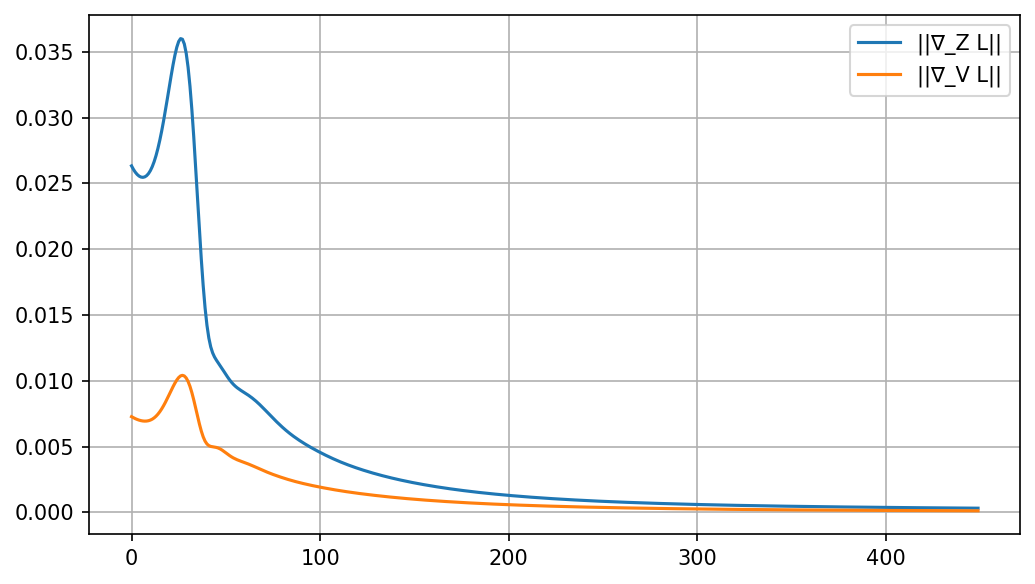}
    \caption{GPCA optimization on resized Cityscapes dataset with latent dimension $d=256$. The left plot shows the evolution of the NLL, Hamming distance, and $e$-distance. The right plot shows the gradient norms during optimization.}
    \label{fig:gpca-cityscapes-gradients-loss-nll}
\end{figure}

Figure~\ref{fig:rec-error} shows the final mean Hamming distance as a function of the latent dimension after GPCA training for $10^4$ iterations. The two curves correspond to Cityscapes and MNIST. Interestingly, zero Hamming distance on the entire training set is achieved with only $d=44$ latent dimensions for MNIST and $d=133$ latent dimensions for Cityscapes.

\begin{figure}[H]
    \centering
    \includegraphics[width=0.62\textwidth]{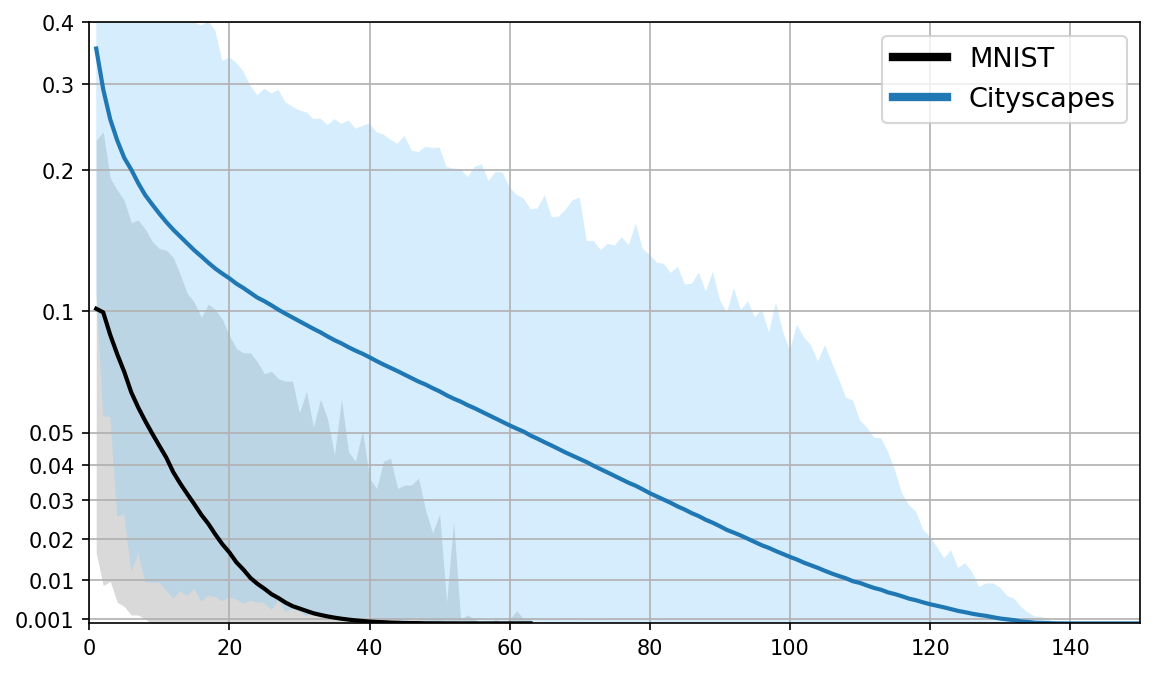}
    \caption{Hamming reconstruction error across latent dimension. The $y$-axis uses inverse-sinh scaling $y'=\sinh^{-1}(y/s)$ to emphasize small errors, with $s=0.05$.}
    \label{fig:rec-error}
\end{figure}

Figures~\ref{fig:MNIST-GPCA-embdding-2D} and~\ref{fig:GPCA-embedding-3D} provide additional visualizations of the GPCA representation obtained on MNIST. 
The visual comparisons in Appendix~\ref{appendix:gpca-visual-comparison} show that a latent dimension of $d=32$ already captures the main structure of the MNIST digits in the GPCA reconstructions.
Figure~\ref{fig:MNIST-GPCA-embdding-2D} shows a two-dimensional embedding of the latent points obtained from a $d=32$ dimensional GPCA subspace. The embedding is computed purely from the geometry of the latent representation and reveals a meaningful organization of the data without using label information. Finally, Figure~\ref{fig:GPCA-embedding-3D} shows a three-dimensional version of the latent embedding, providing a more detailed view of the same organization.


\begin{figure}[H]
\centerline{
\includegraphics[width=0.55\textwidth]{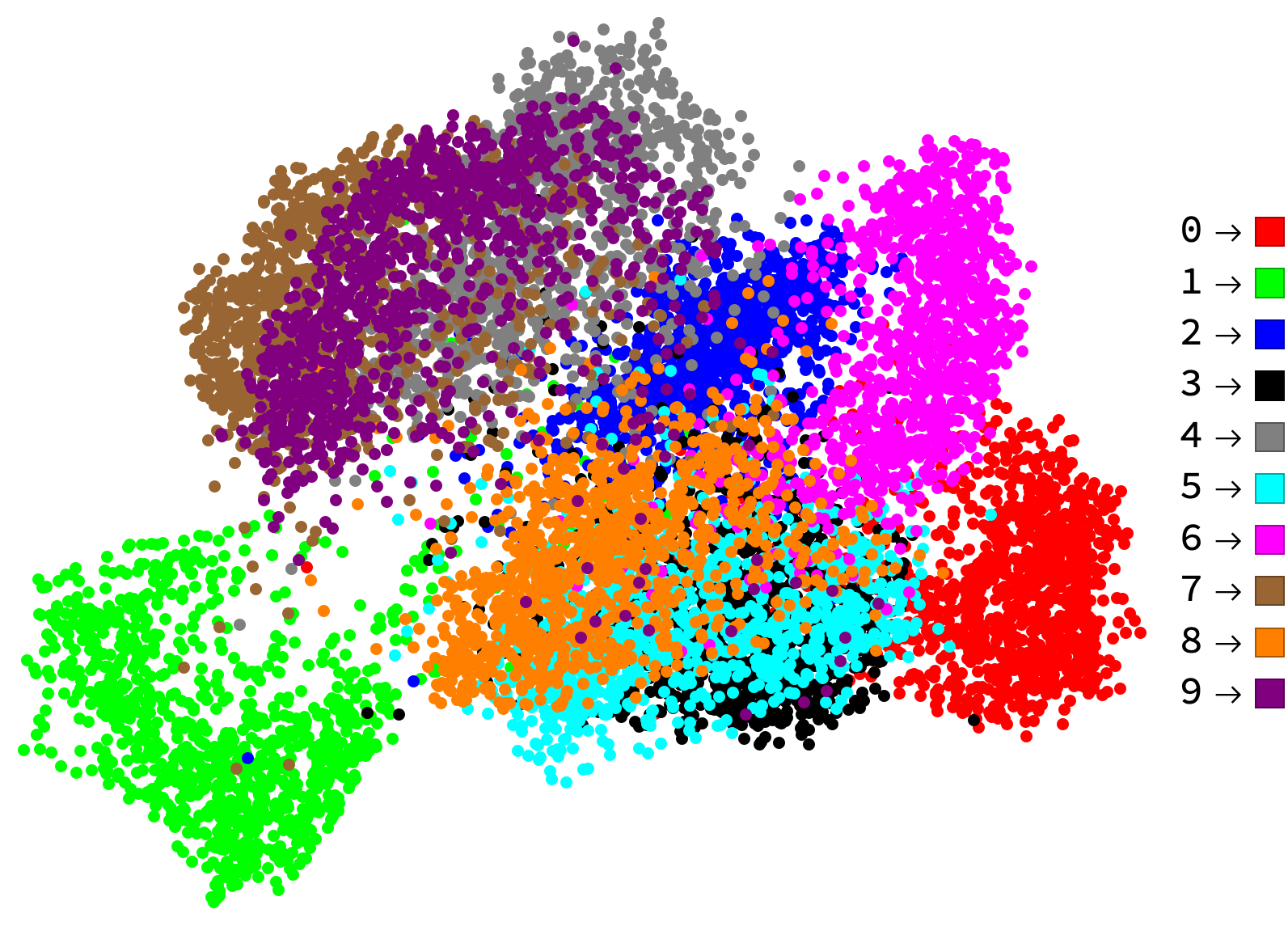}}
\caption{
A 2D embedding of the $d=32$ dimensional latent points $\wh \theta_{i} = V \odot Z_{i}\in\mc{U},\, i\in[N]$ corresponding to $10.000$ samples ${x}_{i}\approx \varphi^{-1}(\wh \theta_{i})$ of the MNIST data base.  This semantically meaningful and \textit{fully unsupervised} embedding of the $4$-nearest neighbor graph with nodes $\wh \theta_{i}$ is \textit{purely geometric}: no learning is involved besides determining the GPCA subspace $\mc{U}$. This result underlines the suitability of latent GPCA representations of discrete data, as illustrated by Figure \ref{fig:H3-GPCA}. }
\label{fig:MNIST-GPCA-embdding-2D}
\end{figure}

\begin{figure}[H]
\centerline{
\includegraphics[width=\textwidth]{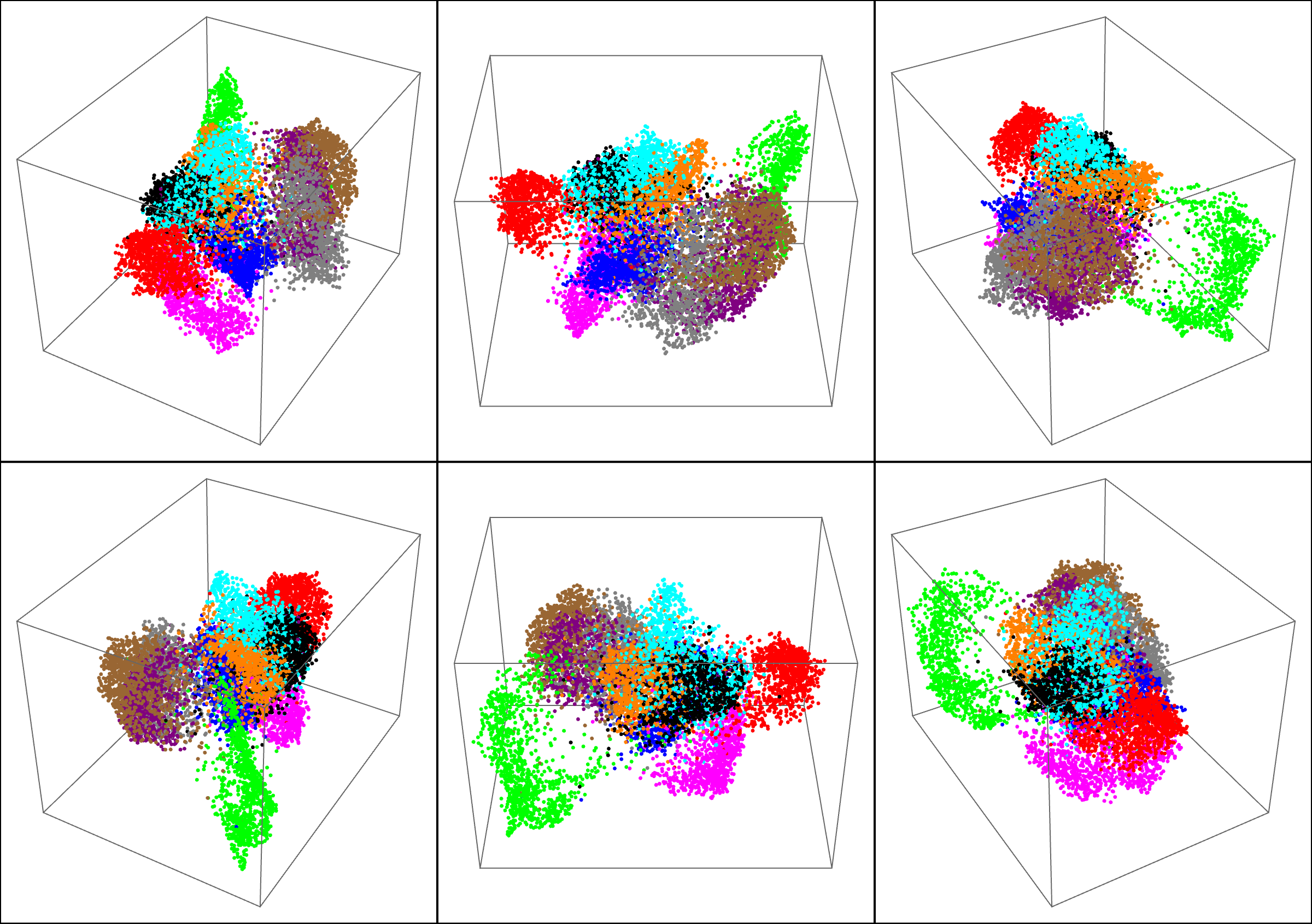}
}
\caption{
A more detailed 3D embedding of the $d=32$-dimensional latent data points representing a corresponding sample from the MNIST data base.}
\label{fig:GPCA-embedding-3D}
\end{figure}

Tables~\ref{tab:nll-mnist-edistance} and~\ref{tab:nll-cityscapes-edistance} show how the Hamming distance in \eqref{eq:hamming-distance} decreases together with the NLL in \eqref{eq:nll-objective} for different latent dimensions on MNIST and Cityscapes. For example, on Cityscapes, even a strongly reduced latent dimension $d=256$, corresponding to a reduction factor of $512$, is sufficient for the NLL to become small and for the Hamming distance to drop to zero within seconds, after only a few iterations. The $e$-distance in \eqref{eq:gpca-e-distance}, however, does not decrease as consistently and remains on the order of $10^3$ when optimizing the NLL.

By contrast, Tables~\ref{tab:kl-mnist-edistance} and \ref{tab:kl-cityscapes-edistance} show that optimizing the KL divergence substantially reduces the $e$-distance, although the resulting values may still remain relatively large. Figure~\ref{fig:gpca-cityscapes-gradients-loss} illustrates this behavior for GPCA optimization on Cityscapes with latent dimension $d=256$. As in Figure~\ref{fig:gpca-cityscapes-gradients-loss-nll}, the loss and the Hamming distance decrease at a similar rate. However, when optimizing the KL divergence, the $e$-distance is considerably smaller, while the final Hamming distance is larger than in the NLL case. The right panel shows that the gradient norms decay stably throughout optimization.

\begin{table}[H]
    \centering
    \caption{\textbf{NLL minimization.} Comparison of the Hamming distance, NLL, and the squared $e$-distance for the \textbf{MNIST} dataset across different latent dimensions.}
    \label{tab:nll-mnist-edistance}
    \begin{tabular}{lccccc}
        \toprule
         & $\boldsymbol{d=8}$ & $\boldsymbol{d=16}$ & $\boldsymbol{d=32}$ & $\boldsymbol{d=64}$ & $\boldsymbol{d=128}$  \\
        \midrule
        \textbf{NLL} \, $\boldsymbol{(\times 10^{-6})}$ & $110123$ & $57950$ & $3803$ & $0.0$ & $0.0$  \\
        \textbf{MHD} \, $\boldsymbol{(\times 10^{-6})}$ & $54501$  & $26023$ & $1352$ & $0.0$ & $0.0$ \\
        \textbf{e-Distance} \, $\boldsymbol{(\times 10^{3})}$  & $73.333$  & $171.252$ & $168.111$ & $3.817$ & $1.713$ \\
        \bottomrule
    \end{tabular}
\end{table}

\begin{table}[H]
    \centering
    \caption{\textbf{NLL minimization.} Comparison of the Hamming distance, NLL, and the squared $e$-distance for the \textbf{Cityscapes} dataset across different latent dimensions.}
    \label{tab:nll-cityscapes-edistance}
    \begin{tabular}{lccccc}
        \toprule
         & $\boldsymbol{d=32}$ & $\boldsymbol{d=64}$ & $\boldsymbol{d=128}$ & $\boldsymbol{d=256}$ & $\boldsymbol{d=512}$  \\
        \midrule
        \textbf{NLL} \, $\boldsymbol{(\times 10^{-6})}$ &  $254804$ & $123109$ & $4$ & $0.0$ & $0.0$ \\
        \textbf{MHD} \, $\boldsymbol{(\times 10^{-6})}$  & $88305$ & $42746$ & $0.0$ & $0.0$ & $0.0$   \\
        \textbf{e-Distance} \, $\boldsymbol{(\times 10^{3})}$  & $493.297$ & $1192.090$ & $212.277$ & $8.332$ & $6.593$   \\
        \bottomrule
    \end{tabular}
\end{table}

\begin{table}[H]
    \centering
    \caption{\textbf{Cross-entropy minimization.} Comparison of the Hamming distance, cross-entropy, and the squared $e$-distance for the \textbf{MNIST} dataset across different latent dimensions.}
    \label{tab:kl-mnist-edistance}
    \begin{tabular}{lccccc}
        \toprule
         & $\boldsymbol{d=8}$ & $\boldsymbol{d=16}$ & $\boldsymbol{d=32}$ & $\boldsymbol{d=64}$ & $\boldsymbol{d=128}$  \\
        \midrule
        \textbf{Cross-Entropy} \, $\boldsymbol{(\times 10^{-6})}$ & $143713$ & $102126$ & $63448$ & $44578$ & $36420$  \\
        \textbf{MHD} \, $\boldsymbol{(\times 10^{-6})}$ & $55357$ & $29450$ & $7705$ & $746$ & $37$  \\
        \textbf{e-Distance} & $4.670$ & $8.117$ & $14.454$ & $7.647$ & $2.550$ \\
        \bottomrule
    \end{tabular}
\end{table}
\begin{table}[H]
    \centering
    \caption{\textbf{Cross-entropy minimization.} Comparison of the Hamming distance, cross-entropy, and the squared $e$-distance for the \textbf{Cityscapes} dataset across different latent dimensions.}
    \label{tab:kl-cityscapes-edistance}
    \begin{tabular}{lccccc}
        \toprule
         & $\boldsymbol{d=32}$ & $\boldsymbol{d=64}$ & $\boldsymbol{d=128}$ & $\boldsymbol{d=256}$ & $\boldsymbol{d=512}$ \\
        \midrule
        \textbf{Cross-Entropy} \, $\boldsymbol{(\times 10^{-6})}$ & $353188$ & $251113$ & $150246$ & $92236$ & $76277$  \\
        \textbf{MHD} \, $\boldsymbol{(\times 10^{-6})}$ & $96587$ & $57206$ & $12970$ & $0.0$ & $0.0$  \\
        \textbf{e-Distance} & $91.55$ & $118.65$ & $157.32$ & $58.23$ & $17.33$ \\
        \bottomrule
    \end{tabular}
\end{table}

\begin{figure}[H]
    \centering
    \begin{subfigure}[t]{0.48\textwidth}
        \centering
        \includegraphics[width=\linewidth]{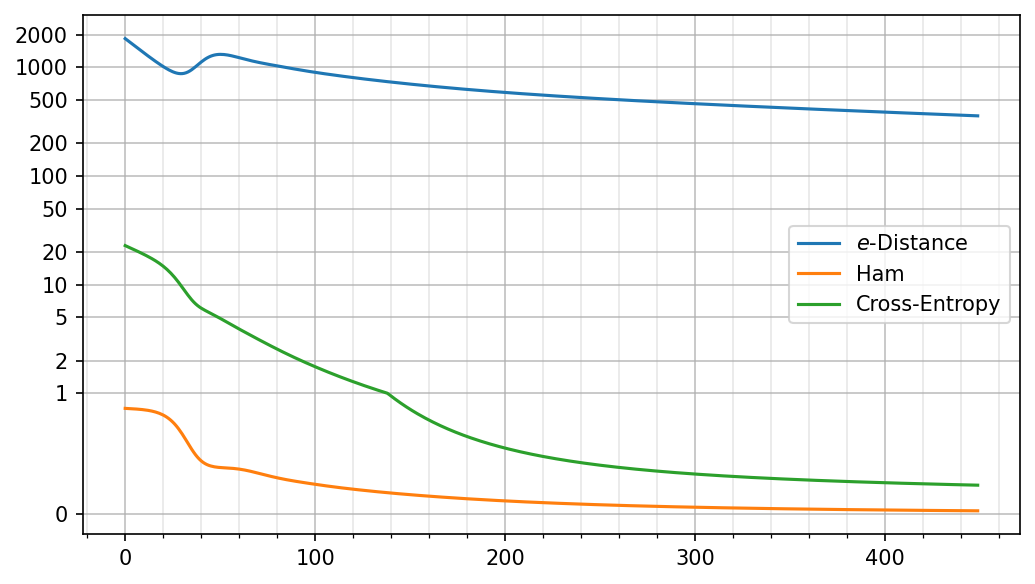}
    \end{subfigure}\hfill
    \begin{subfigure}[t]{0.48\textwidth}
        \centering
        \includegraphics[width=\linewidth]{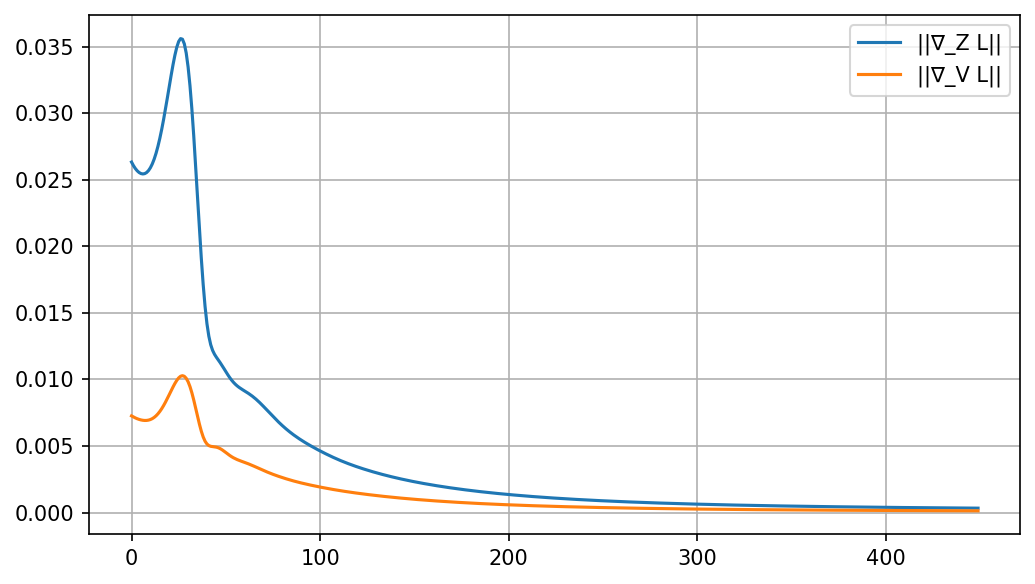}
    \end{subfigure}
    \caption{GPCA optimization on resized Cityscapes dataset with latent dimension $d=256$. The left panel shows the evolution of the KL divergence, Hamming distance, and $e$-distance. The right panel shows the gradient norms during optimization.}
    \label{fig:gpca-cityscapes-gradients-loss}
\end{figure}

\subsection{GPCA objective with $e$-distance regularization} \label{sec:gpca-edistance-reg}

Proposition~\ref{prop:kl-edistance-compact} states that, on a fixed compact set of the assignment manifold $\mc S_c^n$, the KL divergence controls the squared $e$-distance between data points and their reconstructions.Consequently, along the GPCA optimization, decreasing the KL divergence reduces the upper bound on the squared $e$-distance, up to constants depending on a fixed compact set satisfying the assumptions of Proposition~\ref{prop:kl-edistance-compact}, provided the data and their reconstructions remain in this set.
To prevent the reconstructions from approaching the boundary of the simplex too closely, we propose to regularize the GPCA objective by adding an explicit $e$-distance penalty and minimize the objective in \eqref{eq:loss-gpca-e_reg} which reads
\begin{equation*}
\mc L_{\lambda}(V, Z){=}  \frac{1}{Nn} \sum_{i=1}^{N} \sum_{j=1}^n \KL(x_{ij} ,  \varphi^{-1}((V \odot Z_i)_j)) {+} \lambda \, \mathrm{d}_e^2(x_{ij}, \varphi^{-1}((V \odot Z_i)_j)).
\end{equation*}
Here, the second term in \eqref{eq:loss-gpca-e_reg} penalizes reconstructions that are far from the data in $e$-distance, encouraging agreement with the smoothed target probabilities. The parameter $\lambda>0$ balances the contribution of the regularization term relative to the KL term.

Empirically, we observe that smaller $\lambda$ tend to yield lower Hamming distances, whereas larger values enforce stronger control of the $e$-distance. Hence, the main trade-off is between achieving highly accurate discrete reconstructions, as measured by the Hamming distance, and maintaining sufficiently small $e$-distance. This behavior is illustrated in Tables ~\ref{tab:mnist-128-loss-lambda}, and ~\ref{tab:cityscapes-256-loss-lambda}.
For example, choosing $\lambda=10^{-3}$ appears to provide a reasonable compromise for the MNIST dataset with latent dimension $d=128$. It yields a Hamming distance of $\MHD = 9.6\cdot 10^{-5}$ together with a small squared $e$-distance. Interpreted per image, this corresponds on average to about $0.1$  misclassified pixels.
We also observe that, for fixed $\lambda$, increasing the latent dimension decreases the GPCA loss, the Hamming distance, and the $e$-distance. This is consistent with the fact that a higher-dimensional representation allows for sharper and more accurate reconstructions, cf. Tables ~\ref{tab:mnist-dims-lambda-1e-2}--~\ref{tab:cityscapes-dims-lambda-3}.

We also considered an alternative regularized objective in which the $e$-distance penalty is replaced by a negative entropy penalty on the reconstructions. Empirically, this leads to qualitatively similar behavior, i.e. the entropy penalty prevents reconstructions from becoming too close to the corners of the simplex. In addition, smaller regularization weights tend to improve the Hamming distance, whereas larger weights provide stronger control of the $e$-distance. Overall, the results are comparable to those obtained with the $e$-distance penalty, so this modification does not lead to substantial changes in performance.


\begin{table}[H]
    \centering
    \caption{Cross-entropy, Hamming distance and squared $e$-distance for different latent dimensions $\boldsymbol{d}$ on the MNIST dataset with regularization parameter $\boldsymbol{\lambda=10^{-2}}$}
    \label{tab:mnist-dims-lambda-1e-2}
	  \makebox[\linewidth][c]{%
    \begin{tabular}{lcccccc}
        \toprule
         &$\boldsymbol{d=8}$ & $\boldsymbol{d=16}$ & $\boldsymbol{d=32}$ & $\boldsymbol{d=64}$ & $\boldsymbol{d=128}$ & $\boldsymbol{d=256}$ \\
        \midrule
        \textbf{Cross-Entropy} \, $\boldsymbol{(\times 10^{-6})}$ 
        & $173149$ & $129212$ & $89401$ & $62162$ & $45868$ & $36243$ \\
        \textbf{MHD} \, $\boldsymbol{(\times 10^{-6})}$
        & $56267$ & $33396$ & $14260$ & $3604$ & $495$ & $97$ \\
        \textbf{e-Distance} 
        & $2.724$ & $2.206$ & $1.699$ & $1.147$ & $0.689$ &$0.268589$ \\
        \bottomrule
    \end{tabular}
    }
\end{table}


\begin{table}[H]
    \centering
    \caption{Cross-entropy, Hamming distance and squared $e$-distance for different latent dimensions $\boldsymbol{d}$ on the MNIST dataset with regularization parameter $\boldsymbol{\lambda=10^{-3}}$}
    \label{tab:mnist-dims-lambda-1e-3}
    \makebox[\linewidth][c]{
    \begin{tabular}{lcccccc}
        \toprule
         &$\boldsymbol{d=8}$ & $\boldsymbol{d=16}$ & $\boldsymbol{d=32}$ & $\boldsymbol{d=64}$ & $\boldsymbol{d=128}$ & $\boldsymbol{d=256}$ \\
        \midrule
        \textbf{Cross-Entropy} \, $\boldsymbol{(\times 10^{-6})}$ & $173149$ & $129212$ & $89401$ & $62162$ & $38158$ & $36243$ \\
        \textbf{MHD} \, $\boldsymbol{(\times 10^{-6})}$& $56267$ & $33396$& $14260$ & $3602$ & $96$ & $0$ \\
        \textbf{e-Distance} & $2.72$ & $2.21$  & $1.70$ & $1.48$ & $1.38$ & $0.27$ \\
        \bottomrule
    \end{tabular}
    }
\end{table}

\begin{table}[H]
    \centering
    \caption{Cross-entropy, Hamming distance and squared $e$-distance for different latent dimensions $\boldsymbol{d}$ on the Fashion-MNIST dataset with regularization parameter $\boldsymbol{\lambda=10^{-2}}$}
    \label{tab:fashion-mnist-dims-lambda-1e-2}
    \begin{tabular}{lcccc}
        \toprule
         &$\boldsymbol{d=16}$ & $\boldsymbol{d=32}$ & $\boldsymbol{d=64}$ & $\boldsymbol{d=128}$ \\
        \midrule
        \textbf{Cross-Entropy} \, $\boldsymbol{(\times 10^{-6})}$ 
        & $159478$ & $128362$ & $99688$ & $71593$  \\
        \textbf{MHD} \, $\boldsymbol{(\times 10^{-6})}$
        & $44509$ & $31548$ & $19176$ & $6812$  \\
        \textbf{e-Distance} 
        & $2.771$ & $2.292$ & $1.855$ & $1.417$ \\
        \bottomrule
    \end{tabular}
\end{table}

\begin{table}[H]
    \centering
    \caption{Cross-entropy, Hamming distance and squared $e$-distance for different latent dimensions $\boldsymbol{d}$ on the Fashion-MNIST dataset with regularization parameter $\boldsymbol{\lambda=10^{-3}}$}
    \label{tab:fashion-mnist-dims-lambda-1e-3}
    \begin{tabular}{lcccc}
        \toprule
         &$\boldsymbol{d=16}$ & $\boldsymbol{d=32}$ & $\boldsymbol{d=64}$ & $\boldsymbol{d=128}$ \\
        \midrule
        \textbf{Cross-Entropy} \, $\boldsymbol{(\times 10^{-6})}$ 
        & $132078$ & $104965$ & $79654$ & $54938$  \\
        \textbf{MHD} \, $\boldsymbol{(\times 10^{-6})}$
        & $41995$ & $28759$ & $15590$ & $3128$  \\
        \textbf{e-Distance} 
        & $3.964$ & $3.647$ & $3.543$ & $3.424$ \\
        \bottomrule
    \end{tabular}
\end{table}

\begin{table}[H]
    \centering
    \caption{Cross-entropy, Hamming distance and squared $e$-distance for different latent dimensions $\boldsymbol{d}$ on the \textbf{Cityscapes} dataset with regularization parameter $\boldsymbol{\lambda{=}10^{-1}}$}
    \label{tab:cityscapes-dims-lambda-1e-1}
    \makebox[\linewidth][c]{
    \begin{tabular}{lcccccc}
        \toprule
          & $\boldsymbol{d{=}32}$ & $\boldsymbol{d{=}64}$ &$\boldsymbol{d{=}128}$ & $\boldsymbol{d{=}256}$ & $\boldsymbol{d{=}512}$ & $\boldsymbol{d{=}1024}$ \\
        \midrule
        \textbf{Cross-Entropy} \, $\boldsymbol{(\times 10^{-6})}$  & $1405488$ & $1162139$ & $928690$ & $698497$ & $477215$ & $283546$ \\
        \textbf{MHD} \, $\boldsymbol{(\times 10^{-6})}$ &  $125910$ & $98440$ & $71155$ & $41466$ & $13118$ & $752$ \\
        \textbf{e-Distance} &  $9.556$ & $7.990$ & $6.530$ & $5.067$ & $3.548$ & $2.007$ \\
        \bottomrule
    \end{tabular}
    }
\end{table}

\begin{table}[H]
    \centering
    \caption{Cross-entropy, Hamming distance and squared $e$-distance for different latent dimensions $\boldsymbol{d}$ on the \textbf{Cityscapes} dataset with regularization parameter $\boldsymbol{\lambda{=}10^{-2}}$}
    \label{tab:cityscapes-dims-lambda-2}
    \makebox[\linewidth][c]{
    \begin{tabular}{lcccccc}
        \toprule
          & $\boldsymbol{d{=}32}$ & $\boldsymbol{d{=}64}$ & $\boldsymbol{d{=}128}$ & $\boldsymbol{d{=}256}$ & $\boldsymbol{d{=}512}$ & $\boldsymbol{d{=}1024}$ \\
        \midrule
        \textbf{Cross-Entropy} \, $\boldsymbol{(\times 10^{-6})}$ & $510419$ & $404977$ & $301960$ & $207198$ & $137594$ & $96735$ \\
        \textbf{MHD} \, $\boldsymbol{(\times 10^{-6})}$  & $112242$ & $80375$ & $46083$ & $11937$ & $128$ & $0.00$ \\
        \textbf{e-Distance} &  $11.619$ & $10.311$ & $8.97$ & $7.24$ & $4.65$ & $2.25$ \\
        \bottomrule
    \end{tabular}
    }
\end{table}


\begin{table}[H]
    \centering
    \caption{Cross-entropy, Hamming distance and squared $e$-distance for different latent dimensions $\boldsymbol{d}$ on the \textbf{Cityscapes} dataset with regularization parameter $\boldsymbol{\lambda{=}10^{-3}}$}
    \label{tab:cityscapes-dims-lambda-3}
    \makebox[\linewidth][c]{
    \begin{tabular}{lcccccc}
        \toprule
        & $\boldsymbol{d{=}32}$ & $\boldsymbol{d{=}64}$ &$\boldsymbol{d{=}128}$ & $\boldsymbol{d{=}256}$ & $\boldsymbol{d{=}512}$ & $\boldsymbol{d{=}1024}$ \\
        \midrule
        \textbf{Cross-Entropy} \, $\boldsymbol{(\times 10^{-6})}$  & $385214$ & $287447$ & $192681$ & $119147$ & $86937$ & $74136$ \\
        \textbf{MHD} \, $\boldsymbol{(\times 10^{-6})}$ &  $102815$ & $66227$ & $26065$ & $412$ & $0.00$ & $0.00$ \\
        \textbf{e-Distance} &  $23.837$ & $22.684$ & $21.705$ & $18.204$ & $8.418$ & $3.286$ \\
        \bottomrule
    \end{tabular}
    }
\end{table}

\begin{table}[H]
    \centering
    \caption{Cross-entropy, Hamming distance and squared $e$-distance for different values of the regularization parameter $\boldsymbol{\lambda}$ on the MNIST dataset with latent dimension $\boldsymbol{d=128}$}
    \label{tab:mnist-128-loss-lambda}
    \begin{tabular}{lccccc}
        \toprule
         & $\boldsymbol{\lambda=1}$ & $\boldsymbol{\lambda=0.1}$ & $\boldsymbol{\lambda=0.01}$ & $\boldsymbol{\lambda=0.001}$ &  $\boldsymbol{\lambda=0}$ \\
        \midrule
        \textbf{Cross-Entropy} \, $\boldsymbol{(\times 10^{-6})}$ & $586701$ & $97695$ & $45868$ & $38158$ & $36420$  \\
        \textbf{MHD} \, $\boldsymbol{(\times 10^{-6})}$ & $3122$ & $1841$ & $495$ & $96$ & $37$  \\
        \textbf{e-Distance} & $0.54$ & $0.55$ & $0.69$ & $1.38$ & $2.55$ \\
        \bottomrule
    \end{tabular}
\end{table}

\begin{table}[H]
    \centering
    \caption{Cross-entropy, Hamming distance and squared $e$-distance for different values of the regularization parameter $\boldsymbol{\lambda}$ on the \textbf{Cityscapes} dataset with latent dimension $\boldsymbol{d{=}256}$}
    \label{tab:cityscapes-256-loss-lambda}
    \begin{tabular}{lccccc}
        \toprule
         & $\boldsymbol{\lambda{=}1}$ & $\boldsymbol{\lambda{=}0.1}$ & $\boldsymbol{\lambda{=}0.01}$ & $\boldsymbol{\lambda{=}0.001}$ &  $\boldsymbol{\lambda{=}0}$ \\
        \midrule
        \textbf{Cross-Entropy} \, $\boldsymbol{(\times 10^{-6})}$ & $5133136$ & $698159$ & $207198$ & $119147$ & $99236$  \\
        \textbf{MHD} \, $\boldsymbol{(\times 10^{-6})}$ & $143927$ & $41458$ & $11937$ & $412$ & $00$  \\
        \textbf{e-Distance} & $4.91$ & $5.06$ & $7.24$ & $18.20$ & $58.23$ \\
        \bottomrule
    \end{tabular}
\end{table}
\clearpage

\subsection{Visual comparison of GPCA reconstructions}\label{appendix:gpca-visual-comparison}
We further illustrate the effect of the regularization parameter $\lambda$ on the GPCA reconstructions by visualizing the approximations $\widehat{p}_{i}$ for different values of $\lambda$ and different latent dimensions $d$.

\paragraph{Synthetic large-class datasets.}
We consider four synthetic datasets, each represented by a histogram of a discrete distribution on $n=2$ variables with $c=92$ categories. From each dataset, we sample $N=3 \cdot 10^4$ points and then apply GPCA, see Section~\ref{sec:GPCA-Latent}. This yields a low-dimensional representation, which is subsequently used to train our model via Flow Matching, as described in Section~\ref{sec:Flow-Matching}.

\begin{figure}[H]
	\centering

	{\small
	\makebox[0.175\textwidth][c]{\boldmath$\lambda=10^{0}$}\hspace{0.03\textwidth}%
	\makebox[0.175\textwidth][c]{\boldmath$\lambda=10^{-1}$}\hspace{0.03\textwidth}%
	\makebox[0.175\textwidth][c]{\boldmath$\lambda=10^{-2}$}\hspace{0.03\textwidth}%
	\makebox[0.175\textwidth][c]{\boldmath$\lambda=10^{-3}$}\hspace{0.03\textwidth}%
	\makebox[0.175\textwidth][c]{\textbf{Ground truth}}
	\par\medskip
	\captionsetup[subfigure]{justification=centering}
	\begin{subfigure}[t]{\textwidth}
		\centering
        \begingroup\setlength{\fboxsep}{0pt}\fbox{\includegraphics[width=\dimexpr0.175\textwidth-2\fboxrule\relax]{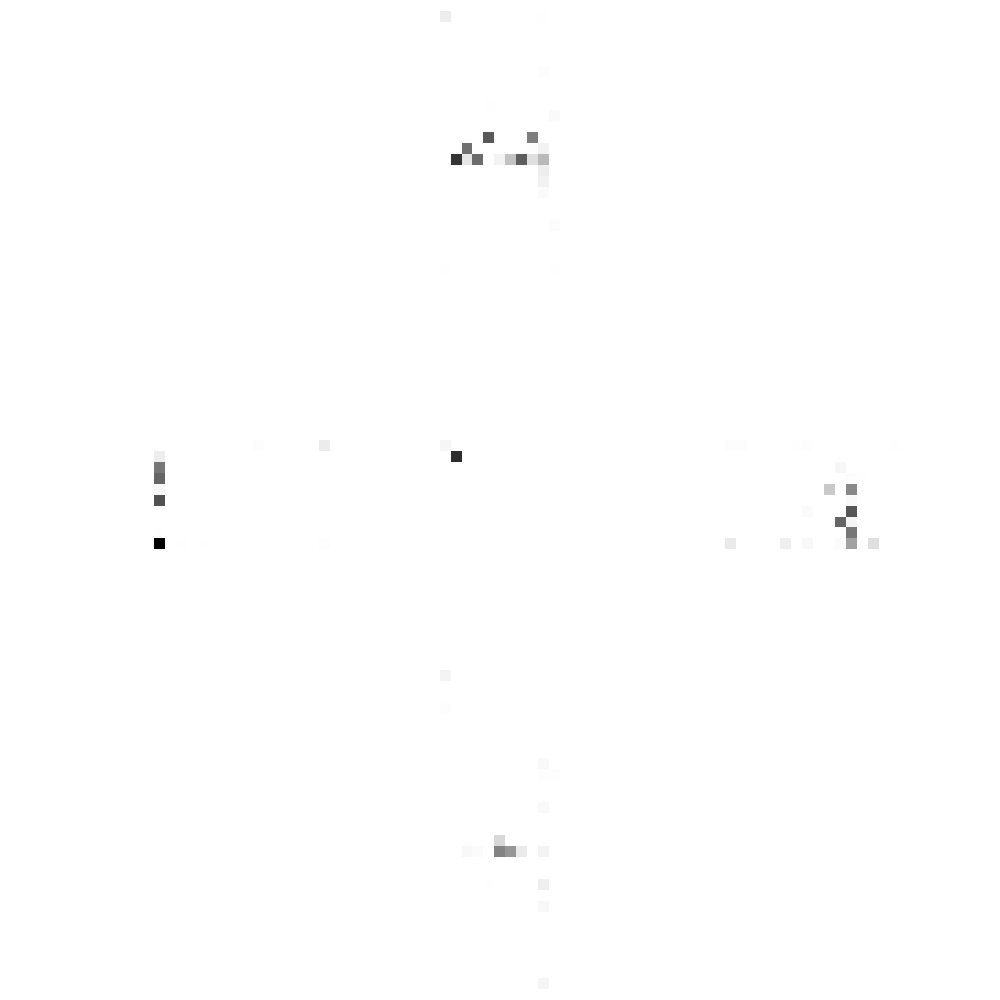}}\endgroup{}\hspace{0.03\textwidth}%
        \begingroup\setlength{\fboxsep}{0pt}\fbox{\includegraphics[width=\dimexpr0.175\textwidth-2\fboxrule\relax]{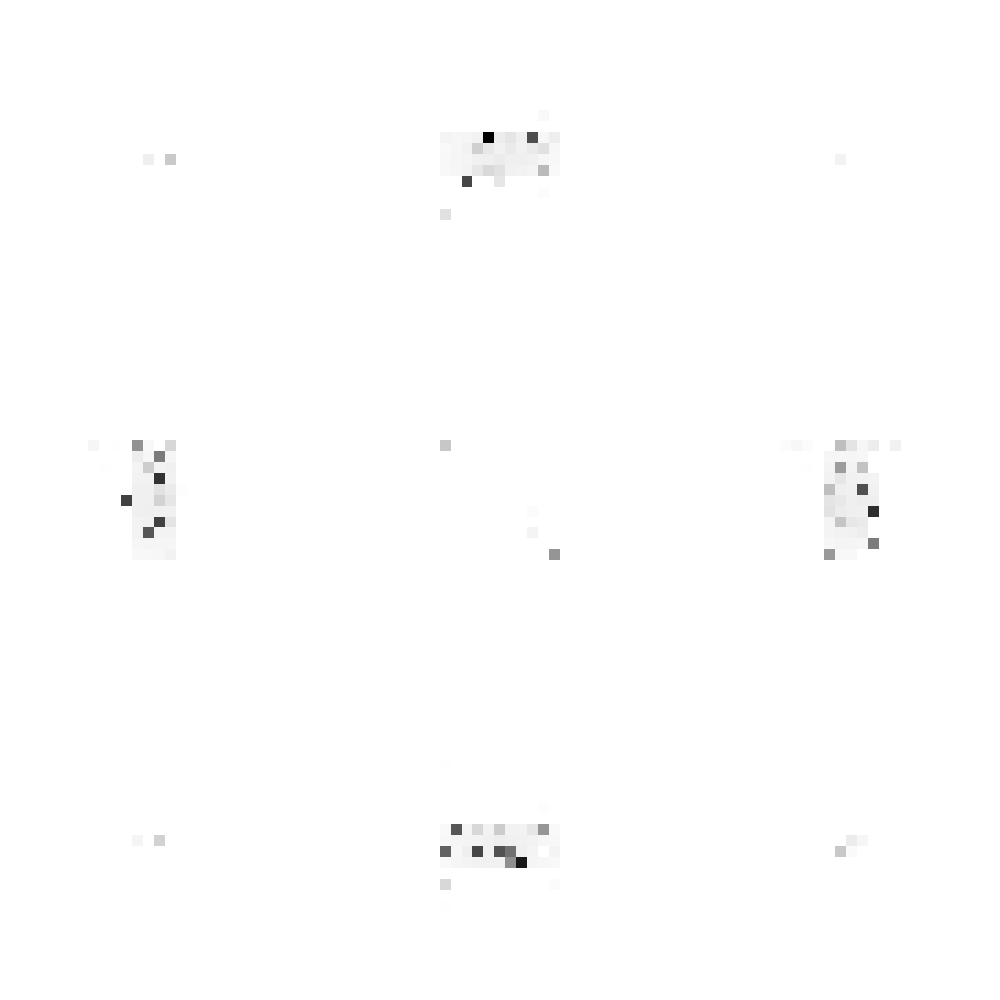}}\endgroup{}\hspace{0.03\textwidth}%
        \begingroup\setlength{\fboxsep}{0pt}\fbox{\includegraphics[width=\dimexpr0.175\textwidth-2\fboxrule\relax]{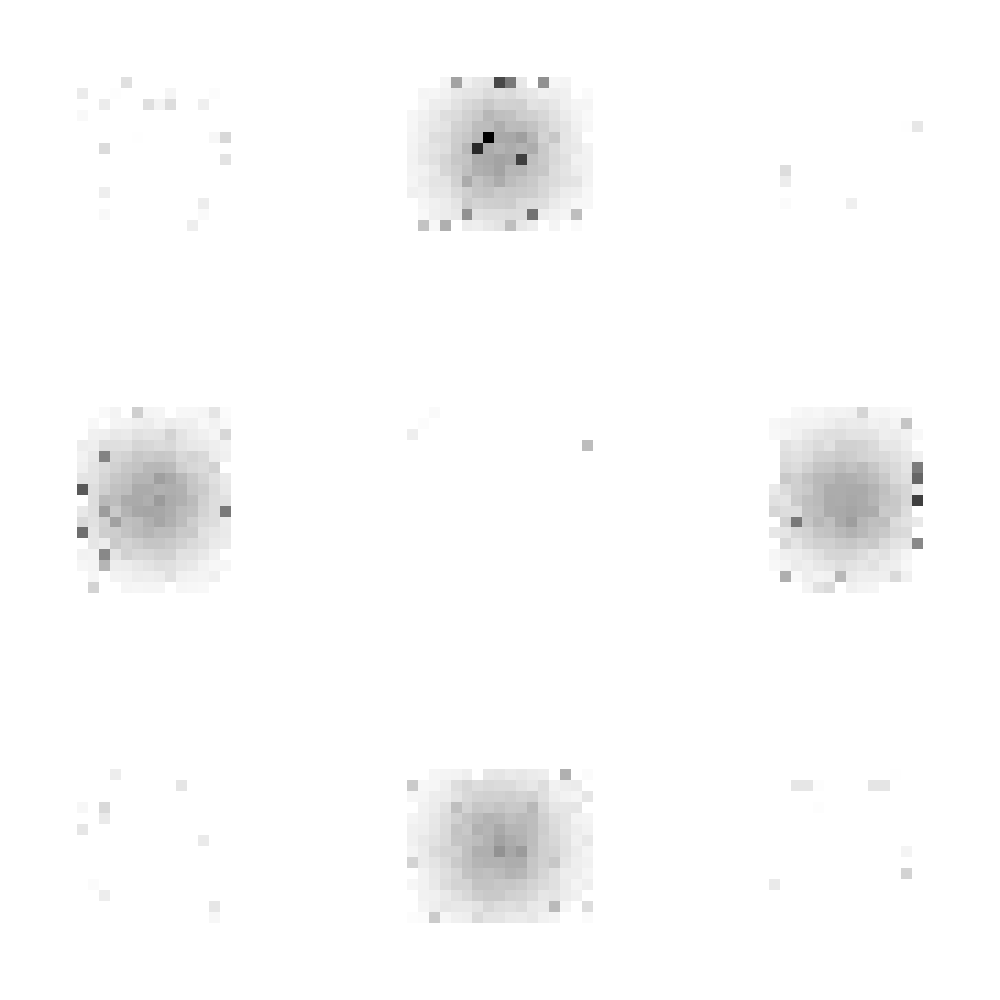}}\endgroup{}\hspace{0.03\textwidth}%
        \begingroup\setlength{\fboxsep}{0pt}\fbox{\includegraphics[width=\dimexpr0.175\textwidth-2\fboxrule\relax]{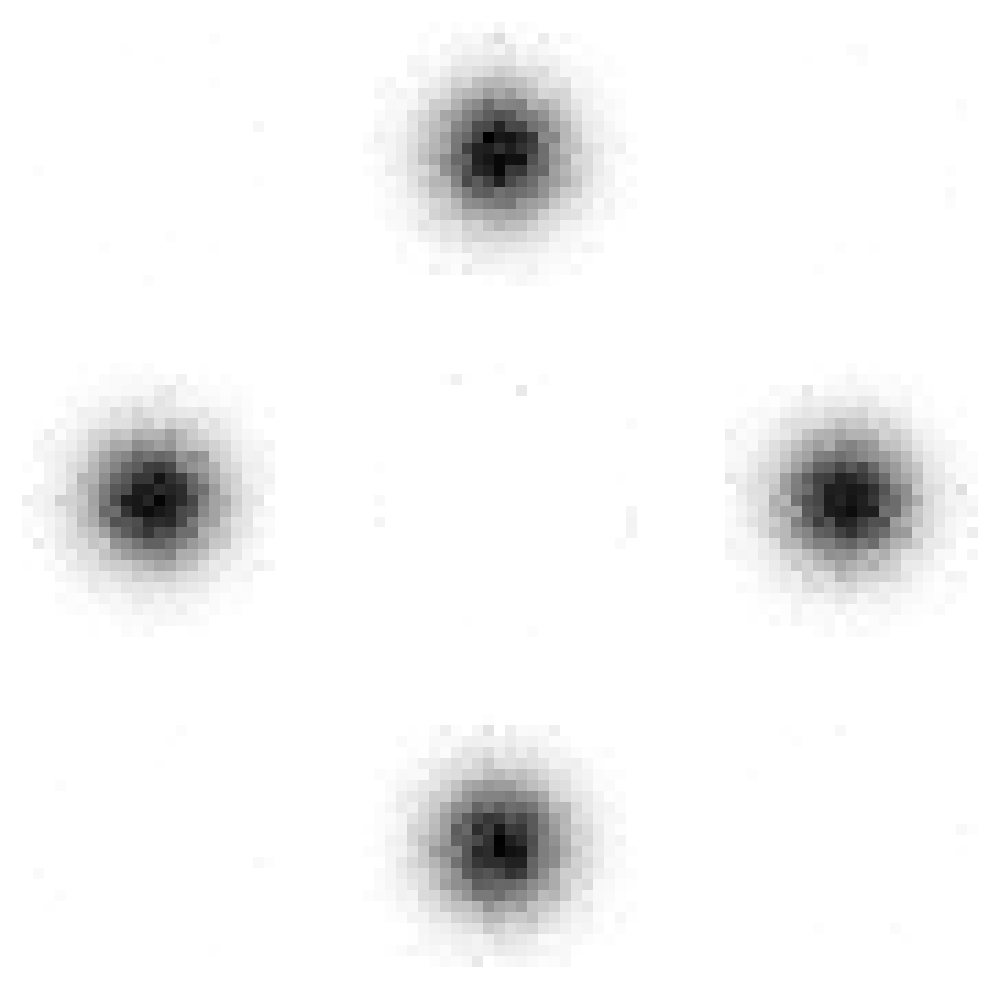}}\endgroup{}\hspace{0.03\textwidth}%
		\includegraphics[width=0.175\textwidth]{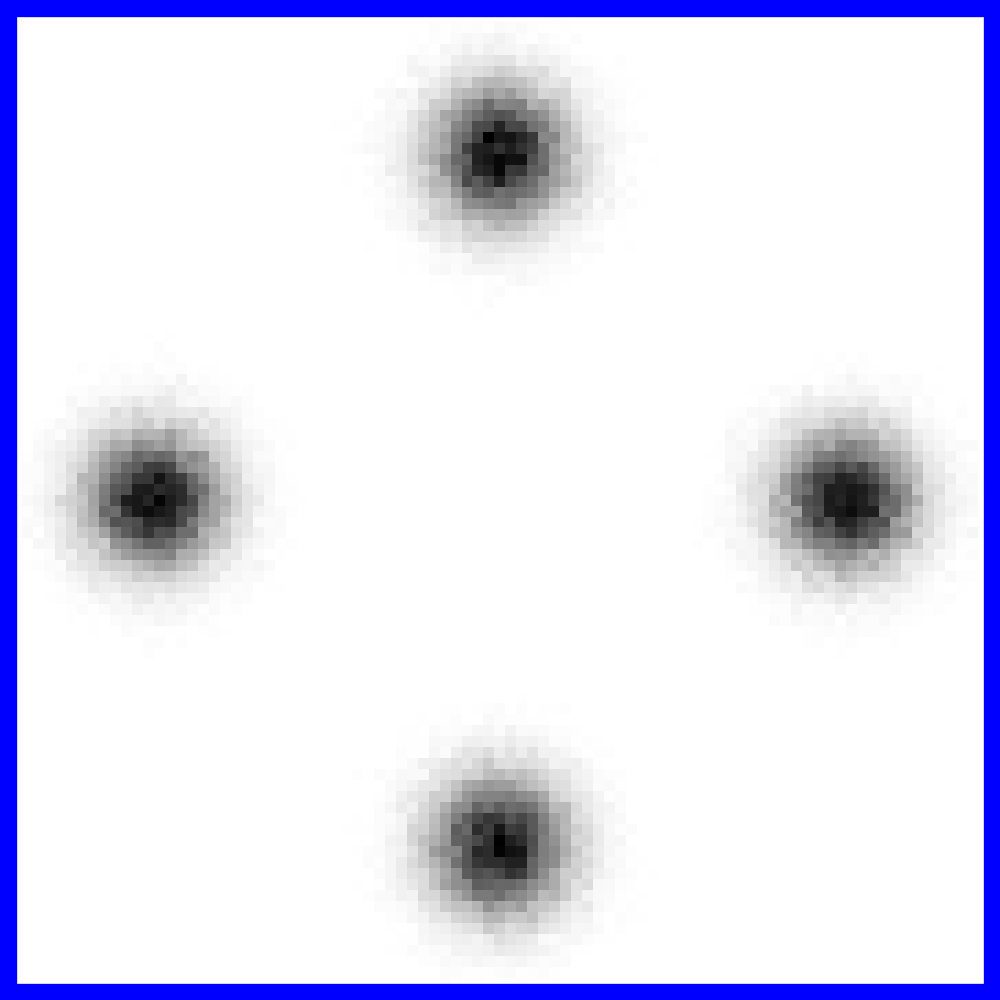}
		\subcaption{Four Gaussians}
		\label{fig:simple-datasets-four-gaussians-lambda}
	\end{subfigure}

	\vspace{0.5em}
	\begin{subfigure}[t]{\textwidth}
		\centering
        \begingroup\setlength{\fboxsep}{0pt}\fbox{\includegraphics[width=\dimexpr0.175\textwidth-2\fboxrule\relax]{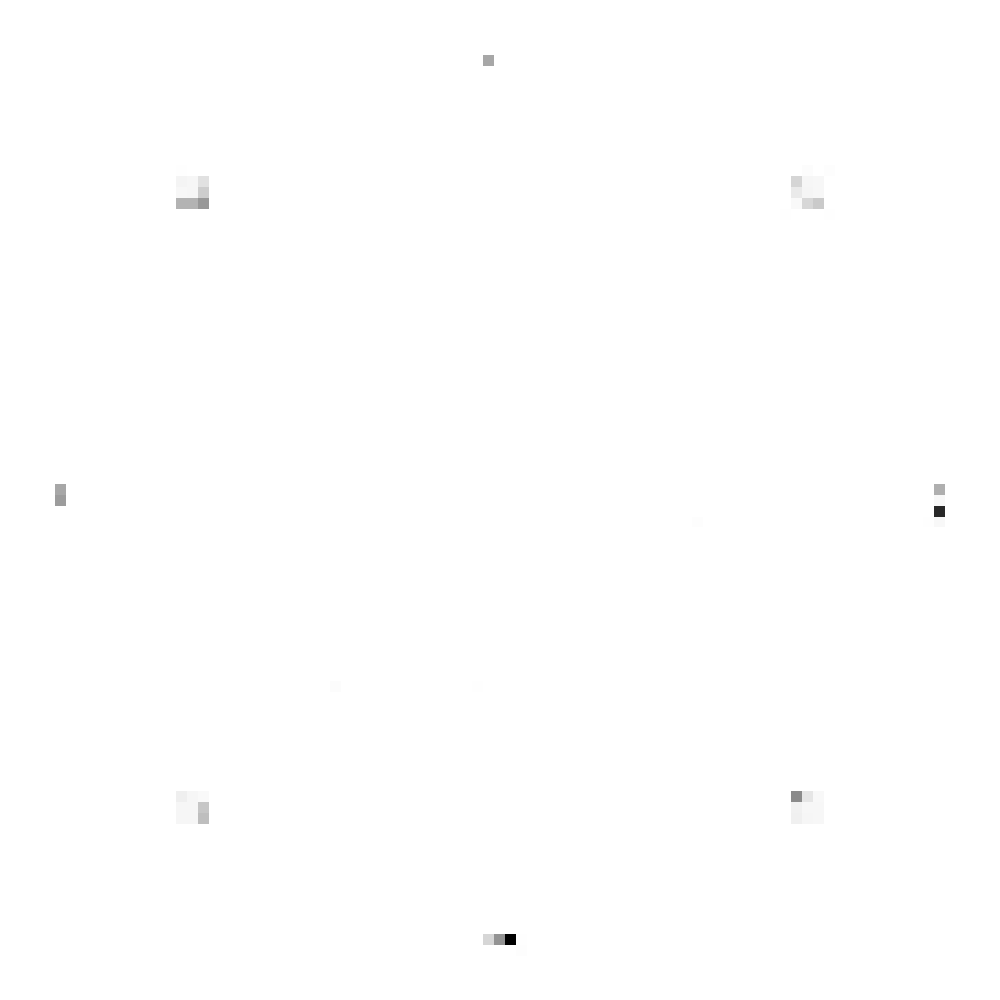}}\endgroup{}\hspace{0.03\textwidth}%
        \begingroup\setlength{\fboxsep}{0pt}\fbox{\includegraphics[width=\dimexpr0.175\textwidth-2\fboxrule\relax]{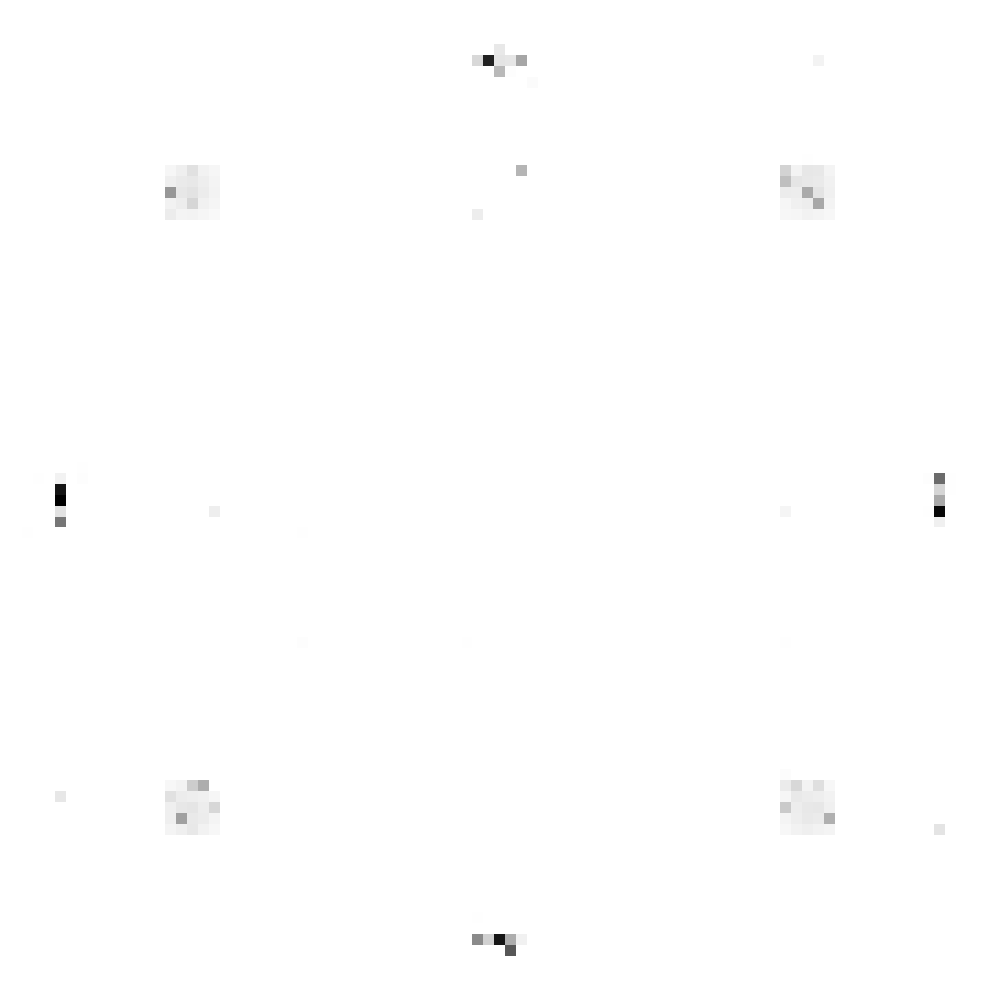}}\endgroup{}\hspace{0.03\textwidth}%
        \begingroup\setlength{\fboxsep}{0pt}\fbox{\includegraphics[width=\dimexpr0.175\textwidth-2\fboxrule\relax]{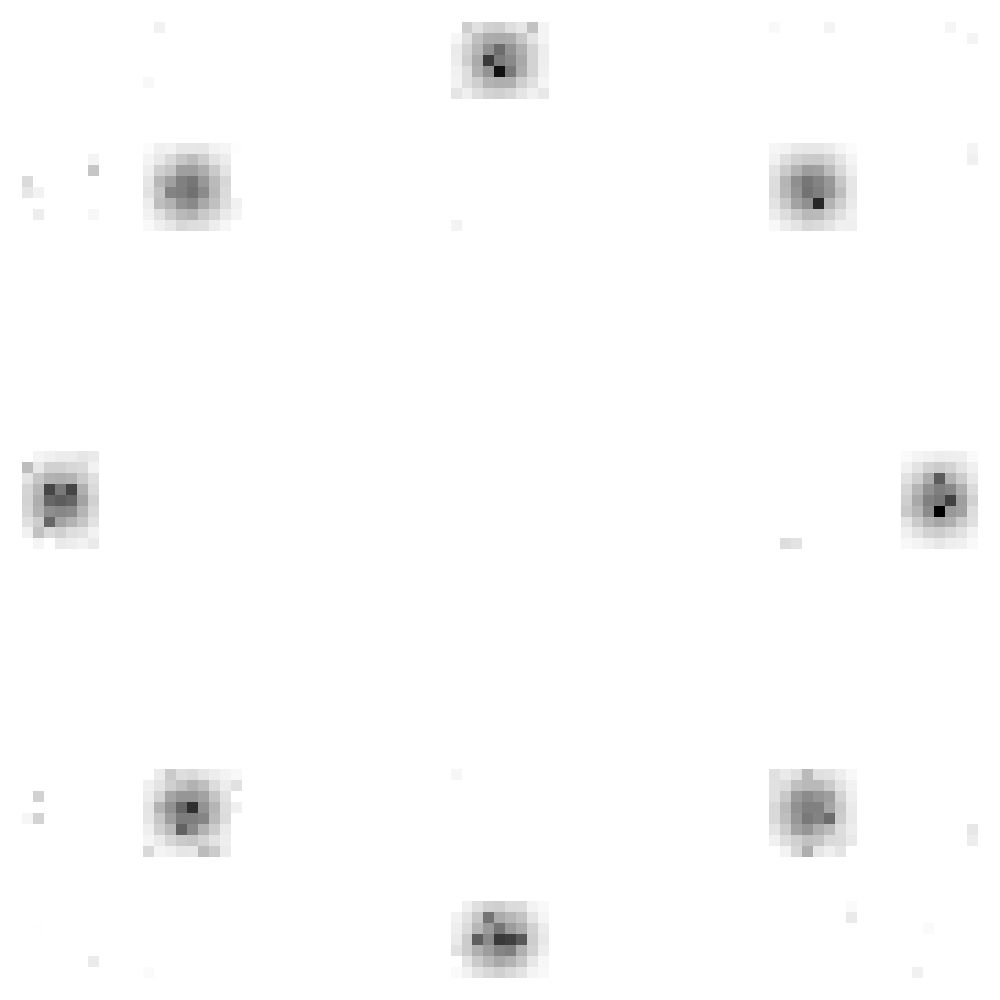}}\endgroup{}\hspace{0.03\textwidth}%
        \begingroup\setlength{\fboxsep}{0pt}\fbox{\includegraphics[width=\dimexpr0.175\textwidth-2\fboxrule\relax]{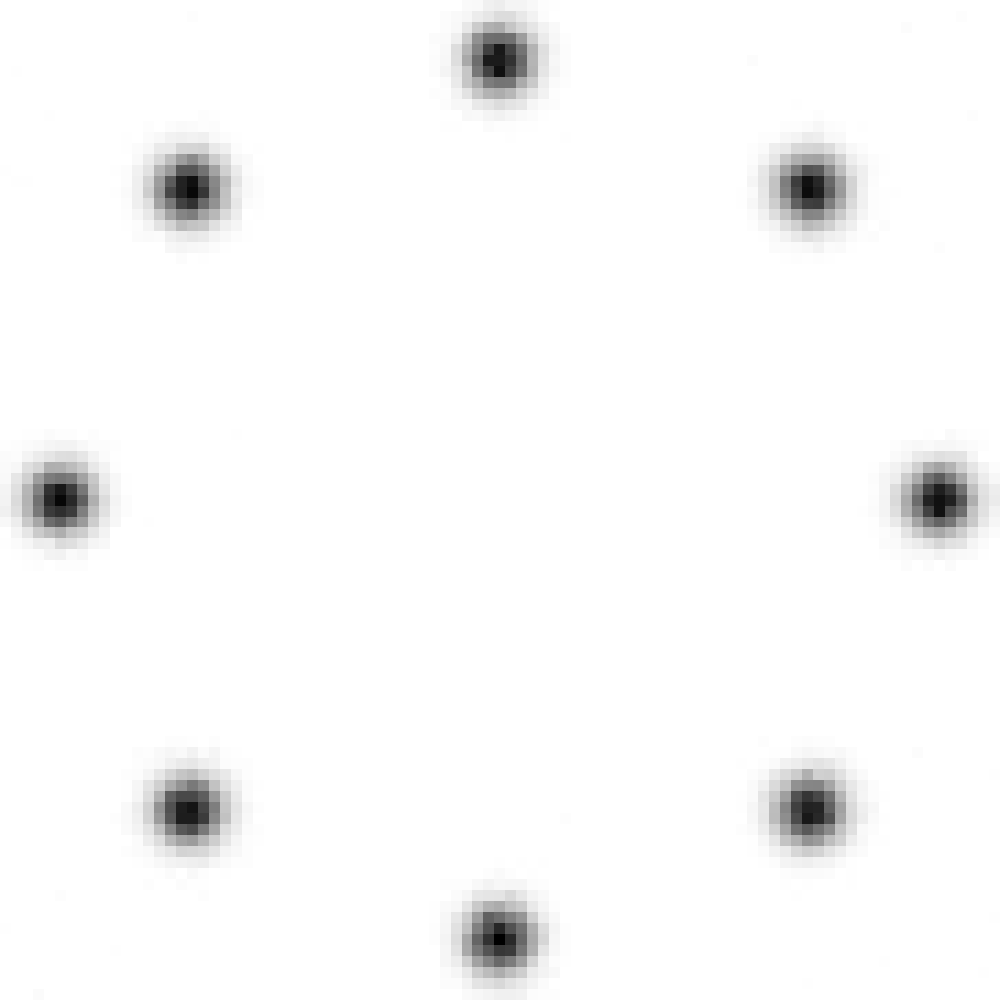}}\endgroup{}\hspace{0.03\textwidth}%
		\includegraphics[width=0.175\textwidth]{Figures/gaussian_mixture.png}
		\subcaption{Gaussian Mixture}
		\label{fig:simple-datasets-gaussian-mixture-lambda}
	\end{subfigure}

	\vspace{0.5em}
	\begin{subfigure}[t]{\textwidth}
		\centering
        \begingroup\setlength{\fboxsep}{0pt}\fbox{\includegraphics[width=\dimexpr0.175\textwidth-2\fboxrule\relax]{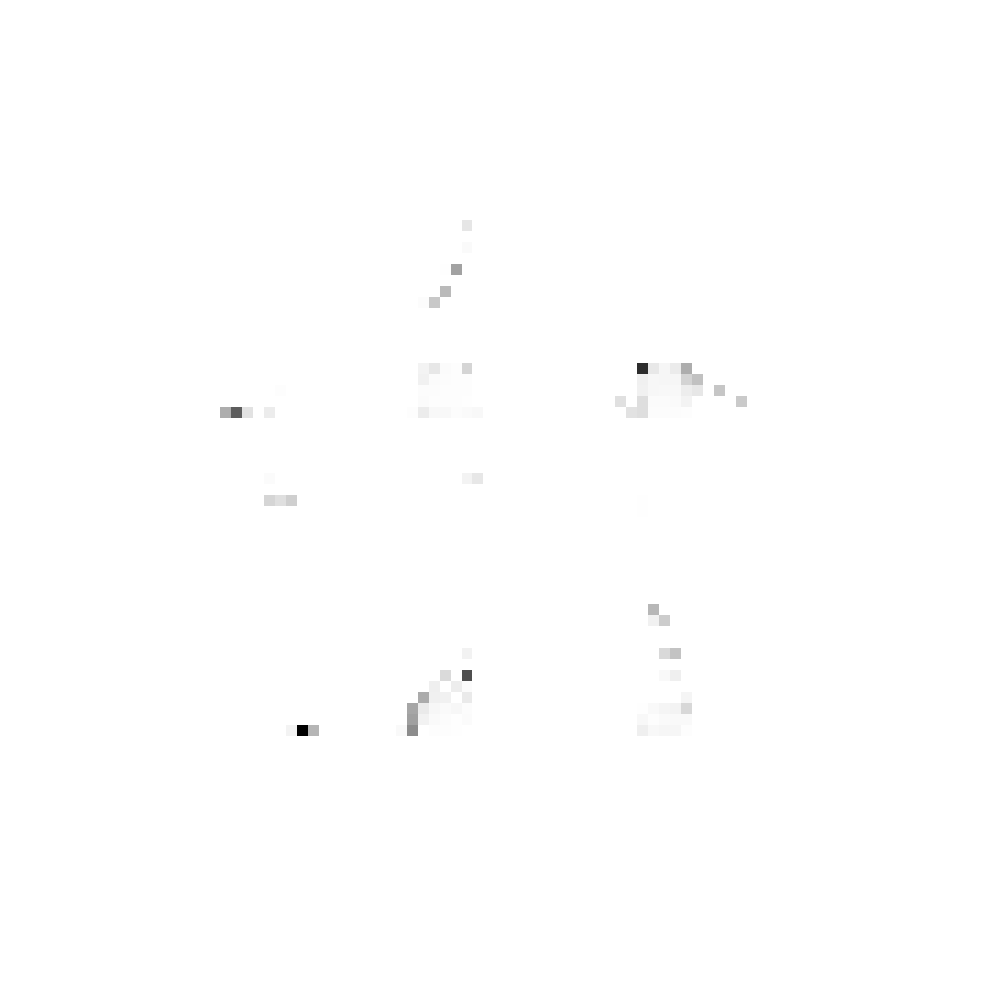}}\endgroup{}\hspace{0.03\textwidth}%
        \begingroup\setlength{\fboxsep}{0pt}\fbox{\includegraphics[width=\dimexpr0.175\textwidth-2\fboxrule\relax]{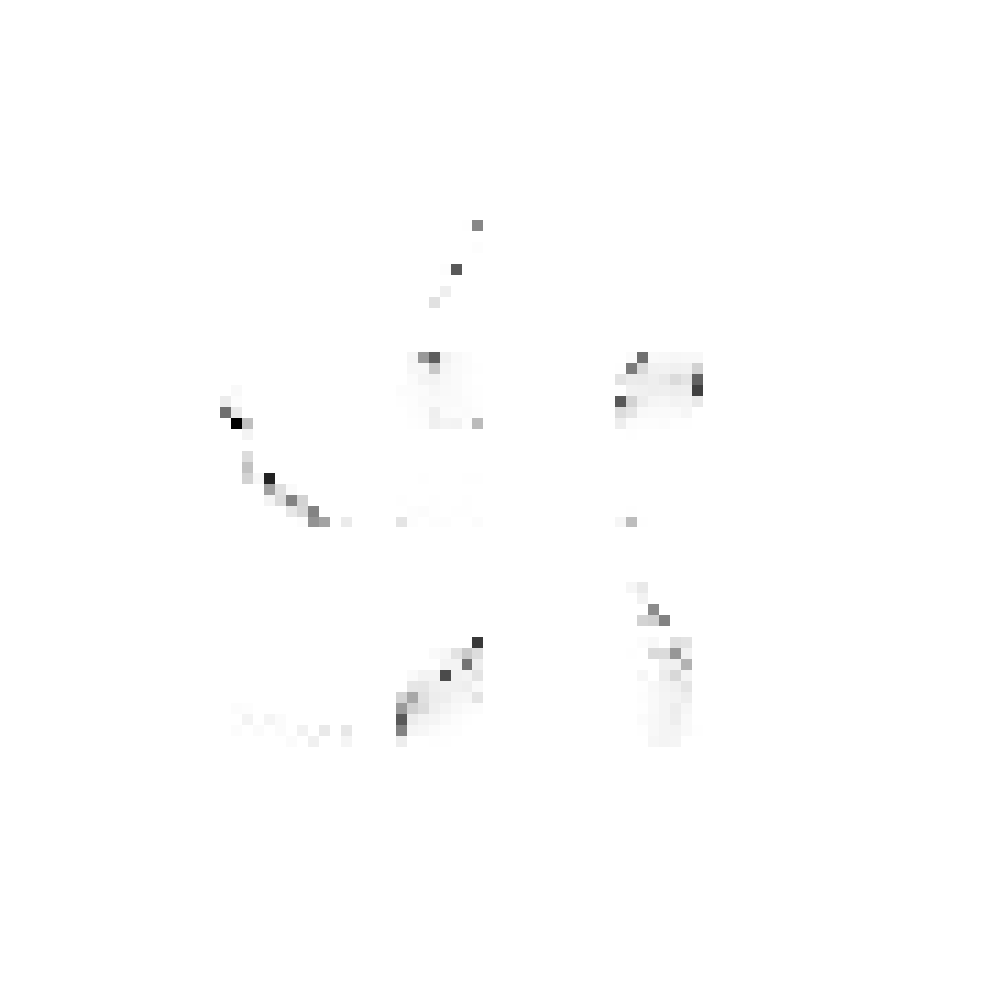}}\endgroup{}\hspace{0.03\textwidth}%
        \begingroup\setlength{\fboxsep}{0pt}\fbox{\includegraphics[width=\dimexpr0.175\textwidth-2\fboxrule\relax]{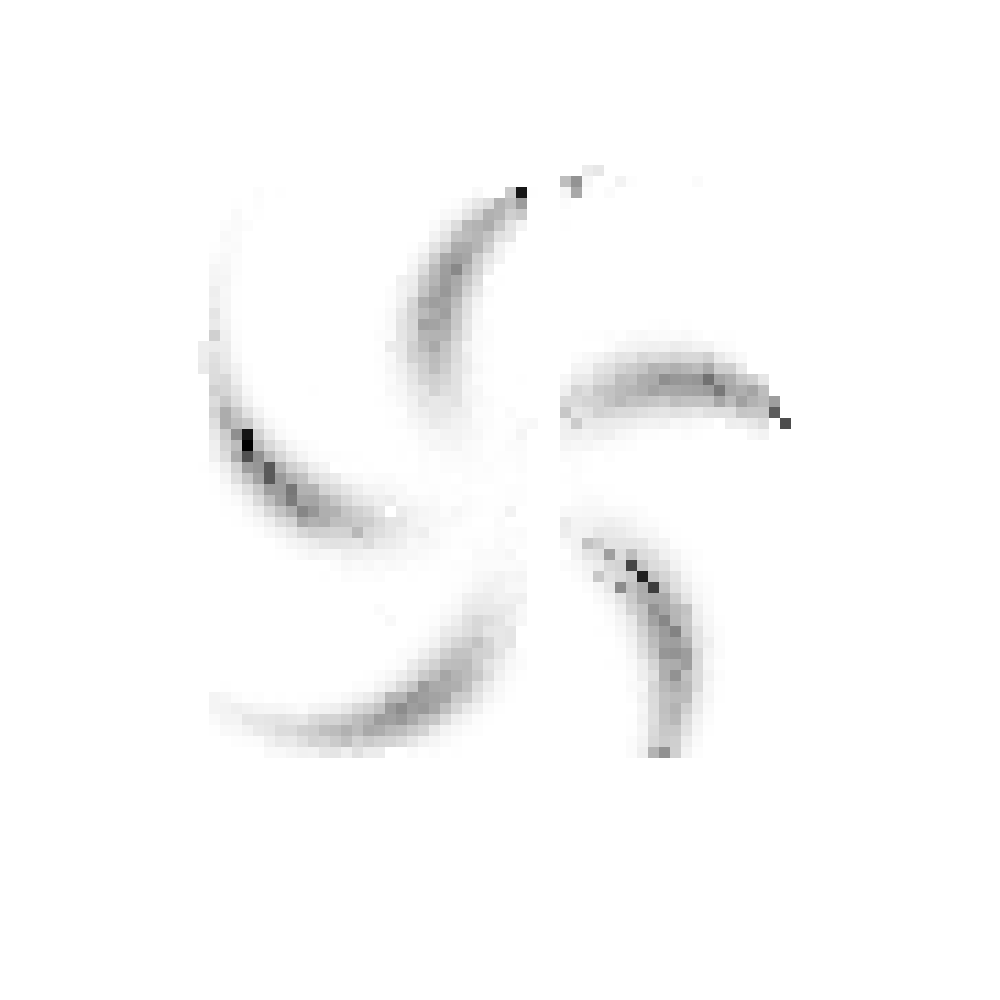}}\endgroup{}\hspace{0.03\textwidth}%
        \begingroup\setlength{\fboxsep}{0pt}\fbox{\includegraphics[width=\dimexpr0.175\textwidth-2\fboxrule\relax]{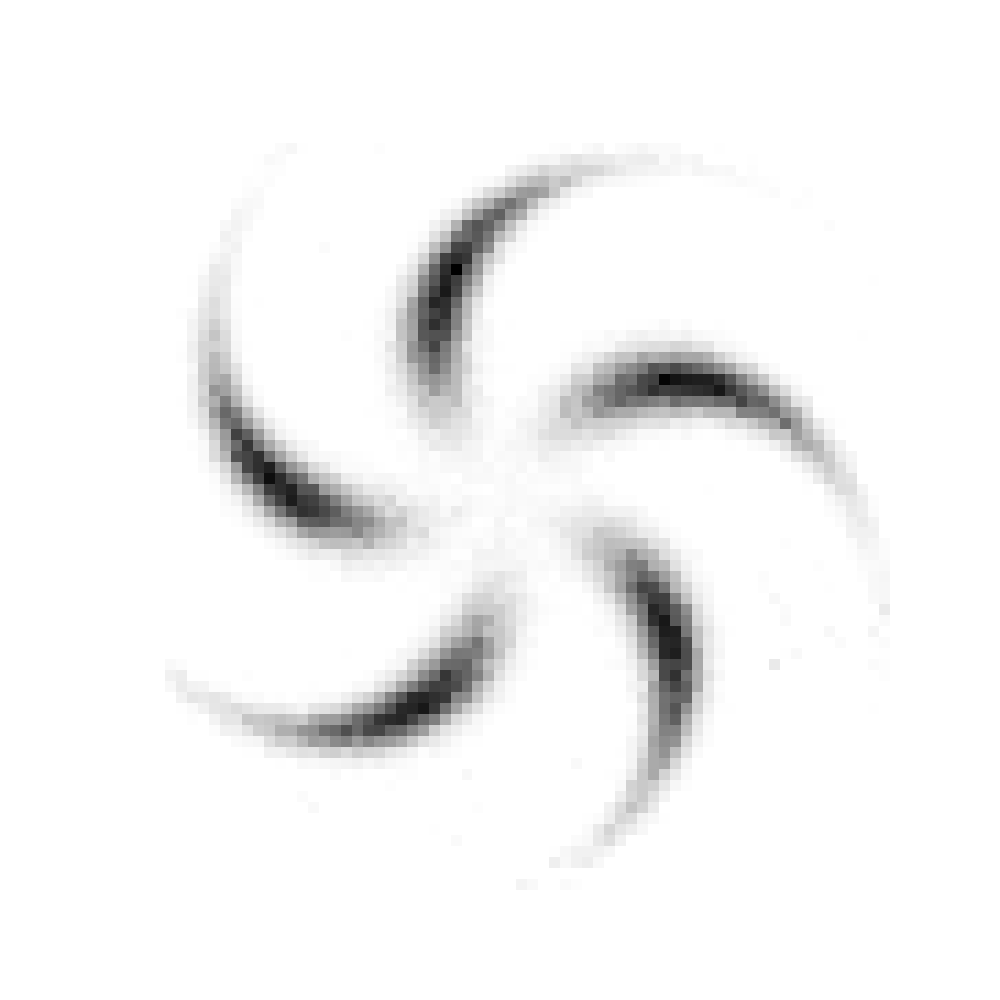}}\endgroup{}\hspace{0.03\textwidth}%
		\includegraphics[width=0.175\textwidth]{Figures/pinwheel.png}
		\subcaption{Pinwheel}
		\label{fig:simple-datasets-pinwheel-lambda}
	\end{subfigure}

	\vspace{0.5em}
	\begin{subfigure}[t]{\textwidth}
		\centering
        \begingroup\setlength{\fboxsep}{0pt}\fbox{\includegraphics[width=\dimexpr0.175\textwidth-2\fboxrule\relax]{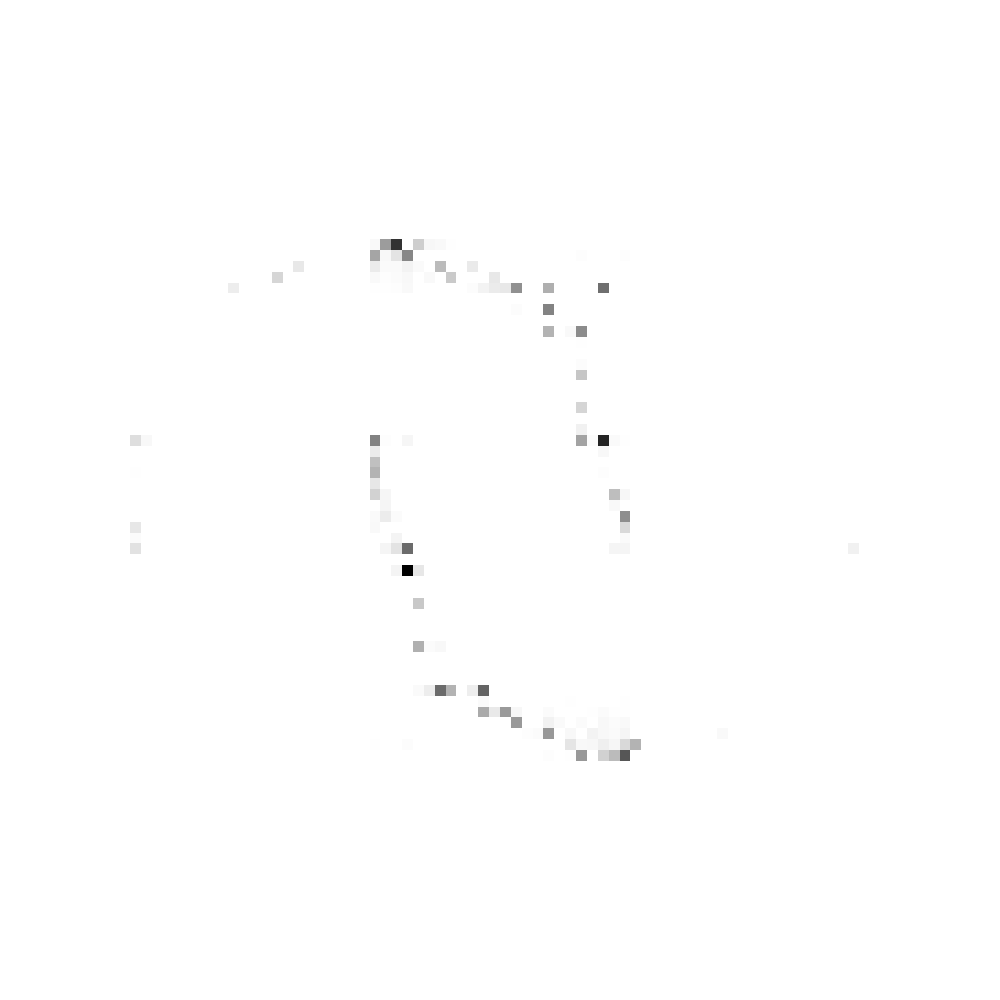}}\endgroup{}\hspace{0.03\textwidth}%
        \begingroup\setlength{\fboxsep}{0pt}\fbox{\includegraphics[width=\dimexpr0.175\textwidth-2\fboxrule\relax]{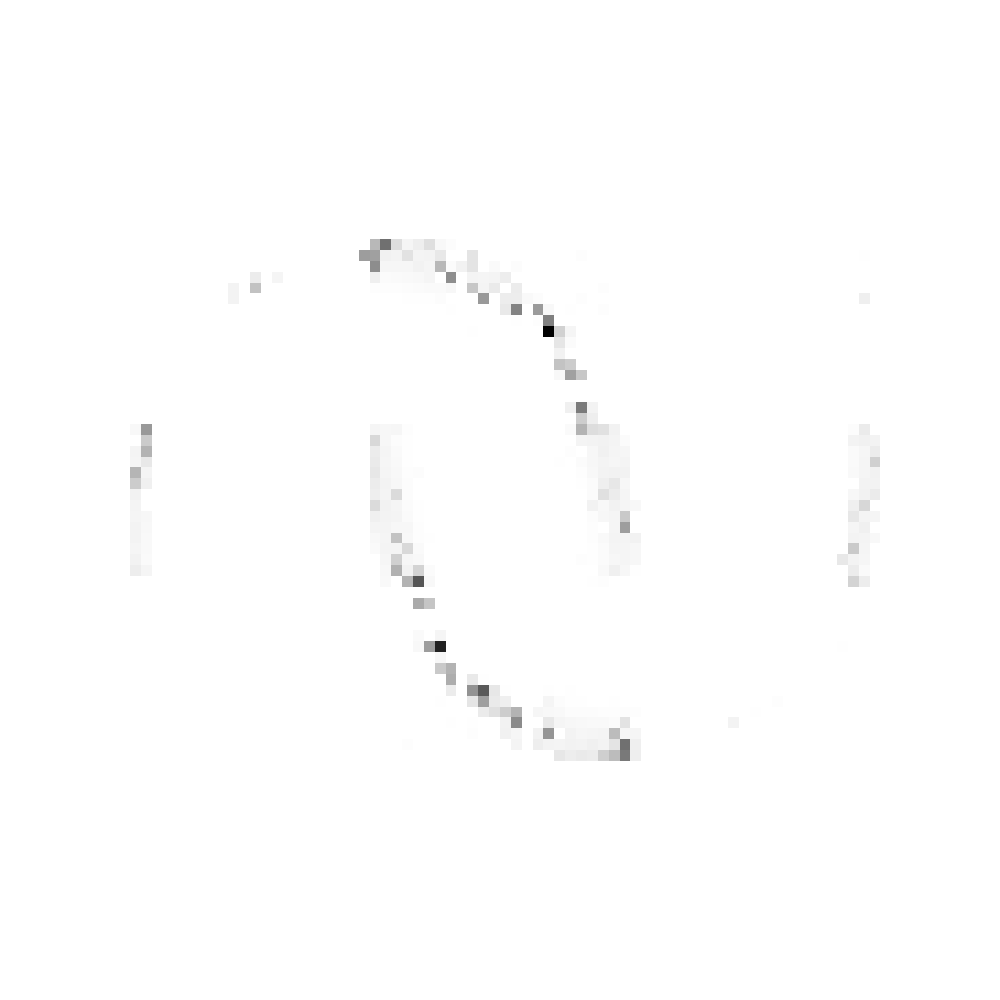}}\endgroup{}\hspace{0.03\textwidth}%
        \begingroup\setlength{\fboxsep}{0pt}\fbox{\includegraphics[width=\dimexpr0.175\textwidth-2\fboxrule\relax]{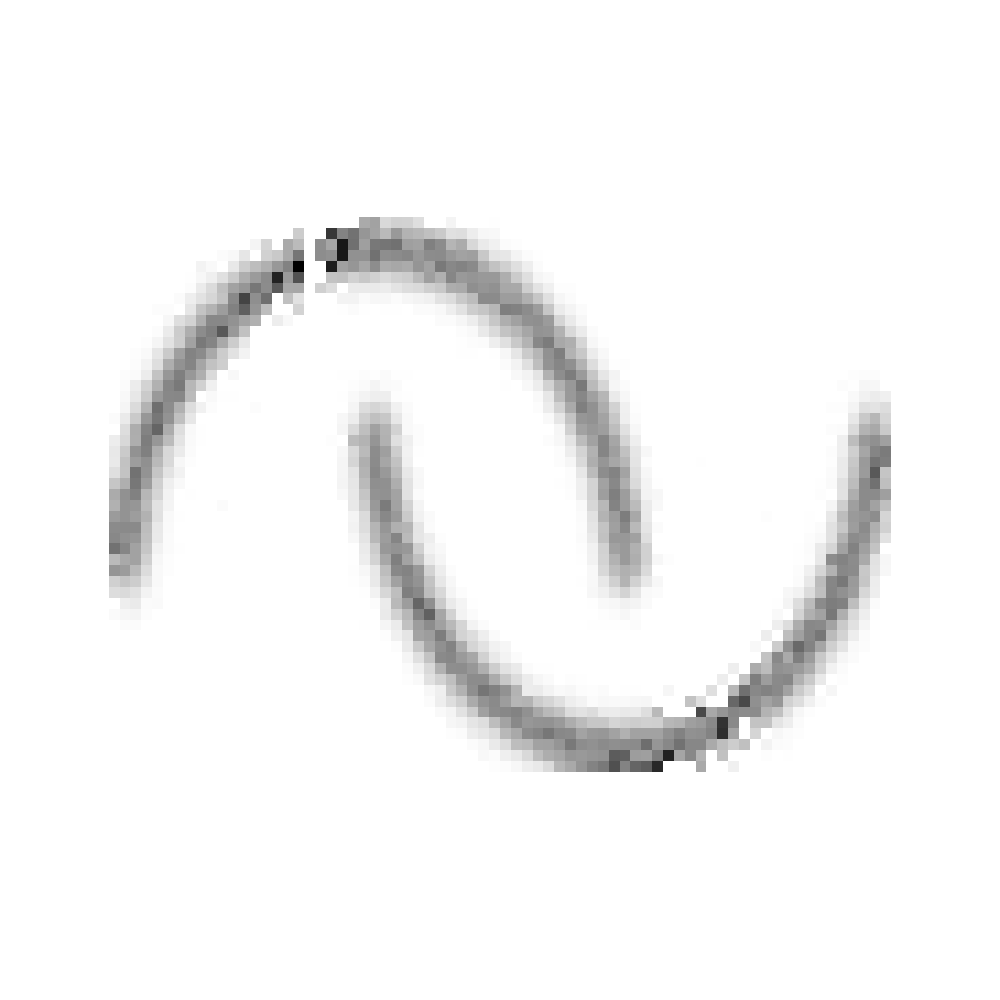}}\endgroup{}\hspace{0.03\textwidth}%
        \begingroup\setlength{\fboxsep}{0pt}\fbox{\includegraphics[width=\dimexpr0.175\textwidth-2\fboxrule\relax]{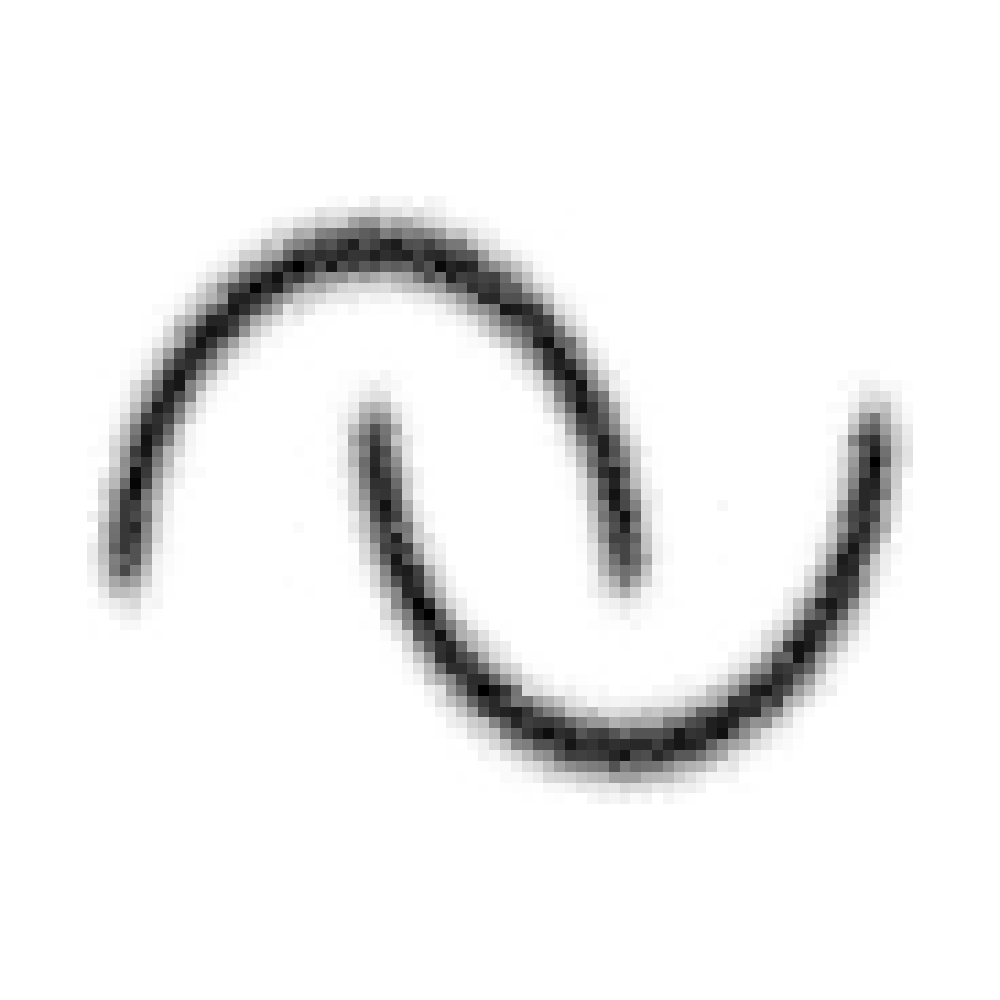}}\endgroup{}\hspace{0.03\textwidth}%
		\includegraphics[width=0.175\textwidth]{Figures/two_moons.png}
		\subcaption{Two Moons}
		\label{fig:simple-datasets-two-moons-lambda}
	\end{subfigure}
	}
	\caption{Visual comparison of GPCA reconstructions of synthetic datasets, with $c=92$ categories, and regularization parameters $\lambda$, cf. Eq. \eqref{eq:loss-gpca-e_reg}, with \textbf{fixed latent dimension} $\boldsymbol{d=16}$. Each row corresponds to one dataset, and each column shows the empirical joint distribution estimated from $N=3 \cdot 10^4$ samples using $\lambda \in \{10^{0},10^{-1},10^{-2},10^{-3}\}$. The rightmost image, highlighted in blue, shows the empirical joint distribution estimated from the ground-truth distribution.}
	\label{fig:simple-datasets-grid-lambda}
\end{figure}

\begin{figure}[H]
	\centering

	{\small
	\makebox[0.175\textwidth][c]{\boldmath$d=2$}\hspace{0.03\textwidth}%
	\makebox[0.175\textwidth][c]{\boldmath$d=4$}\hspace{0.03\textwidth}%
	\makebox[0.175\textwidth][c]{\boldmath$d=8$}\hspace{0.03\textwidth}%
	\makebox[0.175\textwidth][c]{\boldmath$d=16$}\hspace{0.03\textwidth}%
	\makebox[0.175\textwidth][c]{\textbf{Ground truth}}
	\par\medskip
	\captionsetup[subfigure]{justification=centering}
	\begin{subfigure}[t]{\textwidth}
		\centering
        \begingroup\setlength{\fboxsep}{0pt}\fbox{\includegraphics[width=\dimexpr0.175\textwidth-2\fboxrule\relax]{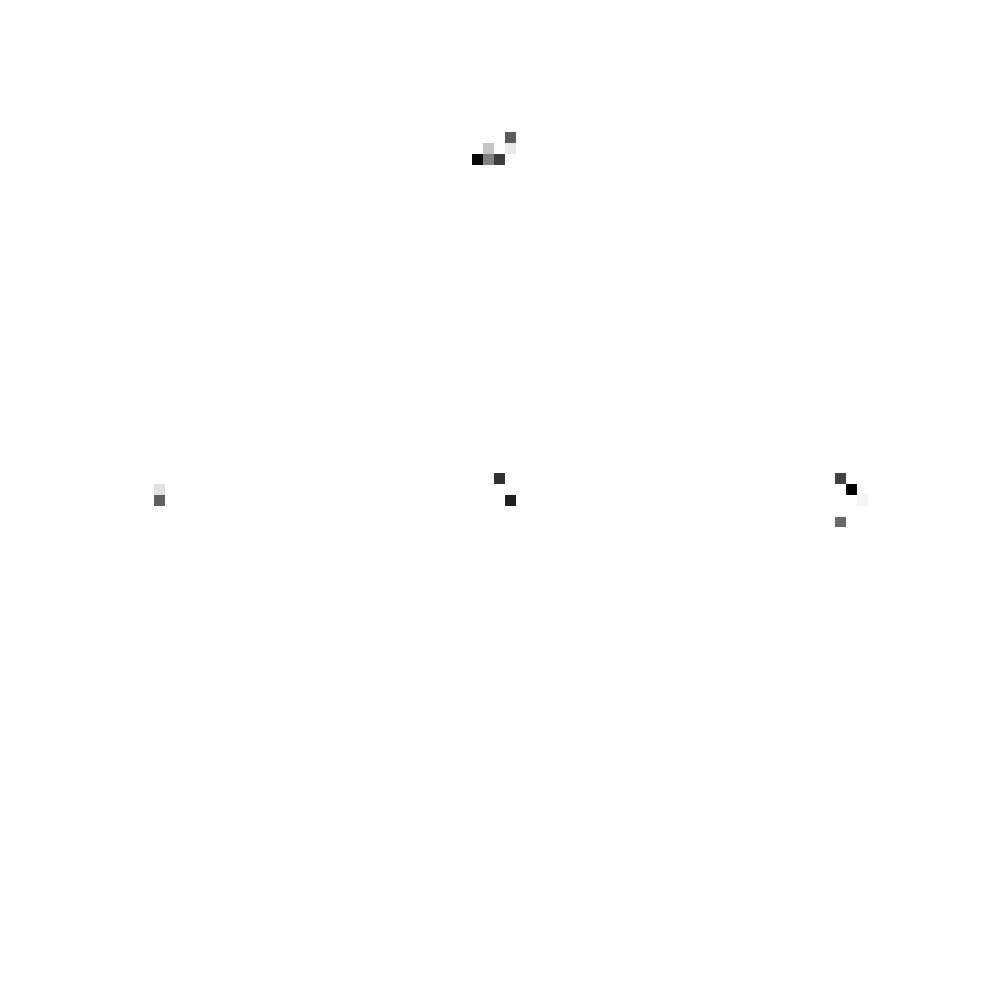}}\endgroup{}\hspace{0.03\textwidth}%
        \begingroup\setlength{\fboxsep}{0pt}\fbox{\includegraphics[width=\dimexpr0.175\textwidth-2\fboxrule\relax]{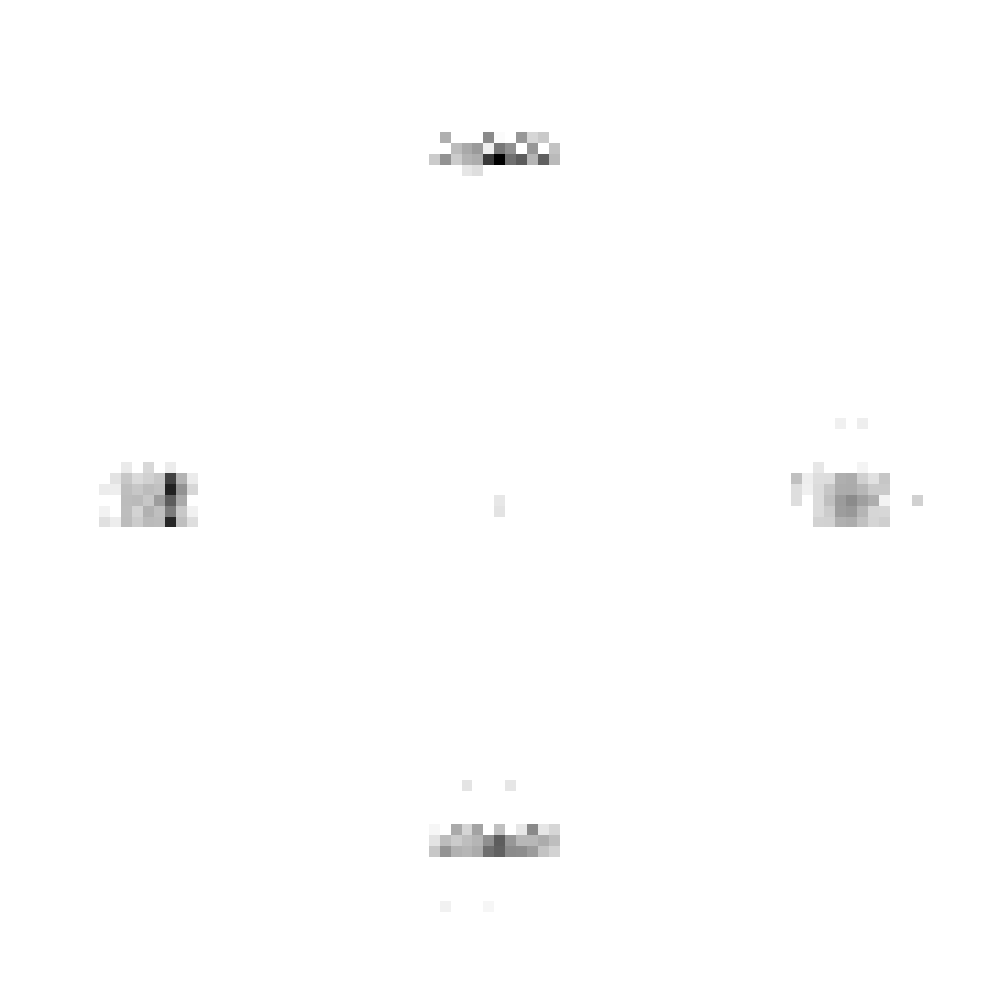}}\endgroup{}\hspace{0.03\textwidth}%
        \begingroup\setlength{\fboxsep}{0pt}\fbox{\includegraphics[width=\dimexpr0.175\textwidth-2\fboxrule\relax]{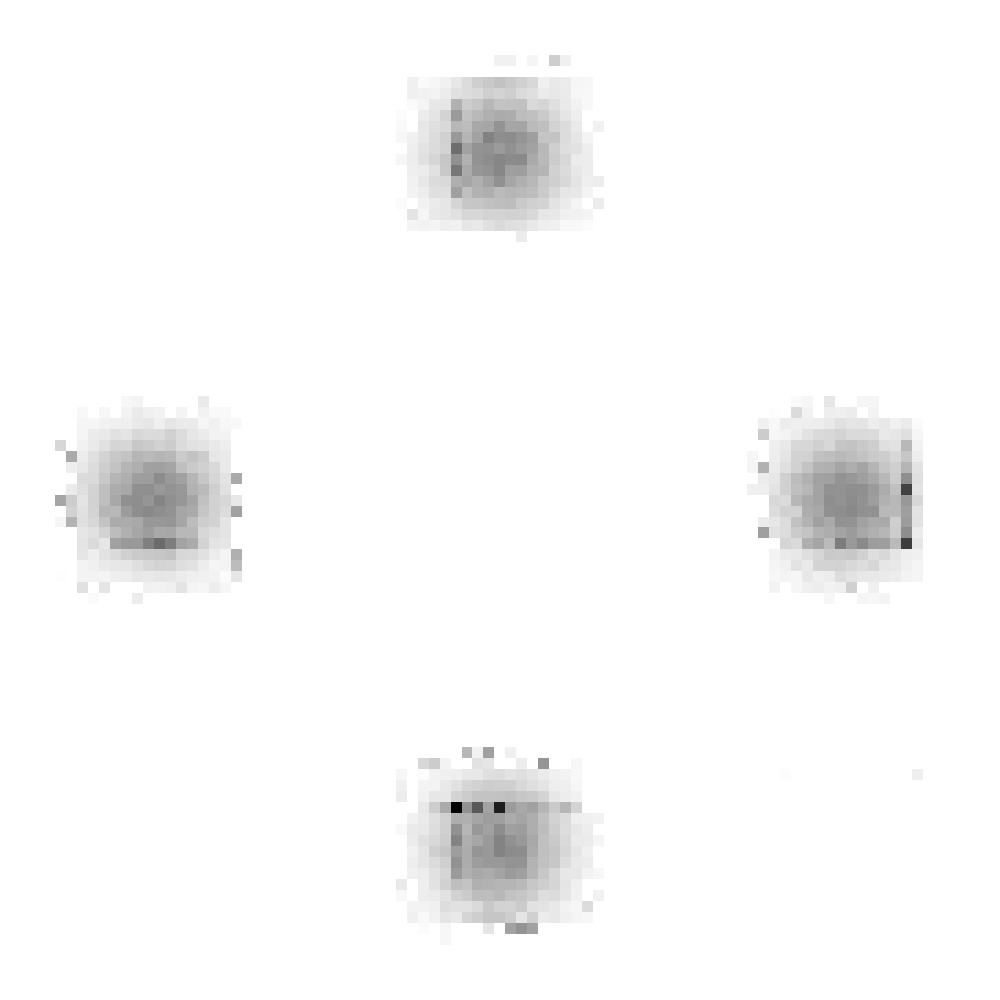}}\endgroup{}\hspace{0.03\textwidth}%
        \begingroup\setlength{\fboxsep}{0pt}\fbox{\includegraphics[width=\dimexpr0.175\textwidth-2\fboxrule\relax]{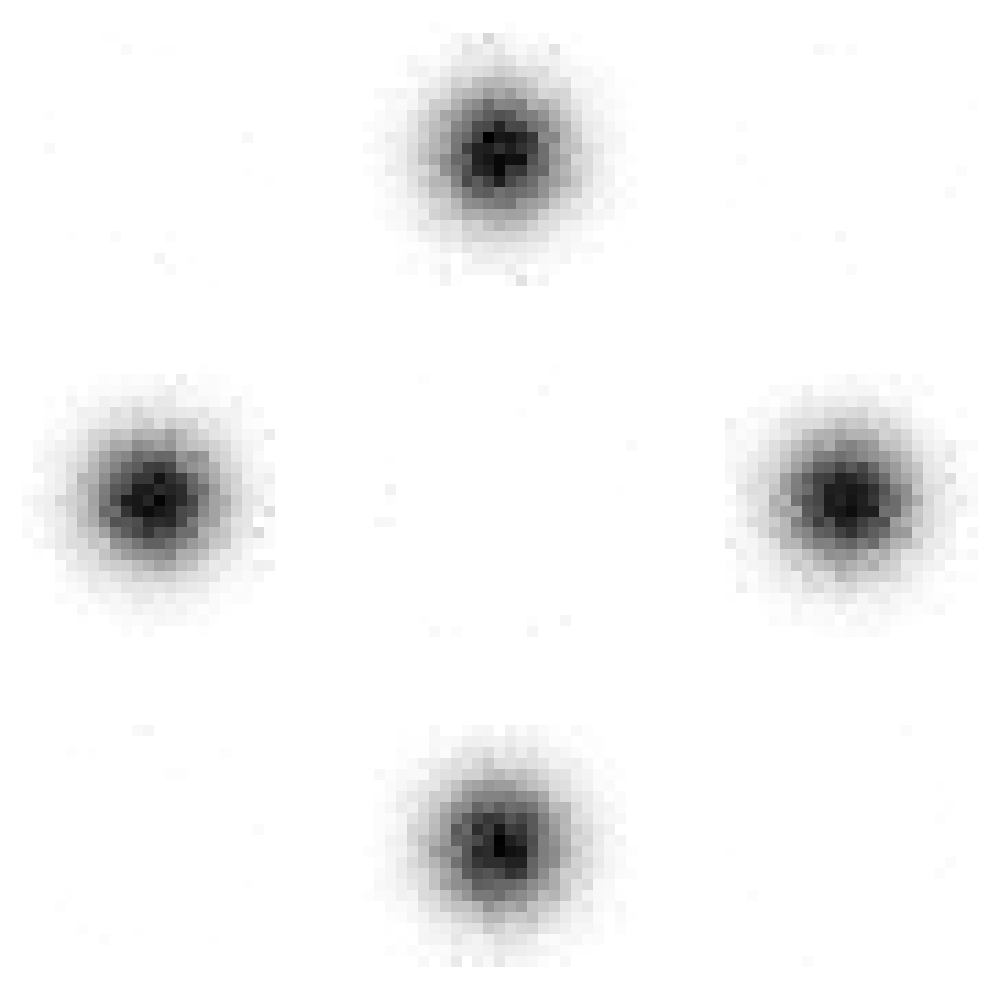}}\endgroup{}\hspace{0.03\textwidth}%
		\includegraphics[width=0.175\textwidth]{Figures/four_gaussians.png}
		\subcaption{Four Gaussians}
		\label{fig:simple-datasets-four-gaussians}
	\end{subfigure}

	\vspace{0.5em}
	\begin{subfigure}[t]{\textwidth}
		\centering
        \begingroup\setlength{\fboxsep}{0pt}\fbox{\includegraphics[width=\dimexpr0.175\textwidth-2\fboxrule\relax]{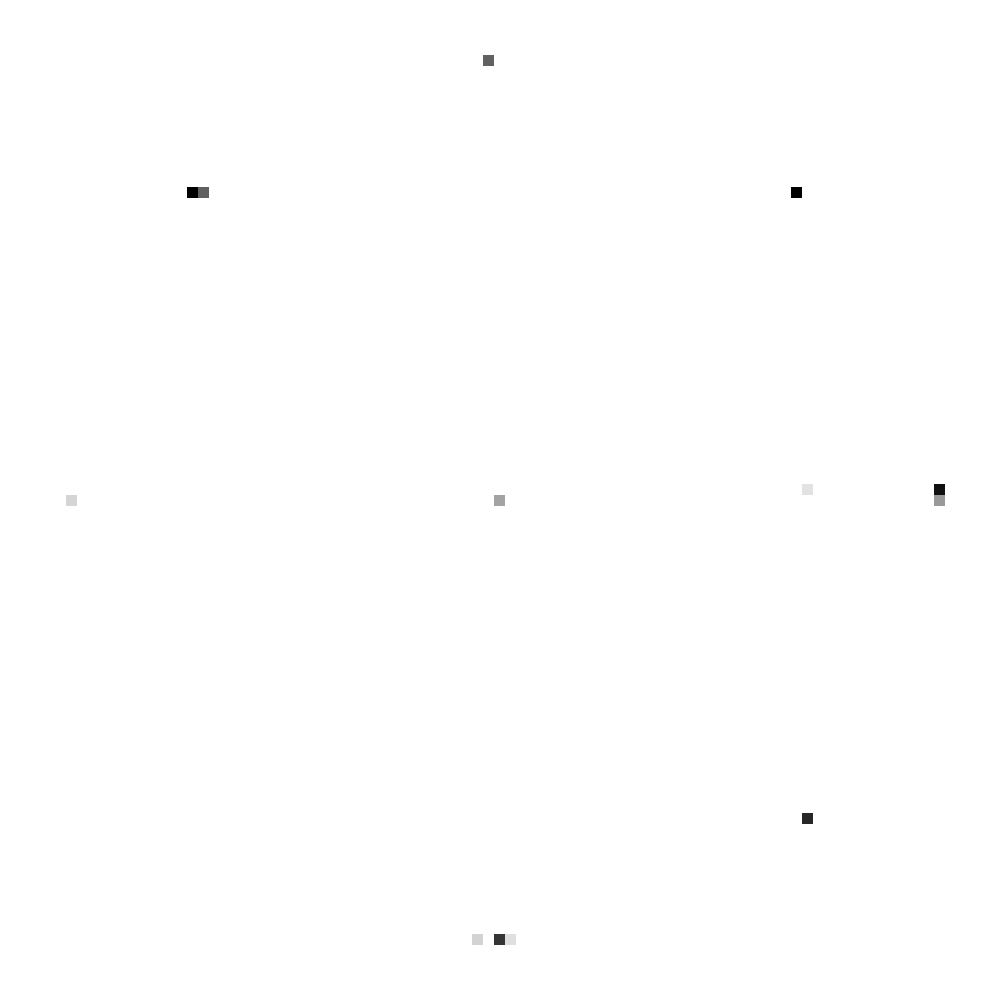}}\endgroup{}\hspace{0.03\textwidth}%
        \begingroup\setlength{\fboxsep}{0pt}\fbox{\includegraphics[width=\dimexpr0.175\textwidth-2\fboxrule\relax]{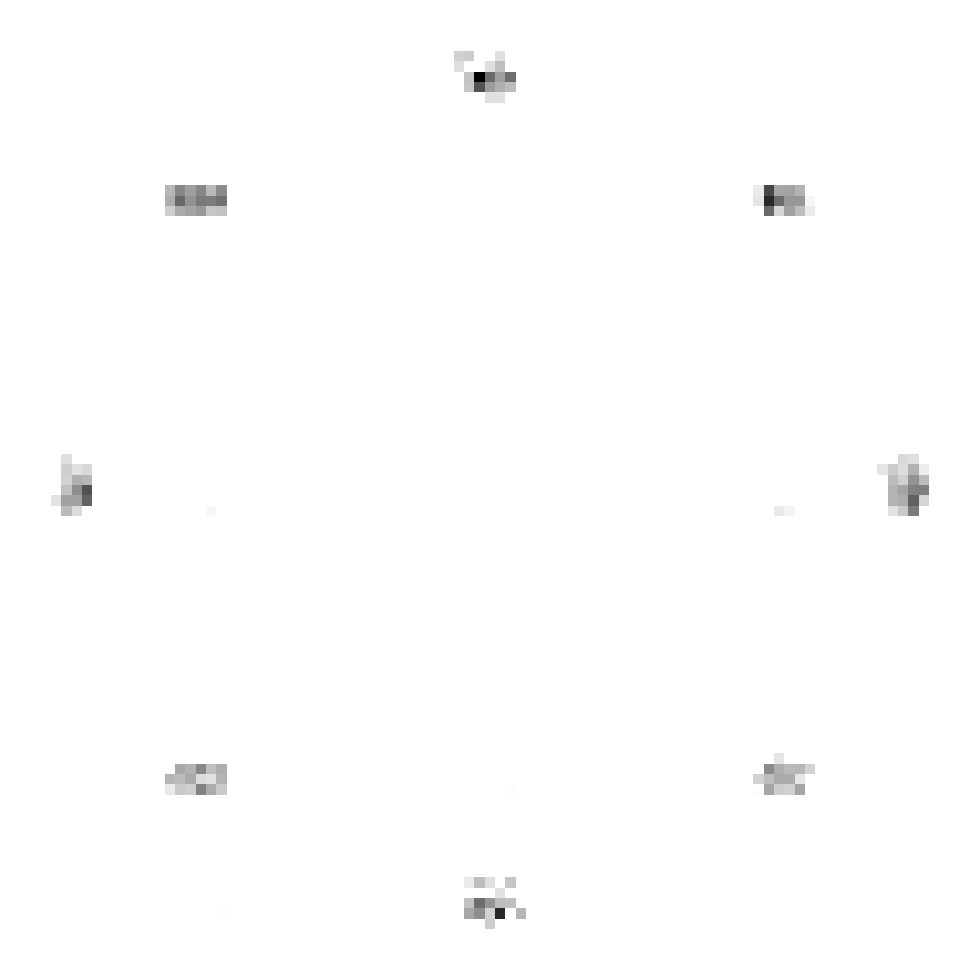}}\endgroup{}\hspace{0.03\textwidth}%
        \begingroup\setlength{\fboxsep}{0pt}\fbox{\includegraphics[width=\dimexpr0.175\textwidth-2\fboxrule\relax]{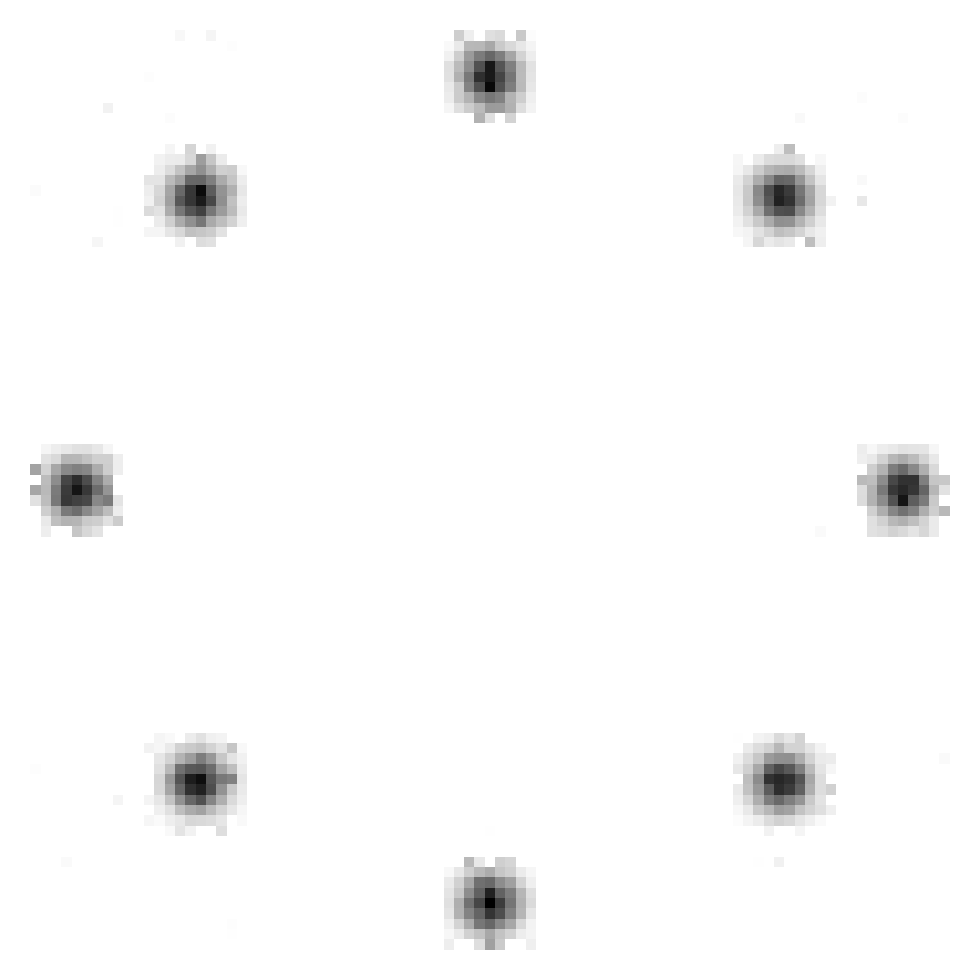}}\endgroup{}\hspace{0.03\textwidth}%
        \begingroup\setlength{\fboxsep}{0pt}\fbox{\includegraphics[width=\dimexpr0.175\textwidth-2\fboxrule\relax]{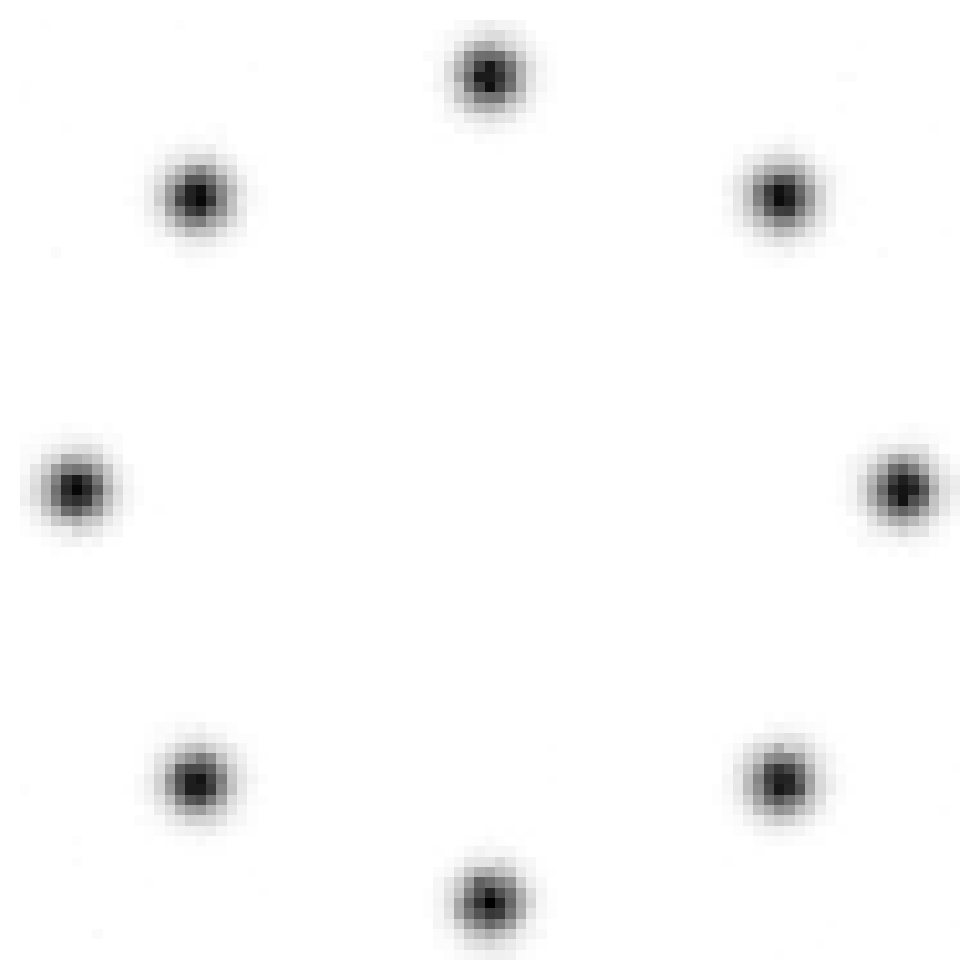}}\endgroup{}\hspace{0.03\textwidth}%
		\includegraphics[width=0.175\textwidth]{Figures/gaussian_mixture.png}
		\subcaption{Gaussian Mixture}
		\label{fig:simple-datasets-gaussian-mixture}
	\end{subfigure}

	\vspace{0.5em}
	\begin{subfigure}[t]{\textwidth}
		\centering
        \begingroup\setlength{\fboxsep}{0pt}\fbox{\includegraphics[width=\dimexpr0.175\textwidth-2\fboxrule\relax]{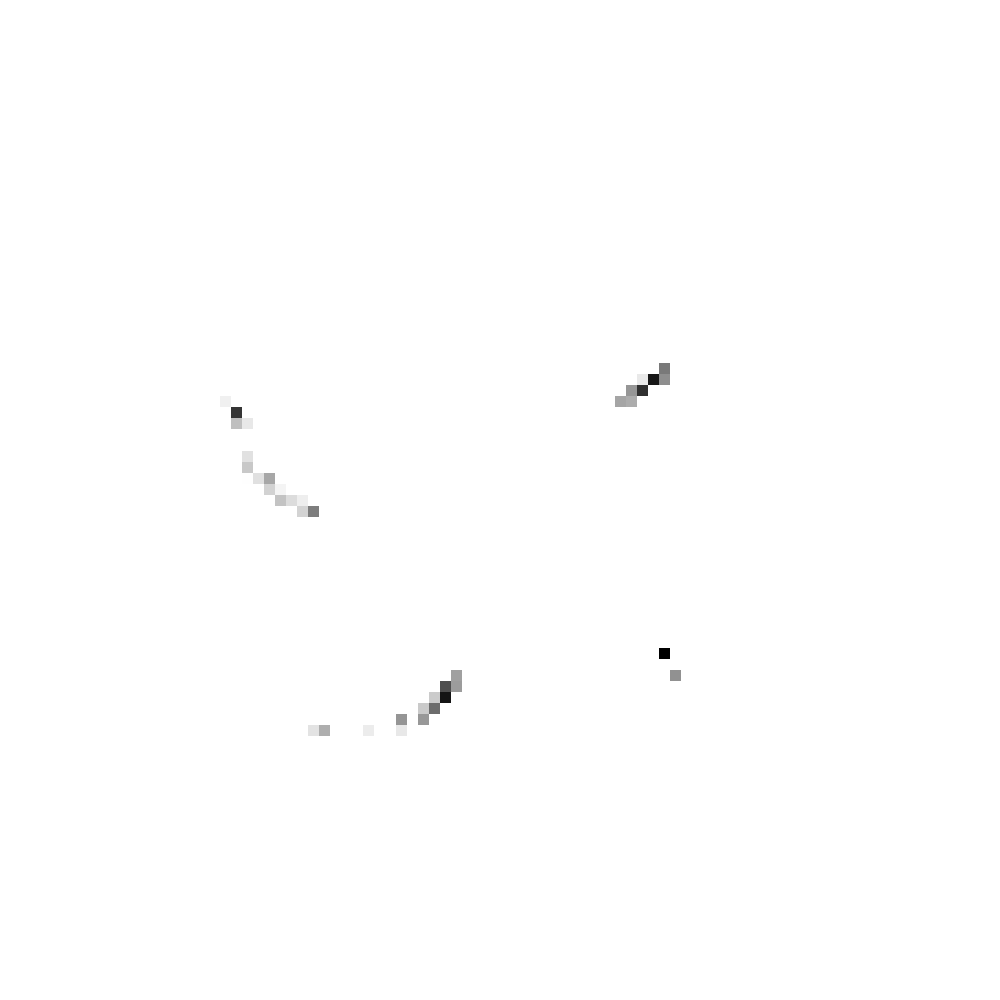}}\endgroup{}\hspace{0.03\textwidth}%
        \begingroup\setlength{\fboxsep}{0pt}\fbox{\includegraphics[width=\dimexpr0.175\textwidth-2\fboxrule\relax]{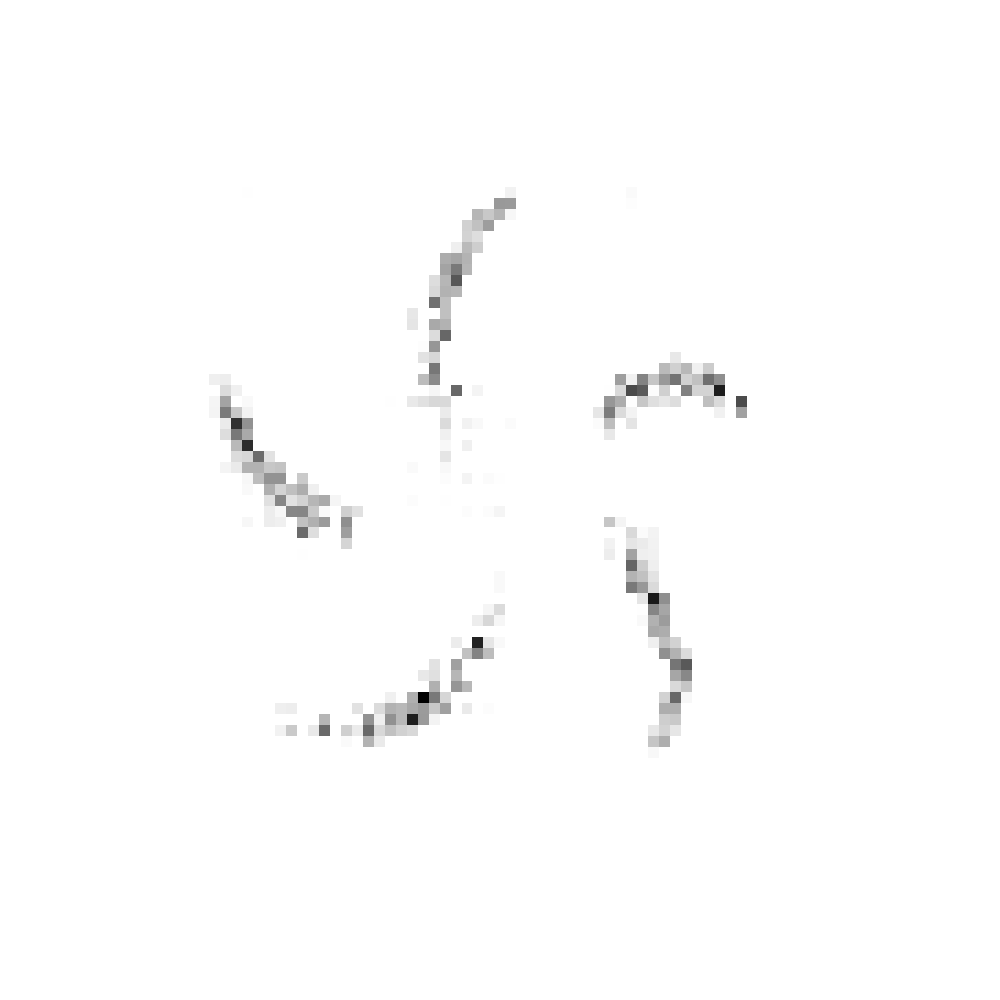}}\endgroup{}\hspace{0.03\textwidth}%
        \begingroup\setlength{\fboxsep}{0pt}\fbox{\includegraphics[width=\dimexpr0.175\textwidth-2\fboxrule\relax]{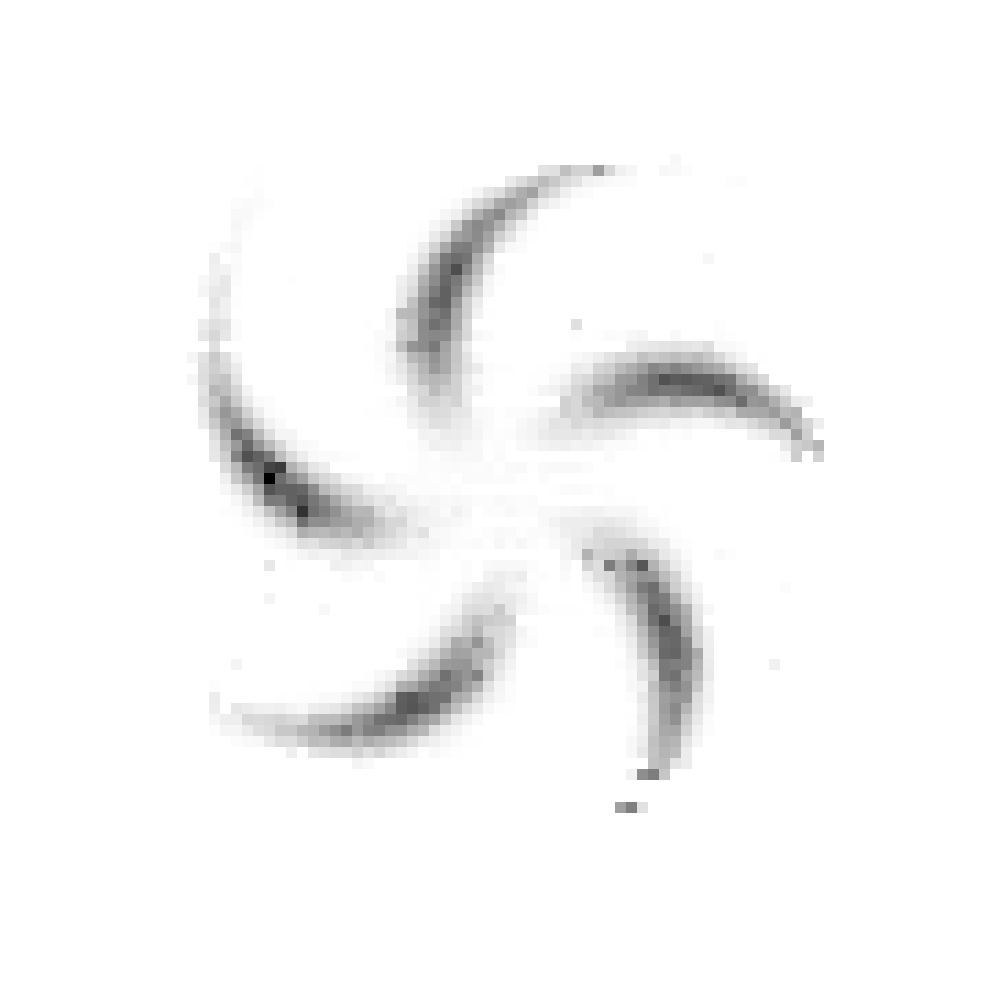}}\endgroup{}\hspace{0.03\textwidth}%
        \begingroup\setlength{\fboxsep}{0pt}\fbox{\includegraphics[width=\dimexpr0.175\textwidth-2\fboxrule\relax]{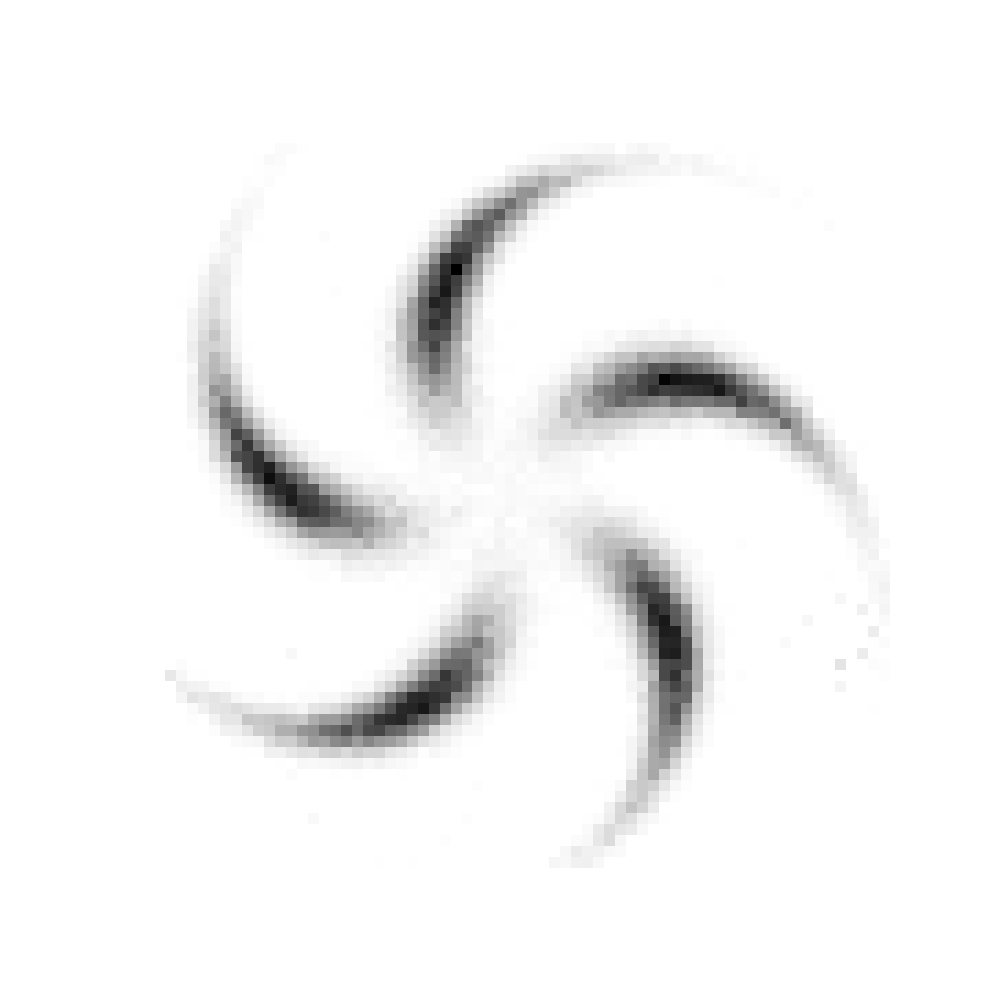}}\endgroup{}\hspace{0.03\textwidth}%
		\includegraphics[width=0.175\textwidth]{Figures/pinwheel.png}
		\subcaption{Pinwheel}
		\label{fig:simple-datasets-pinwheel}
	\end{subfigure}

	\vspace{0.5em}
	\begin{subfigure}[t]{\textwidth}
		\centering
        \begingroup\setlength{\fboxsep}{0pt}\fbox{\includegraphics[width=\dimexpr0.175\textwidth-2\fboxrule\relax]{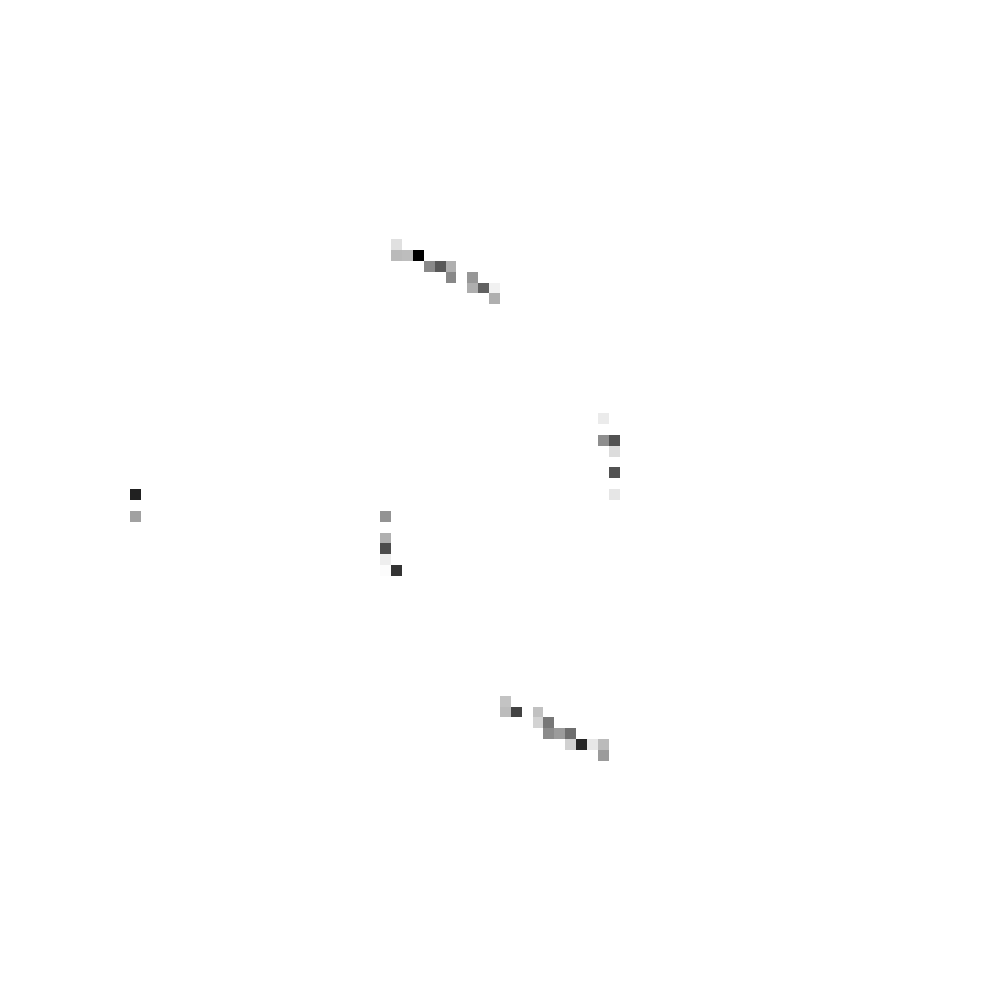}}\endgroup{}\hspace{0.03\textwidth}%
        \begingroup\setlength{\fboxsep}{0pt}\fbox{\includegraphics[width=\dimexpr0.175\textwidth-2\fboxrule\relax]{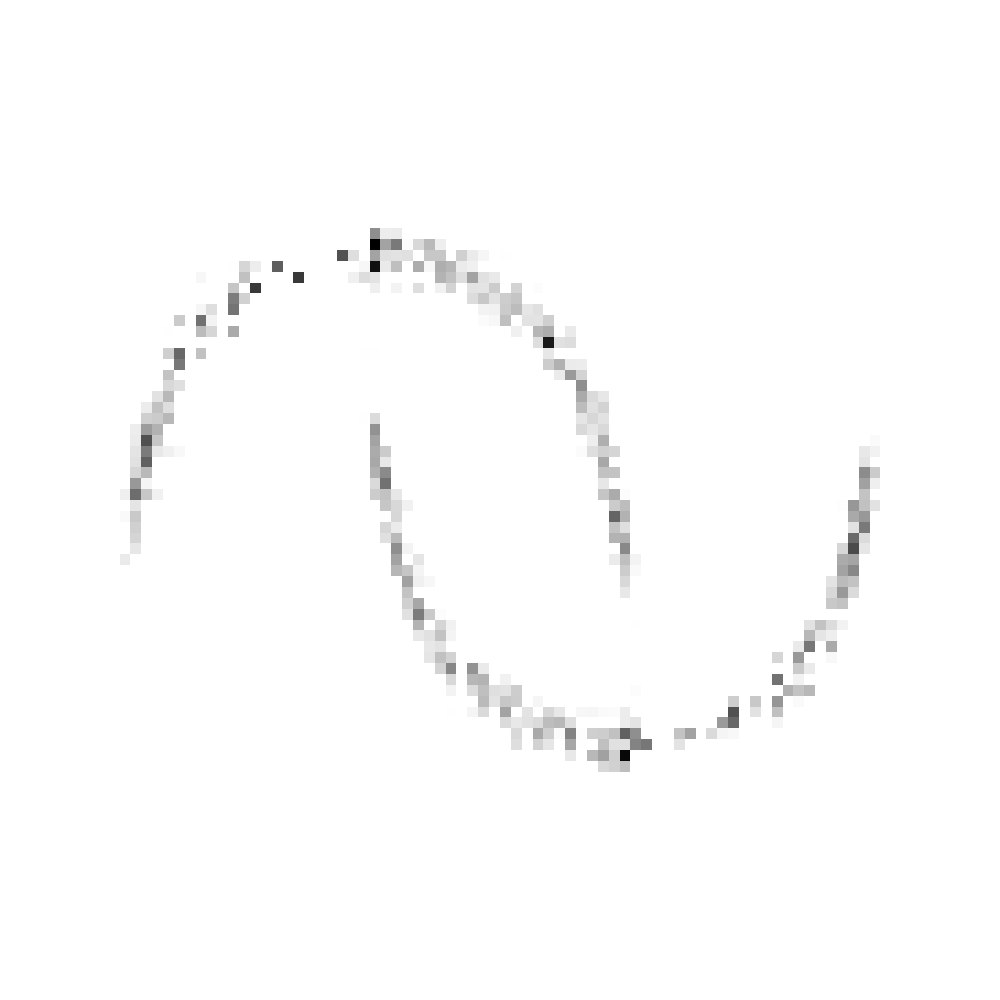}}\endgroup{}\hspace{0.03\textwidth}%
        \begingroup\setlength{\fboxsep}{0pt}\fbox{\includegraphics[width=\dimexpr0.175\textwidth-2\fboxrule\relax]{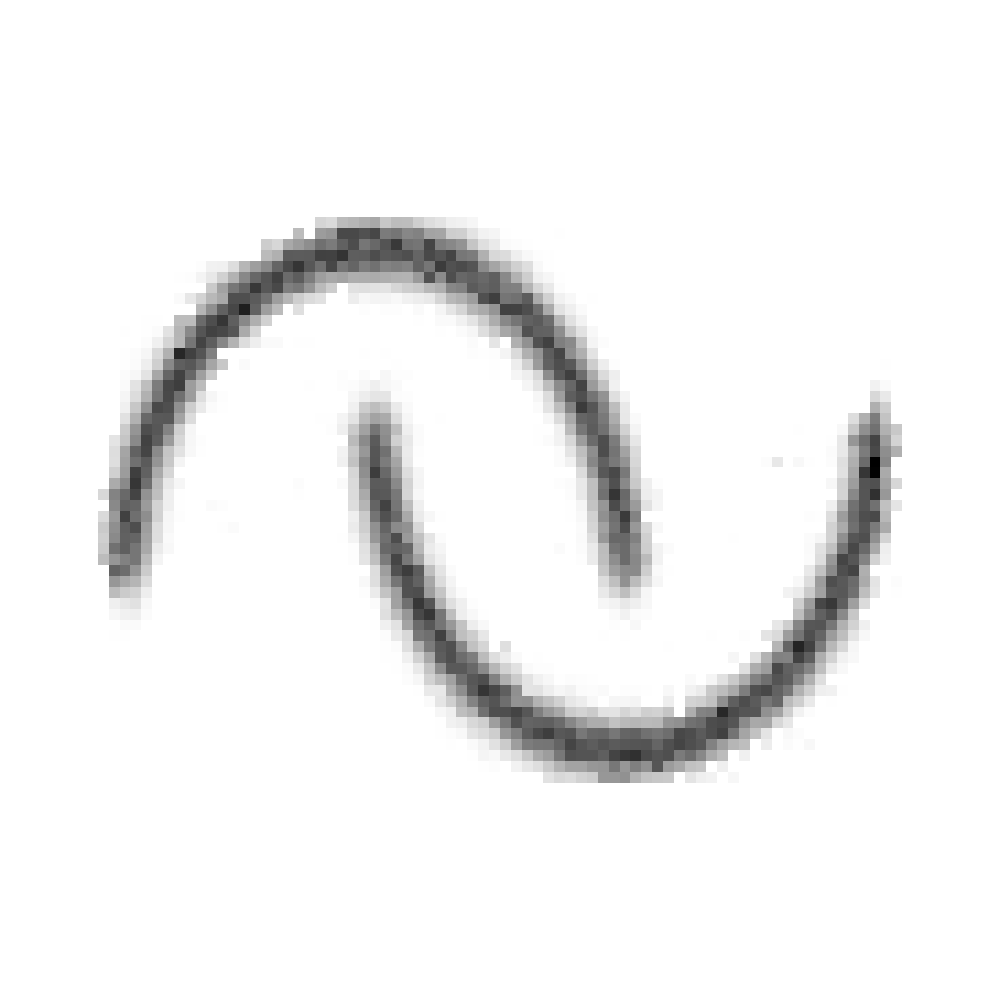}}\endgroup{}\hspace{0.03\textwidth}%
        \begingroup\setlength{\fboxsep}{0pt}\fbox{\includegraphics[width=\dimexpr0.175\textwidth-2\fboxrule\relax]{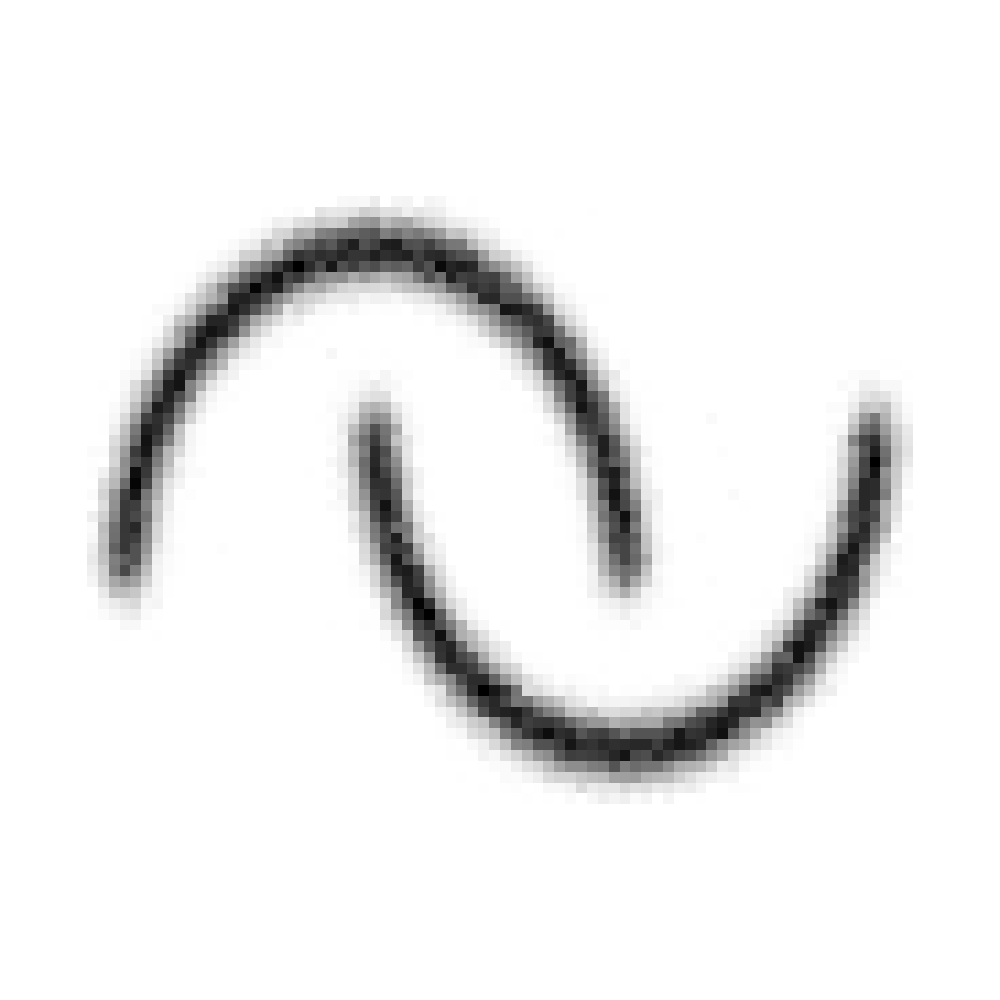}}\endgroup{}\hspace{0.03\textwidth}%
		\includegraphics[width=0.175\textwidth]{Figures/two_moons.png}
		\subcaption{Two Moons}
		\label{fig:simple-datasets-two-moons}
	\end{subfigure}
	}
	\caption{Visual comparison of GPCA reconstructions of synthetic datasets, with $c=92$ categories, and latent dimensions. We use overall a \textbf{fixed} $\boldsymbol{\lambda=10^{-3}}$ for the GPCA optimization, cf. Eq. \eqref{eq:loss-gpca-e_reg}. Each row corresponds to one dataset, and each column shows the empirical joint distribution estimated from $N=3 \cdot 10^4$ samples using latent dimension $d \in \{2,4,8,16\}$. The rightmost image, highlighted in blue, shows the empirical joint distribution estimated from the ground-truth distribution.}
	\label{fig:simple-datasets-grid}
\end{figure}

\clearpage
\paragraph{Cityscapes Dataset.}
We illustrate the effect of the regularization parameter $\lambda$ and the latent dimension $d$ on the GPCA reconstructions of the Cityscapes dataset. Figures~\ref{fig:cityscapes-dims-grid}--\ref{fig:cityscapes-dims-grid-1e-3} correspond to selected values from Tables~\ref{tab:cityscapes-dims-lambda-1e-1}--\ref{tab:cityscapes-dims-lambda-3}. In particular, for $\lambda=10^{-3}$ and $d=256$, we observe almost perfect reconstructions. This is consistent with the very low Hamming distance of $412\cdot 10^{-6}$ reported in Table~\ref{tab:cityscapes-dims-lambda-3}, which corresponds, on average, to about 27 mislabeled pixels per image.

\newcommand{\cityscapesdimslabels}{%
    \par\vspace{0.25em}
    {\fontsize{13}{14}\selectfont
    \makebox[0.188\textwidth][c]{\boldmath$d=32$}\hspace{0.01\textwidth}%
    \makebox[0.188\textwidth][c]{\boldmath$d=64$}\hspace{0.01\textwidth}%
    \makebox[0.188\textwidth][c]{\boldmath$d=128$}\hspace{0.01\textwidth}%
    \makebox[0.188\textwidth][c]{\boldmath$d=256$}\hspace{0.01\textwidth}%
    \makebox[0.188\textwidth][c]{\shortstack{\textbf{Ground}\\\textbf{truth}}}}}

\begin{figure}[H]
    \centering
    \resizebox{0.70\textwidth}{!}{%
    \begin{minipage}{\textwidth}
    \includegraphics[width=0.188\textwidth]{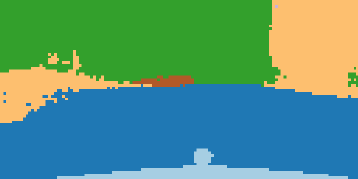}\hspace{0.01\textwidth}%
    \includegraphics[width=0.188\textwidth]{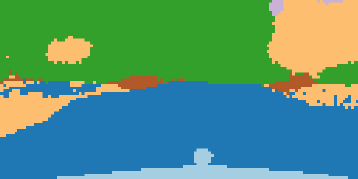}\hspace{0.01\textwidth}%
    \includegraphics[width=0.188\textwidth]{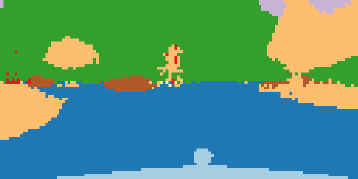}\hspace{0.01\textwidth}%
    \includegraphics[width=0.188\textwidth]{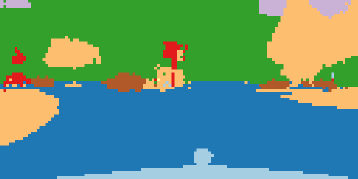}\hspace{0.01\textwidth}%
    \begingroup\setlength{\fboxsep}{0pt}\setlength{\fboxrule}{1.6pt}%
    \fcolorbox{red}{white}{\includegraphics[width=\dimexpr0.188\textwidth-2\fboxrule\relax]{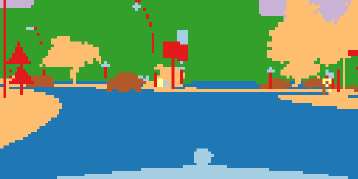}}%
    \endgroup{}
    \vspace{0.25em}
    \includegraphics[width=0.188\textwidth]{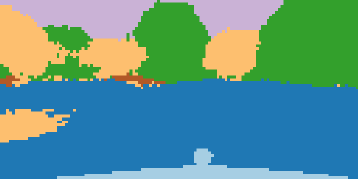}\hspace{0.01\textwidth}%
    \includegraphics[width=0.188\textwidth]{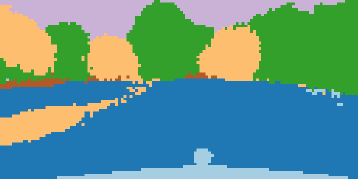}\hspace{0.01\textwidth}%
    \includegraphics[width=0.188\textwidth]{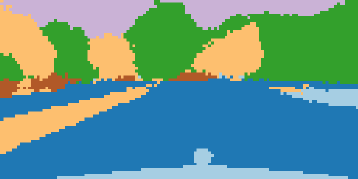}\hspace{0.01\textwidth}%
    \includegraphics[width=0.188\textwidth]{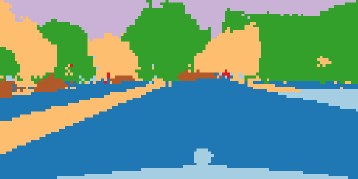}\hspace{0.01\textwidth}%
    \begingroup\setlength{\fboxsep}{0pt}\setlength{\fboxrule}{1.6pt}%
    \fcolorbox{red}{white}{\includegraphics[width=\dimexpr0.188\textwidth-2\fboxrule\relax]{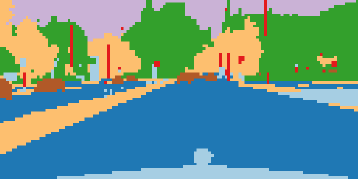}}%
    \endgroup{}
    \vspace{0.25em}
    \includegraphics[width=0.188\textwidth]{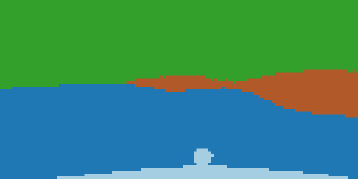}\hspace{0.01\textwidth}%
    \includegraphics[width=0.188\textwidth]{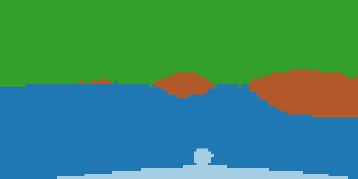}\hspace{0.01\textwidth}%
    \includegraphics[width=0.188\textwidth]{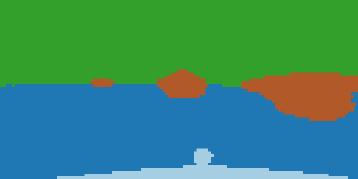}\hspace{0.01\textwidth}%
    \includegraphics[width=0.188\textwidth]{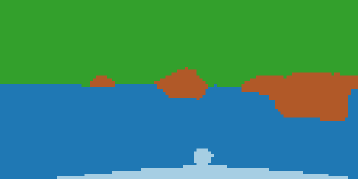}\hspace{0.01\textwidth}%
    \begingroup\setlength{\fboxsep}{0pt}\setlength{\fboxrule}{1.6pt}%
    \fcolorbox{red}{white}{\includegraphics[width=\dimexpr0.188\textwidth-2\fboxrule\relax]{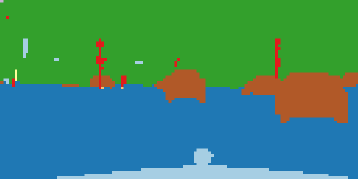}}%
    \endgroup{}
    \vspace{0.25em}
    \includegraphics[width=0.188\textwidth]{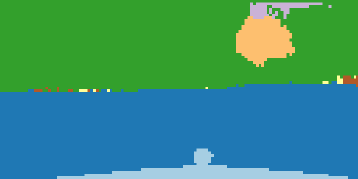}\hspace{0.01\textwidth}%
    \includegraphics[width=0.188\textwidth]{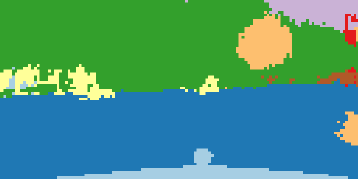}\hspace{0.01\textwidth}%
    \includegraphics[width=0.188\textwidth]{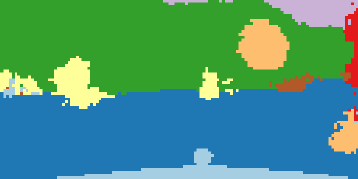}\hspace{0.01\textwidth}%
    \includegraphics[width=0.188\textwidth]{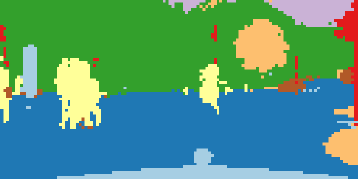}\hspace{0.01\textwidth}%
    \begingroup\setlength{\fboxsep}{0pt}\setlength{\fboxrule}{1.6pt}%
    \fcolorbox{red}{white}{\includegraphics[width=\dimexpr0.188\textwidth-2\fboxrule\relax]{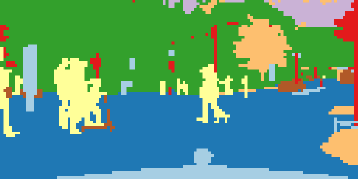}}%
    \endgroup{}
    \vspace{0.25em}
    \includegraphics[width=0.188\textwidth]{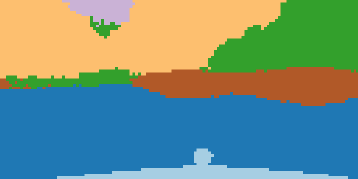}\hspace{0.01\textwidth}%
    \includegraphics[width=0.188\textwidth]{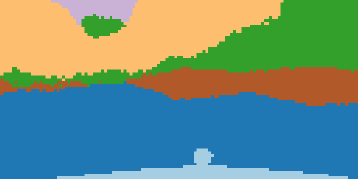}\hspace{0.01\textwidth}%
    \includegraphics[width=0.188\textwidth]{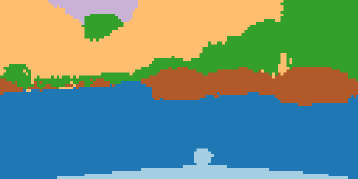}\hspace{0.01\textwidth}%
    \includegraphics[width=0.188\textwidth]{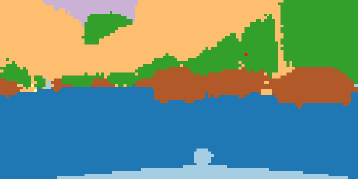}\hspace{0.01\textwidth}%
    \begingroup\setlength{\fboxsep}{0pt}\setlength{\fboxrule}{1.6pt}%
    \fcolorbox{red}{white}{\includegraphics[width=\dimexpr0.188\textwidth-2\fboxrule\relax]{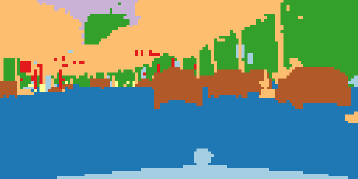}}%
    \endgroup{}
    \vspace{0.25em}
    \includegraphics[width=0.188\textwidth]{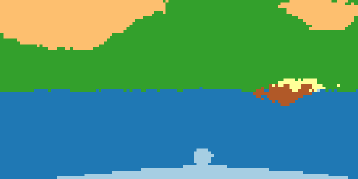}\hspace{0.01\textwidth}%
    \includegraphics[width=0.188\textwidth]{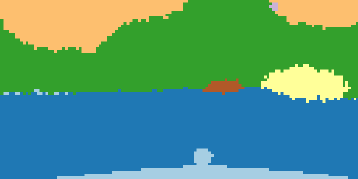}\hspace{0.01\textwidth}%
    \includegraphics[width=0.188\textwidth]{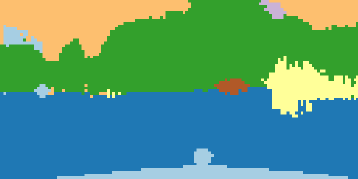}\hspace{0.01\textwidth}%
    \includegraphics[width=0.188\textwidth]{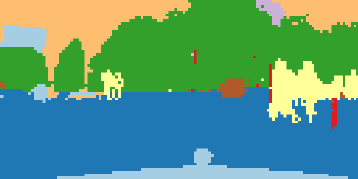}\hspace{0.01\textwidth}%
    \begingroup\setlength{\fboxsep}{0pt}\setlength{\fboxrule}{1.6pt}%
    \fcolorbox{red}{white}{\includegraphics[width=\dimexpr0.188\textwidth-2\fboxrule\relax]{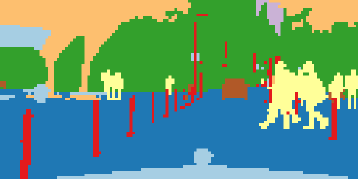}}%
    \endgroup{}
    \cityscapesdimslabels
    \end{minipage}}
    \caption{Visual comparison of Cityscapes GPCA reconstructions for randomly selected training samples. Each column corresponds to one latent dimension $d \in \{32,64,128,256\}$ for \textbf{fixed regularization parameter $\boldsymbol{\lambda = 10^{-1}}$}. The rightmost column, highlighted in red, shows the corresponding ground-truth data $x\in \mc X^n_N$.}
    \label{fig:cityscapes-dims-grid}
\end{figure}
\begin{figure}[H]
    \centering
    \resizebox{0.70\textwidth}{!}{%
    \begin{minipage}{\textwidth}
    \includegraphics[width=0.188\textwidth]{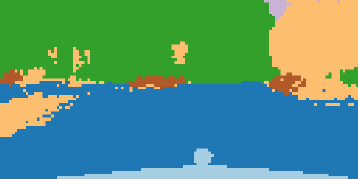}\hspace{0.01\textwidth}%
    \includegraphics[width=0.188\textwidth]{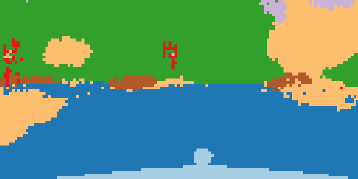}\hspace{0.01\textwidth}%
    \includegraphics[width=0.188\textwidth]{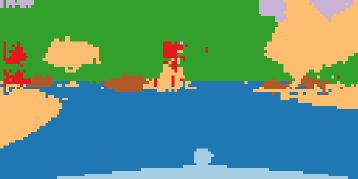}\hspace{0.01\textwidth}%
    \includegraphics[width=0.188\textwidth]{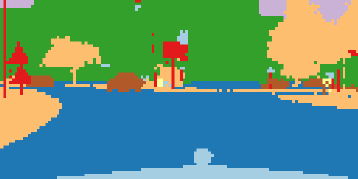}\hspace{0.01\textwidth}%
    \begingroup\setlength{\fboxsep}{0pt}\setlength{\fboxrule}{1.6pt}%
    \fcolorbox{red}{white}{\includegraphics[width=\dimexpr0.188\textwidth-2\fboxrule\relax]{Figures/cityscapes_dims/cityscapes_32_0.pdf}}%
    \endgroup{}
    \vspace{0.25em}
    \includegraphics[width=0.188\textwidth]{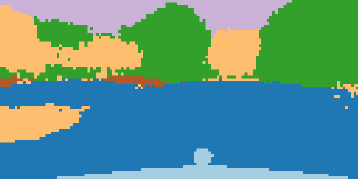}\hspace{0.01\textwidth}%
    \includegraphics[width=0.188\textwidth]{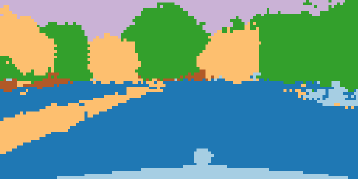}\hspace{0.01\textwidth}%
    \includegraphics[width=0.188\textwidth]{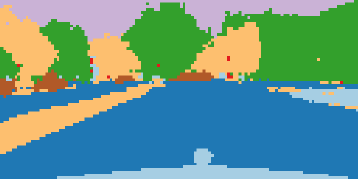}\hspace{0.01\textwidth}%
    \includegraphics[width=0.188\textwidth]{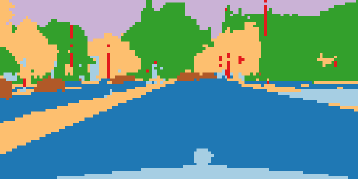}\hspace{0.01\textwidth}%
    \begingroup\setlength{\fboxsep}{0pt}\setlength{\fboxrule}{1.6pt}%
    \fcolorbox{red}{white}{\includegraphics[width=\dimexpr0.188\textwidth-2\fboxrule\relax]{Figures/cityscapes_dims/cityscapes_32_1.pdf}}%
    \endgroup{}
    \vspace{0.25em}
    \includegraphics[width=0.188\textwidth]{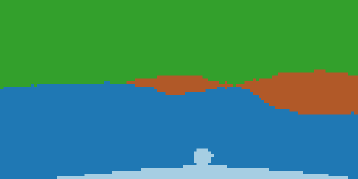}\hspace{0.01\textwidth}%
    \includegraphics[width=0.188\textwidth]{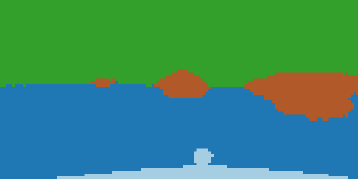}\hspace{0.01\textwidth}%
    \includegraphics[width=0.188\textwidth]{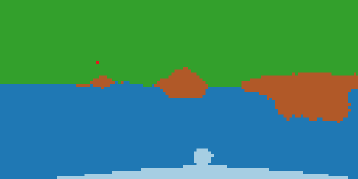}\hspace{0.01\textwidth}%
    \includegraphics[width=0.188\textwidth]{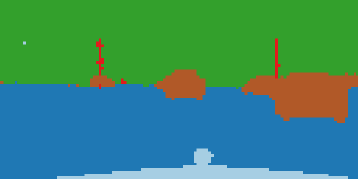}\hspace{0.01\textwidth}%
    \begingroup\setlength{\fboxsep}{0pt}\setlength{\fboxrule}{1.6pt}%
    \fcolorbox{red}{white}{\includegraphics[width=\dimexpr0.188\textwidth-2\fboxrule\relax]{Figures/cityscapes_dims/cityscapes_32_3.pdf}}%
    \endgroup{}
    \vspace{0.25em}
    \includegraphics[width=0.188\textwidth]{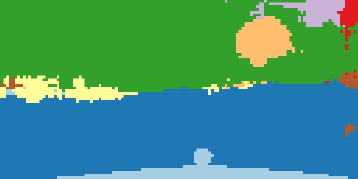}\hspace{0.01\textwidth}%
    \includegraphics[width=0.188\textwidth]{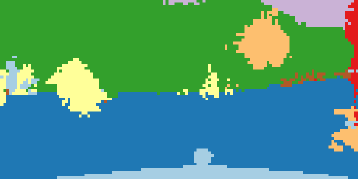}\hspace{0.01\textwidth}%
    \includegraphics[width=0.188\textwidth]{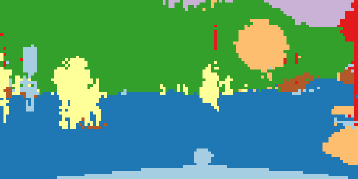}\hspace{0.01\textwidth}%
    \includegraphics[width=0.188\textwidth]{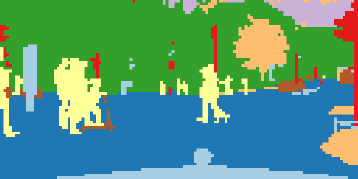}\hspace{0.01\textwidth}%
    \begingroup\setlength{\fboxsep}{0pt}\setlength{\fboxrule}{1.6pt}%
    \fcolorbox{red}{white}{\includegraphics[width=\dimexpr0.188\textwidth-2\fboxrule\relax]{Figures/cityscapes_dims/cityscapes_32_14.pdf}}%
    \endgroup{}
    \vspace{0.25em}
    \includegraphics[width=0.188\textwidth]{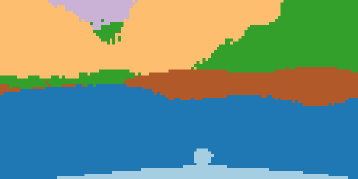}\hspace{0.01\textwidth}%
    \includegraphics[width=0.188\textwidth]{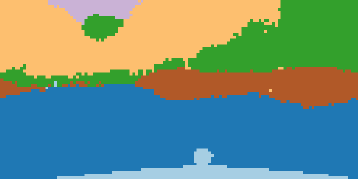}\hspace{0.01\textwidth}%
    \includegraphics[width=0.188\textwidth]{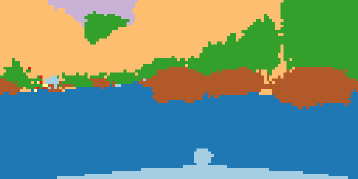}\hspace{0.01\textwidth}%
    \includegraphics[width=0.188\textwidth]{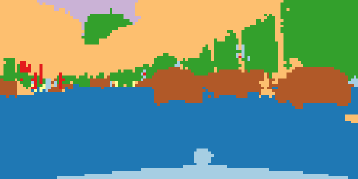}\hspace{0.01\textwidth}%
    \begingroup\setlength{\fboxsep}{0pt}\setlength{\fboxrule}{1.6pt}%
    \fcolorbox{red}{white}{\includegraphics[width=\dimexpr0.188\textwidth-2\fboxrule\relax]{Figures/cityscapes_dims/cityscapes_32_11.pdf}}%
    \endgroup{}
    \vspace{0.25em}
    \includegraphics[width=0.188\textwidth]{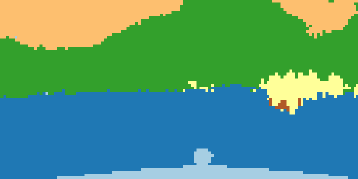}\hspace{0.01\textwidth}%
    \includegraphics[width=0.188\textwidth]{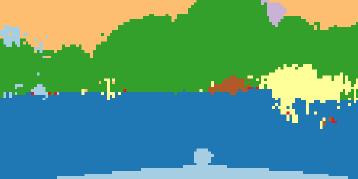}\hspace{0.01\textwidth}%
    \includegraphics[width=0.188\textwidth]{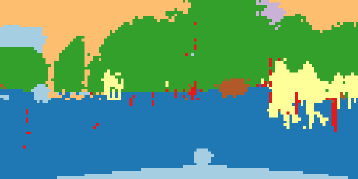}\hspace{0.01\textwidth}%
    \includegraphics[width=0.188\textwidth]{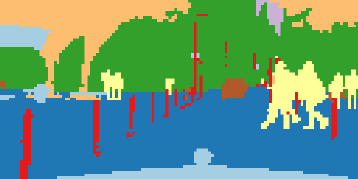}\hspace{0.01\textwidth}%
    \begingroup\setlength{\fboxsep}{0pt}\setlength{\fboxrule}{1.6pt}%
    \fcolorbox{red}{white}{\includegraphics[width=\dimexpr0.188\textwidth-2\fboxrule\relax]{Figures/cityscapes_dims/cityscapes_32_17.pdf}}%
    \endgroup{}
    \cityscapesdimslabels
    \end{minipage}}
    \caption{Visual comparison of Cityscapes GPCA reconstructions for randomly selected training samples. Each column corresponds to one latent dimension $d \in \{32,64,128,256\}$ for \textbf{fixed regularization parameter $\boldsymbol{\lambda = 10^{-2}}$}. The rightmost column, highlighted in red, shows the corresponding ground-truth data $x\in \mc X^n_N$.}
    \label{fig:cityscapes-dims-grid-1e-2}
\end{figure}

\begin{figure}[H]
    \centering
    \resizebox{0.70\textwidth}{!}{%
    \begin{minipage}{\textwidth}
    \includegraphics[width=0.188\textwidth]{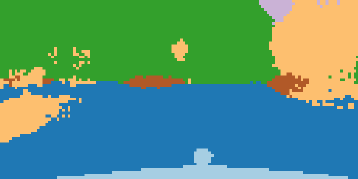}\hspace{0.01\textwidth}%
    \includegraphics[width=0.188\textwidth]{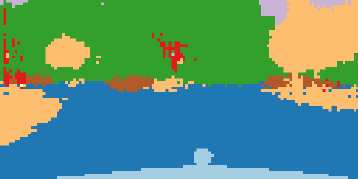}\hspace{0.01\textwidth}%
    \includegraphics[width=0.188\textwidth]{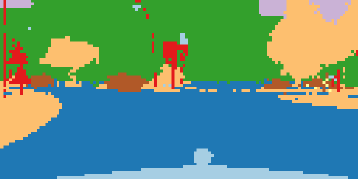}\hspace{0.01\textwidth}%
    \includegraphics[width=0.188\textwidth]{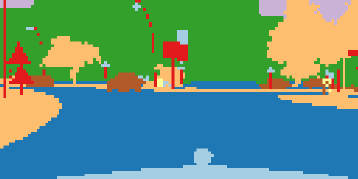}\hspace{0.01\textwidth}%
    \begingroup\setlength{\fboxsep}{0pt}\setlength{\fboxrule}{1.6pt}%
    \fcolorbox{red}{white}{\includegraphics[width=\dimexpr0.188\textwidth-2\fboxrule\relax]{Figures/cityscapes_dims/cityscapes_32_0.pdf}}%
    \endgroup{}
    \vspace{0.25em}
    \includegraphics[width=0.188\textwidth]{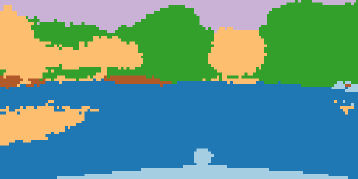}\hspace{0.01\textwidth}%
    \includegraphics[width=0.188\textwidth]{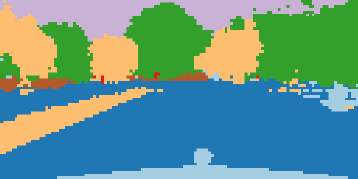}\hspace{0.01\textwidth}%
    \includegraphics[width=0.188\textwidth]{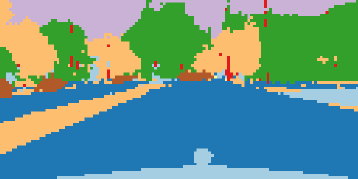}\hspace{0.01\textwidth}%
    \includegraphics[width=0.188\textwidth]{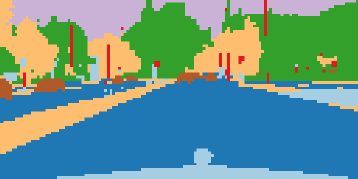}\hspace{0.01\textwidth}%
    \begingroup\setlength{\fboxsep}{0pt}\setlength{\fboxrule}{1.6pt}%
    \fcolorbox{red}{white}{\includegraphics[width=\dimexpr0.188\textwidth-2\fboxrule\relax]{Figures/cityscapes_dims/cityscapes_32_1.pdf}}%
    \endgroup{}
    \vspace{0.25em}
    \includegraphics[width=0.188\textwidth]{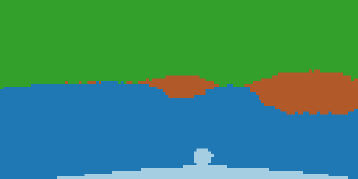}\hspace{0.01\textwidth}%
    \includegraphics[width=0.188\textwidth]{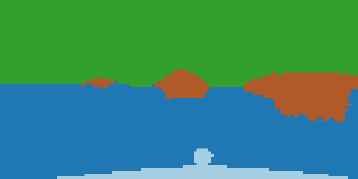}\hspace{0.01\textwidth}%
    \includegraphics[width=0.188\textwidth]{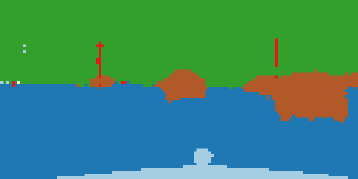}\hspace{0.01\textwidth}%
    \includegraphics[width=0.188\textwidth]{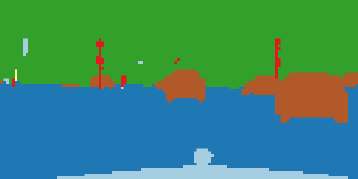}\hspace{0.01\textwidth}%
    \begingroup\setlength{\fboxsep}{0pt}\setlength{\fboxrule}{1.6pt}%
    \fcolorbox{red}{white}{\includegraphics[width=\dimexpr0.188\textwidth-2\fboxrule\relax]{Figures/cityscapes_dims/cityscapes_32_3.pdf}}%
    \endgroup{}
    \vspace{0.25em}
    \includegraphics[width=0.188\textwidth]{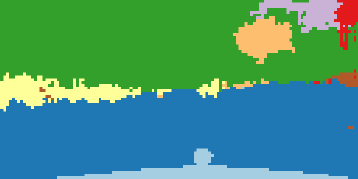}\hspace{0.01\textwidth}%
    \includegraphics[width=0.188\textwidth]{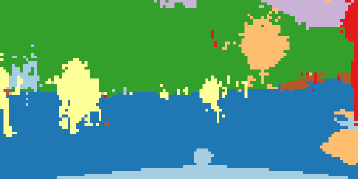}\hspace{0.01\textwidth}%
    \includegraphics[width=0.188\textwidth]{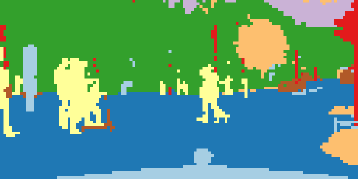}\hspace{0.01\textwidth}%
    \includegraphics[width=0.188\textwidth]{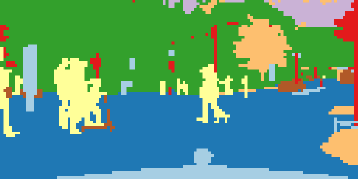}\hspace{0.01\textwidth}%
    \begingroup\setlength{\fboxsep}{0pt}\setlength{\fboxrule}{1.6pt}%
    \fcolorbox{red}{white}{\includegraphics[width=\dimexpr0.188\textwidth-2\fboxrule\relax]{Figures/cityscapes_dims/cityscapes_32_14.pdf}}%
    \endgroup{}
    \vspace{0.25em}
    \includegraphics[width=0.188\textwidth]{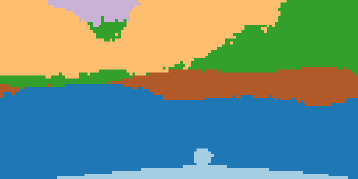}\hspace{0.01\textwidth}%
    \includegraphics[width=0.188\textwidth]{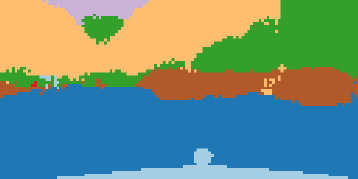}\hspace{0.01\textwidth}%
    \includegraphics[width=0.188\textwidth]{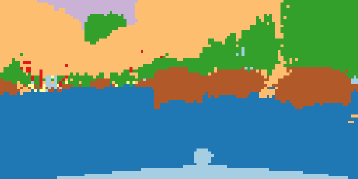}\hspace{0.01\textwidth}%
    \includegraphics[width=0.188\textwidth]{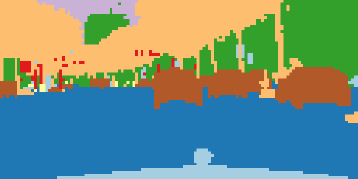}\hspace{0.01\textwidth}%
    \begingroup\setlength{\fboxsep}{0pt}\setlength{\fboxrule}{1.6pt}%
    \fcolorbox{red}{white}{\includegraphics[width=\dimexpr0.188\textwidth-2\fboxrule\relax]{Figures/cityscapes_dims/cityscapes_32_11.pdf}}%
    \endgroup{}
    \vspace{0.25em}
    \includegraphics[width=0.188\textwidth]{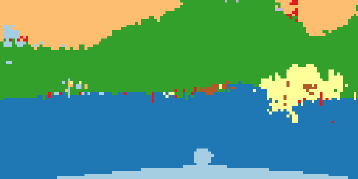}\hspace{0.01\textwidth}%
    \includegraphics[width=0.188\textwidth]{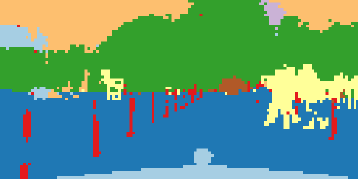}\hspace{0.01\textwidth}%
    \includegraphics[width=0.188\textwidth]{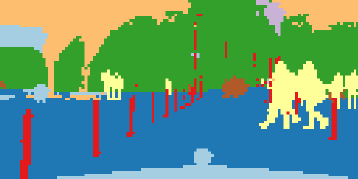}\hspace{0.01\textwidth}%
    \includegraphics[width=0.188\textwidth]{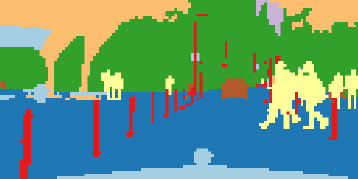}\hspace{0.01\textwidth}%
    \begingroup\setlength{\fboxsep}{0pt}\setlength{\fboxrule}{1.6pt}%
    \fcolorbox{red}{white}{\includegraphics[width=\dimexpr0.188\textwidth-2\fboxrule\relax]{Figures/cityscapes_dims/cityscapes_32_17.pdf}}%
    \endgroup{}
    \cityscapesdimslabels
    \end{minipage}}
    \caption{Visual comparison of Cityscapes GPCA reconstructions for randomly selected training samples. Each column corresponds to one latent dimension $d \in \{32,64,128,256\}$ for \textbf{fixed regularization parameter $\boldsymbol{\lambda = 10^{-3}}$}. The rightmost column, highlighted in red, shows the corresponding ground-truth data $x\in \mc X^n_N$.}
    \label{fig:cityscapes-dims-grid-1e-3}
\end{figure}

\clearpage
\paragraph{MNIST and Fashion-MNIST Datasets.}
We illustrate the effects of the regularization parameter $\lambda$ and the latent dimension $d$ on the GPCA reconstructions of the MNIST and Fashion-MNIST datasets. Figures~\ref{fig:mnist-dims-grid-1e-2} and~\ref{fig:mnist-dims-grid-1e-3} correspond to selected values from Tables~\ref{tab:mnist-dims-lambda-1e-2} and~\ref{tab:mnist-dims-lambda-1e-3}, respectively.
In particular, we observe that, for both $\lambda=10^{-2}$ and $\lambda=10^{-3}$, the reconstructions capture the overall structure of the data for dimensions $d=32$ and $d=64$. These are the latent dimensions used to train our flow matching model, which yielded good performance, cf. Section~\ref{sec:generation-metrics}.
Figures~\ref{fig:fashion-mnist-dims-grid} and~\ref{fig:fashion-mnist-dims-grid-1e-3} correspond to selected values from Tables~\ref{tab:fashion-mnist-dims-lambda-1e-2} and~\ref{tab:fashion-mnist-dims-lambda-1e-3}, respectively.

\newcommand{\mnistdimslabels}{%
    \par\vspace{0.25em}
    {\fontsize{9}{10}\selectfont
    \makebox[0.188\textwidth][c]{\boldmath$d=8$}\hspace{0.01\textwidth}%
    \makebox[0.188\textwidth][c]{\boldmath$d=16$}\hspace{0.01\textwidth}%
    \makebox[0.188\textwidth][c]{\boldmath$d=32$}\hspace{0.01\textwidth}%
    \makebox[0.188\textwidth][c]{\boldmath$d=64$}\hspace{0.01\textwidth}%
    \makebox[0.188\textwidth][c]{\shortstack{\textbf{Ground}\\\textbf{truth}}}}}

\begin{figure}[H]
\centering
\begin{minipage}[t]{0.48\textwidth}
\vspace{0pt}

    \centering
    \resizebox{\linewidth}{!}{%
        \begin{minipage}{\textwidth}
    \includegraphics[width=0.188\textwidth]{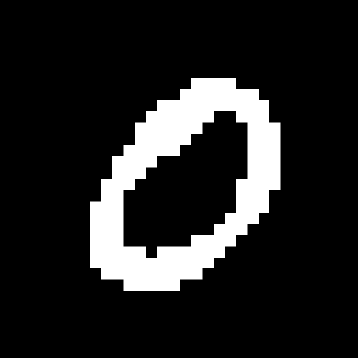}\hspace{0.01\textwidth}%
    \includegraphics[width=0.188\textwidth]{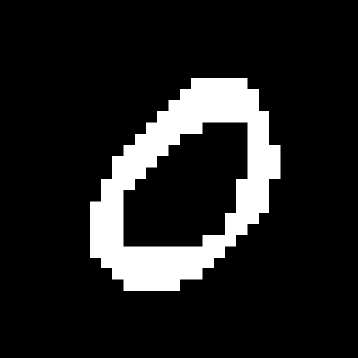}\hspace{0.01\textwidth}%
    \includegraphics[width=0.188\textwidth]{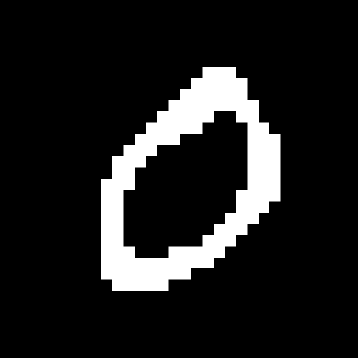}\hspace{0.01\textwidth}%
    \includegraphics[width=0.188\textwidth]{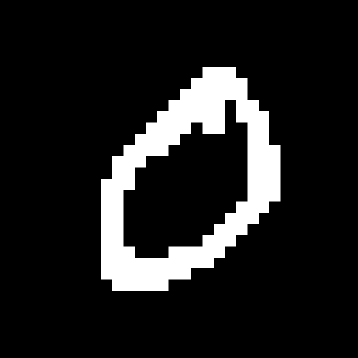}\hspace{0.01\textwidth}%
    \begingroup\setlength{\fboxsep}{0pt}\setlength{\fboxrule}{1.6pt}%
    \fcolorbox{blue}{white}{\includegraphics[width=\dimexpr0.188\textwidth-2\fboxrule\relax]{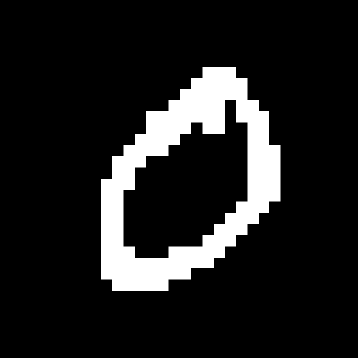}}%
    \endgroup{}
            \vspace{0.25em}
    \includegraphics[width=0.188\textwidth]{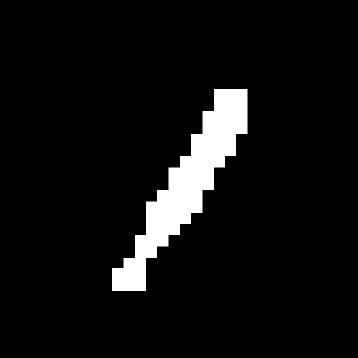}\hspace{0.01\textwidth}%
    \includegraphics[width=0.188\textwidth]{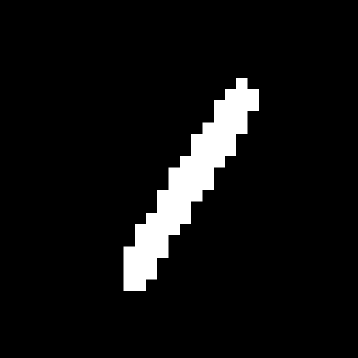}\hspace{0.01\textwidth}%
    \includegraphics[width=0.188\textwidth]{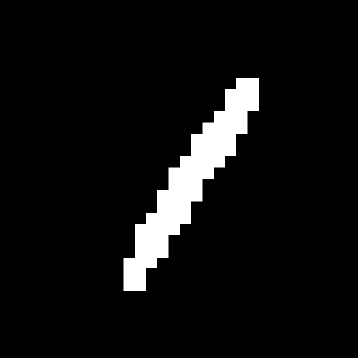}\hspace{0.01\textwidth}%
    \includegraphics[width=0.188\textwidth]{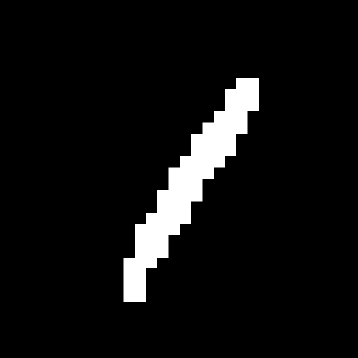}\hspace{0.01\textwidth}%
    \begingroup\setlength{\fboxsep}{0pt}\setlength{\fboxrule}{1.6pt}%
    \fcolorbox{blue}{white}{\includegraphics[width=\dimexpr0.188\textwidth-2\fboxrule\relax]{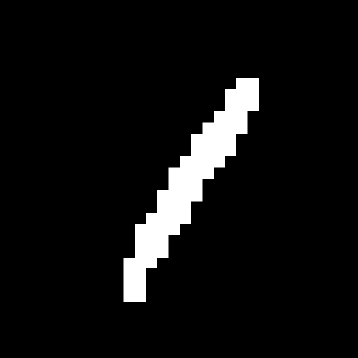}}%
    \endgroup{}
            \vspace{0.25em}
    \includegraphics[width=0.188\textwidth]{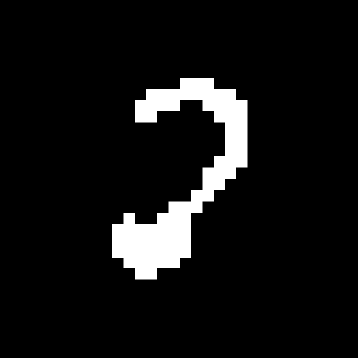}\hspace{0.01\textwidth}%
    \includegraphics[width=0.188\textwidth]{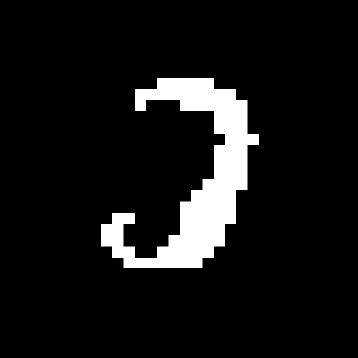}\hspace{0.01\textwidth}%
    \includegraphics[width=0.188\textwidth]{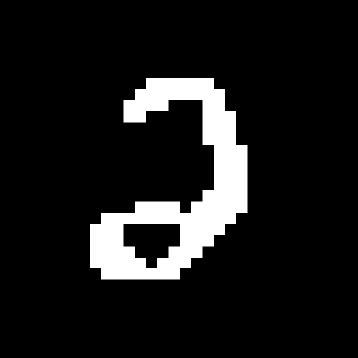}\hspace{0.01\textwidth}%
    \includegraphics[width=0.188\textwidth]{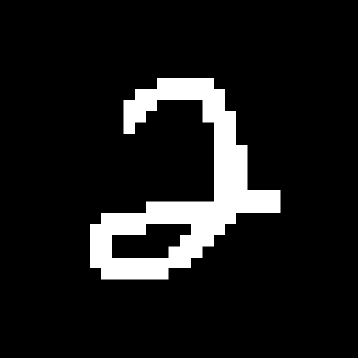}\hspace{0.01\textwidth}%
    \begingroup\setlength{\fboxsep}{0pt}\setlength{\fboxrule}{1.6pt}%
    \fcolorbox{blue}{white}{\includegraphics[width=\dimexpr0.188\textwidth-2\fboxrule\relax]{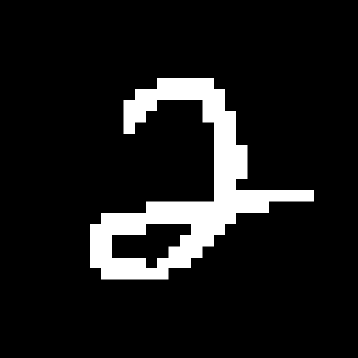}}%
    \endgroup{}
            \vspace{0.25em}
    \includegraphics[width=0.188\textwidth]{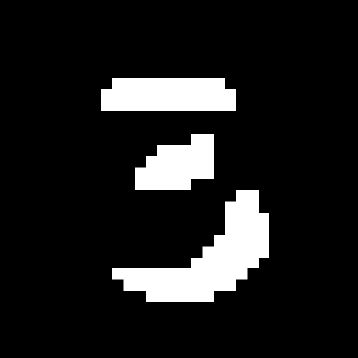}\hspace{0.01\textwidth}%
    \includegraphics[width=0.188\textwidth]{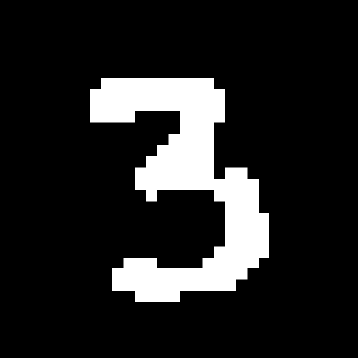}\hspace{0.01\textwidth}%
    \includegraphics[width=0.188\textwidth]{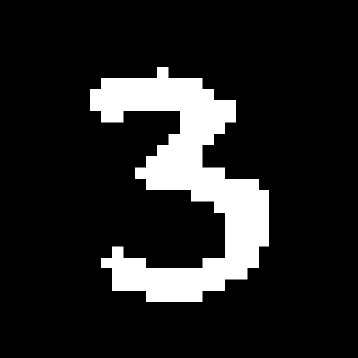}\hspace{0.01\textwidth}%
    \includegraphics[width=0.188\textwidth]{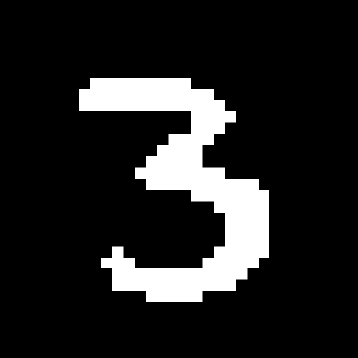}\hspace{0.01\textwidth}%
    \begingroup\setlength{\fboxsep}{0pt}\setlength{\fboxrule}{1.6pt}%
    \fcolorbox{blue}{white}{\includegraphics[width=\dimexpr0.188\textwidth-2\fboxrule\relax]{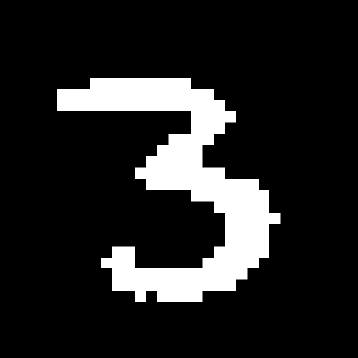}}%
    \endgroup{}
            \vspace{0.25em}
    \includegraphics[width=0.188\textwidth]{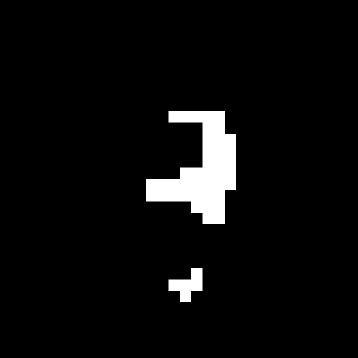}\hspace{0.01\textwidth}%
    \includegraphics[width=0.188\textwidth]{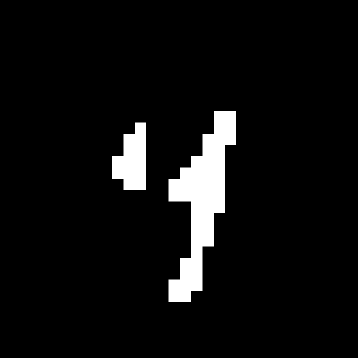}\hspace{0.01\textwidth}%
    \includegraphics[width=0.188\textwidth]{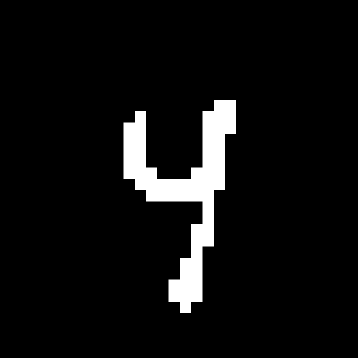}\hspace{0.01\textwidth}%
    \includegraphics[width=0.188\textwidth]{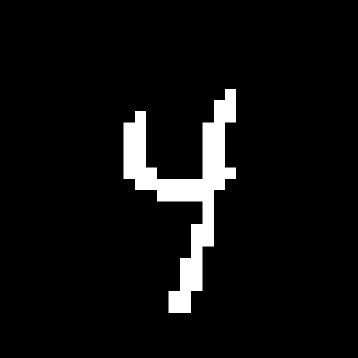}\hspace{0.01\textwidth}%
    \begingroup\setlength{\fboxsep}{0pt}\setlength{\fboxrule}{1.6pt}%
    \fcolorbox{blue}{white}{\includegraphics[width=\dimexpr0.188\textwidth-2\fboxrule\relax]{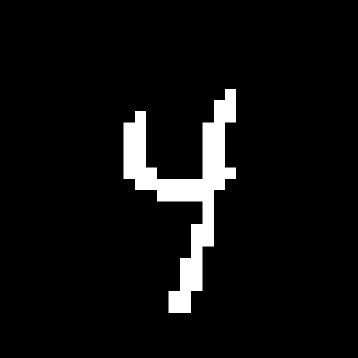}}%
    \endgroup{}
            \vspace{0.25em}
    \includegraphics[width=0.188\textwidth]{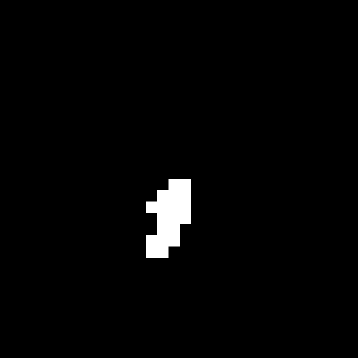}\hspace{0.01\textwidth}%
    \includegraphics[width=0.188\textwidth]{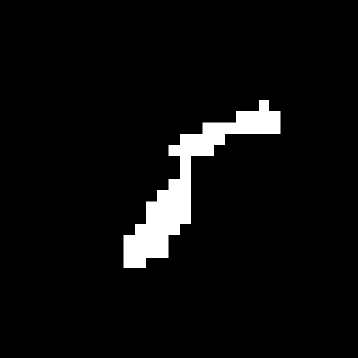}\hspace{0.01\textwidth}%
    \includegraphics[width=0.188\textwidth]{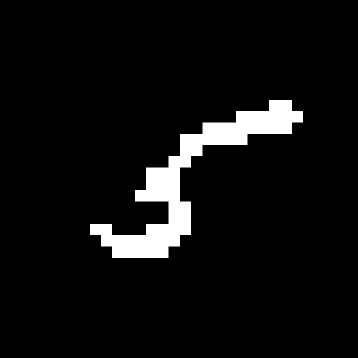}\hspace{0.01\textwidth}%
    \includegraphics[width=0.188\textwidth]{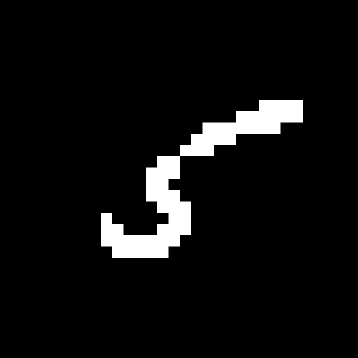}\hspace{0.01\textwidth}%
    \begingroup\setlength{\fboxsep}{0pt}\setlength{\fboxrule}{1.6pt}%
    \fcolorbox{blue}{white}{\includegraphics[width=\dimexpr0.188\textwidth-2\fboxrule\relax]{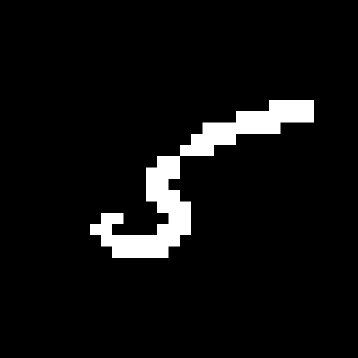}}%
    \endgroup{}
            \vspace{0.25em}
    \includegraphics[width=0.188\textwidth]{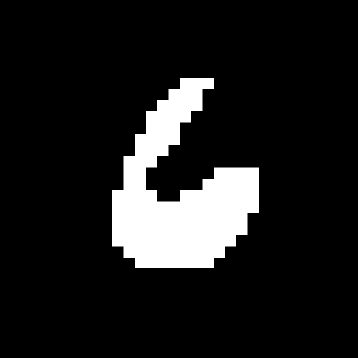}\hspace{0.01\textwidth}%
    \includegraphics[width=0.188\textwidth]{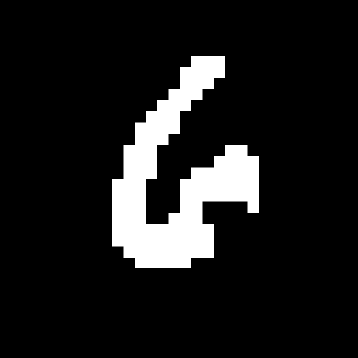}\hspace{0.01\textwidth}%
    \includegraphics[width=0.188\textwidth]{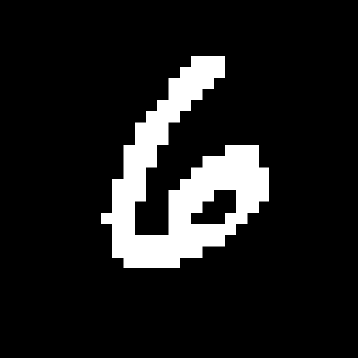}\hspace{0.01\textwidth}%
    \includegraphics[width=0.188\textwidth]{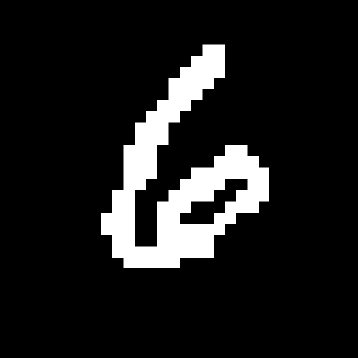}\hspace{0.01\textwidth}%
    \begingroup\setlength{\fboxsep}{0pt}\setlength{\fboxrule}{1.6pt}%
    \fcolorbox{blue}{white}{\includegraphics[width=\dimexpr0.188\textwidth-2\fboxrule\relax]{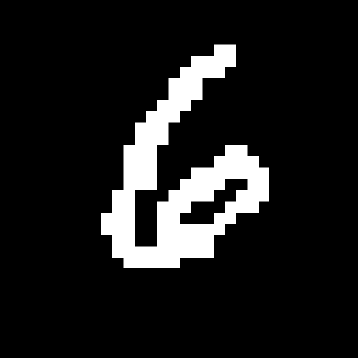}}%
    \endgroup{}
            \vspace{0.25em}
    \includegraphics[width=0.188\textwidth]{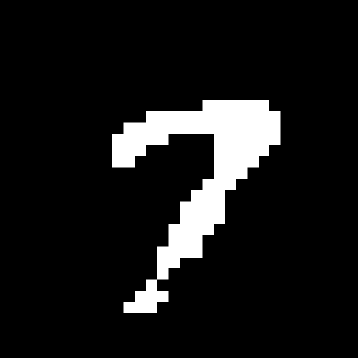}\hspace{0.01\textwidth}%
    \includegraphics[width=0.188\textwidth]{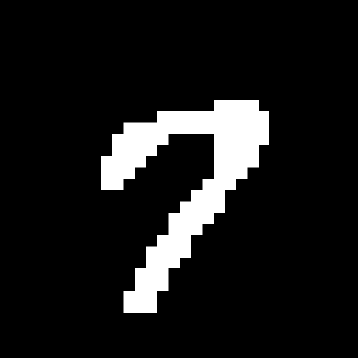}\hspace{0.01\textwidth}%
    \includegraphics[width=0.188\textwidth]{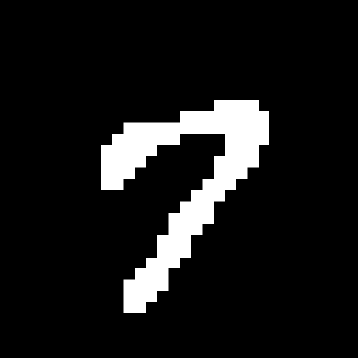}\hspace{0.01\textwidth}%
    \includegraphics[width=0.188\textwidth]{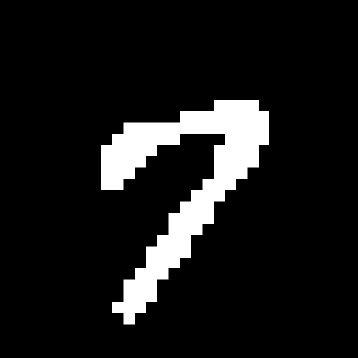}\hspace{0.01\textwidth}%
    \begingroup\setlength{\fboxsep}{0pt}\setlength{\fboxrule}{1.6pt}%
    \fcolorbox{blue}{white}{\includegraphics[width=\dimexpr0.188\textwidth-2\fboxrule\relax]{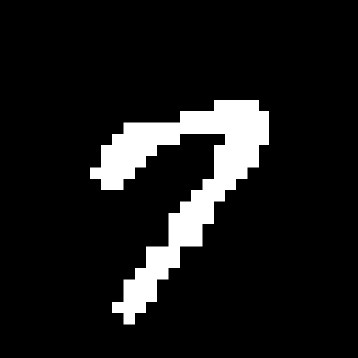}}%
    \endgroup{}
            \vspace{0.25em}
    \includegraphics[width=0.188\textwidth]{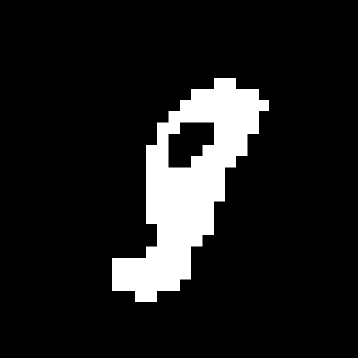}\hspace{0.01\textwidth}%
    \includegraphics[width=0.188\textwidth]{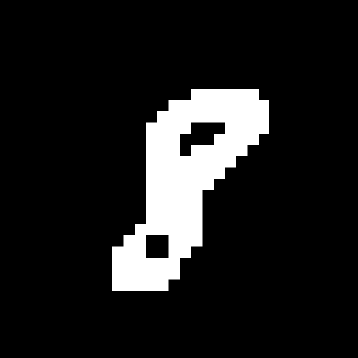}\hspace{0.01\textwidth}%
    \includegraphics[width=0.188\textwidth]{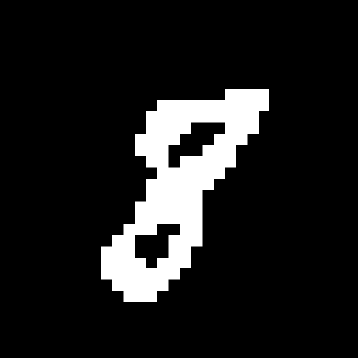}\hspace{0.01\textwidth}%
    \includegraphics[width=0.188\textwidth]{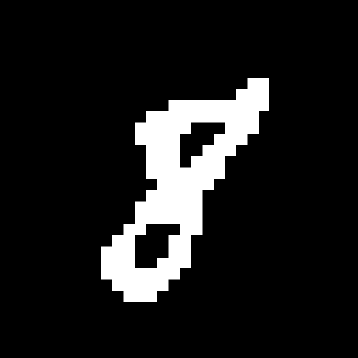}\hspace{0.01\textwidth}%
    \begingroup\setlength{\fboxsep}{0pt}\setlength{\fboxrule}{1.6pt}%
    \fcolorbox{blue}{white}{\includegraphics[width=\dimexpr0.188\textwidth-2\fboxrule\relax]{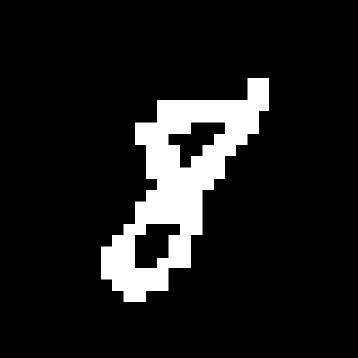}}%
    \endgroup{}
            \vspace{0.25em}
    \includegraphics[width=0.188\textwidth]{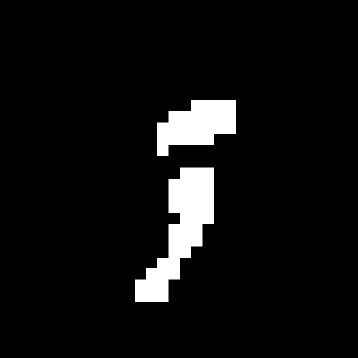}\hspace{0.01\textwidth}%
    \includegraphics[width=0.188\textwidth]{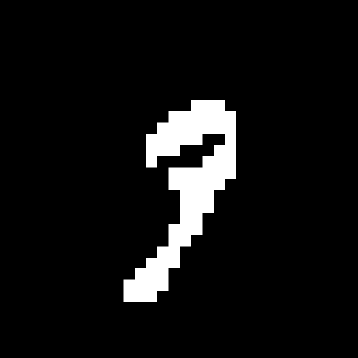}\hspace{0.01\textwidth}%
    \includegraphics[width=0.188\textwidth]{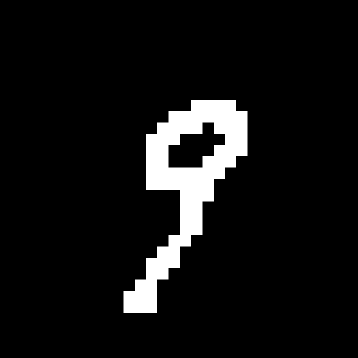}\hspace{0.01\textwidth}%
    \includegraphics[width=0.188\textwidth]{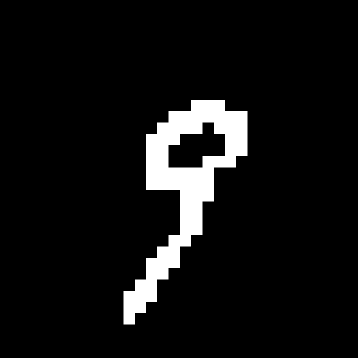}\hspace{0.01\textwidth}%
    \begingroup\setlength{\fboxsep}{0pt}\setlength{\fboxrule}{1.6pt}%
    \fcolorbox{blue}{white}{\includegraphics[width=\dimexpr0.188\textwidth-2\fboxrule\relax]{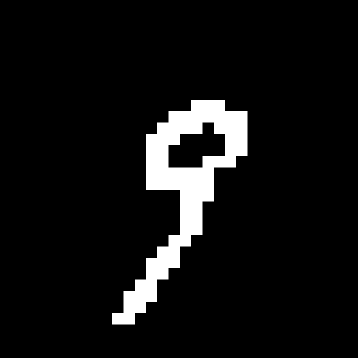}}%
    \endgroup{}
            \mnistdimslabels
        \end{minipage}%
    }
    \caption{Visual comparison of MNIST GPCA reconstructions for randomly selected training samples. Each column corresponds to one latent dimension $d \in \{8,16,32,64\}$ for \textbf{fixed regularization parameter $\boldsymbol{\lambda = 10^{-2}}$}. The rightmost column, highlighted in blue, shows the corresponding ground-truth data $x\in \mc X^n_N$.}
    \label{fig:mnist-dims-grid-1e-2}
\end{minipage}\hfill
\begin{minipage}[t]{0.48\textwidth}
\vspace{0pt}

    \centering
    \resizebox{\linewidth}{!}{%
        \begin{minipage}{\textwidth}
    \includegraphics[width=0.188\textwidth]{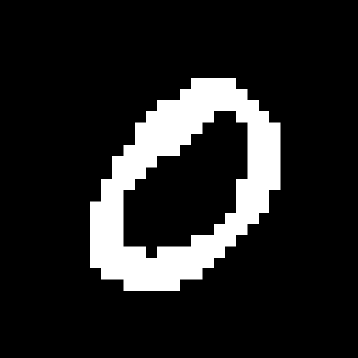}\hspace{0.01\textwidth}%
    \includegraphics[width=0.188\textwidth]{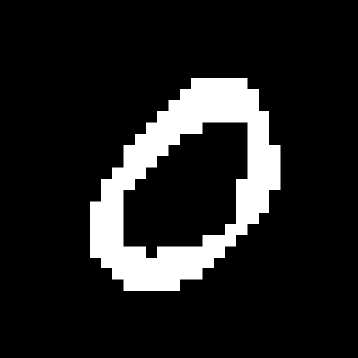}\hspace{0.01\textwidth}%
    \includegraphics[width=0.188\textwidth]{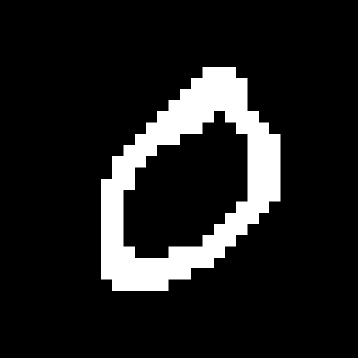}\hspace{0.01\textwidth}%
    \includegraphics[width=0.188\textwidth]{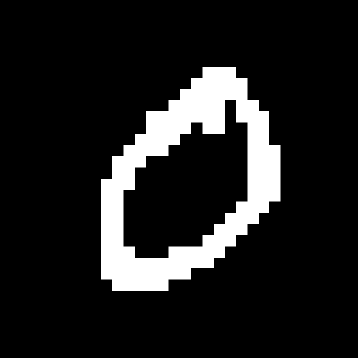}\hspace{0.01\textwidth}%
    \begingroup\setlength{\fboxsep}{0pt}\setlength{\fboxrule}{1.6pt}%
    \fcolorbox{blue}{white}{\includegraphics[width=\dimexpr0.188\textwidth-2\fboxrule\relax]{Figures/mnist_dims/mnist_8_1.pdf}}%
    \endgroup{}
            \vspace{0.25em}
    \includegraphics[width=0.188\textwidth]{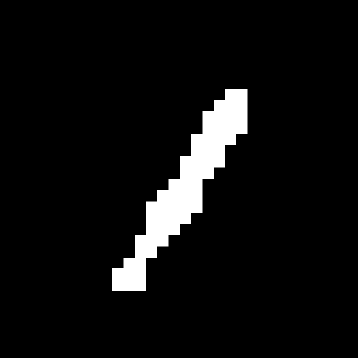}\hspace{0.01\textwidth}%
    \includegraphics[width=0.188\textwidth]{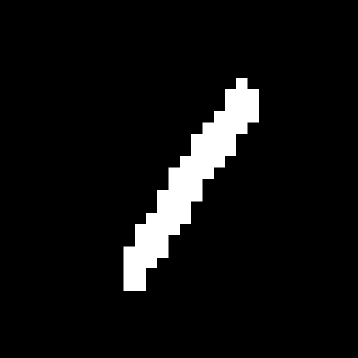}\hspace{0.01\textwidth}%
    \includegraphics[width=0.188\textwidth]{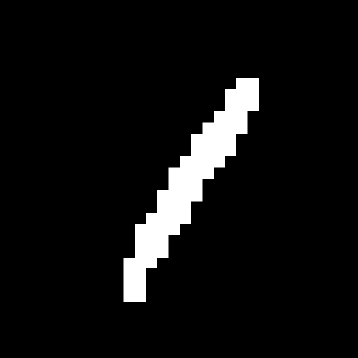}\hspace{0.01\textwidth}%
    \includegraphics[width=0.188\textwidth]{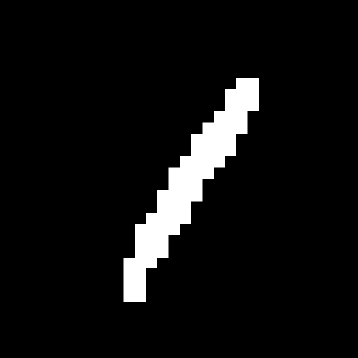}\hspace{0.01\textwidth}%
    \begingroup\setlength{\fboxsep}{0pt}\setlength{\fboxrule}{1.6pt}%
    \fcolorbox{blue}{white}{\includegraphics[width=\dimexpr0.188\textwidth-2\fboxrule\relax]{Figures/mnist_dims/mnist_8_23.pdf}}%
    \endgroup{}
            \vspace{0.25em}
    \includegraphics[width=0.188\textwidth]{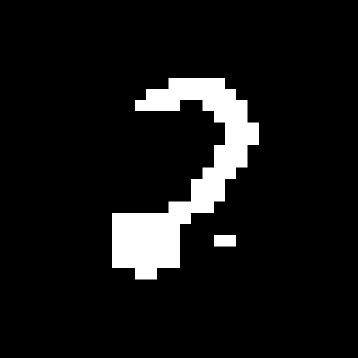}\hspace{0.01\textwidth}%
    \includegraphics[width=0.188\textwidth]{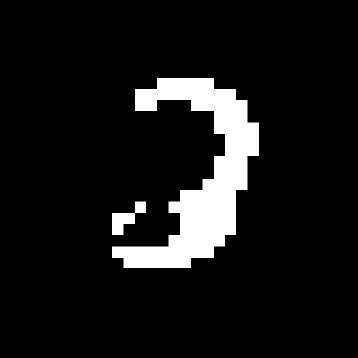}\hspace{0.01\textwidth}%
    \includegraphics[width=0.188\textwidth]{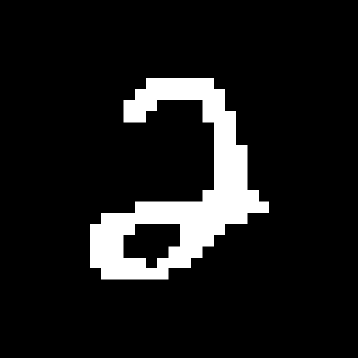}\hspace{0.01\textwidth}%
    \includegraphics[width=0.188\textwidth]{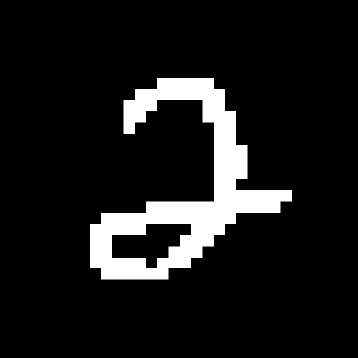}\hspace{0.01\textwidth}%
    \begingroup\setlength{\fboxsep}{0pt}\setlength{\fboxrule}{1.6pt}%
    \fcolorbox{blue}{white}{\includegraphics[width=\dimexpr0.188\textwidth-2\fboxrule\relax]{Figures/mnist_dims/mnist_8_16.pdf}}%
    \endgroup{}
            \vspace{0.25em}
    \includegraphics[width=0.188\textwidth]{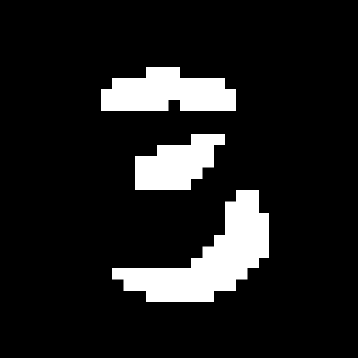}\hspace{0.01\textwidth}%
    \includegraphics[width=0.188\textwidth]{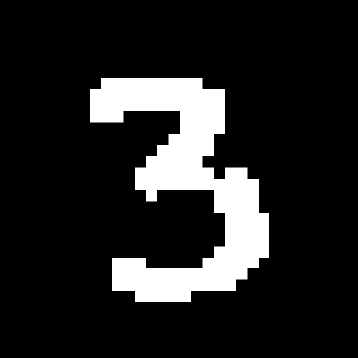}\hspace{0.01\textwidth}%
    \includegraphics[width=0.188\textwidth]{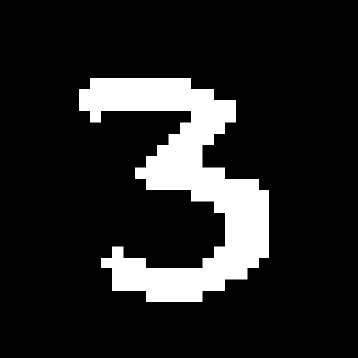}\hspace{0.01\textwidth}%
    \includegraphics[width=0.188\textwidth]{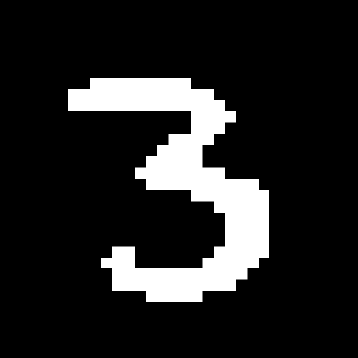}\hspace{0.01\textwidth}%
    \begingroup\setlength{\fboxsep}{0pt}\setlength{\fboxrule}{1.6pt}%
    \fcolorbox{blue}{white}{\includegraphics[width=\dimexpr0.188\textwidth-2\fboxrule\relax]{Figures/mnist_dims/mnist_8_12.pdf}}%
    \endgroup{}
            \vspace{0.25em}
    \includegraphics[width=0.188\textwidth]{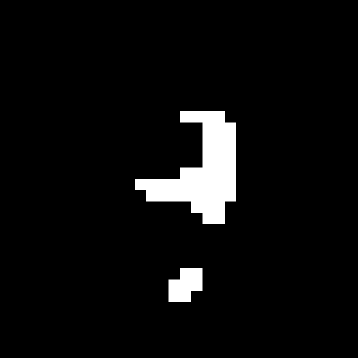}\hspace{0.01\textwidth}%
    \includegraphics[width=0.188\textwidth]{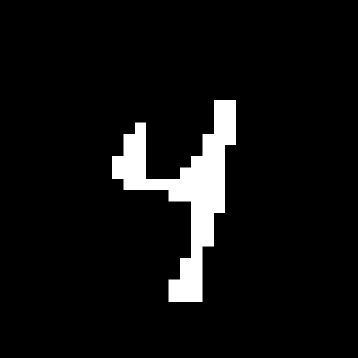}\hspace{0.01\textwidth}%
    \includegraphics[width=0.188\textwidth]{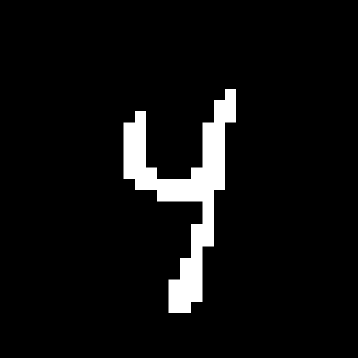}\hspace{0.01\textwidth}%
    \includegraphics[width=0.188\textwidth]{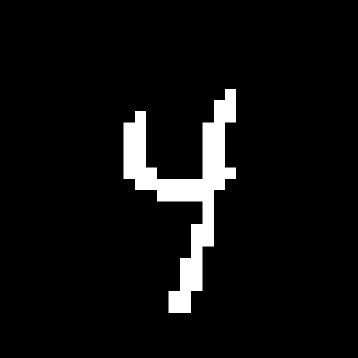}\hspace{0.01\textwidth}%
    \begingroup\setlength{\fboxsep}{0pt}\setlength{\fboxrule}{1.6pt}%
    \fcolorbox{blue}{white}{\includegraphics[width=\dimexpr0.188\textwidth-2\fboxrule\relax]{Figures/mnist_dims/mnist_8_26.pdf}}%
    \endgroup{}
            \vspace{0.25em}
    \includegraphics[width=0.188\textwidth]{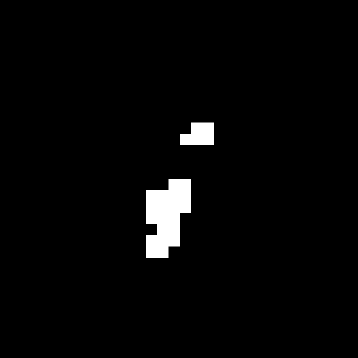}\hspace{0.01\textwidth}%
    \includegraphics[width=0.188\textwidth]{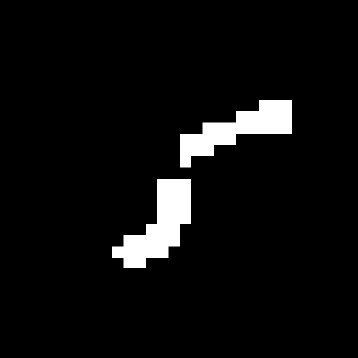}\hspace{0.01\textwidth}%
    \includegraphics[width=0.188\textwidth]{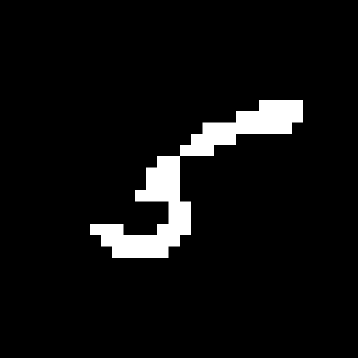}\hspace{0.01\textwidth}%
    \includegraphics[width=0.188\textwidth]{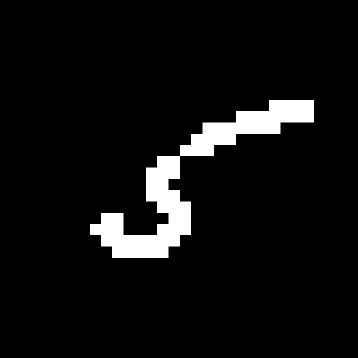}\hspace{0.01\textwidth}%
    \begingroup\setlength{\fboxsep}{0pt}\setlength{\fboxrule}{1.6pt}%
    \fcolorbox{blue}{white}{\includegraphics[width=\dimexpr0.188\textwidth-2\fboxrule\relax]{Figures/mnist_dims/mnist_8_11.pdf}}%
    \endgroup{}
            \vspace{0.25em}
    \includegraphics[width=0.188\textwidth]{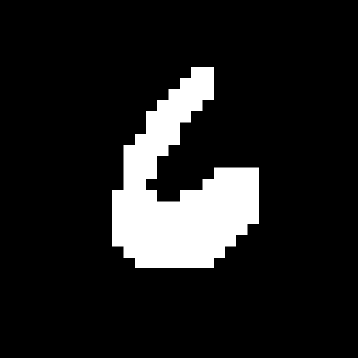}\hspace{0.01\textwidth}%
    \includegraphics[width=0.188\textwidth]{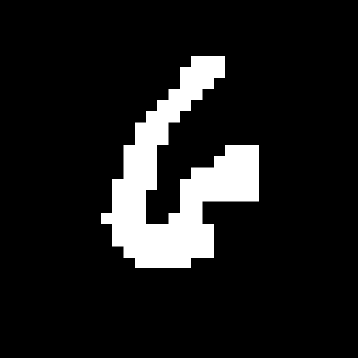}\hspace{0.01\textwidth}%
    \includegraphics[width=0.188\textwidth]{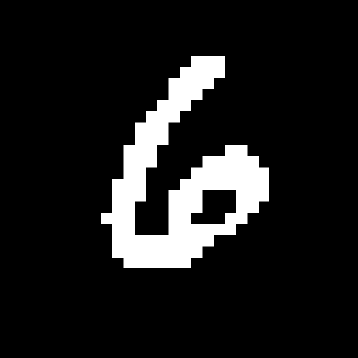}\hspace{0.01\textwidth}%
    \includegraphics[width=0.188\textwidth]{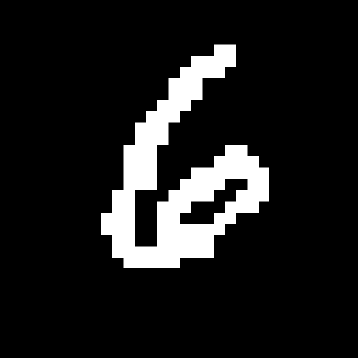}\hspace{0.01\textwidth}%
    \begingroup\setlength{\fboxsep}{0pt}\setlength{\fboxrule}{1.6pt}%
    \fcolorbox{blue}{white}{\includegraphics[width=\dimexpr0.188\textwidth-2\fboxrule\relax]{Figures/mnist_dims/mnist_8_13.pdf}}%
    \endgroup{}
            \vspace{0.25em}
    \includegraphics[width=0.188\textwidth]{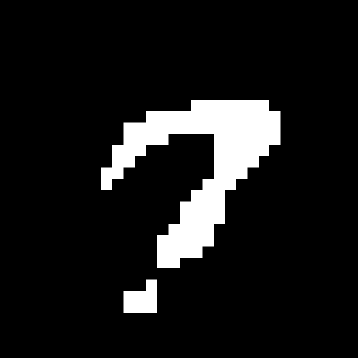}\hspace{0.01\textwidth}%
    \includegraphics[width=0.188\textwidth]{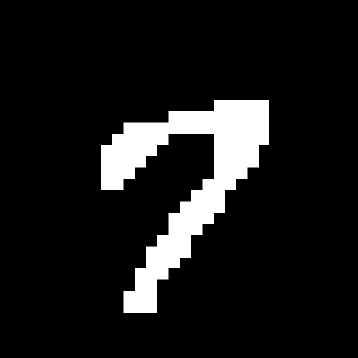}\hspace{0.01\textwidth}%
    \includegraphics[width=0.188\textwidth]{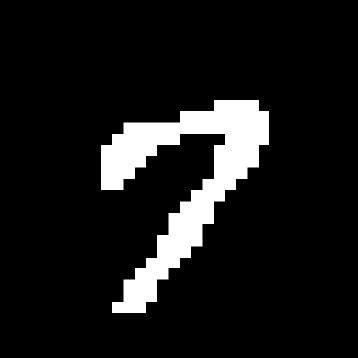}\hspace{0.01\textwidth}%
    \includegraphics[width=0.188\textwidth]{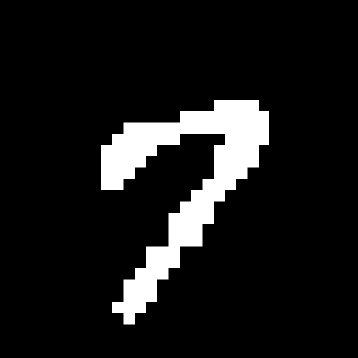}\hspace{0.01\textwidth}%
    \begingroup\setlength{\fboxsep}{0pt}\setlength{\fboxrule}{1.6pt}%
    \fcolorbox{blue}{white}{\includegraphics[width=\dimexpr0.188\textwidth-2\fboxrule\relax]{Figures/mnist_dims/mnist_8_15.pdf}}%
    \endgroup{}
            \vspace{0.25em}
    \includegraphics[width=0.188\textwidth]{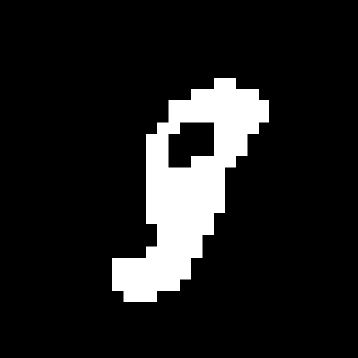}\hspace{0.01\textwidth}%
    \includegraphics[width=0.188\textwidth]{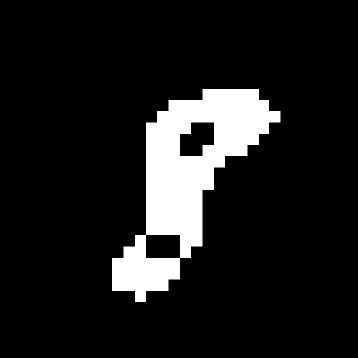}\hspace{0.01\textwidth}%
    \includegraphics[width=0.188\textwidth]{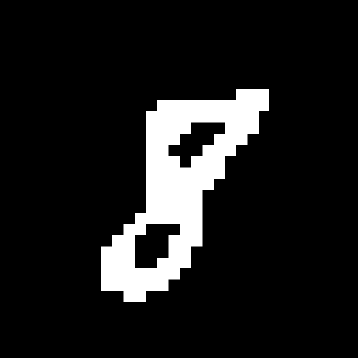}\hspace{0.01\textwidth}%
    \includegraphics[width=0.188\textwidth]{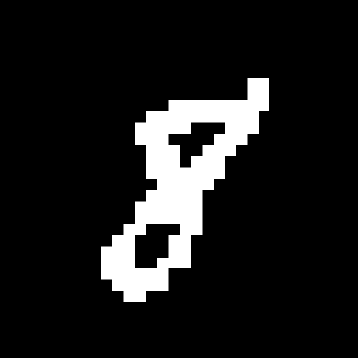}\hspace{0.01\textwidth}%
    \begingroup\setlength{\fboxsep}{0pt}\setlength{\fboxrule}{1.6pt}%
    \fcolorbox{blue}{white}{\includegraphics[width=\dimexpr0.188\textwidth-2\fboxrule\relax]{Figures/mnist_dims/mnist_8_17.pdf}}%
    \endgroup{}
            \vspace{0.25em}
    \includegraphics[width=0.188\textwidth]{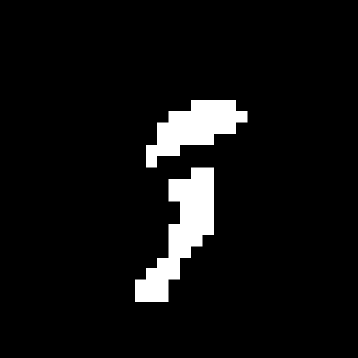}\hspace{0.01\textwidth}%
    \includegraphics[width=0.188\textwidth]{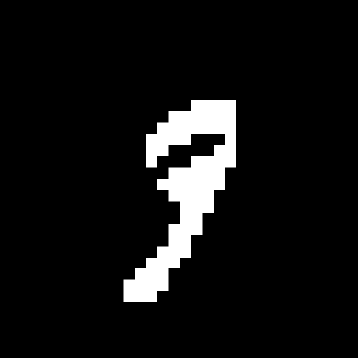}\hspace{0.01\textwidth}%
    \includegraphics[width=0.188\textwidth]{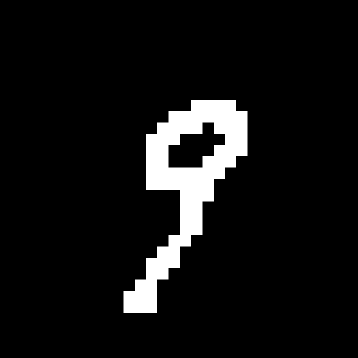}\hspace{0.01\textwidth}%
    \includegraphics[width=0.188\textwidth]{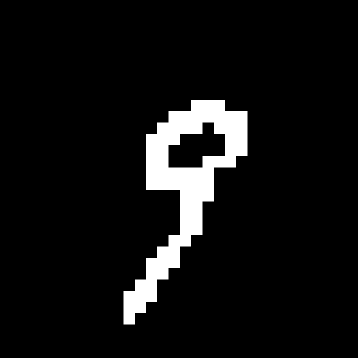}\hspace{0.01\textwidth}%
    \begingroup\setlength{\fboxsep}{0pt}\setlength{\fboxrule}{1.6pt}%
    \fcolorbox{blue}{white}{\includegraphics[width=\dimexpr0.188\textwidth-2\fboxrule\relax]{Figures/mnist_dims/mnist_8_19.pdf}}%
    \endgroup{}
            \mnistdimslabels
        \end{minipage}%
    }
    \caption{Visual comparison of MNIST GPCA reconstructions for randomly selected training samples. Each column corresponds to one latent dimension $d \in \{8,16,32,64\}$ for \textbf{fixed regularization parameter $\boldsymbol{\lambda = 10^{-3}}$}. The rightmost column, highlighted in blue, shows the corresponding ground-truth data $x\in \mc X^n_N$.}
    \label{fig:mnist-dims-grid-1e-3}
\end{minipage}
\end{figure}

\newcommand{\fashiondimslabels}{%
    \par\vspace{0.25em}
    {\fontsize{18}{20}\selectfont
    \makebox[0.188\textwidth][c]{\boldmath$d=16$}\hspace{0.01\textwidth}%
    \makebox[0.188\textwidth][c]{\boldmath$d=32$}\hspace{0.01\textwidth}%
    \makebox[0.188\textwidth][c]{\boldmath$d=64$}\hspace{0.01\textwidth}%
    \makebox[0.188\textwidth][c]{\boldmath$d=128$}\hspace{0.01\textwidth}%
    \makebox[0.188\textwidth][c]{\shortstack{\textbf{Ground}\\\textbf{truth}}}}}

\begin{figure}[H]
    \centering
    \resizebox{0.50\textwidth}{!}{%
        \begin{minipage}{\textwidth}
    \includegraphics[width=0.188\textwidth]{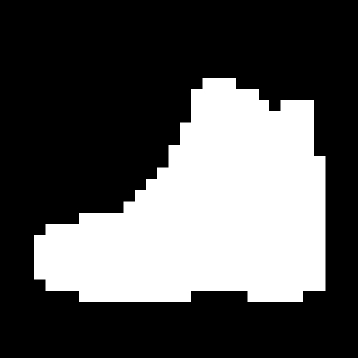}\hspace{0.01\textwidth}%
    \includegraphics[width=0.188\textwidth]{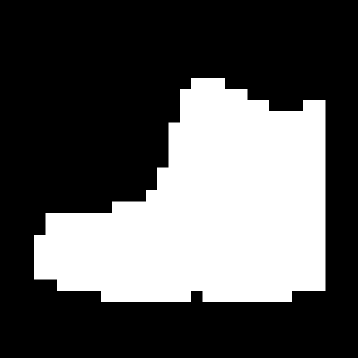}\hspace{0.01\textwidth}%
    \includegraphics[width=0.188\textwidth]{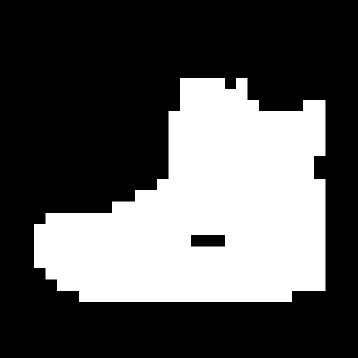}\hspace{0.01\textwidth}%
    \includegraphics[width=0.188\textwidth]{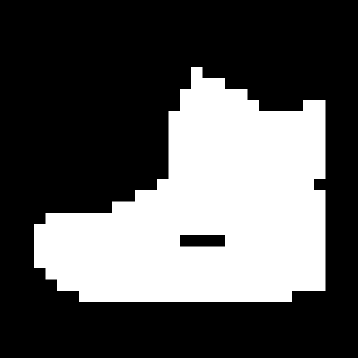}\hspace{0.01\textwidth}%
    \begingroup\setlength{\fboxsep}{0pt}\setlength{\fboxrule}{1.6pt}%
    \fcolorbox{blue}{white}{\includegraphics[width=\dimexpr0.188\textwidth-2\fboxrule\relax]{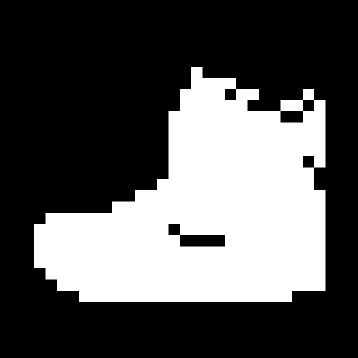}}%
    \endgroup{}
            \vspace{0.25em}
    \includegraphics[width=0.188\textwidth]{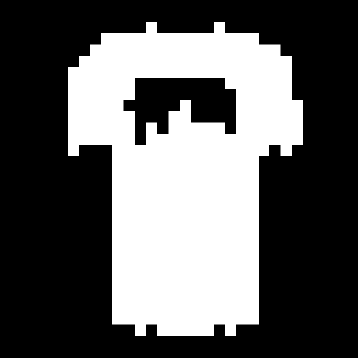}\hspace{0.01\textwidth}%
    \includegraphics[width=0.188\textwidth]{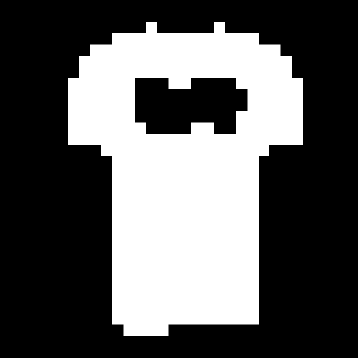}\hspace{0.01\textwidth}%
    \includegraphics[width=0.188\textwidth]{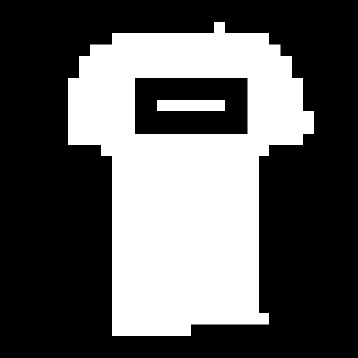}\hspace{0.01\textwidth}%
    \includegraphics[width=0.188\textwidth]{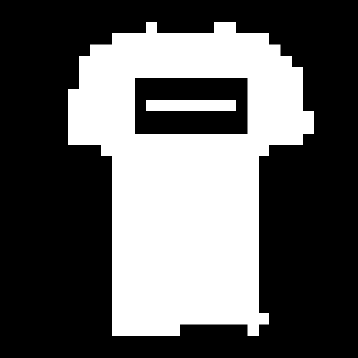}\hspace{0.01\textwidth}%
    \begingroup\setlength{\fboxsep}{0pt}\setlength{\fboxrule}{1.6pt}%
    \fcolorbox{blue}{white}{\includegraphics[width=\dimexpr0.188\textwidth-2\fboxrule\relax]{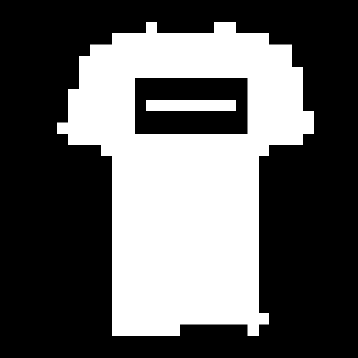}}%
    \endgroup{}
            \vspace{0.25em}
    \includegraphics[width=0.188\textwidth]{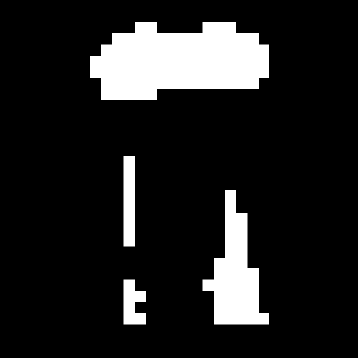}\hspace{0.01\textwidth}%
    \includegraphics[width=0.188\textwidth]{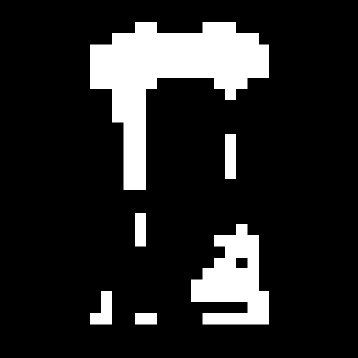}\hspace{0.01\textwidth}%
    \includegraphics[width=0.188\textwidth]{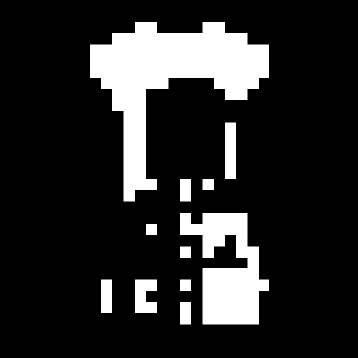}\hspace{0.01\textwidth}%
    \includegraphics[width=0.188\textwidth]{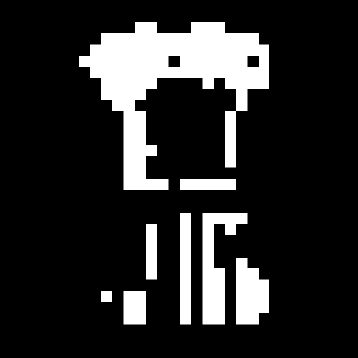}\hspace{0.01\textwidth}%
    \begingroup\setlength{\fboxsep}{0pt}\setlength{\fboxrule}{1.6pt}%
    \fcolorbox{blue}{white}{\includegraphics[width=\dimexpr0.188\textwidth-2\fboxrule\relax]{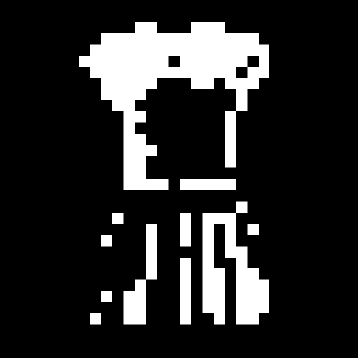}}%
    \endgroup{}
            \vspace{0.25em}
    \includegraphics[width=0.188\textwidth]{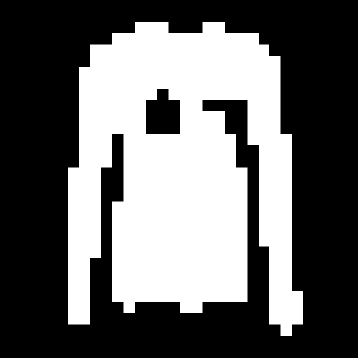}\hspace{0.01\textwidth}%
    \includegraphics[width=0.188\textwidth]{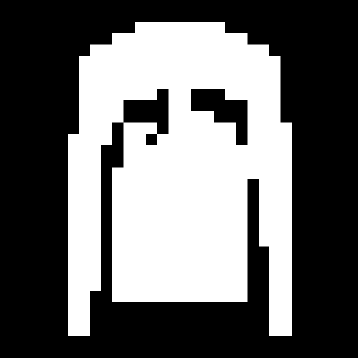}\hspace{0.01\textwidth}%
    \includegraphics[width=0.188\textwidth]{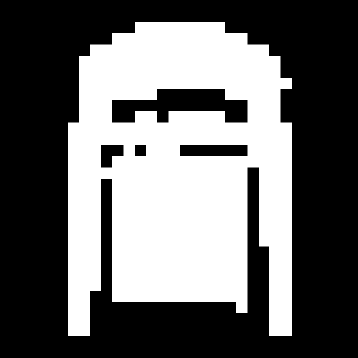}\hspace{0.01\textwidth}%
    \includegraphics[width=0.188\textwidth]{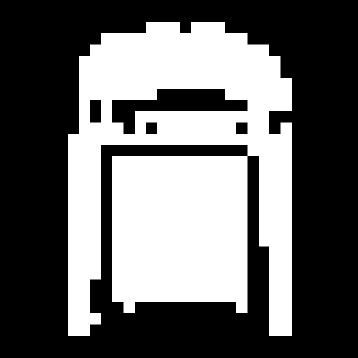}\hspace{0.01\textwidth}%
    \begingroup\setlength{\fboxsep}{0pt}\setlength{\fboxrule}{1.6pt}%
    \fcolorbox{blue}{white}{\includegraphics[width=\dimexpr0.188\textwidth-2\fboxrule\relax]{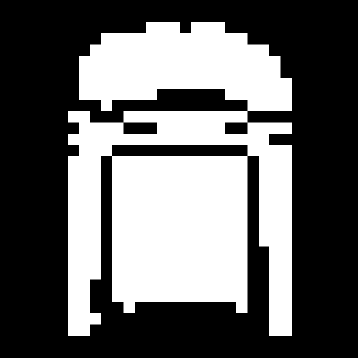}}%
    \endgroup{}
            \vspace{0.25em}
    \includegraphics[width=0.188\textwidth]{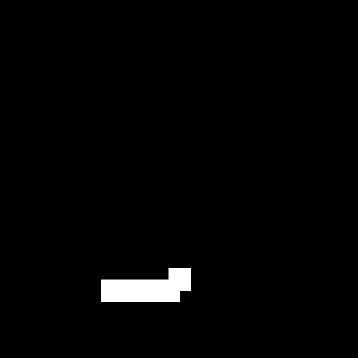}\hspace{0.01\textwidth}%
    \includegraphics[width=0.188\textwidth]{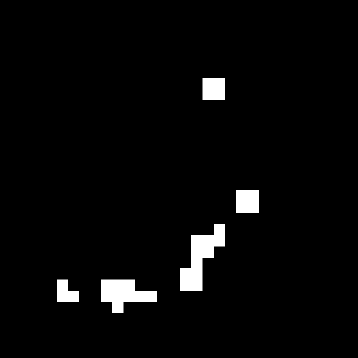}\hspace{0.01\textwidth}%
    \includegraphics[width=0.188\textwidth]{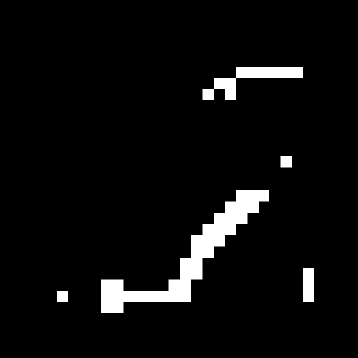}\hspace{0.01\textwidth}%
    \includegraphics[width=0.188\textwidth]{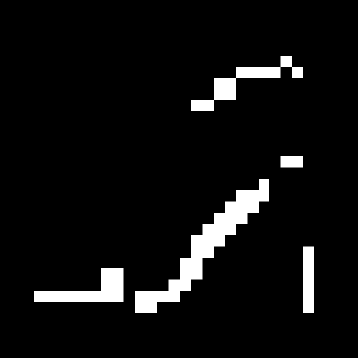}\hspace{0.01\textwidth}%
    \begingroup\setlength{\fboxsep}{0pt}\setlength{\fboxrule}{1.6pt}%
    \fcolorbox{blue}{white}{\includegraphics[width=\dimexpr0.188\textwidth-2\fboxrule\relax]{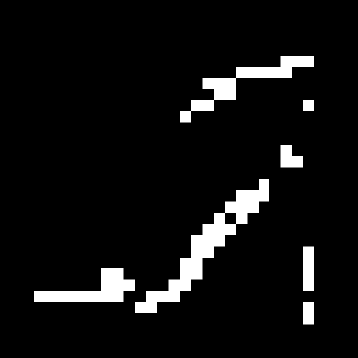}}%
    \endgroup{}
            \vspace{0.25em}
    \includegraphics[width=0.188\textwidth]{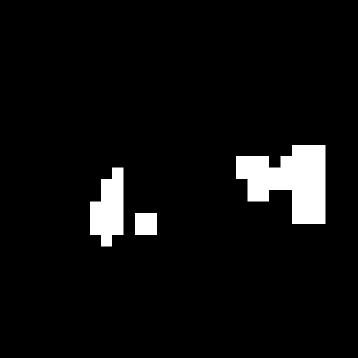}\hspace{0.01\textwidth}%
    \includegraphics[width=0.188\textwidth]{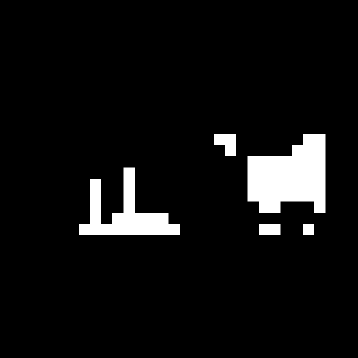}\hspace{0.01\textwidth}%
    \includegraphics[width=0.188\textwidth]{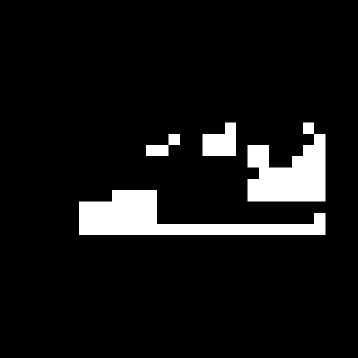}\hspace{0.01\textwidth}%
    \includegraphics[width=0.188\textwidth]{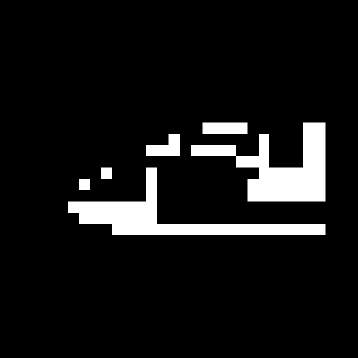}\hspace{0.01\textwidth}%
    \begingroup\setlength{\fboxsep}{0pt}\setlength{\fboxrule}{1.6pt}%
    \fcolorbox{blue}{white}{\includegraphics[width=\dimexpr0.188\textwidth-2\fboxrule\relax]{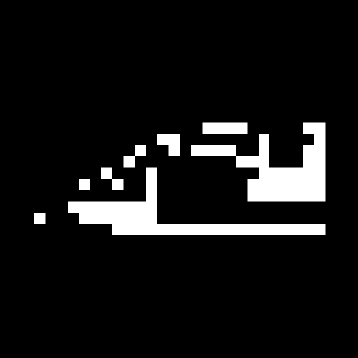}}%
    \endgroup{}
            \fashiondimslabels
        \end{minipage}%
    }
    \caption{Visual comparison of Fashion-MNIST GPCA reconstructions for randomly selected training samples. Each column corresponds to one latent dimension $d \in \{16,32,64,128\}$ for \textbf{fixed regularization parameter $\boldsymbol{\lambda = 10^{-2}}$}. The rightmost column, highlighted in blue, shows the corresponding ground-truth data $x\in \mc X^n_N$.}
    \label{fig:fashion-mnist-dims-grid}
\end{figure}

\begin{figure}[H]
    \centering
    \resizebox{0.50\textwidth}{!}{%
        \begin{minipage}{\textwidth}
    \includegraphics[width=0.188\textwidth]{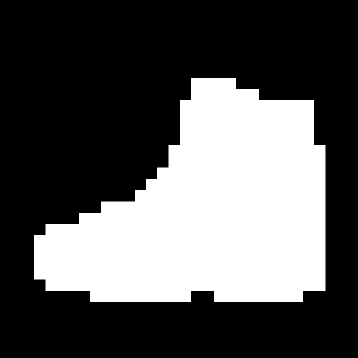}\hspace{0.01\textwidth}%
    \includegraphics[width=0.188\textwidth]{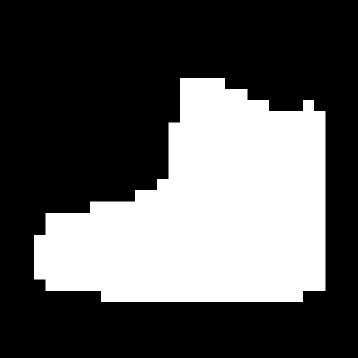}\hspace{0.01\textwidth}%
    \includegraphics[width=0.188\textwidth]{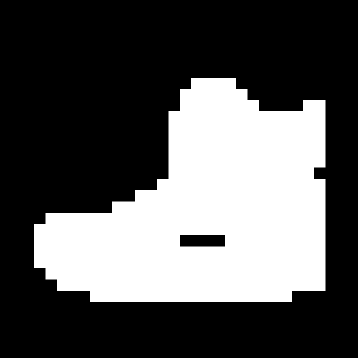}\hspace{0.01\textwidth}%
    \includegraphics[width=0.188\textwidth]{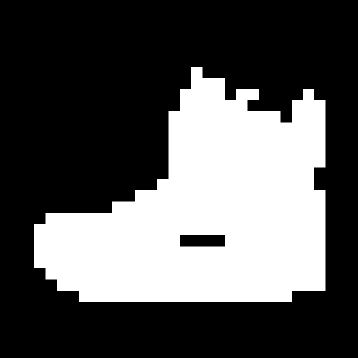}\hspace{0.01\textwidth}%
    \begingroup\setlength{\fboxsep}{0pt}\setlength{\fboxrule}{1.6pt}%
    \fcolorbox{blue}{white}{\includegraphics[width=\dimexpr0.188\textwidth-2\fboxrule\relax]{Figures/fashion_dims/fashion_mnist_16_0.pdf}}%
    \endgroup{}
            \vspace{0.25em}
    \includegraphics[width=0.188\textwidth]{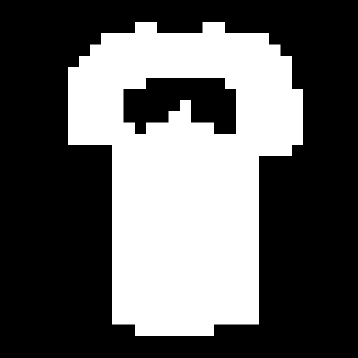}\hspace{0.01\textwidth}%
    \includegraphics[width=0.188\textwidth]{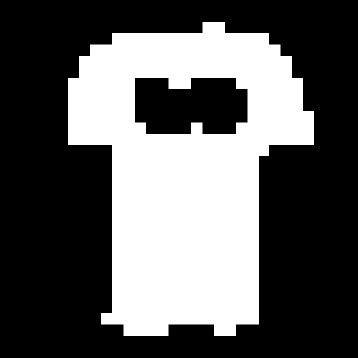}\hspace{0.01\textwidth}%
    \includegraphics[width=0.188\textwidth]{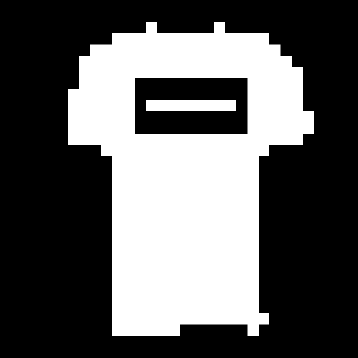}\hspace{0.01\textwidth}%
    \includegraphics[width=0.188\textwidth]{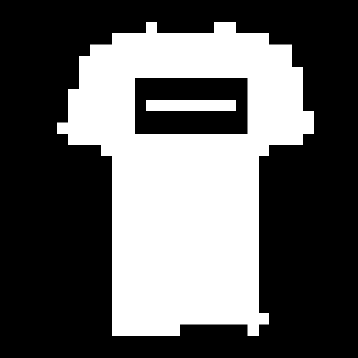}\hspace{0.01\textwidth}%
    \begingroup\setlength{\fboxsep}{0pt}\setlength{\fboxrule}{1.6pt}%
    \fcolorbox{blue}{white}{\includegraphics[width=\dimexpr0.188\textwidth-2\fboxrule\relax]{Figures/fashion_dims/fashion_mnist_16_1.pdf}}%
    \endgroup{}
            \vspace{0.25em}
    \includegraphics[width=0.188\textwidth]{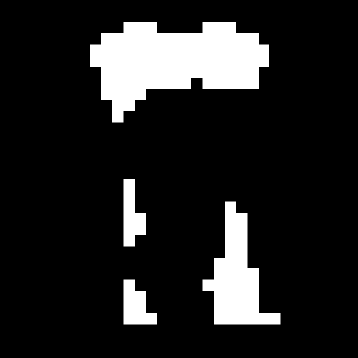}\hspace{0.01\textwidth}%
    \includegraphics[width=0.188\textwidth]{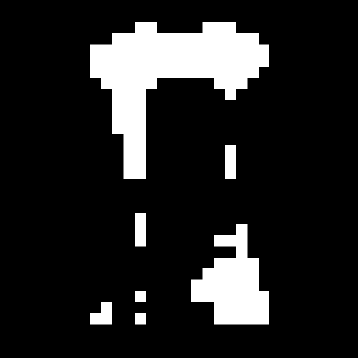}\hspace{0.01\textwidth}%
    \includegraphics[width=0.188\textwidth]{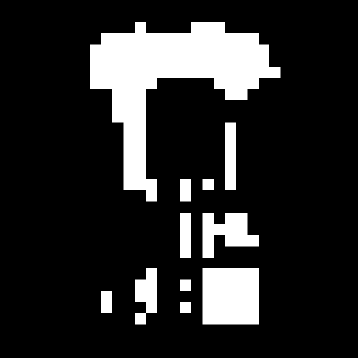}\hspace{0.01\textwidth}%
    \includegraphics[width=0.188\textwidth]{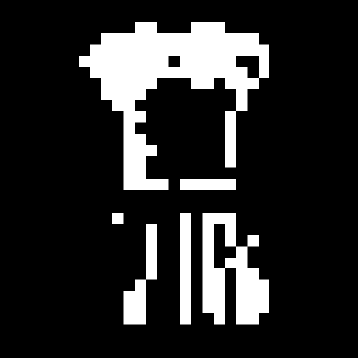}\hspace{0.01\textwidth}%
    \begingroup\setlength{\fboxsep}{0pt}\setlength{\fboxrule}{1.6pt}%
    \fcolorbox{blue}{white}{\includegraphics[width=\dimexpr0.188\textwidth-2\fboxrule\relax]{Figures/fashion_dims/fashion_mnist_16_3.pdf}}%
    \endgroup{}
            \vspace{0.25em}
    \includegraphics[width=0.188\textwidth]{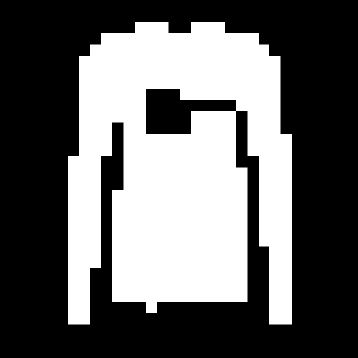}\hspace{0.01\textwidth}%
    \includegraphics[width=0.188\textwidth]{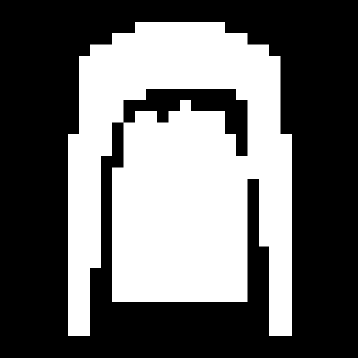}\hspace{0.01\textwidth}%
    \includegraphics[width=0.188\textwidth]{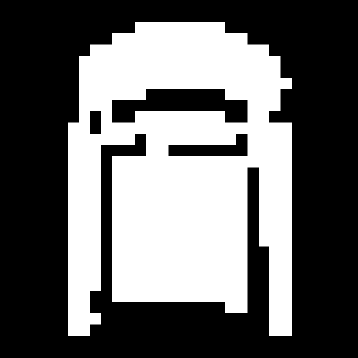}\hspace{0.01\textwidth}%
    \includegraphics[width=0.188\textwidth]{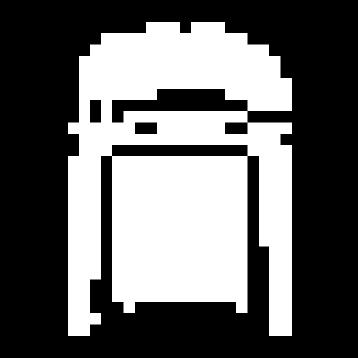}\hspace{0.01\textwidth}%
    \begingroup\setlength{\fboxsep}{0pt}\setlength{\fboxrule}{1.6pt}%
    \fcolorbox{blue}{white}{\includegraphics[width=\dimexpr0.188\textwidth-2\fboxrule\relax]{Figures/fashion_dims/fashion_mnist_16_5.pdf}}%
    \endgroup{}
            \vspace{0.25em}
    \includegraphics[width=0.188\textwidth]{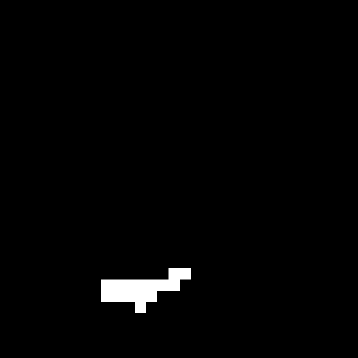}\hspace{0.01\textwidth}%
    \includegraphics[width=0.188\textwidth]{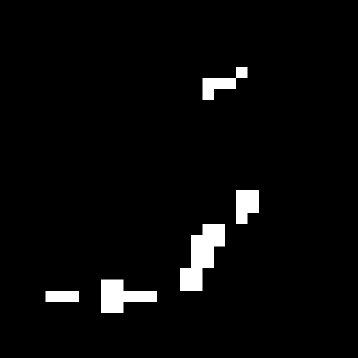}\hspace{0.01\textwidth}%
    \includegraphics[width=0.188\textwidth]{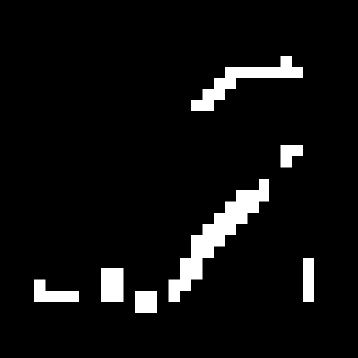}\hspace{0.01\textwidth}%
    \includegraphics[width=0.188\textwidth]{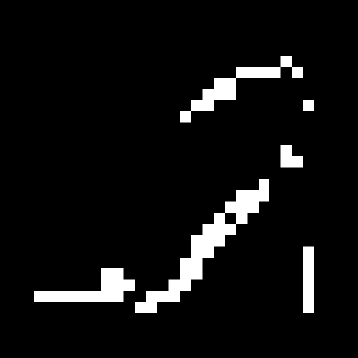}\hspace{0.01\textwidth}%
    \begingroup\setlength{\fboxsep}{0pt}\setlength{\fboxrule}{1.6pt}%
    \fcolorbox{blue}{white}{\includegraphics[width=\dimexpr0.188\textwidth-2\fboxrule\relax]{Figures/fashion_dims/fashion_mnist_16_8.pdf}}%
    \endgroup{}
            \vspace{0.25em}
    \includegraphics[width=0.188\textwidth]{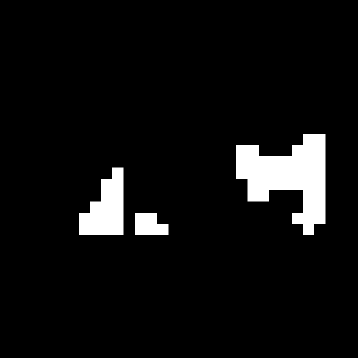}\hspace{0.01\textwidth}%
    \includegraphics[width=0.188\textwidth]{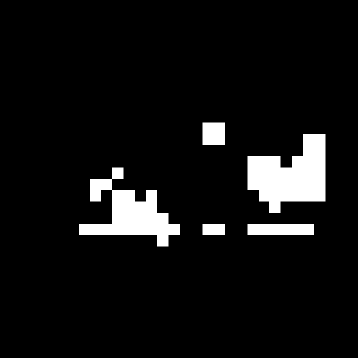}\hspace{0.01\textwidth}%
    \includegraphics[width=0.188\textwidth]{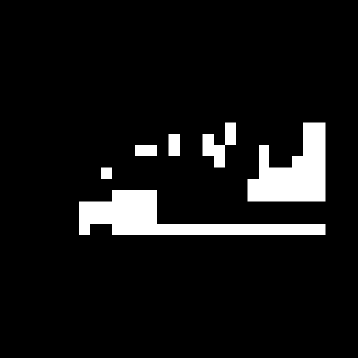}\hspace{0.01\textwidth}%
    \includegraphics[width=0.188\textwidth]{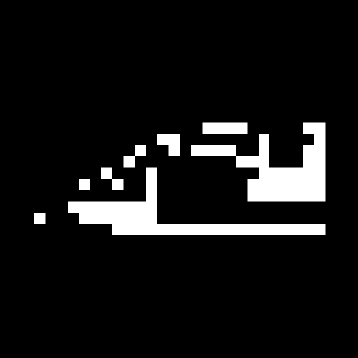}\hspace{0.01\textwidth}%
    \begingroup\setlength{\fboxsep}{0pt}\setlength{\fboxrule}{1.6pt}%
    \fcolorbox{blue}{white}{\includegraphics[width=\dimexpr0.188\textwidth-2\fboxrule\relax]{Figures/fashion_dims/fashion_mnist_16_13.pdf}}%
    \endgroup{}
            \fashiondimslabels
        \end{minipage}%
    }
    \caption{Visual comparison of Fashion-MNIST GPCA reconstructions for randomly selected training samples. Each column corresponds to one latent dimension $d \in \{16,32,64,128\}$ for \textbf{fixed regularization parameter $\boldsymbol{\lambda = 10^{-3}}$}. The rightmost column, highlighted in blue, shows the corresponding ground-truth data $x\in \mc X^n_N$.}
    \label{fig:fashion-mnist-dims-grid-1e-3}
\end{figure}

\clearpage

\section{Experiments: generation quality for MNIST with binary inputs} \label{sec:generation-metrics}

We evaluate unconditional samples produced by models trained on the MNIST handwritten-digit dataset \cite{Lecun:2010} with binary inputs. To study the effect of latent dimensionality, we vary $d\in\{8,16,32,64,128\}$ while fixing the regularization weight to $\lambda=10^{-4}$. We also vary $\lambda\in\{0,10^{-5},10^{-4},10^{-3},10^{-2},10^{-1}\}$ at fixed latent dimension $d=32$. In both experiments, we assess the quality and diversity of generated samples using a pretrained classifier.

For each latent dimension $d$, we generate $M$ unconditional samples $w_1,\dots,w_M$ and compute the mean classifier confidence
\begin{align}
Q(d) = \frac{1}{M} \sum_{i=1}^{M} \max_{k=0,\dots,9} p_{\phi}( k \mid w_i) ,
\end{align}
where $p_\phi(k \mid w_i)$ denotes the predicted class probability assigned by the pretrained classifier. This quantity measures how confidently the generated samples are recognized as valid digits. Larger values of $Q(d)$ indicate that samples are, on average, more digit-like according to the classifier.
To assess class diversity, we also compute the entropy of the empirical predicted class histogram
\begin{align}
H(d) = - \sum_{k=0}^9 \widehat p_d(k) \log \widehat p_d(k) , \qquad \widehat p_d(k) = \frac{1}{M} \sum_{i=1}^M \eins[\arg \max_{c} p_\phi(c \mid w_i) = k ],
\end{align}
which measures the class coverage. The maximal class confidence is $Q(d)=1$, the maximal class coverage is $H(d)=\log(10)\approx2.3026$. 

Table~\ref{tab:gpca-dims-summary} reports the metrics for varying latent dimension $d$ at fixed regularization weight $\lambda=10^{-4}$.

\begin{table}[H]
    \centering
    \caption{Classifier-based metrics for different latent dimensions ${d}$ at fixed regularization weight ${\lambda=10^{-4}}$.}
    \label{tab:gpca-dims-summary}
    \begin{tabular}{lcccccc}
        \toprule
        & $\boldsymbol{d=8}$ & $\boldsymbol{d=16}$ & $\boldsymbol{d=32}$ & $\boldsymbol{d=64}$  & \textbf{Test data} \\
        \midrule
        $Q(d)$ & $0.831$ & $0.944$ & $0.982$ & $0.981$ & $0.991$ \\
        $H(d)$ & $2.224$ & $2.293$ & $2.301$ & $2.301$ & $2.301$ \\
        \bottomrule
    \end{tabular}
\end{table}

For fixed $d=32$, Table~\ref{tab:gpca-lambda-summary} reports the classifier confidence and class-histogram entropy as $\lambda$ varies. The highest confidence is obtained for $\lambda=10^{-4}$, while class-histogram entropy remains close to the test-data reference across all settings. We use $Q$ and $H$ for the same metrics defined above, evaluated for each regularization weight.

\begin{table}[H]
    \centering
    \caption{Classifier-based metrics for varying regularization weight $\lambda$ at fixed latent dimension $d=32$. The test-data reference is evaluated using the same classifier.}
    \label{tab:gpca-lambda-summary}
    \setlength{\tabcolsep}{5pt}
    \begin{tabular}{lccccccc}
        \toprule
        & $\boldsymbol{\lambda=0}$ & $\boldsymbol{\lambda=10^{-5}}$ & $\boldsymbol{\lambda=10^{-4}}$ & $\boldsymbol{\lambda=10^{-3}}$ & $\boldsymbol{\lambda=10^{-2}}$ & $\boldsymbol{\lambda=10^{-1}}$ & \textbf{Test data} \\
        \midrule
        $Q(d)$ & $0.972$ & $0.978$ & $0.982$ & $0.981$ & $0.978$ & $0.969$ & $0.991$ \\
        $H(d)$ & $2.300$ & $2.300$ & $2.301$ & $2.301$ & $2.301$ & $2.299$ & $2.301$ \\
        \bottomrule
    \end{tabular}
\end{table}

\clearpage


\section{Experiments: Generation of novel data points} \label{sec:generation-novelty}
To assess whether the generative model produces novel samples rather than reproducing training examples, we compare randomly generated outputs with their nearest images from the training set. Models are trained with regularization parameter $\lambda = 10^{-2}$ and latent dimensions $32$, $64$, and $512$ for MNIST, FashionMNIST, and Cityscapes, respectively. For each generated sample, we retrieve the four nearest training images using the Hamming distance in \eqref{eq:hamming-distance} and display them alongside the generated sample. The results indicate that the generated samples are not direct copies of training images, suggesting that the model does not simply memorize the training set.

\begin{figure}[H]
    \centering
\begin{minipage}[t]{0.48\linewidth}
    \vspace{0pt}
    \centering
    \includegraphics[width=\linewidth]{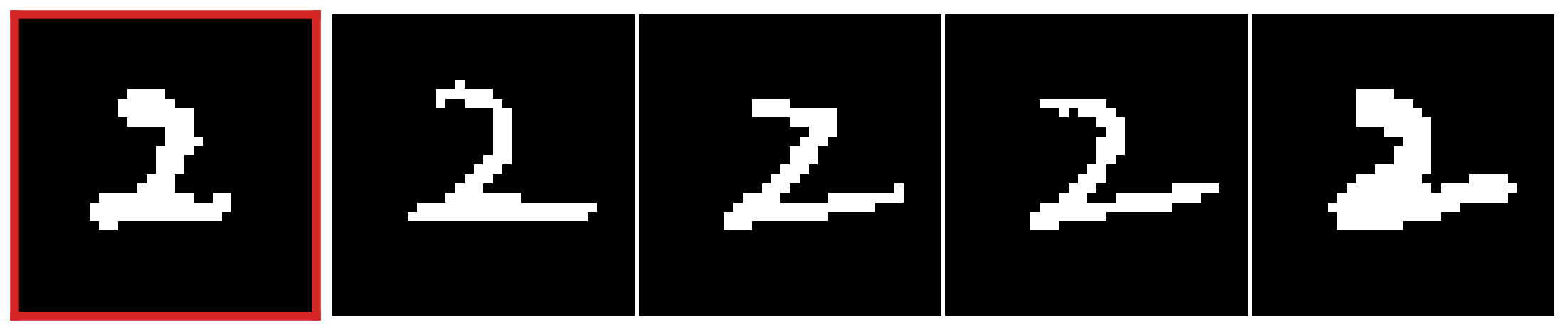}
    \vspace{0.1em}

    \includegraphics[width=\linewidth]{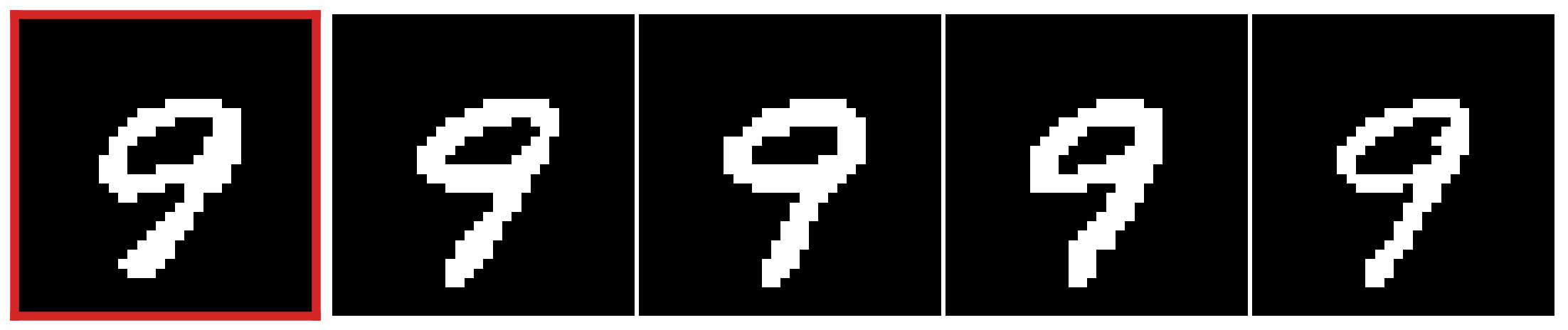}
    \vspace{0.1em}

    \includegraphics[width=\linewidth]{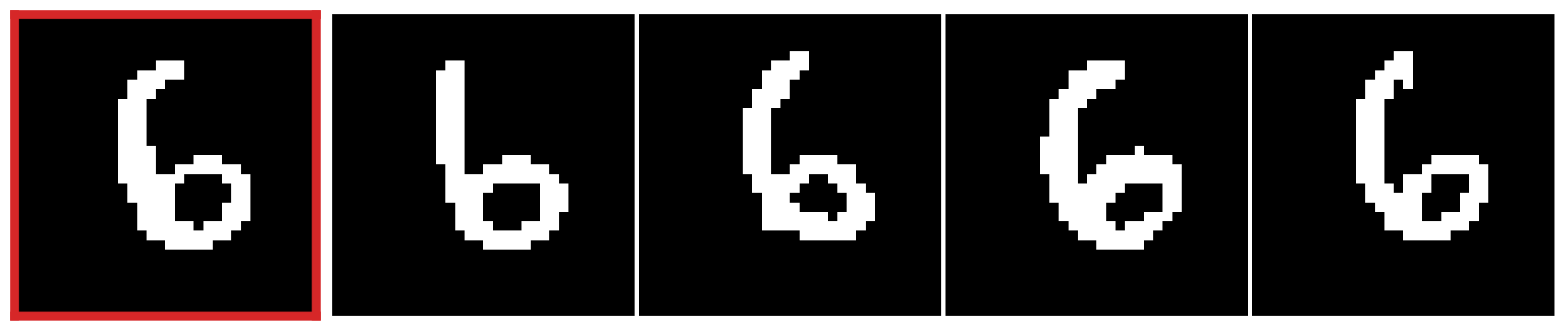}
    \vspace{0.1em}

    \includegraphics[width=\linewidth]{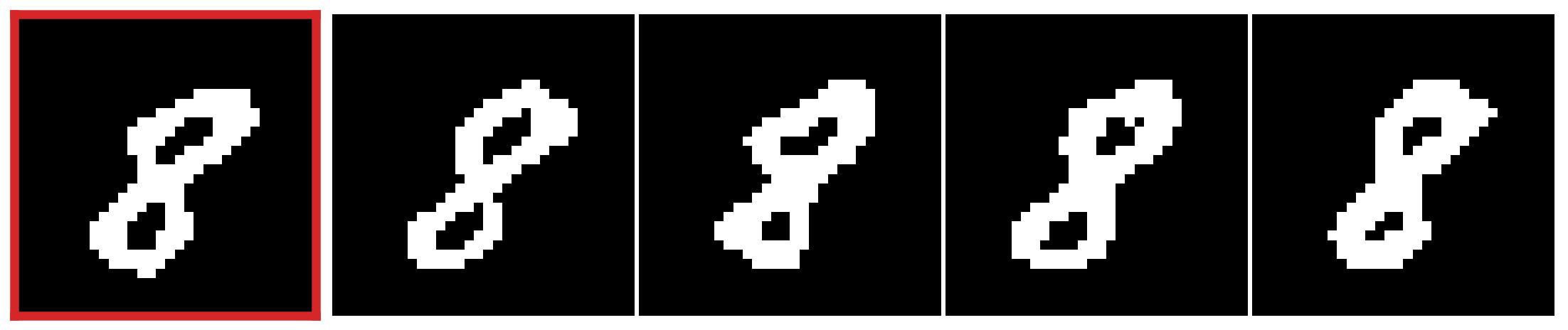}
    \vspace{0.1em}

    \includegraphics[width=\linewidth]{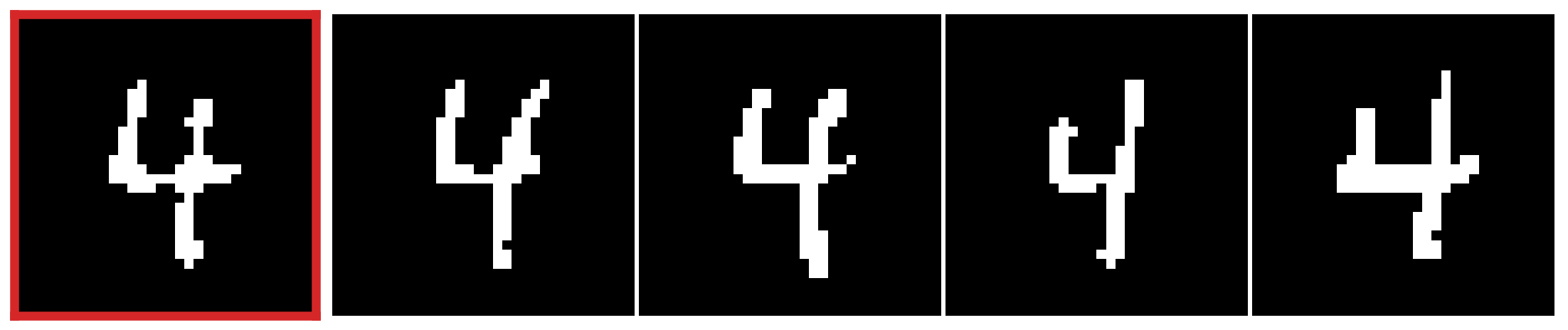}
    \caption{\textbf{MNIST.} Generated samples (left, marked red) and their four nearest training images, with $d=32$ and $\lambda=10^{-2}$.}
    \label{fig:mnist-generation-novelty}
\end{minipage}\hfill
\begin{minipage}[t]{0.48\linewidth}
    \vspace{0pt}
    \centering
    \includegraphics[width=\linewidth]{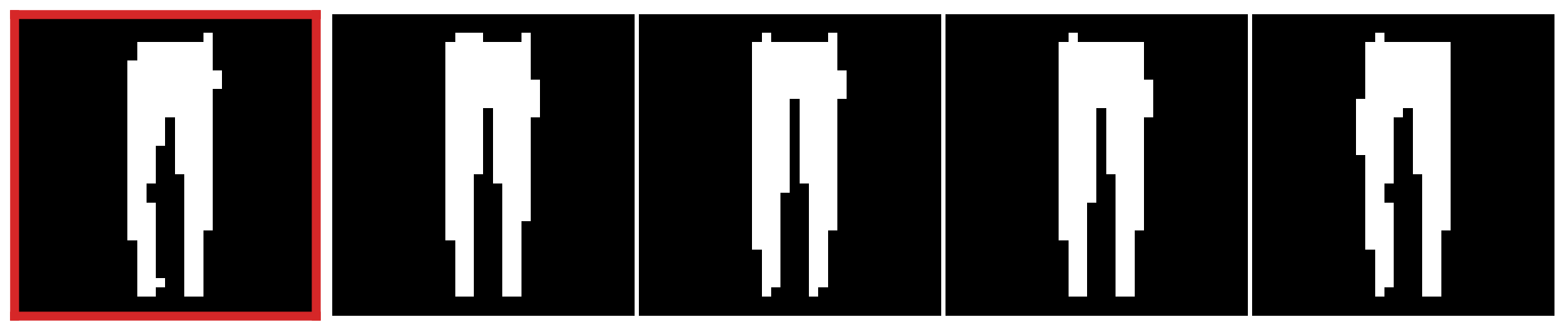}
    \vspace{0.1em}

    \includegraphics[width=\linewidth]{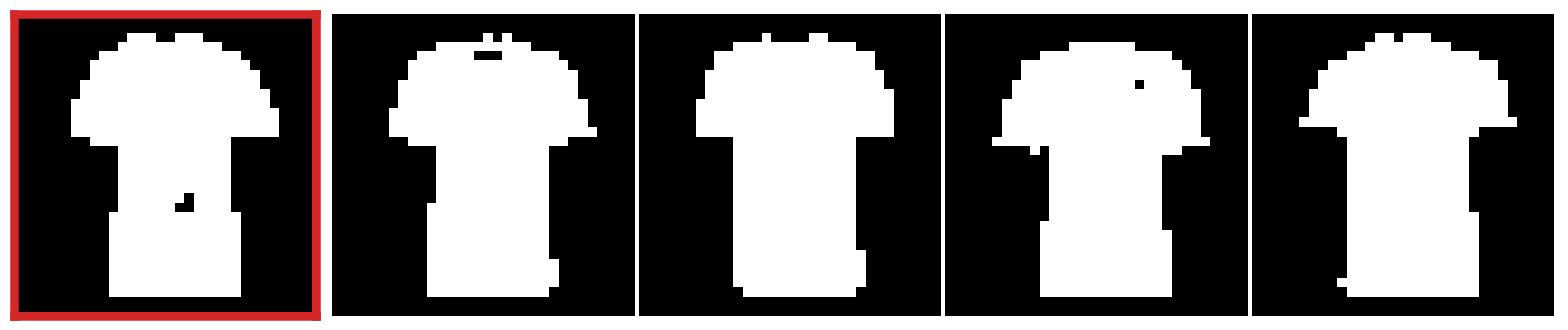}
    \vspace{0.1em}

    \includegraphics[width=\linewidth]{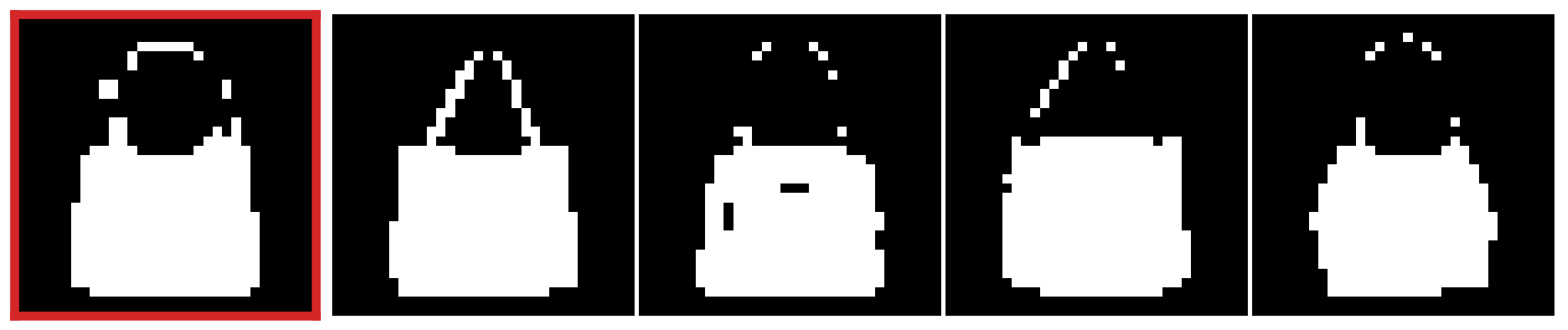}
    \vspace{0.1em}

    \includegraphics[width=\linewidth]{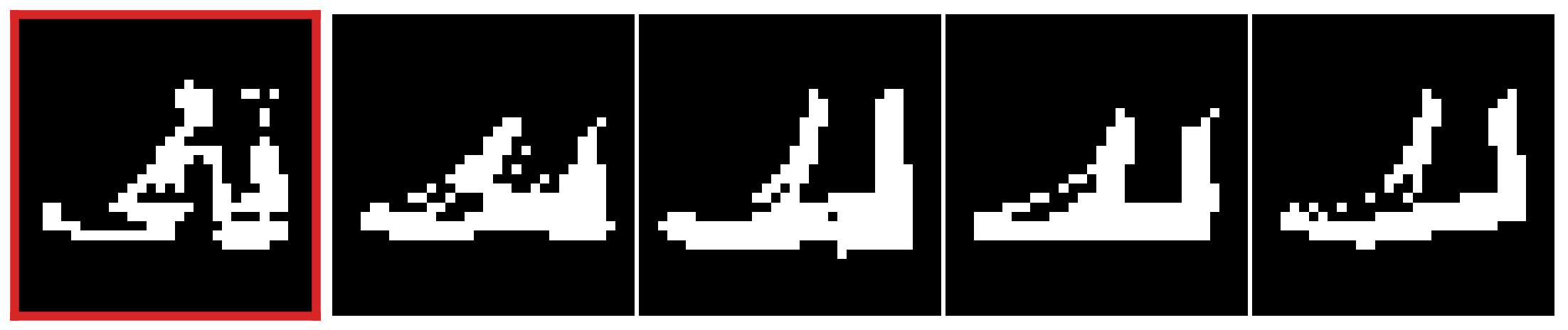}
    \vspace{0.1em}

    \includegraphics[width=\linewidth]{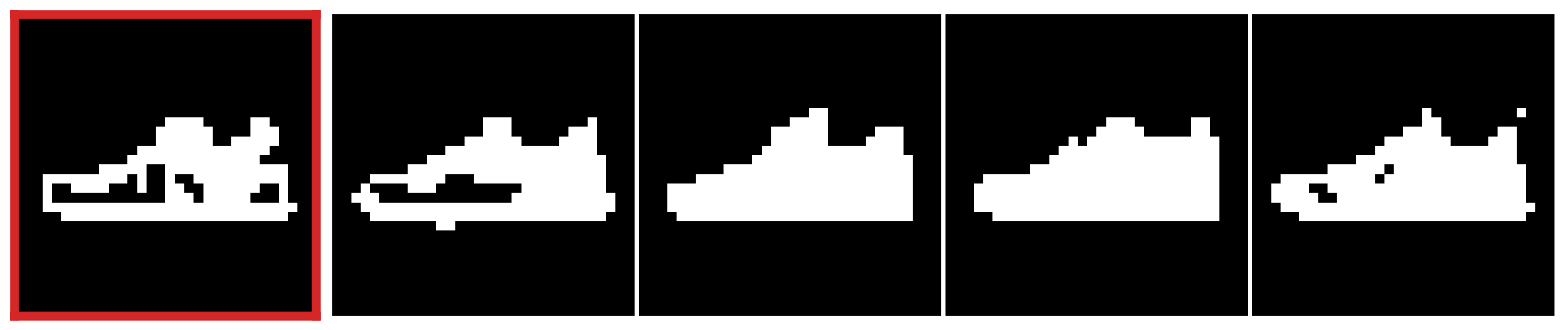}
    \caption{\textbf{FashionMNIST.} Generated samples (left, marked red) and their four nearest training images, with $d=64$ and $\lambda=10^{-2}$.}
    \label{fig:fashion-generation-novelty}

\end{minipage}
\end{figure}

\begin{figure}[H]
    \centering
    \includegraphics[width=0.68\linewidth]{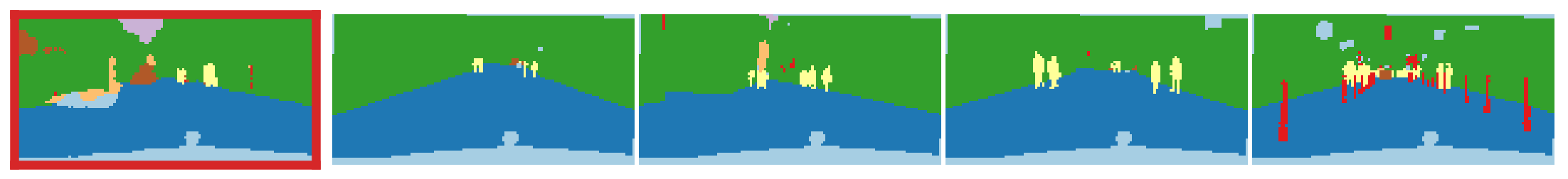}
    \vspace{0.1em}

    \includegraphics[width=0.68\linewidth]{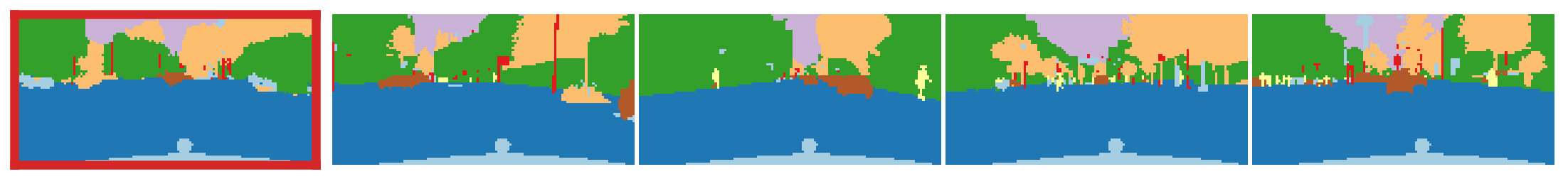}
    \vspace{0.1em}

    \includegraphics[width=0.68\linewidth]{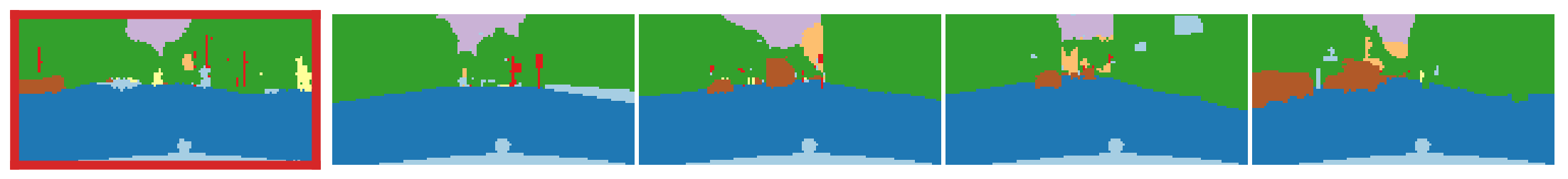}
    \vspace{0.1em}

    \includegraphics[width=0.68\linewidth]{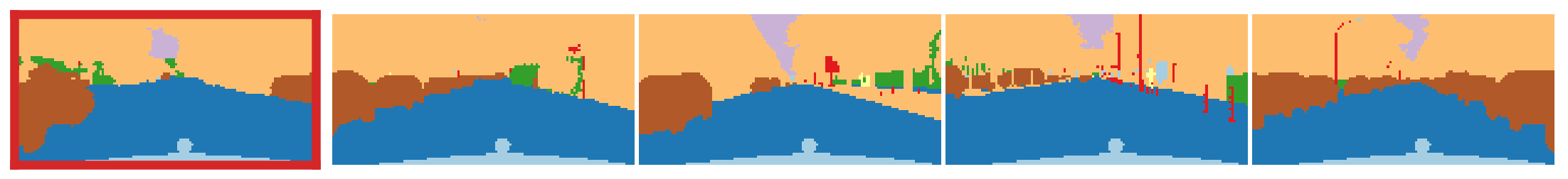}
    \vspace{0.1em}

    \includegraphics[width=0.68\linewidth]{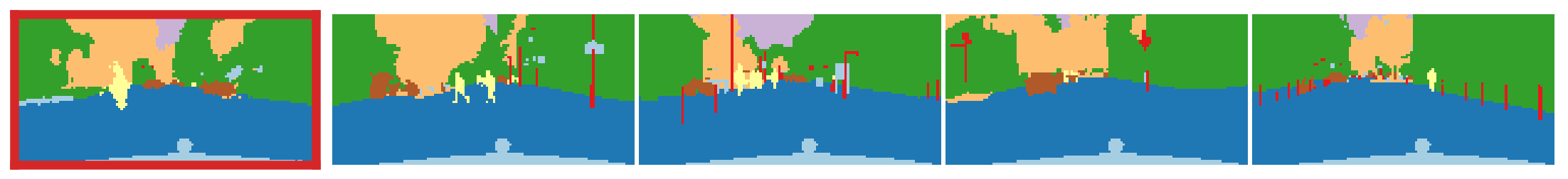}
    \caption{\textbf{Cityscapes.} Generated samples (left, marked red) and their four nearest training images, with $d=512$ and $\lambda=10^{-2}$.}
    \label{fig:cityscapes-generation-novelty}
\end{figure}

\section{Experiments: conditional flow matching on enhancer DNA datasets}\label{appendix:DNA}
\subsubsection{GPCA reconstructions on DNA datasets}\label{appendix:gpca-dna}
We consider three DNA datasets, with each nucleotide represented by one of $c=4$ categories. The \textsc{Promoter} dataset contains $100000$ human DNA sequences of length $n=1024$, centered at annotated transcription start sites and paired with transcription initiation signal profiles. The \textsc{Melanoma} and \textsc{FlyBrain} datasets contain $88870$ and $104665$ sequences of length $n=500$ from human melanoma and fly brain cells, respectively. Further details on these datasets, preparation generation tasks see \citet[Appendices~C.4 and F]{Avdeyev:2023} and \citet[Sections~4.2--4.3 and Appendices~B.1--B.2]{Stark:2024}.

A limitation of GPCA concerns the compressibility of the data. Overall, we observe that the image datasets considered in this work admit more accurate reconstructions than the DNA datasets when using the same latent dimension and regularization parameters, compare Tables~\ref{tab:kl-dna-promoter}, \ref{tab:kl-dna-flybrain}, and \ref{tab:kl-dna-edistance} below with the results in Section~\ref{sec:gpca-edistance-reg}.

A further limitation concerns extending our formulation of flow matching in the coefficient space $\mc Z$ to conditional generation. The conditioning information need not be categorical or share the same representation as the sequences, but may instead be continuous, as in the \textsc{Promoter} dataset. In this case, we learn GPCA from the sequences alone, while the conditioning profiles are subsequently used for conditional flow matching. We incorporate the conditioning information using network architectures employed in the $\mc U$ space, as such architectures are typically used by conditional generation models in practice. Consequently, we cannot directly take advantage of operating in $\mc Z$ and minimizing the conditional flow matching objective in \eqref{eq:cfm-objective-Z}. Developing a formulation that jointly incorporates the latent representation in $\mc Z$ and the conditioning information therefore remains a direction for future work.

\paragraph{Promoter dataset.}
We use $N=88570$ training sequences to learn the GPCA representation. Table~\ref{tab:kl-dna-promoter} reports results for $d=256,512,1024$ using the regularization parameter $\lambda=10^{-3}$ in \eqref{eq:loss-gpca-e_reg}. These dimensions correspond to compression factors of $16$, $8$, and $4$, respectively, relative to the $4096$ one-hot coordinates. We observe that accurate label reconstruction requires a relatively large latent dimension and does not by itself ensure a small geometric approximation error.

\begin{table}[H]
    \centering
    \caption{GPCA reconstruction metrics for \textsc{Promoter} with $\lambda=10^{-3}$ across latent dimensions.}
    \label{tab:kl-dna-promoter}
    \begin{tabular}{lccc}
        \toprule
        & $\boldsymbol{d=256}$ & $\boldsymbol{d=512}$ & $\boldsymbol{d=1024}$ \\
        \midrule
        \textbf{Cross-Entropy} & $1.0925$ & $0.7977$ & $0.2442$ \\
        \textbf{MHD} & $0.4583$ & $0.2903$ & $0.0041$ \\
        \textbf{$e$-distance} & $7.9725$ & $9.0613$ & $9.9992$ \\
        \bottomrule
    \end{tabular}
\end{table}


\paragraph{Enhancer datasets.}
Tables~\ref{tab:kl-dna-flybrain} and \ref{tab:kl-dna-edistance} report GPCA reconstruction results for \textsc{FlyBrain} and \textsc{Melanoma} at $d=500$, varying the regularization parameter $\lambda$ in \eqref{eq:loss-gpca-e_reg}. This dimension corresponds to a compression factor of $4$ relative to the $4\cdot500=2000$ one-hot coordinates. Increasing $\lambda$ reduces the $e$-distance but generally increases the reconstruction error, illustrating the trade-off between accurate label reconstruction and a small geometric approximation error mentioned in Section~\ref{sec:Experiments}.

\begin{table}[H]
    \centering
    \caption{GPCA reconstruction metrics for \textsc{FlyBrain} with $d=500$.}
    \label{tab:kl-dna-flybrain}
    \begin{tabular}{lccc}
        \toprule
        & $\boldsymbol{\lambda=0.1}$ & $\boldsymbol{\lambda=0.01}$ & $\boldsymbol{\lambda=0.001}$ \\
        \midrule
        \textbf{Cross-Entropy} & $0.5643$ & $0.4312$ & $0.2507$ \\
        \textbf{MHD} & $0.1951$ & $0.0100$ & $0.0100$ \\
        \textbf{$e$-distance} & $3.8992$ & $4.6061$ & $9.5124$ \\
        \bottomrule
    \end{tabular}
\end{table}

\begin{table}[H]
    \centering
    \caption{GPCA reconstruction metrics for \textsc{Melanoma} with $d=500$.}
    \label{tab:kl-dna-edistance}
    \begin{tabular}{lccc}
        \toprule
        & $\boldsymbol{\lambda=0.1}$ & $\boldsymbol{\lambda=0.01}$ & $\boldsymbol{\lambda=0.001}$ \\
        \midrule
        \textbf{Cross-Entropy} & $0.5121$ & $0.2927$ & $0.2305$ \\
        \textbf{MHD} & $0.1610$ & $0.0342$ & $0.0100$ \\
        \textbf{$e$-distance} & $3.8571$ & $6.5329$ & $9.2930$ \\
        \bottomrule
    \end{tabular}
\end{table}

\subsubsection{Conditional flow matching on DNA datasets}

\paragraph{Promoter dataset.}
We compare our method (GPCA + FM) with the promoter generation results reported by \cite{Cheng:2024}, using latent dimension $d=1024$ and regularization parameter $\lambda=10^{-3}$. Generation is conditioned on transcription initiation signal profiles. Table~\ref{tab:promoter-generation-mse} reports SP-MSE, the mean squared difference between the promoter activities predicted by Sei~\cite{Chen:2022Sequence} for generated sequences and the corresponding reference sequences. Lower values indicate closer agreement in predicted promoter activity. We use the experimental configurations described in \citet[Appendix C.6]{Cheng:2024}, select the checkpoint with the lowest validation SP-MSE and report its test SP-MSE. Baseline values are taken from their Figure 4, with the Language Model and DirichletFM results attributed to \cite{Stark:2024}. Our method achieves a test SP-MSE of $0.0335$, close to the value of $0.0315$ obtained in our reproduction of SFM without OT. We note that we could not reproduce its published SP-MSE of $0.0258$.

\begin{table}[H]
    \centering
    \caption{SP-MSE on generated promoter DNA experiment. Baseline values are taken from \citet{Cheng:2024}, who report the entries marked $\dagger$ from \citet{Stark:2024}.}
    \label{tab:promoter-generation-mse}
    \begin{tabular}{lc}
        \toprule
        \textbf{Model} & \textbf{SP-MSE}$\downarrow$ \\
        \midrule
        SFM w/ OT~\cite{Cheng:2024} & $0.0279$ \\
        SFM w/o OT~\cite{Cheng:2024} & $\mathbf{0.0258}$ \\
        LinearFM~\cite{Cheng:2024} & $0.0282$ \\
        DDSM~\cite{Avdeyev:2023} & $0.0334$ \\
        D3PM-uniform~\cite{Austin:2021} & $0.0375$ \\
        Bit-Diffusion (one-hot)~\cite{Chen:2022} & $0.0395$ \\
        Bit-Diffusion (bit)~\cite{Chen:2022} & $0.0414$ \\
        Language Model & $0.0333^{\dagger}$ \\
        DirichletFM~\cite{Stark:2024} & $0.0269^{\dagger}$ \\
        \midrule
        SFM w/o OT (our reproduction) & $0.0315$ \\
        GPCA + FM (ours), $d=1024$, \ $\lambda=10^{-3}$ & $0.0335$ \\
        \bottomrule
    \end{tabular}
\end{table}

\paragraph{Enhancer datasets.}
We compare our method (GPCA + FM) with the enhancer generation results reported by \citet{Davis:2024} on the \textsc{Melanoma} and \textsc{FlyBrain} datasets, using latent dimension $d=500$ and regularization parameter $\lambda=10^{-3}$. Training is unconditional. Table~\ref{tab:dna-fn} reports generation quality measured by the Fr\'echet Biological Distance (FBD), with lower values indicating closer agreement between the distributions of generated and reference sequences in the feature space used for evaluation. We use the same neural network architectures and hyperparameters as \citet{Davis:2024}, performing flow matching in the natural-parameter subspace $\mc U$. Consequently, this experiment does not realize the computational savings from training directly in the latent coefficient space $\mc Z$.

For Fisher-Flow, we report the mean FBD obtained in our reproduction. On \textsc{Melanoma}, our value of $28.4$ is close to the published result of $27.5\pm2.6$ \cite{Davis:2024}. On \textsc{FlyBrain}, we identified a dataset-loading issue in the released experiment configuration, which selects the Melanoma dataset instead of FlyBrain. With the corrected configuration, we obtain an FBD of $16.9$ and could not reproduce the published value of $3.8$. This configuration issue may account for the discrepancy.  

\begin{table}[H]
    \centering
    \caption{Generation quality on enhancer DNA datasets measured by FBD. Published baseline results are taken from \citet{Davis:2024}.}
    \label{tab:dna-fn}
    \begin{tabular}{lcc}
        \toprule
        \textbf{Method} & \textbf{Melanoma FBD}$\downarrow$ & \textbf{FlyBrain FBD}$\downarrow$ \\
        \midrule
        Random Sequence & $619.0\pm0.8$ & $832.4\pm0.3$ \\
        Language Model & $35.4\pm0.5$ & $25.7\pm1.0$ \\
        DirichletFM~\cite{Stark:2024} & $\mathbf{7.3\pm1.2}$ & $6.8\pm1.8$ \\
        Fisher-Flow~\cite{Davis:2024} & $27.5\pm2.6$ & $\mathbf{3.8\pm0.3}$ \\
        \midrule
        Fisher-Flow (our reproduction) & $28.4 \pm 0.7$ & $16.9 \pm 0.5$ \\
        GPCA + FM (ours) & $29.6 \pm 1.1$ & $17.3 \pm 0.4$ \\
        \bottomrule
    \end{tabular}
\end{table}


\begin{thebibliography}{47}
\providecommand{\natexlab}[1]{#1}
\providecommand{\url}[1]{\texttt{#1}}
\expandafter\ifx\csname urlstyle\endcsname\relax
  \providecommand{\doi}[1]{doi: #1}\else
  \providecommand{\doi}{doi: \begingroup \urlstyle{rm}\Url}\fi

\bibitem[Aitchison(1982)]{Aitchison:1982vt}
J.~Aitchison.
\newblock {The Statistical Analysis of Compositional Data}.
\newblock \emph{J. Royal Statistical Soc. B}, 44\penalty0 (2):\penalty0 139--177, 1982.

\bibitem[Albergo \& Vanden-Eijnden(2022)Albergo and Vanden-Eijnden]{Albergo:2022}
Michael~S Albergo and Eric Vanden-Eijnden.
\newblock Building normalizing flows with stochastic interpolants.
\newblock \emph{arXiv preprint arXiv:2209.15571}, 2022.

\bibitem[Amari \& Nagaoka(2000)Amari and Nagaoka]{Amari:2000aa}
S.-I. Amari and H.~Nagaoka.
\newblock \emph{{Methods of Information Geometry}}.
\newblock Amer. Math. Soc. and Oxford Univ. Press, 2000.

\bibitem[Amari(2016)]{Amari:2016}
Shun-ichi Amari.
\newblock \emph{{Information Geometry and Its Applications}}, volume 194 of \emph{Applied Mathematical Sciences}.
\newblock Springer, Tokyo, 2016.
\newblock ISBN 978-4-431-55978-8.
\newblock \doi{10.1007/978-4-431-55978-8}.

\bibitem[Austin et~al.(2021)Austin, Johnson, Ho, Tarlow, and Van Den~Berg]{Austin:2021}
Jacob Austin, Daniel~D Johnson, Jonathan Ho, Daniel Tarlow, and Rianne Van Den~Berg.
\newblock {Structured Denoising Diffusion Models in Discrete State-Spaces}.
\newblock \emph{Advances in neural information processing systems}, 34:\penalty0 17981--17993, 2021.

\bibitem[Avdeyev et~al.(2023)Avdeyev, Shi, Tan, Dudnyk, and Zhou]{Avdeyev:2023}
Pavel Avdeyev, Chenlai Shi, Yuhao Tan, Kseniia Dudnyk, and Jian Zhou.
\newblock {Dirichlet Diffusion Score Model for Biological Sequence Generation}.
\newblock In \emph{International Conference on Machine Learning}, pp.\  1276--1301. PMLR, 2023.

\bibitem[Ay et~al.(2017)Ay, Jost, L{\^e}, and Schwachh{\"o}fer]{Ay:2017aa}
N.~Ay, J.~Jost, H.~V. L{\^e}, and L.~Schwachh{\"o}fer.
\newblock \emph{{Information Geometry}}.
\newblock Springer, 2017.

\bibitem[Barndorff-Nielsen(1978)]{Barndorff-Nielsen:1978aa}
O.~E. Barndorff-Nielsen.
\newblock \emph{{Information and Exponential Families in Statistical Theory}}.
\newblock Wiley, Chichester, 1978.

\bibitem[Basseville(2013)]{Basseville:2013aa}
M.~Basseville.
\newblock {Divergence Measures for Statistical Data Processing -- An Annotated Bibliography}.
\newblock \emph{Signal Proc.}, 93\penalty0 (4):\penalty0 621--633, 2013.

\bibitem[Boll et~al.(2024)Boll, Gonzalez-Alvarado, and Schn\"{o}rr]{Boll:2024ab}
B.~Boll, D.~Gonzalez-Alvarado, and C.~Schn\"{o}rr.
\newblock {Generative Modeling of Discrete Joint Distributions by E-Geodesic Flow Matching on Assignment Manifolds}.
\newblock \emph{arXiv:2402.07846}, 2024.

\bibitem[Boll et~al.(2025)Boll, Gonzalez-Alvarado, Petra, and Schn\"{o}rr]{Boll:2025aa}
B.~Boll, D.~Gonzalez-Alvarado, S.~Petra, and C.~Schn\"{o}rr.
\newblock {Generative Assignment Flows for Representing and Learning Joint Distributions of Discrete Data}.
\newblock \emph{J. Math. Imaging Vision}, 67\penalty0 (34), 2025.

\bibitem[Boyd \& Vandenberghe(2004)Boyd and Vandenberghe]{Vandenberghe:2004}
Stephen Boyd and Lieven Vandenberghe.
\newblock \emph{{Convex Optimization}}, volume~1.
\newblock Cambridge university press Cambridge, 2004.

\bibitem[Brown(1986)]{Brown:1986vy}
L.~D. Brown.
\newblock \emph{{Fundamentals of Statistical Exponential Families}}.
\newblock Institute of Mathematical Statistics, Hayward, CA, 1986.

\bibitem[Calin \& Udriste(2014)Calin and Udriste]{calin2014geometric}
O.~Calin and C.~Udriste.
\newblock \emph{{Geometric Modeling in Probability and Statistics}}.
\newblock Springer, 2014.

\bibitem[Chen et~al.(2022{\natexlab{a}})Chen, Wong, Troyanskaya, and Zhou]{Chen:2022Sequence}
Kathleen~M Chen, Aaron~K Wong, Olga~G Troyanskaya, and Jian Zhou.
\newblock {A Sequence-Based Global Map of Regulatory Activity for Deciphering Human Genetics}.
\newblock \emph{Nature genetics}, 54\penalty0 (7):\penalty0 940--949, 2022{\natexlab{a}}.

\bibitem[Chen \& Lipman(2024)Chen and Lipman]{Chen:2023}
R.~T.~Q. Chen and Y.~Lipman.
\newblock {Riemannian Flow Matching on General Geometries}.
\newblock In \emph{ICLR}, 2024.

\bibitem[Chen et~al.(2022{\natexlab{b}})Chen, Zhang, and Hinton]{Chen:2022}
Ting Chen, Ruixiang Zhang, and Geoffrey Hinton.
\newblock {Analog Bits: Generating Discrete Data Using Diffusion Models with Self-Conditioning}.
\newblock \emph{arXiv preprint arXiv:2208.04202}, 2022{\natexlab{b}}.

\bibitem[Cheng et~al.(2024)Cheng, Li, Peng, and Liu]{Cheng:2024}
Chaoran Cheng, Jiahan Li, Jian Peng, and Ge~Liu.
\newblock {Categorical Flow Matching on Statistical Manifolds}.
\newblock \emph{arXiv:2405.16441}, 2024.

\bibitem[Cheng et~al.(2025)Cheng, Li, Fan, and Liu]{Cheng:2025}
Chaoran Cheng, Jiahan Li, Jiajun Fan, and Ge~Liu.
\newblock {{$\alpha$}-Flow: A Unified Framework for Continuous-State Discrete Flow Matching Models}.
\newblock \emph{arXiv:2504.10283}, 2025.

\bibitem[Collins et~al.(2001)Collins, Dasgupta, and Schapire]{Collins:2001aa}
M.~Collins, S.~Dasgupta, and R.~E. Schapire.
\newblock {A Generalization of Principal Component Analysis to the Exponential Family}.
\newblock In \emph{NIPS}, 2001.

\bibitem[Cordts et~al.(2016)Cordts, Omran, Ramos, Rehfeld, Enzweiler, Benenson, Franke, Roth, and Schiele]{Cordts:2016}
Marius Cordts, Mohamed Omran, Sebastian Ramos, Timo Rehfeld, Markus Enzweiler, Rodrigo Benenson, Uwe Franke, Stefan Roth, and Bernt Schiele.
\newblock {The Cityscapes Dataset for Semantic Urban Scene Understanding}.
\newblock In \emph{Proceedings of the IEEE Conference on Computer Vision and Pattern Recognition (CVPR)}, 2016.

\bibitem[Critchley et~al.(1994)Critchley, Marriott, and Salmon]{Critchley:1994aa}
F.~Critchley, P.~Marriott, and M.~Salmon.
\newblock {Preferred Point Geometry and the Local Differential Geometry of the Kullback-Leibler Divergence}.
\newblock \emph{Ann. Statistics}, 22\penalty0 (3):\penalty0 1587--1602, 1994.

\bibitem[Dao et~al.(2023)Dao, Phung, Nguyen, and Tran]{Dao:2023}
Quan Dao, Hao Phung, Binh Nguyen, and Anh Tran.
\newblock {Flow Matching in Latent Space}.
\newblock \emph{arXiv:2307.08698}, 2023.

\bibitem[Davis et~al.(2024)Davis, Kessler, Petrache, Ceylan, Bronstein, and Bose]{Davis:2024}
Oscar Davis, Samuel Kessler, Mircea Petrache, Ismail~Ilkan Ceylan, Michael Bronstein, and Avishek~Joey Bose.
\newblock {Fisher Flow Matching for Generative Modeling over Discrete Data}.
\newblock \emph{arXiv:2405.14664}, 2024.

\bibitem[Drton et~al.(2009)Drton, Sturmfels, and Sullivant]{Drton:2009}
Mathias Drton, Bernd Sturmfels, and Seth Sullivant.
\newblock \emph{{Lectures on Algebraic Statistics}}, volume~39 of \emph{Oberwolfach Seminars}.
\newblock Birkh{\"a}user, 2009.
\newblock \doi{10.1007/978-3-7643-8905-5}.

\bibitem[Erb \& Ay(2021)Erb and Ay]{Erb:2021aa}
I.~Erb and N.~Ay.
\newblock {The Information-Geometric Perspective of Compositional Data Analysis}.
\newblock In P.~Filzmoser, K.~Hron, J.~A. Mart\'{i}n-Fern\'{a}ndez, and J.~Palarea-Albaladejo (eds.), \emph{{Advances in Compositional Data Analysis}}, pp.\  21--43. Springer, 2021.

\bibitem[Fefferman et~al.(2016)Fefferman, Mitter, and Narayanan]{Fefferman:2016aa}
C.~Fefferman, S.~Mitter, and H.~Narayanan.
\newblock {Testing the Manifold Hypothesis}.
\newblock \emph{J. Amer. Math. Soc.}, 29\penalty0 (4):\penalty0 983--1049, 2016.

\bibitem[Filzmoser et~al.(2018)Filzmoser, Hron, and Templ]{Filzmoser:2018vq}
P.~Filzmoser, K.~Hron, and M.~Templ.
\newblock \emph{{Applied Compositional Data Analysis}}.
\newblock Springer, 2018.

\bibitem[Gat et~al.(2024)Gat, Remez, Shaul, Kreuk, Chen, Synnaeve, Adi, and Lipman]{Gat:2024}
Itai Gat, Tal Remez, Neta Shaul, Felix Kreuk, Ricky~TQ Chen, Gabriel Synnaeve, Yossi Adi, and Yaron Lipman.
\newblock {Discrete Flow Matching}.
\newblock \emph{Advances in Neural Information Processing Systems}, 37:\penalty0 133345--133385, 2024.

\bibitem[Goldt et~al.(2020)Goldt, M\'{e}zard, Krzakala, and Zdeborov\'{a}]{Goldt:2020aa}
S.~Goldt, M.~M\'{e}zard, F.~Krzakala, and L.~Zdeborov\'{a}.
\newblock {Modeling the Influence of Data Structure on Learning in Neural Networks: The Hidden Manifold Model}.
\newblock \emph{Physical Review X}, 10\penalty0 (041044), 2020.

\bibitem[Jolliffe(2002)]{Jolliffe:2002aa}
I.~T. Jolliffe.
\newblock \emph{{Principal Component Analysis}}.
\newblock Springer, 2nd edition, 2002.

\bibitem[Kivinen \& Warmuth(1997)Kivinen and Warmuth]{Kivinen:1997aa}
J.~Kivinen and M.~Warmuth.
\newblock {Exponentiated Gradient versus Gradient Descent for Linear Predictors}.
\newblock \emph{Inform. Comput.}, 132:\penalty0 1--63, 1997.

\bibitem[Kivinen \& Warmuth(2001)Kivinen and Warmuth]{Warmuth:2001aa}
J.~Kivinen and M.~K. Warmuth.
\newblock {Relative Loss Bounds for Multidimensional Regression Problems}.
\newblock \emph{Machine Learning}, 45:\penalty0 301--329, 2001.

\bibitem[Lauritzen(1987)]{Lauritzen:1987aa}
S.~L. Lauritzen.
\newblock {Chapter 4: Statistical Manifolds}.
\newblock In Shanti~S. Gupta, S.~I. Amari, O.~E. Barndorff-Nielsen, R.~E. Kass, S.~L. Lauritzen, and C.~R. Rao (eds.), \emph{{Differential Geometry in Statistical Inference}}, pp.\  163--216. Institute of Mathematical Statistics, Hayward, CA, 1987.

\bibitem[LeCun et~al.(2010)LeCun, Cortes, and Burges]{Lecun:2010}
Yann LeCun, Corinna Cortes, and Christopher~J.C. Burges.
\newblock {MNIST Handwritten Digit Database}.
\newblock \emph{ATT Labs [Online]}, 2010.

\bibitem[Lipman et~al.(2023)Lipman, Chen, Ben-Hamu, Nickel, and Le]{Lipman:2023}
Yaron Lipman, Ricky T.~Q. Chen, Heli Ben-Hamu, Maximilian Nickel, and Matthew Le.
\newblock {Flow Matching for Generative Modeling}.
\newblock In \emph{The Eleventh International Conference on Learning Representations}, 2023.

\bibitem[Liu et~al.(2022)Liu, Gong, and Liu]{Liu:2022}
Xingchao Liu, Chengyue Gong, and Qiang Liu.
\newblock Flow straight and fast: Learning to generate and transfer data with rectified flow.
\newblock \emph{arXiv preprint arXiv:2209.03003}, 2022.

\bibitem[Mahony \& Williamson(2001)Mahony and Williamson]{Mahony:2001aa}
R.~E. Mahony and R.~C. Williamson.
\newblock {Prior Knowledge and Preferential Structures in Gradient Descent Learning Algorithms}.
\newblock \emph{J. Mach. Learning Res.}, 1:\penalty0 311--355, 2001.

\bibitem[Matumoto(1993)]{Matumoto:1993aa}
T.~Matumoto.
\newblock {Any Statistical Manifold has a Contrast Function---on the $C^{3}$-Functions Taking the Minimum at the Diagonal of the Product Manifold}.
\newblock \emph{Hiroshima Math. J.}, 23\penalty0 (2):\penalty0 327--332, 1993.

\bibitem[McCullagh \& Nelder(1989)McCullagh and Nelder]{McCullagh:1989aa}
P.~McCullagh and J.A. Nelder.
\newblock \emph{Generalized Linear Models}.
\newblock Chapman and Hall, 1989.

\bibitem[Meil\u{a} \& Zhang(2024)Meil\u{a} and Zhang]{Meila:2024aa}
M.~Meil\u{a} and H.~Zhang.
\newblock {Manifold Learning: What, How, and Why}.
\newblock \emph{Annual Review of Statistics and Its Application}, 11:\penalty0 393--417, 2024.

\bibitem[Samaddar et~al.(2025)Samaddar, Sun, Nilsson, and Madireddy]{samaddar2025efficient}
Anirban Samaddar, Yixuan Sun, Viktor Nilsson, and Sandeep Madireddy.
\newblock Efficient flow matching using latent variables.
\newblock \emph{arXiv:2505.04486}, 2025.

\bibitem[Stark et~al.(2024)Stark, Jing, Wang, Corso, Berger, Barzilay, and Jaakkola]{Stark:2024}
Hannes Stark, Bowen Jing, Chenyu Wang, Gabriele Corso, Bonnie Berger, Regina Barzilay, and Tommi Jaakkola.
\newblock {Dirichlet Flow Matching with Applications to DNA Sequence Design}.
\newblock \emph{arXiv:2402.05841}, 2024.

\bibitem[Tang et~al.(2025)Tang, Zhang, Tong, and Chatterjee]{Tang:2025}
Sophia Tang, Yinuo Zhang, Alexander Tong, and Pranam Chatterjee.
\newblock {Gumbel-Softmax Flow Matching with Straight-Through Guidance for Controllable Biological Sequence Generation}.
\newblock \emph{arXiv:2503.17361}, 2025.

\bibitem[Williams et~al.(2025)Williams, Yeom-Song, Hartmann, and Klami]{Williams:2025}
Bernardo Williams, Victor~M Yeom-Song, Marcelo Hartmann, and Arto Klami.
\newblock {Simplex-to-Euclidean Bijections for Categorical Flow Matching}.
\newblock \emph{arXiv:2510.27480}, 2025.

\bibitem[Xiao et~al.(2017)Xiao, Rasul, and Vollgraf]{Xiao:2017}
Han Xiao, Kashif Rasul, and Roland Vollgraf.
\newblock {Fashion-Mnist: a Novel Image Dataset for Benchmarking Machine Learning Algorithms}.
\newblock \emph{arXiv:1708.07747}, 2017.

\bibitem[Zhang(2004)]{Zhang:2004aa}
J.~Zhang.
\newblock {Divergence Function, Duality, and Convex Analysis}.
\newblock \emph{Neural Computation}, 16\penalty0 (1):\penalty0 159--195, 2004.

\end{thebibliography}
\end{document}